%% file: main.tex
\documentclass{article}
\PassOptionsToPackage{table}{xcolor}

\usepackage{mymacros}

\usepackage[accepted]{icml2025}

\allowdisplaybreaks[4]

\icmltitlerunning{Learning Utilities from Demonstrations in MDPs}

\begin{document}

\twocolumn[
\icmltitle{Learning Utilities from Demonstrations in Markov Decision Processes}

\icmlsetsymbol{equal}{*}

\begin{icmlauthorlist}
\icmlauthor{Filippo Lazzati}{yyy}
\icmlauthor{Alberto Maria Metelli}{yyy}
\end{icmlauthorlist}

\icmlaffiliation{yyy}{Politecnico di Milano, Milan, Italy}

\icmlcorrespondingauthor{Filippo Lazzati}{filippo.lazzati@polimi.it}

\icmlkeywords{Inverse Reinforcement Learning, Risk, Imitation Learning, Theory}

\vskip 0.3in
]

\printAffiliationsAndNotice{}  %

\begin{abstract}
Although it is well-known that humans commonly engage in \emph{risk-sensitive}
behaviors in the presence of stochasticity, most Inverse Reinforcement Learning
(IRL) models assume a \emph{risk-neutral} agent.
As such, beyond $(i)$ introducing model misspecification, $(ii)$ they do not
permit direct inference of the risk attitude of the observed agent, which can be
useful in many applications.
In this paper, we propose a novel model of behavior to cope with these issues.
By allowing for risk sensitivity, our model alleviates $(i)$, and by explicitly
representing risk attitudes through (learnable) \emph{utility} functions, it
solves $(ii)$.
Then, we characterize the partial identifiability of an agent's utility under
the new model and note that demonstrations from multiple environments mitigate
the problem.
We devise two provably-efficient algorithms for learning utilities in a
finite-data regime, and we conclude with some proof-of-concept experiments to
validate \emph{both} our model and our algorithms.
\end{abstract}

\section{Introduction}\label{sec: introduction}

The ultimate goal of Artificial Intelligence (AI) is to construct artificial
rational autonomous agents \citep{russel2010ai}. Such agents will interact with
each other and with human beings to achieve the tasks that \emph{we} assign to
them. In this vision, a crucial feature is being able to correctly model the
observed behavior of other agents. This allows a variety of applications:
\emph{descriptive}, to understand the intent of the observed agent
\citep{russell1998learning}, \emph{predictive}, to anticipate the behavior of
the observed agent (potentially in new scenarios) \citep{arora2020survey}, and
\emph{normative}, to imitate the observed agent because they are behaving in the
``right way'' \citep{osa2018IL}.

Nowadays, Inverse Reinforcement Learning (IRL) provides the most popular and
powerful models of the behavior of the observed agent, named ``expert''.
IRL models assume the existence of a reward function that \emph{rationalizes}
the expert's behavior and differ from each other based on the specific
assumption of how the expert behaves based on the given reward.
For instance, \citet{ng2000algorithms} considers the expert as playing an
optimal policy, \citet{poiani2024inversereinforcementlearningsuboptimal}
considers an $\epsilon$-optimal policy, \citet{malik2021icrl} considers an
optimal policy satisfying some constraints,
\citet{Ziebart2010ModelingPA,Fu2017LearningRR} assume that the expert plays
actions proportionally to their (soft) Q-functions, and
\citet{Ramachandran2007birl} assumes this probability to depend on the optimal
advantage function.

These IRL models represent the expert as a \textit{risk-neutral} agent, i.e., an
agent interested in maximizing the \textit{expected} return. However, there are
many scenarios in which rational agents and humans adopt risk-sensitive
strategies in the presence of stochasticity, like finance
\citep{foellmer2004stochastic}, revenue management \citep{barz2007revenue},
driving \citep{Bernhard2019AddressingIU}, and many other choice problems
\citep{kahneman1979prospect,Kreps1988NotesOT}.
In these settings, agents are not only interested in the \textit{expected}
return, but in its full \emph{distribution} \citep{bellemare2023distributional}.
Thus, IRL models incur in \emph{misspecification}, that can crucially affect the
descriptive, predictive, and normative power of the inferred reward function
\citep{skalse2024quantifyingsensitivityinversereinforcement}.

In this context, in addition to misspecification, another issue of IRL models is
that they do not \emph{explicitly} represent the risk attitude of the expert,
which is only indirectly captured by the reward function.
We desire two different representations for the \emph{task} of the expert
(through a reward) and for its \emph{risk attitude} (e.g., through a utility
function), analogously to what is done in Inverse Constrained Reinforcement
Learning (ICRL) \citep{malik2021icrl}. Here, the behavior of the expert is
described by two parameters, a reward for modelling the task, and a cost for
modelling the constraints.
In this way, the reward is more easily interpretable since it does not have to
capture both the task and the constraints, and, also, we can use the learned
cost for performing new tasks safely \citep{kim2023learningshared}.
For these reasons, if we were able to directly learn the risk attitude of the
expert \emph{separately} from its reward function, then we could more easily
understand its intent and anticipate its choices in new, unseen, scenarios
\citep{Kreps1988NotesOT}.

\textbf{Contributions.}~~In this paper, we introduce a new risk-sensitive model
of behavior that encodes the risk attitude of an agent with a utility function.
Thanks to this model, we will show that it is possible to overcome the
limitations mentioned above.
Our main contributions are:
\begin{itemize}[leftmargin=*, noitemsep, topsep=-2pt]
  \item We present a new simple yet powerful \emph{model} of behavior in Markov
  Decision Processes (MDPs) that rationalizes non-Markovian demonstrations
  (Section \ref{sec: motivation and problem setting}).
  \item We formulate \emph{Utility Learning} (UL) as the problem of inferring
  the risk attitude of an agent under the new model of behavior, we characterise
  the partial identifiability of its utility, and we show that demonstrations in
  multiple environments alleviate the issue (Section \ref{sec: utility
  learning}).
  \item We introduce \caty and \tractor, two novel algorithms for solving the UL
  problem with finite data in a provably-efficient manner (Section \ref{sec:
  online utility learning}).
  \item We conclude with proof-of-concept \emph{experiments} that serve as an
  empirical validation of both the proposed model and the presented algorithms.
  (Section \ref{sec: experiments}).
\end{itemize}
The proofs of all results are reported in Appendix \ref{apx: section
3}-\ref{apx: section 5}.

\section{Preliminaries}\label{sec: preliminaries}

The main paper's notation is below. Additional notation for the
supplemental is in Appendix \ref{apx: additional notation}.

\textbf{Notation.}~~For any $N \in \Nat$, we write
$\dsb{N}\coloneqq\{1,\dotsc,N\}$. Given set $\cX$, we denote by $\Delta^\cX$ the
probability simplex on $\cX$. Given compact $\cX\subseteq\RR^d,y\in\RR^d$, we
define $\Pi_{\cX}(y) \coloneqq \argmin_{x\in\cX}\|y-x\|_2$.
A real-valued function $f:\RR\to\RR$ is \emph{$L$-Lipschitz} if, for all
$x,y\in\RR$, we have $|f(x)-f(y)|\le L|x-y|$. $f$ is \emph{increasing} if, for
all $x<y\in\RR$, it holds $f(x)\le f(y)$, and it is \emph{strictly-increasing} if
$f(x)< f(y)$.
The probability distribution that puts all its mass on $z\in\RR$ is denoted by
$\delta_z$ and is called the \emph{Dirac delta}. We represent distributions on
finite support as mixtures of Dirac deltas.

\textbf{Markov Decision Processes (MDPs).}~~A tabular episodic Markov Decision
Process (MDP) \citep{puterman1994markov} is a tuple
$\cM=\tuple{\cS,\cA,H,s_0,p,r}$, where $\cS$ and $\cA$ are the finite state
($S\coloneqq|\cS|$) and action ($A\coloneqq|\cA|$) spaces, $H$ is the time
horizon, $s_0\in\cS$ is the initial state, $p:\SAH\to\Delta^\cS$ is the
transition model, and $r:\SAH\to [0,1]$ is the \emph{deterministic} reward
function.
The interaction of an agent with $\cM$ generates trajectories. Let
$\Omega_h\coloneqq (\cS\times\cA)^{h-1}\times\cS$ be the set of
state-action trajectories of length $h$ for all $h\in\dsb{H+1}$, and
$\Omega\coloneqq \Omega_{H+1}$.
A deterministic
\emph{non-Markovian} policy $\pi=\{\pi_h\}_{h\in\dsb{H}}$ is a sequence of
functions $\pi_h:\Omega_h\to\cA$ that, given the history up to stage $h$, i.e., 
$\omega=\tuple{s_1,a_1,\dotsc,s_{h-1},a_{h-1},s_h}\in\Omega_h$,
prescribes an action.
A \emph{Markovian} policy $\pi=\{\pi_h\}_{h\in\dsb{H}}$ is a sequence of
functions $\pi_h:\cS\to\cA$ that depend on the current state only.
We use $g:\bigcup_{h\in\{2,\dotsc,H+1\}}\Omega_h\to[0,H]$ to denote the return
of a (partial) trajectory $\omega\in\Omega_h$, i.e., $g(\omega)\coloneqq\sum_{h'\in\dsb{h-1}}
r_{h'}(s_{h'},a_{h'})$.
With abuse of notation, we denote by $\P_{p,r,\pi}$ the probability distribution
over trajectories of any length induced by $\pi$ in $\cM$ (we omit $s_0$ for
simplicity), and by $\E_{p,r,\pi}$ the expectation w.r.t.
$\P_{p,r,\pi}$.
We define the \emph{return distribution} $\eta^{p,r,\pi}\in\Delta^{[0,H]}$ of
policy $\pi$ as
$\eta^{p,r,\pi}(y)\coloneqq\sum_{\omega\in\Omega:\,g(\omega)=y}\P_{p,r,\pi}(\omega)$
for all $y\in[0,H]$. The set of possible returns at $h\in\dsb{H+1}$ is
$\cG^{p,r}_h \coloneqq\{y\in[0,h-1]\,|\,\exists \omega\in\Omega_h,\exists
\pi:\,g(\omega)=y\wedge \P_{p,r,\pi}(\omega)>0\}$, and
$\cG^{p,r}\coloneqq\cG^{p,r}_{H+1}$. We remark that $\cG^{p,r}_h$ has finite
cardinality for all $h$.
The performance of policy $\pi$ is given by $J^\pi(p,r)\coloneqq
\E_{p,r,\pi}[\sum_{h=1}^H r_{h}(s_h,a_h)]$, and note that
$J^\pi(p,r)=\E_{G\sim\eta^{p,r,\pi}}[G]$. We define the optimal
performance as $J^*(p,r)\coloneqq \max_\pi J^{\pi}(p,r)$, and the optimal policy
as $\pi^*\in\argmax_\pi J^{\pi}(p,r)$.

\textbf{Risk-Sensitive Markov Decision Processes (RS-MDPs).}~~A Risk-Sensitive
Markov Decision Process (RS-MDP) \citep{wu2023risksensitive} is a pair
$\cM_U\coloneqq\tuple{\cM,U}$, where
$\cM=\tuple{\cS,\cA,H,s_0,p,r}$ is an MDP, and $U\in\fU$ is a utility function
in set $\fU\coloneqq\{U':[0,H]\to[0,H]\,|\,U'(0)=0,U'(H)=H\wedge U'\text{ is
strictly-increasing and continuous}\}$. Differently from
\citet{wu2023risksensitive}, w.l.o.g., our utilities satisfy $U(H)=H$ to settle
the scale.
The interaction with $\cM_U$ is the same as with
$\cM$, and the notation described earlier still applies, except for the
performance of policies.
The performance of policy $\pi$ is $J^\pi(U;p,r)\coloneqq\E_{p,r,\pi}[ U(\sum_{h=1}^H
r_h(s_h,a_h))]$, and note that
$J^\pi(U;p,r)=\E_{G\sim\eta^{p,r,\pi}}[U(G)]$. We define the optimal
performance as $J^*(U;p,r)\coloneqq \max_\pi J^{\pi}(U;p,r)$, the optimal policy
as $\pi^*\in\argmax_\pi J^{\pi}(U;p,r)$, and the set of optimal policies for
$\cM_U$ as $\Pi^*_{p,r}(U)$.

\begin{figure*}[h!]
  \centering
  \begin{tikzpicture}[node distance=3.5cm]
      \node[state,initial] at (-0.5,0) (s0) {$s_0$};
      \node[state] at (2.5,0) (s1) {$s$};
      \node[state] at (7,1.2) (ar) {$s_1$};
      \node[state] at (7,0) (as) {$s_2$};
      \node[state] at (7,-1.2) (aq) {$s_3$};
      \node[draw=none,fill=black] at (5,0.6) (s2) {};
      \draw (s0) edge[->, solid, above,bend left=45] node{\footnotesize $a_1,r=0$\texteuro} (s1);
      \draw (s0) edge[->, solid, below,bend right=45] node{\footnotesize $a_2,r=1000$\texteuro} (s1);
      \draw (s1) edge[->, solid, above,sloped] node{\footnotesize $a_{\text{risky}},r=0$\texteuro} (s2);
      \draw (s1) edge[->, solid, below,sloped] node{\small $a_{\text{safe}},r=0$\texteuro} (aq);
      \draw (s2) edge[->, solid, above,sloped] node{\small $p=0.5$} (ar);
      \draw (s2) edge[->, solid, below,sloped] node{\small $p=0.5$} (as);
      \draw (ar) edge[->, solid, loop right] node{\small $r=150$\texteuro} (ar);
      \draw (as) edge[->, solid, loop right] node{\small $r=0$\texteuro} (as);
      \draw (aq) edge[->, solid, loop right] node{\small $r=50$\texteuro} (aq);
    \end{tikzpicture}
  \caption{The MDP considered in Example \ref{example: non markovian policy}.}
  \label{fig: MDP example non markovian policy}
\end{figure*}

\textbf{Enlarged state space approach.}~~ In MDPs, there always exists a
\emph{Markovian} optimal policy \citep{puterman1994markov}, but in RS-MDPs this
does not hold.
The \emph{enlarged state space approach} \citep{wu2023risksensitive} is a
method, proposed by \citet{bauerle2014more}, to compute an optimal policy in a
RS-MDP.
Given RS-MDP {\thickmuskip=1mu \medmuskip=1mu \thinmuskip=1mu
$\cM_U=\tuple{\cS,\cA,H,s_0,p,r,U}$}, we construct the \emph{enlarged} state
space MDP
$\fE[\cM_U]=\tuple{\cS',\cA,H,\tuple{s_0,0},\fp,\fr}$,
with a different state space $\cS'=\cS\times\cG^{p,r}_h$ at every
$h$.\footnote{Actually, \citet{bauerle2014more} use state space
$\cS\times\RR_{\ge 0}$, while \citet{wu2023risksensitive} use $\cS\times[h-1]$
for all $h\in\dsb{H}$. Instead, we consider sets $\cS\times\{\cG^{p,r}_h\}_h$ to
capture the minimal size required.}
For every $h\in\dsb{H}$ and {\thickmuskip=2mu
\medmuskip=1.5mu$(s,y,a)\in\cS\times\cG^{p,r}_h\times\cA$}, the
reward function $\fr$ is {\thickmuskip=2mu
\medmuskip=1.5mu$\fr_h(s,y,a)=U(y+r_h(s,a))\indic{h=H}$}, while the dynamics
$\fp$ assigns to the next state {\thickmuskip=2mu
\medmuskip=1.5mu$(s',y')\in\cS\times\cG^{p,r}_{h+1}$} the
probability: {\thickmuskip=2mu
\medmuskip=1.5mu$\fp_h(s',y'|s,y,a)\coloneqq p_h(s'|s,a)\indic{y'=y+r_h(s,a)}$}.
In words, the state space is enlarged with a component that keeps track of
the cumulative reward in the original RS-MDP, and the reward $\fr$, bounded in
$[0,H]$,
provides the utility of the accumulated reward at the end of the episode.
A Markovian policy $\psi=\{\psi_h\}_{h\in\dsb{H}}$ for $\fE[\cM_U]$ is a
sequence of mappings $\psi_h:\cS\times\cG^{p,r}_h\to\cA$. Being an MDP, we adopt
for $\fE[\cM_U]$ the same notation presented earlier for MDPs, by replacing
$p,r,\pi$ with $\fp,\fr,\psi$.
Let $\psi^*$ be the optimal \emph{Markovian} policy for $\fE[\cM_U]$.
Then, Theorem 3.1 of \citet{bauerle2014more} shows that the (non-Markovian)
policy $\pi^*$, defined for all $h\in\{2,\dotsc,H\}$ and {\thickmuskip=1mu
\medmuskip=1mu \thinmuskip=1mu$\omega\in\Omega_h$} as
$\pi^*_h(\omega)\coloneqq
\psi^*_h(s_h,\sum_{h'\in\dsb{h-1}}r_{h'}(s_{h'},a_{h'}))$, and
{\thickmuskip=1mu
\medmuskip=1mu \thinmuskip=1mu $\pi^*_1(s_0)=\psi^*_1(s_0,0)$}, is optimal for $\cM_U$.

\textbf{Inverse Reinforcement Learning (IRL).}~~In IRL we are given
demonstrations of behavior from the expert's policy $\pi^E$, and the goal is to
recover the reward of the expert $r^E$
\citep{russell1998learning}.
As explained in Section \ref{sec: introduction}, a \emph{model of behavior}
describes how the expert's policy $\pi^E$ is generated from $r^E$.
A model suffers from \emph{partial identifiability} if the knowledge of $\pi^E$
does not permit to recover $r^E$ (almost) uniquely
\citep{cao2021identifiability,metelli2021provably}.

\textbf{Miscellaneous.}~~For $L>0$, we write $\fU_L\coloneqq\{U\in\fU\,|\,U$ is
$L$-Lipschitz$\}$. For any finite set $\cX\subseteq[0,H]$ we define
{\thickmuskip=1mu \medmuskip=1mu
\thinmuskip=1mu$\overline{\fU}^{\cX}\coloneqq\{\overline{U}\in[0,H]^{|\cX|}\,|\,
\exists U\in\fU,\,\forall x\in\cX:\, \overline{U}(x)=U(x)\}$}, and
$\overline{\fU}_L^\cX\coloneqq\{\overline{U}\in\overline{\fU}^\cX\,|\,\exists
U\in\fU_L,\,\forall x\in\cX:\, \overline{U}(x)=U(x)\}$.
We will denote by $\cM_{\overline{U}}$ some RS-MDPs with
$\overline{U}\in\overline{\fU}^{\cX}$.

\section{A New Model of Behavior}
\label{sec: motivation and problem setting}

We aim to devise a realistic model of behavior for humans and rational agents in
MDPs that complies with their sensitivity to risk. In fact, due to the
stochasticity of the environment, they are likely to behave in a risk-sensitive
manner. Our insight is that risk-sensitivity in MDPs gives rise to
\emph{non-Markovian} policies for both rational agents (see
\citet{bellemare2023distributional}) and humans:
\begin{example}\label{example: non markovian policy}
In the MDP of Fig. \ref{fig: MDP example non markovian policy},
we expect most people to decide what action to play in state $s$
\emph{depending} on the amount of reward earned so far, since, intuitively, it
makes more sense to take the risky action $a_{\text{risky}}$, that sometimes
gives a large return ($150$\texteuro) but sometimes gives nothing
($0$\texteuro), when we are guaranteed to have at least $1000$\texteuro\ in our
wallet (i.e., we have reached $s$ from $a_2$), while it may be better to take
the safe action $a_{\text{safe}}$, that gives $50$\texteuro\ for sure, if we
reached $s$ with no reward (i.e., from $a_1$).
This kind of behavior is known as ``decreasing'' risk-aversion.
\citep{pratt1964riskaversion,Kreps1988NotesOT,Wakker2010prospect}.
\end{example}
In short, demonstrations of behavior from risk-sensitive agents in MDPs are
likely to be collected by \emph{non-Markovian} policies, whose dependency on the
past history is restricted to the \emph{amount of reward} collected so far.
However, none of the existing IRL models of behavior (see Sections \ref{sec:
introduction} and \ref{sec: related works}) contemplate non-Markovian policies,
and, thus, they result in misspecification.\footnote{ Re-modelling the MDP
including the sum of the past rewards into the state would make the demonstrated
policy Markovian, but, as explained in Appendix \ref{apx: drawbacks reward into
state}, it would create various issues like a state space with a size
exponential in the horizon. }

For these reasons, we introduce a new model of behavior that contemplates
non-Markovian policies. Given demonstrations from the expert's policy $\pi^E$ in
an environment $\tuple{\cS,\cA,H,s_0,p}$, we assume the existence of a reward
function $r^E$ and a utility function $U^E\in\fU$ such that:
{\thinmuskip=1mu \medmuskip=1mu \thickmuskip=2mu
\begin{equation}\label{eq: model expert rs mdp}
  \pi^E\in\argmax\limits_{\pi}
  \E_{p,r^E,\pi}
  \Big[
  U^E\Big(\sum\limits_{h=1}^H
  r_h^E(s_h,a_h)\Big)
  \Big],
\end{equation}
}\.
i.e., we model the expert as an \emph{optimal agent in a RS-MDP}.
The reward $r^E$ aims to capture the task of the expert, while the utility $U^E$
represents its risk attitude.
Intuitively, if $p$ is deterministic, then the trajectories with the largest
returns under $r^E$ are preferred. However, in presence of stochasticity, the
utility $U^E$ associates weights to the returns of the trajectories to represent
their true ``values'' for the expert.
If $U^E$ is linear, then $\argmax_\pi J^\pi(U^E;p,r^E) = \argmax_\pi
J^\pi(p,r^E)$ and the expert values each trajectory by its return under $r^E$,
i.e., it is \emph{risk-neutral}. However, if $U^E$ is convex (resp. concave),
then the expert amplifies (resp. attenuates) the desirability of high-return
trajectories, so that it will accept even more (resp. less) variance to play
them. In this case, $U^E$ represents a \emph{risk-seeking} (resp.
\emph{risk-averse}) expert \citep{Kreps1988NotesOT,bauerle2014more}.

There are many arguments that support this model:
\begin{enumerate}[leftmargin=*, noitemsep, topsep=-2pt]
  \item it generalizes the IRL model of \citet{ng2000algorithms}, that we get if
  the expert is risk-neutral ($U^E$ is linear);
  \item it is justified by the famous expected utility theory
  \citep{vnm1947theory}, as we can interpret each policy $\pi$ as a choice that
  induces a lottery $\eta^{p,r^E,\pi}$ over the set of prizes (i.e., returns)
  $\cG^{p,r^E}$;
  \item it contemplates the existence of non-Markovian policies that depend only
  on the cumulative reward so far (see \citet{bauerle2014more});
  \item the corresponding planning problem enjoys practical tractability
  \citep{wu2023risksensitive};
  \item $U^E$ can be learned efficiently, as we show in Section \ref{sec: online
  utility learning}.
\end{enumerate}
\textbf{Some considerations.}~~ If $U^E$ is linear, Eq. \eqref{eq: model expert
rs mdp} admits a \emph{Markovian} optimal policy \citep{puterman1994markov}.
Otherwise, the more $U^E$ deviates from linearity, the more non-Markovian
policies \emph{may} outperform Markovian policies:
\begin{restatable}{prop}{loseinperformancewithmarkovianity}
  \label{prop: lose in performance with markovianity}
  There exists a RS-MDP in which the difference between the optimal performance
  and the performance of the best Markovian policy is $0.5$.
\end{restatable}
Next, note that, in absence of stochasticity, $U^E$ plays no role, and Eq.
\eqref{eq: model expert rs mdp} traces back to risk-neutral behavior, as
desired:
\begin{restatable}{prop}{envdeterministic}
  \label{prop: env deterministic}
  If $p$ is deterministic, then $\argmax_\pi J^\pi(U^E;p,r^E) = \argmax_\pi
  J^\pi(p,r^E)$.
\end{restatable}
We remark that, by complying with non-Markovian policies, our model of behavior
suffers from less misspecification than common IRL models. Moreover, by using
$U^E$, it permits to learn a succinct and transferrable representation of the
risk attitude of the expert, as we shall see later.

\section{Utility Learning}\label{sec: utility learning}

In this and in the following sections, we focus on the problem of learning the
utility $U^E$ of the expert under the assumption that it behaves as in Eq.
\eqref{eq: model expert rs mdp}. Here, we assume that the expert's policy
$\pi^E$ and the dynamics $s_0,p$ are known, while in Section \ref{sec: online
utility learning} we will estimate them from finite data.

\textbf{Problem definition and partial identifiability.}~~%
Given demonstrations collected by a policy $\pi^E$ satisfying Eq. \eqref{eq:
model expert rs mdp}, three different learning problems arise:
\begin{enumerate}[leftmargin=*, noitemsep, topsep=-2pt]
  \item given $r^E$, learn $U^E$;
  \item given $U^E$, learn $r^E$; %
  \item learn both $r^E$ and $U^E$.
\end{enumerate}
Problem 3 is the most interesting and challenging, because it makes the least
assumptions,
while Problem 2, i.e., IRL, has been extensively studied in literature when
$U^E$ is linear \citep{ng2000algorithms}.
In this paper, in analogy to ICRL \citep{malik2021icrl} where the goal is to
learn the constraints when $r^E$ is known, we focus on Problem 1 because it has
relevant applications per se (see later in this section) and because it
represents a significant step toward solving Problem 3. Thus, we will consider
$r^E$ to be given and denote it with $r$ for simplicity.
Let us formalize Problem 1:
\begin{defi}[Utility Learning (UL)]\label{def: ul}
  Let $\cM=\tuple{\cS,\cA,H, s_0,p,r}$ be an MDP and $\pi^E$ a (potentially
 non-Markovian) policy. Under the assumption that $\pi^E$ satisfies Eq.
 \eqref{eq: model expert rs mdp} in $\cM$ for some unknown $U^E$, the goal of
 \emph{Utility Learning (UL)} is to find $U^E$.
\end{defi}
Does the knowledge of $\pi^E$ and $\cM$ suffice to \emph{uniquely} identify
$U^E$? Analogously to IRL \citep{cao2021identifiability} and ICRL
\citep{kim2023learningshared}, the answer is \emph{negative}, as shown in the
following example (details in Appendix \ref{apx: section 4}).
\begin{restatable}{example}{examplepartialidentifiability}
  \label{example: UE partially identifiable}
  {\thickmuskip=2mu \medmuskip=1.5mu \thinmuskip=1.5mu Consider the MDP $\cM$ in
  Fig. \ref{fig: example fs} (left), where
  $H=2,r_1(s_0,a_1)=1,r_1(s_0,a_2)=0.5$. Let the expert's policy $\pi^E$
  prescribe $a_1$ in $s_0$. 
 Then, all the utility functions $U\in\fU$ that take on values in the blue
 region of Fig. \ref{fig: example fs} (middle) for returns $G=1,G=1.5$, make
 $\pi^E$ optimal in $\cM_U$.}
\end{restatable}
\begin{figure*}[t!]
  \begin{minipage}[t!]{0.32\textwidth}
    \scalebox{0.90}{
    \begin{tikzpicture}[scale=0.8]
        \node[state] at (0,0) (s0) {$s_0$};
        \node[state] at (3,2) (s1) {$s_1$};
        \node[state] at (3,0) (s2) {$s_2$};
        \node[state] at (3,-2) (s3) {$s_3$};
        \node[draw=none] at (-1,0) (s00) {};
        \node[draw=none, fill=black] at (0.8,1.4) (a1) {};
        \node[draw=none,fill=black] at (0.8,-1.4) (a2) {};
        \node[state, draw=none] at (5,2) (ss1) {};
        \node[state, draw=none] at (5,0) (ss2) {};
        \node[state, draw=none] at (5,-2) (ss3) {};
        \draw (s00) edge[->, solid, above] node{} (s0);
        \draw (s0) edge[-, solid, above, pos=0.05] node[above=0.2cm]{\scriptsize$a_1$} (a1);
        \draw (s0) edge[-, solid, above, pos=0.9] node[above=0.2cm]{\scriptsize$a_2$} (a2);
        \draw (a1) edge[->, solid, above, sloped] node{\scriptsize$0.4$} (s1);
        \draw (a1) edge[->, solid, above, sloped] node{\scriptsize$0.5$} (s2);
        \draw (a1) edge[->, solid, above, pos=0.5, sloped] node{\scriptsize$0.1$} (s3);
        \draw (a2) edge[->, solid, above, sloped] node{\scriptsize$0.2$} (s2);
        \draw (a2) edge[->, solid, above, sloped] node{\scriptsize$0.8$} (s3);
        \draw (s1) edge[->, solid, above] node{\scriptsize$r\!=\!0$} (ss1);
        \draw (s2) edge[->, solid, above, pos=0.6] node{\scriptsize$r\!=\!0.5$} (ss2);
        \draw (s3) edge[->, solid, above] node{\scriptsize$r\!=\!1$} (ss3);
  \end{tikzpicture}}
 \end{minipage}
 \hfill
  \begin{minipage}[t!]{0.29\textwidth}
    \scalebox{0.90}{
    \begin{tikzpicture}
      \draw[->] (-1,0) -- (3,0) node[right] {\scriptsize$U(1)$};
      \draw[->] (0,-1) -- (0,3) node[above] {\scriptsize$U(1.5)$};
      
      \draw[very thin, gray!50, dashed] (-0.9,-0.9) grid[step=0.5] (2.8,2.8);
 
      \node at (-0.15,-0.15) {0};
      
      \draw (1,0.1) -- (1,-0.1) node[below] {1};
      \draw (2,0.1) -- (2,-0.1) node[below] {2};
      \draw (0.1,1) -- (-0.1,1) node[left] {1};
      \draw (0.1,2) -- (-0.1,2) node[left] {2};
 
      \draw[thick] (0,0) rectangle (2,2);
 
      \coordinate (A) at (0,0);
      \coordinate (B) at (0,2/3);
      \coordinate (C) at (2,2);
      
      \filldraw[fill=blue!30, draw=black, thick] (A) -- (B) -- (C) -- cycle;
 
      \fill[red] (0.1,0.7) circle (1.5pt);
 
      \node at (0.2,0.5) {\scriptsize $U'$};
 
  \end{tikzpicture}}
 \end{minipage}
 \hfill
 \begin{minipage}[t!]{0.29\textwidth}
  \scalebox{0.90}{
    \begin{tikzpicture}
      \draw[->] (-1,0) -- (3,0) node[right] {\scriptsize$ G$};
      \draw[->] (0,-1) -- (0,3) node[above] {\scriptsize$ U'$};
      
      \draw[very thin, gray!50, dashed] (-0.9,-0.9) grid[step=0.5] (2.8,2.8);
 
      \node at (-0.15,-0.15) {0};
      
      \draw (0.5,0.1) -- (0.5,-0.1) node[below] {\scriptsize 0.5};
      \draw (1,0.1) -- (1,-0.1) node[below] {\scriptsize 1};
      \draw (1.5,0.1) -- (1.5,-0.1) node[below] {\scriptsize 1.5};
      \draw (2,0.1) -- (2,-0.1) node[below] {\scriptsize 2};
      \draw (0.1,1) -- (-0.1,1) node[left] {1};
      \draw (0.1,2) -- (-0.1,2) node[left] {2};
 
      \fill[red] (0,0) circle (2pt);
      \fill[red] (1,0.1) circle (2pt);
      \fill[red] (1.5,0.7) circle (2pt);
      \fill[red] (2,2) circle (2pt);
 
      \draw[red] (0,0) -- (1,0.1) node {};
      \draw[red] (1,0.1) -- (1.5,0.7) node {};
      \draw[red] (1.5,0.7) -- (2,2) node {};
  \end{tikzpicture}}
 \end{minipage}
 \caption{(Left) MDP of Example \ref{example: UE partially identifiable}.
 (Middle) its feasible set with a sample utility $U'$. (Right) plot of $U'$ with
 linear interpolation.
}
 \label{fig: example fs}
\end{figure*}
Simply put, in UL, the only information available on the unknown utility $U^E$
is that it belongs to $\fU$ and it makes $\pi^E$ an optimal policy in the
corresponding RS-MDP. Since Example \ref{example: UE partially identifiable}
shows that, in general, there is a set of utilities $U\in\fU$ satisfying this
condition, we realize that $U^E$ is partially identifiable.
Analogously to \citet{metelli2021provably,metelli2023towards}, we call such set
the \emph{feasible set} of utilities ``compatible'' with $\pi^E$ in
$\cM$:\footnote{In Appendix \ref{apx: section 4} we provide a more explicit
expression.}
\begin{align}\label{eq: fs def}
  \cU_{p,r,\pi^E}\coloneqq\{U\in\fU\,|\,J^{\pi^E}(U;p,r)=J^*(U;p,r)\}.
\end{align}

\textbf{Applications.}~~%
If we knew the risk attitude of the expert, i.e., its utility $U^E$ in our
model, then we could use it for applications like $(i)$ \emph{predicting} the
behavior of the expert in a new environment, $(ii)$ \emph{imitating} the expert,
or $(iii)$ \emph{assessing} how valuable a certain behavior is from the
viewpoint of the expert.
UL represents an appealing problem setting for learning $U^E$ from
demonstrations of behavior.
However, due to partial identifiability, no learning algorithm can recover
$U^E$, but, at best, it can find an \emph{arbitrary} utility in the feasible set
$\cU_{p,r,\pi^E}$. Is this ambiguity tolerated by the applications $(i),(ii)$,
and $(iii)$ above? In other words, we are interested in understanding whether
\emph{all} the utilities contained into $\cU_{p,r,\pi^E}$ can be used in place
of the true $U^E$ without incurring in large
errors.\footnote{\citet{skalse2023invariancepolicyoptimisationpartial} conduct
an analogous study for IRL.}
Unfortunately, the following propositions answer \emph{negatively} for all
$(i),(ii)$, and $(iii)$.
Nevertheless, Proposition \ref{prop: multiple demonstrations} shows that the
availability of expert demonstrations from multiple environments is a possible
mitigation for the issue.

Let us begin with $(i)$. We say that a utility $U$ permits to
\emph{predict} the behavior of an agent with utility $U^E$ in a new MDP $\cM'$
if $U$ and $U^E$ induce in $\cM'$ the same optimal policies.
The next two propositions show that if the transition model or the reward
function of $\cM'$ differ from those of the original MDP $\cM$, then there are
utilities in the feasible set that get wrong in predicting the behavior of the
agent with $U^E$:
\begin{restatable}%
  {prop}{proptransferabilityp}\label{prop: transferring utilities p}
  There exist two MDPs $\cM=\tuple{\cS,\cA,H,s_0,p,r}$,
 $\cM'=\tuple{\cS,\cA,H,s_0,p',r}$, with $p\neq p'$, for which there exist a
 policy $\pi^E$ and a pair of utilities $U_1,U_2\in\cU_{p,r,\pi^E}$ such that
 $\Pi_{p',r}^*(U_1)\cap\Pi_{p',r}^*(U_2)=\{\}$.
\end{restatable}
\begin{restatable}%
{prop}{proptransferabilityr}\label{prop: transferring utilities r}
 There exist two MDPs $\cM=\tuple{\cS,\cA,H,s_0,p,r}$,
 $\cM'=\tuple{\cS,\cA,H,s_0,p,r'}$, with $r\neq r'$, for which there exist a
 policy $\pi^E$ and a pair of utilities $U_1,U_2\in\cU_{p,r,\pi^E}$ such that
 $\Pi_{p,r'}^*(U_1)\cap\Pi_{p,r'}^*(U_2)=\{\}$.
 \end{restatable}
Consider now $(ii)$. We say that a utility $U$ permits to \emph{imitate} the
behavior of an agent with utility $U^E$ if optimizing $U$ provides policies with
a large expected utility w.r.t. $U^E$.
The reason behind this definition is that, differently from IRL, we wish to
imitate also the risk attitude of the observed agent.
However, UL does not always allow to perform meaningful imitations:
\begin{restatable}{prop}{propimitation}
\label{prop: imitating the expert}
There exists an MDP $\cM=\tuple{\cS,\cA,H,s_0,p,r}$ and a policy $\pi^E$ for
which there are utilities $U_1,U_2\in\cU_{p,r,\pi^E}$ such that, for any
$\epsilon\ge 0$ smaller than some universal constant, there exists a policy
$\pi_\epsilon$ such that $J^*(U_1;p,r)-J^{\pi_\epsilon}(U_1;p,r)=\epsilon$ and
$J^*(U_2;p,r)-J^{\pi_\epsilon}(U_2;p,r)\ge 1$.
\end{restatable}
Concerning $(iii)$, we say that $U$ and $U^E$ \emph{assess} behavior in a
similar way if, given any policy, they provide close values of performance.
The intuition is that the expert values policies based on their alignment with
its risk attitude $U^E$ w.r.t. its task $r$.
Formally, we want $U$ such that {\thickmuskip=2mu \medmuskip=1.5mu
\thinmuskip=1.5mu$d_{p,r}^{\text{all}}(U^E,U)\coloneqq \max_{\pi}
\big|J^\pi(U^E;p,r)-J^\pi(U;p,r)\big|$} is small \cite{zhao2023inverse}.
Nonetheless, \emph{not} all the utilities in the feasible set are close to each
other w.r.t. $d^{\text{all}}_{p,r}$:
\begin{restatable}{prop}{propnobounddall}
  \label{prop: no bound d all}
  There exists an MDP $\cM=\tuple{\cS,\cA,H,s_0,p,r}$ and a policy $\pi^E$ for
     which there exists a pair of utilities $U_1,U_2\in\cU_{p,r,\pi^E}$ such
     that $d^{\text{all}}_{p,r}(U_1,U_2)= 1$.
\end{restatable}
Propositions \ref{prop: transferring utilities p}-\ref{prop: no bound d all}
tell us that demonstrations of behavior in a \emph{single} MDP do not provide
enough information on $U^E$ for applications $(i),(ii)$, and
$(iii)$.\footnote{Actually, for $(ii)$ only, we can try to learn $\pi^E$
directly without passing through $U^E$, as in behavioral cloning
\citep{osa2018IL}.}
Thus, we might hope that expert demonstrations in \emph{multiple} environments
can help in mitigating this issue, similarly to what is done in IRL
\citep{amin2016resolvingunidentifiabilityinversereinforcement,cao2021identifiability}
and ICRL \citep{kim2023learningshared}.
Formally, we extend the UL problem of Definition \ref{def: ul} to a set of MDPs
$\{\cM^i\}_i$, with $\cM^i=\tuple{\cS^i,\cA^i,H,s_0^i,p^i,r^i}$,\footnote{For
simplicity, we let $H$ be shared.} and policies $\{\pi^{E,i}\}_i$ by assuming
that there exists a single utility $U^E$ for which Eq. \eqref{eq: model expert
rs mdp} is satisfied for all $i$, i.e., such that $\pi^{E,i}$ is optimal for
$\cM^i_{U^E}$ for all $i$.
In this extended problem setting, the feasible set will be the intersection of
all the feasible sets $\cU_{p^i,r^i,\pi^{E,i}}$.
The following result proves that demonstrations in multiple environments is a
\emph{possible} solution to the partial identifiability problem. 
\begin{restatable}{prop}{propmultipledemonstrations}
\label{prop: multiple demonstrations}
{\thickmuskip=1.5mu \medmuskip=1.5mu \thinmuskip=1.5mu Let $\cS,\cA,H$ be any
state space, action space, and horizon, satisfying $S\ge 3,A\ge2,H\ge 2$, and
let $U^E\in\fU$ be any utility. If, for any possible dynamics $s_0,p$ and reward
$r$, we are given the set of \emph{all} the deterministic optimal policies of
the corresponding RS-MDP $\tuple{\cS,\cA,H, s_0,p,r,U^E}$, then we can
\emph{uniquely} identify $U^E$.}
\end{restatable}

\section{Online UL with Generative Model}
\label{sec: online utility learning}

In this section, we present two provably-efficient algorithms for solving the UL
problem in a \emph{finite-data} regime.

\subsection{Problem Setting}\label{sec: problem setting online UL}

We consider a finite-data version of the UL problem with demonstrations in
multiple environments presented in Section \ref{sec: utility learning}.
We let $\{\cM^i\}_{i\in\dsb{N}}$, with
$\cM^i=\tuple{\cS^i,\cA^i,H,s_0^i,p^i,r^i}$, be the $N$ MDPs with shared horizon
$H$ in which an expert with utility $U^E\in\fU$ provides demonstrations of
behavior. Specifically, for each MDP $\cM^i$, the expert provides us with a
batch dataset
$\cD^{E,i}=\{\tuple{s_1^j,a_1^j,s_2^j,\dotsc,s_H^j,a_H^j,s_{H+1}^j}
\}_{j\in\dsb{\tau^{E,i}}}$ of $\tau^{E,i}$ trajectories collected by executing a
policy $\pi^{E,i}$, which is optimal for the RS-MDP $\cM^i_{U^E}$.
Moreover, for every $\cM^i$, we let $\cS^i,\cA^i,H,s_0^i,r^i$ be known, and we
consider access to a \emph{generative sampling model} \citep{azar2013minimax}
for the transition model $p^i$, which allows us to collect a sample $s'\sim
p^i_h(\cdot|s,a)$ from any triple $s,a,h$ at our choice.
In short, we assume access to \emph{offline} data for the expert and to
\emph{online} data for the environments, as is common in the IRL literature
\citep{ho2016generativeadversarialimitationlearning}.

Due to partial identifiability, the feasible set $\bigcap_i
\cU_{p^i,r^i,\pi^{E,i}}$ might contain multiple utilities. Thus, we will develop
two different algorithms, one that aims to classify utilities as inside or
outside the feasible set $\bigcap_i \cU_{p^i,r^i,\pi^{E,i}}$ (\caty, Section
\ref{sec: caty}), and the other that aims to compute a single utility contained
into it (\tractor, Section \ref{sec: tractor}).
Intuitively, \caty and \tractor together permit to fully characterize the
feasible set, by, respectively, learning a classification boundary and a
representative item.
Nevertheless, note that, because of finite data, we will be able to provide
guarantees only for a relaxation of the feasible set
$\cU_\Delta\supseteq\bigcap_i \cU_{p^i,r^i,\pi^{E,i}}$ for some $\Delta\ge0$:
\begin{align}\label{eq: U delta def}
  \cU_\Delta\coloneqq\Big\{U\in\fU\,|\,
  \sum\limits_{i\in\dsb{N}}\overline{\cC}_{p^i,r^i,\pi^{E,i}}(U)
\le\Delta\Big\},
\end{align}
where $\overline{\cC}_{p^i,r^i,\pi^{E,i}}(U)$ quantifies the \emph{(non)compatibility} of
utility $U$ with demonstrations from $\pi^{E,i}$ in $\cM^i$
\citep{lazzati2024compatibility,lazzati2025journal}:
\begin{align}\label{eq: compatibility def}
  \overline{\cC}_{p^i,r^i,\pi^{E,i}}(U)\coloneqq
J^*(U;p^i,r^i)-J^{\pi^{E}}(U;p^i,r^i).
\end{align}
Intuitively, $\cU_\Delta$ enlarges the feasible set by accepting utilities that
make the policies $\pi^{E,i}$ at most $\Delta$-suboptimal. Note that, for
$\Delta=0$, we have $\cU_\Delta=\bigcap_i \cU_{p^i,r^i,\pi^{E,i}}$.

\subsection{Challenges and Solution}\label{sec: our solution}

To develop \emph{practical} algorithms, some dimensionality challenges must be
addressed. In this section, we explain how we will face them. In short, our
solution permits to work with tractable approximations whose complexity is
controlled by a discretization parameter $\epsilon_0>0$.
First, we need some notation. We use symbol $\cY_h$ to denote an
$\epsilon_0$-discretization of the real-valued interval $[0,h-1]$, i.e., we set
$\cY_h\coloneqq\{0,\epsilon_0,2\epsilon_0
,\dotsc,\floor{(h-1)/\epsilon_0}\epsilon_0\}$ $\forall h\in\dsb{H+1}$.
Moreover, we introduce ad-hoc symbols $\cR,\cY$ for the discretization of
intervals $[0,1]$ and $[0,H]$, namely, we let $\cR\coloneqq\cY_2$,
$\cY\coloneqq\cY_{H+1}$. We also set $d\coloneqq|\cY|=\floor{H/\epsilon_0}$.

\textbf{Working with continuous utilities.}~~Utilities in $\fU$ are defined over
the real-valued interval $[0,H]$, making them incompatible with the finite
precision of computers. For this reason, we will consider \emph{discretized}
utilities.
Formally, we will approximate any $U\in\fU$ with a $d$-dimensional vector
$\overline{U}\in\overline{\fU}\coloneqq$\scalebox{0.9}{$\overline{\fU}^\cY$},
such that $\overline{U}(y)=U(y)$ $\forall y\in\cY$.

\textbf{Return distributions.}~~%
In MDPs with dynamics $p$ and reward $r$, return distributions $\eta$ are
supported on the set of possible returns {\thinmuskip=1mu \medmuskip=1mu
\thickmuskip=1mu$\cG^{p,r}\subset[0,H]$. However, in general, the size of this
set grows exponentially in the horizon $|\cG^{p,r}|\propto (SA)^H$, causing any
exact representation of $\eta$ to explode even for small $H$.}
{\thinmuskip=1mu \medmuskip=1mu \thickmuskip=1mu Thus, we adopt a
\emph{categorical representation} for return distributions
\citep{bellemare2023distributional}, that, roughly speaking, aims to approximate
a distribution on $[0,H]$ with a distribution on $\cY\subset[0,H]$.}
{\thinmuskip=1mu \medmuskip=1mu \thickmuskip=1mu Formally, given any
$\eta\in\Delta^{[0,H]}$ with finite support, its categorical representation
$\text{Proj}_\cC(\eta)$ is the distribution in
$\cQ\coloneqq\{q\in\Delta^{\dsb{d}}\,|\, \sum_{j\in\dsb{d}}q_j\delta_{y_j}\}$
($y_j$ are the items of $\cY$) obtained through the categorical projection
operator $\text{Proj}_\cC$ \citep{rowland2018analysis}, reported in Eq.
\eqref{eq: definition proj c}-\eqref{eq: property proj c} in Appendix \ref{apx:
additional notation}.}

\textbf{Optimal policies in RS-MDPs.}~~%
To compute an optimal policy in RS-MDP $\cM_U$ with dynamics $p$ and reward $r$,
the \emph{enlarged state space approach} of \citet{bauerle2014more} presented in
Section \ref{sec: preliminaries} requires the computation of an optimal policy
in the MDP $\fE[\cM_U]$, whose state space is $\cS\times \cG^{p,r}_h$ $\forall
h$.
Unfortunately, we suffer again from an exponential dependence on the horizon
$|\cG^{p,r}_h|\propto (SA)^{h-1}$, that causes any exact representation of the
state space and of the optimal policy of MDP $\fE[\cM_U]$ to explode. 
To avoid this issue, we adopt the \emph{discretization} approach of
\citet{wu2023risksensitive}, which, in short, amounts to approximate sets
$\cG^{p,r}_h$ with $\cY_h$ by simply replacing reward $r$ with the discretized
version $\overline{r}$, defined as: $\overline{r}_h(s,a)\coloneqq
\Pi_{\cR}[r_h(s,a)]$ for all $s,a,h$.
Crucially, since $\overline{r}_h(s,a)\in\cR$, then the sum of $h$ rewards
$\overline{r}$ belongs to $\cY_{h+1}$.
In this manner, the sets of partial returns satisfy
$\cG^{p,\overline{r}}_h\subseteq \cY_h\subseteq\cY$ for all $h$, thus, the state
space of the enlarged MDP has now a cardinality at most $Sd\le
\cO(SH/\epsilon_0)$, which is no longer exponential in the horizon.
In the following, we will denote by $\overline{r}^i$ the discretized version of
reward $r^i$ for all $i\in\dsb{N}$.

\subsection{\caty (\catylong)}
\label{sec: caty}

\input{algorithms/caty_classification.tex}

The goal of \caty (Algorithm \ref{alg: caty classification}) is to classify
input utilities $U\in\fU$ based on whether they belong to set $\cU_\Delta$ or
not for some $\Delta\ge0$.
We implement it using two different subroutines
\citep{lazzati2024compatibility,lazzati2025journal}. First, Algorithm \ref{alg:
caty exploration} (reported in Appendix \ref{apx: section 5} for its simplicity)
actively explores the $N$ environments $\cM^i$ \emph{uniformly}, by collecting
$\tau^i$ samples from each transition model $p^i$, and uses these samples to
construct estimates $\widehat{p}^i$.
Next, Algorithm \ref{alg: caty classification} uses these estimates along with
the expert's data $\cD^{E,i}$ to classify any input utility $U\in\fU$ w.r.t.
$\cU_\Delta$. To perform the classification, based on Eq. \eqref{eq: U delta
def}, \caty computes estimates $\widehat{\cC}^i(U)\approx
\cC_{p^i,r^i,\pi^{E,i}}(U)$ for all $i\in\dsb{N}$, and, then, it outputs whether
$\sum_{i \in \dsb{N}}\widehat{\cC}^i(U)\le\Delta$ (see Line \ref{line:
classify}).
To compute $\widehat{\cC}^i(U)$, driven by Eq. \eqref{eq: compatibility def},
\caty computes two separate estimates $\widehat{J}^{E,i}(U)\approx
J^{\pi^{E,i}}(U;p^i,r^i)$ (Lines \ref{line: erd caty}-\ref{line: est JE caty})
and $\widehat{J}^{*,i}(U)\approx J^*(U;p^i,r^i)$ (Line \ref{line: planning
caty}), and then combines them (Line \ref{line: estimate C}).
Specifically, the \texttt{ERD} (Estimate Return Distribution) subroutine
(Algorithm \ref{alg: erd}) permits to construct an estimate
$\widehat{\eta}^{E,i}$ of the categorical projection
$\text{Proj}_\cC(\eta^{p^i,r^i,\pi^{E,i}})$ of the expert's return distribution
$\eta^{p^i,r^i,\pi^{E,i}}$ from dataset $\cD^{E,i}$, which is used at Line
\ref{line: est JE caty} to compute $\widehat{J}^{E,i}(U)$.
Instead, quantity $\widehat{J}^{*,i}(U)$ is calculated at Line \ref{line:
planning caty} as the optimal performance in the RS-MDP $
\tuple{\cS^i,\cA^i,H,s_0^i,\widehat{p}^i,\overline{r}^i,\overline{U}}$, which
is computed through value iteration \citep{puterman1994markov} in the
corresponding \emph{enlarged state space MDP} using the \texttt{PLANNING}
subroutine (Algorithm \ref{alg: planning}).
\caty enjoys the following guarantee for $L$-Lipschitz input utilities:
\begin{restatable}{thr}{thrupperboundcatyoneu}
  \label{thr: caty upper bound 1 u}
  Let $L>0$, $\epsilon,\delta\in(0,1)$, and let $\cU\subseteq\fU_L$ be the set
  of utilities to classify.
  For all $i\in\dsb{N}$, in case $|\cU|=1$, let the number of samples satisfy:
\begin{align*}
  \scalebox{0.88}{$  \displaystyle \tau^{E,i}\ge
  \widetilde{\cO}\Big(\frac{N^2H^2}{\epsilon^2}\log\frac{N}{\delta}\Big),
  \;\tau^i\ge\widetilde{\cO}\Big(\frac{N^2SAH^4}{\epsilon^2}
  \log\frac{SAHNL}{\delta\epsilon}\Big). $}
\end{align*}
Otherwise, if $|\cU|>1$, let the number of samples satisfy:
\begin{align}\label{eq: bound all U thr caty}
  \begin{split}
    &\scalebox{0.88}{
  $  \displaystyle\tau^{E,i}\ge
    \widetilde{\cO}\Big(\frac{N^4H^4 L^2}{\epsilon^4}\log\frac{HNL}{\delta\epsilon}\Big),$}\\
    &\scalebox{0.88}{$  \displaystyle
    \tau^i\ge\widetilde{\cO}\Big(\frac{N^2SAH^5}{\epsilon^2}
    \Big(S+\log\frac{SAHN}{\delta}\Big)\Big). $}
  \end{split}
\end{align}
  Then, setting $\epsilon_0=\epsilon^2/(72HL^2N^2)$, w.p. at least $1-\delta$,
  for any $\Delta\ge0$, \caty \emph{correctly classifies} all the $U\in\cU$
  lying inside $\cU_{\Delta-\epsilon}$ or outside $\cU_{\Delta+\epsilon}$.
\end{restatable}
Roughly speaking, this theorem says that, for a number of samples independent of
$\Delta$, \caty correctly recognizes all the utilities in
$\cU_{\Delta-\epsilon}\subseteq\cU_\Delta$ and outside
$\cU_{\Delta+\epsilon}\supseteq\cU_\Delta$ as, respectively, inside and outside
set $\cU_\Delta$.
Intuitively, the Lipschitzianity assumption is necessary for approximating
functions $U\in\cU$ with vectors in $\overline{\fU}$.
We remark that, if $|\cU|=1$, then $\propto S$ queries to the generative model
suffice. Otherwise, we require $\propto S^2$ samples.

\subsection{\tractor (\tractorlong)}
\label{sec: tractor}

\input{algorithms/tractor.tex}

For simplicity of presentation, we introduce some notation. For any $L>0$, let
$\overline{\fU}_L\coloneqq$\scalebox{0.9}{$\overline{\fU}_L^\cY$}, and let
$\underline{\fU},\underline{\fU}_L,
\overline{\underline{\fU}},\overline{\underline{\fU}}_L,\underline{\cU}_\Delta$
be the analogous of, respectively, $\fU,\fU_L,
\overline{\fU},\overline{\fU}_L,\cU_\Delta$, but containing \emph{increasing}
functions instead of \emph{strictly-increasing} functions.
\tractor (Algorithm \ref{alg: tractor}) is a more ``practical'' UL algorithm, in
that it aims to compute a utility function contained into the feasible set
$\bigcap_i \cU_{p^i,r^i,\pi^{E,i}}$.
As \caty, it comprises an initial exploration phase (Algorithm \ref{alg: caty
exploration}), that collects $\tau^i$ samples to compute estimates
$\widehat{p}^i$ of the transition models $p^i$, and an extraction phase
(Algorithm \ref{alg: tractor}), where these estimates and the expert's data
$\cD^{E,i}$ are used to compute a utility (almost) in the feasible set.
Specifically, since the utilities $U$ in the feasible set satisfy $\sum_i
\overline{\cC}_{p^i,r^i,\pi^{E,i}}(U)=0$, \tractor aims to find a minimum of
function $\sum_i \overline{\cC}_{p^i,r^i,\pi^{E,i}}(\cdot)$ over the set
$\underline{\fU}_L$. So, starting from an initial $d$-dimensional utility
$\overline{U}_0\in\overline{\underline{\fU}}_L$, \tractor computes a sequence
$\overline{U}_1,\dotsc, \overline{U}_T$ by performing \emph{online projected
gradient descent}\footnote{This approach is based on
\citep{syed2007game,schlaginhaufen2024transferability}.}
\citep{orabona2023modernintroductiononlinelearning} in the space of discretized
$L$-Lipschitz utilities $\overline{\underline{\fU}}_L$, where the gradient $g_t$
is computed at Line \ref{line: compute gt tractor}, and the update is carried
out at Line \ref{line: gd update tractor}.
With infinite data, the gradient $g_t$ at iteration $t$ would be
{\thinmuskip=1mu \medmuskip=1mu \thickmuskip=1mu $\sum_i
(\eta^{p^i,r^i,\pi^{*,i}_t}-\eta^{p^i,r^i,\pi^{E,i}})$}, where $\pi^{*,i}_t$ is
any optimal policy in RS-MDP $\cM^i_{U_{t}}$, in which $U_t\in\underline{\fU}_L$
is any utility satisfying $U_t(y)=\overline{U}_t(y)$ for all $y\in\cY$.
In our case, \tractor uses $\widehat{\eta}^{E,i}$, computed at Line \ref{line:
erd tractor} in the same way as in \caty, to approximate
$\eta^{p^i,r^i,\pi^{E,i}}$, and it uses $\widehat{\eta}_t^i$, computed at Lines
\ref{line: planning tractor}-\ref{line: line compute etati tractor}, to
approximate $\eta^{p^i,r^i,\pi^{*,i}_t}$ by estimating
$\eta^{\widehat{p}^i,\overline{r}^i,\widehat{\pi}^{*,i}_t}$, where
$\widehat{\pi}^{*,i}_t$ is the optimal policy for the RS-MDP {\thinmuskip=1mu
\medmuskip=1mu \thickmuskip=1mu $\tuple{\cS^i,\cA^i,H,s_0^i, \widehat{p}^i,
\overline{r}^i,\overline{U}_t}$}. In short, Line \ref{line: planning tractor}
computes policy $\widehat{\pi}^{*,i}_t$ through the enlarged state space
approach, which is subsequently played in MDP {\thinmuskip=1mu \medmuskip=1mu
\thickmuskip=1mu $\tuple{\cS^i,\cA^i,H,s_0^i, \widehat{p}^i, \overline{r}^i}$}
(\texttt{ROLLOUT} subroutine, Algorithm \ref{alg: rollout}, Line \ref{line:
rollout tractor}) to construct a dataset $\cD$ of $K$ trajectories that is used
at Line \ref{line: line compute etati tractor} to compute $\widehat{\eta}^i_t$.
\tractor enjoys the following guarantee:
\begin{restatable}{thr}{thrupperboundtractor}\label{thr: tractor upper bound}
   Let $L>0$, $\epsilon,\delta\in(0,1)$, and $U^E\in\underline{\fU}_L$. Assume
   that the projection operator $\Pi_{\overline{\underline{\fU}}_L}$ is
   implemented exactly. Let the number of samples satisfy Eq. \eqref{eq: bound
   all U thr caty}. There exist values of $\epsilon_0,K,\alpha,\overline{U}_0$
   (see Appendix \ref{apx: analysis tractor}) such that, if we run \tractor for
   a number of gradient iterations:
   \begin{align*}
    T\ge \cO\big(N^4H^4L^2/\epsilon^4\big),
   \end{align*}
   then, w.p. at least $1-\delta$, any utility $U\in\underline{\fU}_L$ such that
   $U(y)=\widehat{U}(y)$ $\forall y\in\cY$ belongs to
   $\underline{\cU}_\epsilon$.
\end{restatable}
In other words, with high probability, \tractor is guaranteed to find a utility
$U$ with small $\sum_i \overline{\cC}_{p^i,r^i,\pi^{E,i}}(U)\le\epsilon$, i.e.,
$U$ is close to the feasible set.
Note that we consider \emph{increasing} utilities $\underline{\fU}_L$ instead of
\emph{strictly-increasing} $\fU_L$ to guarantee the closedness of the set onto
which we project.
Observe also that assuming that
$\Pi$\scalebox{0.8}{$_{\overline{\underline{\fU}}_L}$} can be implemented
exactly simplifies the theoretical analysis, but, in practice, we are satisfied
with approximations that can be computed efficiently since set
\scalebox{0.9}{$\overline{\underline{\fU}}_L$} is made of
$\cO(H^2/\epsilon_0^2)$ linear constraints (Appendix \ref{apx: computational
complexity projection}). 

\section{Numerical Simulations}\label{sec: experiments}

In this section, we present \emph{proof-of-concept} experiments using data
collected from lab members to provide empirical evidence to support both our
model and algorithms.

\textbf{The data.}~~We asked to 15 participants to describe the actions they
would play in an MDP with horizon $H=5$ (see Appendix \ref{apx: experimental
details}), at varying of the state, the stage, and the \emph{cumulative reward}
collected. The reward has a monetary interpretation. To answer the questions,
the participants have been provided with complete information about the
MDP.\footnote{The data collected is not personal.}

\textbf{Experiment 1 - Validation of the model.}~~%
Our model of behavior, presented in Eq. \eqref{eq: model expert rs mdp}, is the
first IRL model that contemplates non-Markovian policies. To understand if this
new model is more suitable than existing IRL models to describe human behavior
in MDPs, we count how many participants to the study exhibited non-Markovian
behavior. Intuitively, the more non-Markovianity, the better our model.
What we found is that \emph{10 participants out of 15} demonstrated a
non-Markovian policy even in this very small environment, providing consistent
evidence on the importance of our new model.
See Appendix \ref{apx: additional experiment} for additional analysis of our
model on this data.

\textbf{Experiment 2 - Validation of \tractor.}~~%
To understand how \tractor performs in practice, we have run it on both the
real-world data described earlier and on simulated data.
Crucially, the executions on the participants' data reveal that, irrespective of
the initial utility $\overline{U}_0$ adopted, the algorithm converges much
faster using large values of step size $\alpha$. For instance, as shown in
Fig. \ref{fig: tractor main}, the best step size for using \tractor to compute
a utility representative of the behavior of participant 10 is $\alpha=100$.
Intuitively, this is explained by the presence of a large number of utilities in
the feasible set (since we are considering demonstrations in a single
environment $N=1$), and by the projection step onto
$\overline{\underline{\fU}}_L$, that results in small changes of utility even
with large steps (see Appendix \ref{apx: explanation large learning rate}).
Next, we have run \tractor on simulated data to analyze its performance on
larger MDPs (increment of $S,A$) and with multiple environments (increment of
$N$). We found that the number of gradient iterations necessary to achieve a
certain level of (non)compatibility is affected by the increment of $N$, but not
of $S,A$, as predicted by Theorem \ref{thr: tractor upper bound}. However, note
that larger $S,A$ require more execution time, because of the value iteration
subroutine.
Moreover, we observed that, when $N>1$, the best step size $\alpha$ can be much
smaller than $\alpha=100$. Intuitively, the feasible set contains less utilities
now, thus, we need more accurate (smaller) gradient steps to find them.
More details on this experiment in Appendix \ref{apx: analysis on simulated
data}.

  \begin{figure}[t]
    \includegraphics[width=0.48\textwidth]{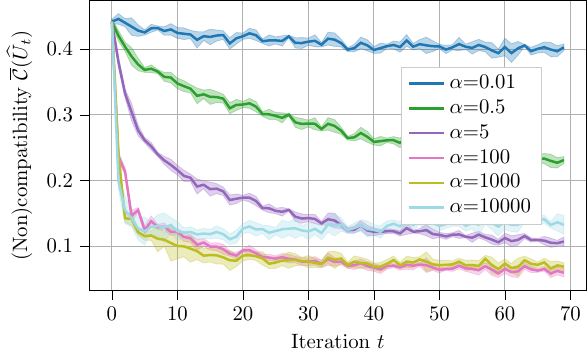}
    \caption{Simulations of \tractor with various step sizes $\alpha$. The
    shaded regions are the standard deviation over 5 seeds.}
    \label{fig: tractor main}
  \end{figure}                

\section{Related Work}\label{sec: related works}

In \emph{risk-sensitive IRL} \citep{majumdar2017risk}, the learner is either
provided with the reward of the expert and it must infer some parameters
representing its risk attitude, or the learner must infer both the reward and
the risk attitude from demonstrations of behavior
\citep{singh2018risksensitive,chen2019active,ratliff2017inverse,cheng2023eliciting,cao2024inferenceutilitiestimepreference}.
However, these works consider problem settings and models of behavior fairly
different from ours.
Specifically, \citet{majumdar2017risk,singh2018risksensitive,chen2019active}
focus on the so-called ``prepare-react model'', which is a model of environment
less expressive than an MDP. Instead,
\citet{ratliff2017inverse,cheng2023eliciting,cao2024inferenceutilitiestimepreference}
consider models of behavior in which the expert's policy is \emph{Markovian}.
More on the related works in Appendix \ref{apx: additional related works}.
  
\section{Conclusion}\label{sec: conclusion}

In this paper, we proposed a novel risk-aware model of behavior that
rationalizes non-Markovian policies in MDPs, and we presented two
provably-efficient algorithms for learning the risk attitude of an agent from
demonstrations.
Interesting directions for future works include extending our algorithms to
high-dimensional settings, studying the problem of learning both $r$ and $U$
from demonstrations, and exploring new methods to alleviate the partial
identifiability.
  
\section*{Acknowledgements}
    AI4REALNET has received funding from European Union's Horizon Europe
    Research and Innovation programme under the Grant Agreement No 101119527.
    Views and opinions expressed are however those of the author(s) only and do
    not necessarily reflect those of the European Union. Neither the European
    Union nor the granting authority can be held responsible for them.

    Funded by the European Union - Next Generation EU within the project NRPP
    M4C2, Investment 1.,3 DD. 341 - 15 march 2022 - FAIR - Future Artificial
    Intelligence Research - Spoke 4 - PE00000013 - D53C22002380006.

\section*{Impact Statement}

This paper presents work whose goal is to advance the field of Machine Learning.
There are many potential societal consequences of our work, none of which we
feel must be specifically highlighted here.

\bibliography{refs}
\bibliographystyle{icml2025}

\newpage
\appendix
\onecolumn

 \section{More on Related Work}
 \label{apx: additional related works}
 
 We describe here more in detail the most relevant related works. First, we
 describe IRL papers with risk, i.e., those works that consider MDPs, and try to
 learn either the reward function or the utility or both. Next, we analyze the
 works that aim to learn the risk attitude (i.e., a utility function) from
 demonstrations of behavior (potentially in problems other than MDPs). Finally,
 we present other connected works.
 
 \paragraph{Inverse Reinforcement Learning with risk.}
 
 \citet{majumdar2017risk} introduces the risk-sensitive IRL problem in decision
 problems different from MDPs. The authors analyze two settings, one in which
 the expert takes a single decision, and one in which there are multiple
 decisions in sequence. They model the expert as a risk-aware decision-making
 agent acting according to a \emph{coherent risk metric}
 \citep{artzner1999coherent}, and they consider both the case in which the
 reward function is known, and they try to learn the risk attitude (coherent
 risk metric) of the expert, and the case in which the reward is unknown, and
 they aim to estimate both the risk attitude and the reward function.
 Nevertheless, the authors analyze a model of environment, called
 \emph{prepare-react model}, rather different from an MDP, since, simply put, it
 can be seen as an MDP in which the stochasticity is shared by all the
 state-action pairs at each stage $h\in\dsb{H}$. Moreover, they consider
 Markovian policies.
 \citet{singh2018risksensitive} generalizes the work of
 \citet{majumdar2017risk}. Specifically, the biggest improvement is to consider
 nested optimization stages. However, the model of the environment is still
 rather different from an MDP.
 We mention also the work of \citet{chen2019active} who extend
 \citet{majumdar2017risk} by devising an active learning framework to improve the
 efficiency of their learning algorithms.
 
 Another important work is that of \citet{ratliff2017inverse}, who study the
 risk-sensitive IRL problem in MDPs, by proposing a parametric model of behavior
 for the expert based on prospect theory \cite{kahneman1979prospect}, and they
 devise a gradient-based IRL algorithm that minimizes a loss function defined on
 the observed behavior. This work differs from ours in that it assumes that the
 expert plays actions based on a softmax distribution, i.e., using a Markovian
 policy.

 \citet{cheng2023eliciting} proposes a model of behavior in MDPs using the
 conditional value-at-risk \citep{rockafellar2000cvar} instead of a utility
 function \citep{vnm1947theory}. Moreover, differently from ours, their model
 does not contemplate non-Markovian policies. Similarly to us, they analyze the
 partial identifiability of the parameters representing the risk attitude from
 demonstrations in a single environment, and propose a strategy for designing
 the environments in which collecting additional demonstrations in order to
 reduce the partial identifiability.
 
 We shall mention also the recent pre-print of
 \citet{cao2024inferenceutilitiestimepreference} that proposes a novel
 stochastic control framework in continuous time that includes two utility
 functions and a generic discounting scheme under a time-varying rate. Assuming
 to know both the utilities and the discounting scheme, the authors show that,
 through state augmentation, the control problem is well-posed. In addition, the
 authors provide sufficient conditions for the identification of both the
 utilities and the discounting scheme given demonstrations of behavior.
 We note that there are crucial differences between this work and ours. First,
 the author model the expert as solving an optimization problem in which the
 utility function is applied to the per \emph{stage} reward, instead we apply
 the utility to the entire return (see Eq. \eqref{eq: model expert rs mdp}).
 Next, they model the expert using Markovian policies.
 
 \paragraph{Learning utilities from demonstrations.}
 
 \citet{Chajewska2001learning} considers an approach similar to IRL
 \citet{ng2000algorithms}. Their goal is not to perform \emph{active} preference
 elicitation, but, similarly to us, to use demonstrations to infer preferences.
 Specifically, they aim to learn utilities in sequential decision-making problems
 from demonstrations. However, they model the problems through decision trees,
 which are different from MDPs, and this represents the main difference between
 their work and ours. Indeed, decision trees are simpler since there is no notion
 of reward function at intermediate states. In this manner, they are able to
 devise (backward induction) algorithms to learn utilities in decision trees
 through linear constraints similar to those devised by \citet{ng2000algorithms}
 in IRL. It is interesting to notice that they adopt a Bayesian approach to
 extract a single utility from the feasible set constructed, and not an heuristic
 like that of \citet{ng2000algorithms}. They assume a prior $p(u)$ over the true
 utility function $u$, and approximate the posterior w.r.t. the feasible set of
 utilities $\cU$ using Markov Chain Monte Carlo (MCMC).

 \citet{shukla2017learningvideo} faces the problem of learning human utilities
 from (video) demonstrations, with the aim of generating meaningful tasks based
 on the learned utilities. However, differently from us, they consider the
 stochastic context-free And-Or graph (STC-AOG) framework \citep{xiong2016robot},
 instead of MDPs.
 
 \citet{bai2020influencediagram} considers the problem of learning utilities from
 demonstrations similarly to \citet{Chajewska2001learning}, but with the
 difference of considering \emph{influence diagrams} instead of decision trees.
 Since any influence diagram can be expanded into a decision tree, authors adopt
 a strategy similar to \citet{Chajewska2001learning}.

 \paragraph{Others.}
 
 \citet{shah2019feasibility} aims to learn the behavioral model of the expert
 from demonstrations. However, they do not consider a specific model like us
 (i.e., Eq. \eqref{eq: model expert rs mdp}), but use a differentiable planner
 (neural network) to learn the planner. In principle, they can fit any
 behavioral model (also risk-sentitive models) given the huge expressive power
 of neural networks. However, their approach requires a lot of demonstrations,
 even across multiple MDPs. Moreover, this approach does not permit to learn a
 utility function as a simple, interpretable, and transferrable representation
 of the risk attitude of the expert.
 
 \section{Additional Notation}
 \label{apx: additional notation}
 
 In this appendix, we introduce additional notation that will be used in other
 appendices.
 
 \paragraph{Miscellaneous.}
 For any probability distribution $\nu\in\Delta^\RR$, we denote its cumulative
 density function by $F_{\nu}$. Let $\nu\in\Delta^\RR$ be a probability
 distribution on $\RR$; then, for any $y\in[0,1]$, we define the
 \emph{generalized inverse} $F^{-1}_\nu(y)$ as:
 \begin{align*}
   F^{-1}_\nu(y)\coloneqq \inf\limits_{x\in\RR}\{F_\nu(x)\ge y\}.
 \end{align*}
 We define the \emph{1-Wasserstein distance}
 $w_1:\Delta^\RR\times\Delta^\RR\to [0,\infty]$ between two probability
 distributions $\nu,\mu$ as:
 \begin{align}\label{eq: definition wasserstein distance w1}
   w_1(\nu,\mu)\coloneqq \int_0^1 \big|
   F_\nu^{-1}(y)-F^{-1}_\mu(y)\big|dy.
 \end{align}
 In addition, we define the \emph{Cramér distance} $\ell_2:\Delta^\RR\times\Delta^\RR\to[0,\infty]$
 between two probability distributions $\nu,\mu$ as:
 \begin{align}\label{eq: definition cramer distance l2}
   \ell_2(\nu,\mu)\coloneqq \Big(
   \int_\RR (F_\nu(y)-F_{\mu}(y))^2dy
   \Big)^{1/2}.
 \end{align}
 We will use notation:
 \begin{align*}
   \V_{X\sim Q}[X]&\coloneqq \E_{X\sim Q}[(X-\E_{X\sim Q}[X])^2],
 \end{align*}
 to denote the variance of a random variable $X\sim Q$ distributed as $Q$.
 Given two random variables $X\sim Q_1,Y\sim Q_2$, we denote their
 covariance as:
 \begin{align*}
   \text{Cov}_{X\sim Q_1,Y\sim Q_2}[X,Y]\coloneqq
   \E_{X\sim Q_1,Y\sim Q_2}[(X-\E_{X\sim Q_1}[X])(Y-\E_{Y\sim Q_2}[Y])].
 \end{align*}
 We define the \emph{categorical projection operator} $\text{Proj}_\cC$
 (mentioned in Section \ref{sec: online utility learning}), that projects onto
 set $\cY=\{y_1,y_2,\dotsc,y_d\}$ (the items of $\cY$ are ordered: $y_1\le
 y_2\le\dotsc\le y_d$, with
 $y_1=0,y_2=\epsilon_0,y_3=2\epsilon_0,\dotsc,y_d=\floor{H/\epsilon_0}\epsilon_0$),
 based on \citet{rowland2018analysis}. For single Dirac measures on an arbitrary
 $y\in\RR$, we write:
 \begin{align}\label{eq: definition proj c}
   \text{Proj}_\cC(\delta_y)\coloneqq\begin{cases}
     \delta_{y_1}&\text{if }y\le y_1\\
     \frac{y_{i+1}-y}{y_{i+1}-y_i}\delta_{y_i}+
     \frac{y-y_i}{y_{i+1}-y_i}\delta_{y_{i+1}}& \text{if }y_i<y\le y_{i+1}\\
     \delta_{y_{d}}&\text{if }y> y_{d}\\
   \end{cases},
 \end{align}
 and we extend it affinely to finite mixtures of $M$ Dirac distributions, so
 that:
 \begin{align}\label{eq: property proj c}
   \text{Proj}_\cC\Big(\sum\limits_{j\in\dsb{M}}q_j\delta_{z_j}\Big)=\sum\limits_{j\in\dsb{M}}q_j
   \text{Proj}_\cC(\delta_{z_j}),
 \end{align}
 for some set of real values $\{z_j\}_{j\in\dsb{M}}$ and weights $\{q_j\}_{j\in\dsb{M}}$.
 
 \paragraph{Value functions.}
 Given an MDP $\cM=\tuple{\cS,\cA,H,s_0,p,r}$ and a policy $\pi$, we define the
 $V$- and $Q$-functions of policy $\pi$ in MDP $\cM$ at every $(s,a,h)\in\SAH$
 respectively as $V^\pi_h(s;p,r)\coloneqq \E_{p,r,\pi}[\sum_{t=h}^H
 r_{t}(s_t,a_t)|s_h=s]$ and $Q^{\pi}_h(s,a;p,r)\coloneqq \E_{p,r,\pi}[\sum_{t=h}^H
 r_{t}(s_t,a_t)|s_h=s,a_h=a]$. We define the optimal $V$- and $Q$-functions as
 $V^*_h(s;p,r)\coloneqq\sup_{\pi} V^{\pi}_h(s;p,r)$ and
 $Q^*_h(s,a;p,r)\coloneqq\sup_{\pi} Q^{\pi}_h(s,a;p,r)$.
 
 For MDPs with an enlarged state space, e.g.,
 $\tuple{\{\cS\times\cY_h\}_h,\cA,H,(s_0,0),\fp,\fr}$, and a policy
 $\psi=\{\psi_h\}_h$, for all $h\in\dsb{H}$ and $(s,y,a)\in\SYAh$  we denote the
 $V$- and $Q$-functions respectively as $V^\psi_h(s,y;\fp,\fr)\coloneqq
 \E_{\fp,\fr,\psi}[\sum_{t=h}^H \fr_{t}(s_t,y_t,a_t)|s_h=s,y_h=y]$ and
 $Q^{\psi}_h(s,y,a;\fp,\fr)\coloneqq \E_{\fp,\fr,\psi}[\sum_{t=h}^H
 \fr_{t}(s_t,y_t,a_t)|s_h=s,y_h=y,a_h=a]$. We denote the optimal $V$- and
 $Q$-functions as $V^*_h(s,y;\fp,\fr)\coloneqq\sup_{\psi} V^{\psi}_h(s,y;\fp,\fr)$
 and $Q^*_h(s,y,a;\fp,\fr)\coloneqq\sup_{\psi} Q^{\psi}_h(s,y,a;\fp,\fr)$.

 Observe that the notation just introduced will be extended in a straightforward
 manner to MDPs (MDPs with enlarged state space) that have an estimated
 transition model $\widehat{p}$ ($\widehat{\fp}$), and/or a discretized reward
 function $\overline{r}$ ($\overline{\fr}$).
 
\section{Additional Results and Proofs for Section \ref{sec: motivation and problem setting}}
\label{apx: section 3}

In Appendix \ref{apx: drawbacks reward into state}, we explain why including the
past rewards into the state is not satisfactory, in Appendix \ref{apx:
observation model} we provide an observation on Eq. \eqref{eq: model expert rs
mdp}, while in Appendix \ref{apx: missing proofs sec 3} we provide the missing
proofs for Section \ref{sec: motivation and problem setting}.

\subsection{Drawbacks of Re-modelling the MDP}\label{apx: drawbacks reward into
state}

Re-modelling the MDP including the sum of the past rewards into the state would
make the demonstrated policy Markovian, and so, in principle, it would allow to
apply the existing IRL models meaningfully.
However, since in IRL the reward function is unknown, to adopt this trick one
should include into the state representation the entire sequence of past
state-action pairs, causing the size of the new state space to explode, and also
causing the reward function to become non-Markovian w.r.t. the original state
space.
If instead the reward function was known, and one just wanted to apply one of
the risk-sensitive IRL models of behavior presented in Section \ref{sec: related
works} to learn the (parameters of the) risk attitude, then re-modelling the MDP
would still cause the size of the new state space to explode in tabular MDPs,
since, in general, there is a number of cumulative reward values that is
exponential in the horizon. Moreover, it is not clear why the considered model
of behavior, that was designed for the original state space (not including the
past rewards), should be realistic in the new state space.

\subsection{An Observation on the Model}\label{apx: observation model}

If we restrict the optimization problem in Eq. \eqref{eq: model expert rs mdp}
to Markovian policies, we note that non-stationarity (i.e., the dependence of
the policy on the stage $h$) and stochasticity (i.e., if the policy prescribes a
lottery over actions instead of a single action) can improve the performance
w.r.t. Markovian stationary deterministic policies even in stationary
environments. Intuitively, the reason is that they permit to consider larger
ranges of return distributions w.r.t. Markovian stationary deterministic
policies.
\begin{restatable}{prop}{policynonstationary}
  \label{prop: policy non stationary}
  There exists a RS-MDP with stationary transition model and reward in which
  the best \emph{Markovian} policy is non-stationary, and the best
  \emph{stationary Markovian} policy is stochastic.
\end{restatable}
 
 \subsection{Proofs for Section \ref{sec: motivation and problem setting}}
 \label{apx: missing proofs sec 3}
 
 \loseinperformancewithmarkovianity*
 \begin{proof}
   \begin{figure}[t]
     \centering
     \begin{tikzpicture}[node distance=3.5cm]
     \node[state,initial] at (0,0) (s0) {$s_{\text{init}}$};
     \node[state] at (3,1) (s1) {$s_1$};
     \node[state] at (3,-1) (s2) {$s_2$};
     \node[state] at (6,0) (s3) {$s_3$};
     \node[state] at (9,2) (s4) {$s_4$};
     \node[state] at (9,0) (s5) {$s_5$};
     \node[state] at (9,-2) (s6) {$s_6$};
     \node[state, draw=none] at (11,2) (ss4) {};
     \node[state, draw=none] at (11,0) (ss5) {};
     \node[state, draw=none] at (11,-2) (ss6) {};
     \node[draw=none, fill=black] at (1.5,0) (ss0) {};
     \node[draw=none,fill=black] at (7.5,1) (ss3) {};
     \draw (s0) edge[-, solid, above] node{\scriptsize$a_1,a_2$} (ss0);
     \draw (ss0) edge[-, solid, above] node{\scriptsize$1/2$} (s1);
     \draw (ss0) edge[-, solid, above] node{\scriptsize$1/2$} (s2);
     \draw (s1) edge[->, solid, above] node{\scriptsize$r=1$} (s3);
     \draw (s2) edge[->, solid, above] node{\scriptsize$r=0$} (s3);
     \draw (s3) edge[->, solid, above] node{\scriptsize$a_1$} (ss3);
     \draw (s3) edge[->, solid, above] node{\scriptsize$a_2$} (s6);
     \draw (ss3) edge[->, solid, above,sloped] node{\scriptsize$x/3.99$} (s4);
     \draw (ss3) edge[->, solid, above,sloped] node{\scriptsize$1-x/3.99$} (s5);
     \draw (s4) edge[->, solid, above] node{\scriptsize$r=1$} (ss4);
     \draw (s5) edge[->, solid, above] node{\scriptsize$r=0$} (ss5);
     \draw (s6) edge[->, solid, above] node{\scriptsize$r=0.5$} (ss6);
   \end{tikzpicture}
   \caption{MDP for the proof of Proposition \ref{prop: lose in performance with markovianity}.}
   \label{fig: mdp for proof lose in performance}
   \end{figure}
   For reasons that will be clear later, let us define symbol $x\approx 2.6$
   as the solution of $x-\frac{x^2}{3.99}-0.1=1$.
   
   Consider the RS-MDP $\cM_U=\tuple{\cS,\cA,H,s_0,p,r,U}$ in Fig. \ref{fig: mdp
   for proof lose in performance}, where
   $\cS=\{s_{\text{init}},s_1,s_2,s_3,s_4,s_5,s_6\}$, $\cA=\{a_1,a_2\}$, $H=4$,
   $s_0=s_{\text{init}}$, transition model $p$ such that:
   \begin{align*}
     &p_1(s_1|s_{\text{init}},a)=p_1(s_2|s_{\text{init}},a)=1/2 \quad\forall a\in\cA,\\
     &p_2(s_3|s_1,a)=p_2(s_3|s_2,a)=1 \quad\forall a\in\cA,\\
     &p_3(s_4|s_3,a_1)=x/3.99,p_3(s_5|s_3,a_1)=1-x/3.99,p_3(s_6|s_3,a_2)=1,
   \end{align*}
   reward function $r$ defined as:
   \begin{align*}
     &r_1(s_{\text{init}},a)=0\quad\forall a\in\cA,\\
     &r_2(s_1,a)=1\quad\forall a\in\cA,\\
     &r_2(s_2,a)=0\quad\forall a\in\cA,\\
     &r_3(s_3,a)=0\quad\forall a\in\cA,\\
     &r_4(s_4,a)=1\quad\forall a\in\cA,\\
     &r_4(s_5,a)=0\quad\forall a\in\cA,\\
     &r_4(s_6,a)=0.5\quad\forall a\in\cA,
   \end{align*}
   and utility function $U\in\fU$ that satisfies:
   \begin{align*}
     U(y)=\begin{cases}
       x-0.1 & \text{if }y=0.5\\
       x & \text{if }y=1\\
       x+0.1 & \text{if }y=1.5\\
       3.99 & \text{if }y=2
     \end{cases}.
   \end{align*}
   Note that this entails that:
   \begin{align}\label{eq: for example prop lost performance markovianity}
     \frac{x}{3.99}U(2)+U(1)=U(0.5)+U(1.5).
   \end{align}
   Note also that the support of the return function of this (RS-)MDP is
   $\cG^{p,r}=\{0,0.5,1,1.5,2\}$.
   
   For $\alpha\in[0,1]$, let $\pi^\alpha$ be the generic Markovian policy that plays
   action $a_1$ in $s_3$ w.p. $\alpha$ (the actions played in other states are
   not relevant). Then, its expected utility is:
   \begin{align*}
     J^{\pi^\alpha}(U;p,r)&=\frac{1}{2}\Big[
     \alpha\Big(\frac{x}{3.99}U(2)+(1-\frac{x}{3.99})U(1)\Big)
     +(1-\alpha)U(1.5)
     \Big]\\
     &\qquad
     +\frac{1}{2}\Big[
     \alpha\Big(\frac{x}{3.99}U(1)+(1-\frac{x}{3.99})U(0)\Big)
     +(1-\alpha)U(0.5)
     \Big]\\
     &\markref{(1)}{=}
     \frac{1}{2}\Big[
     \alpha\Big(\frac{x}{3.99}U(2)+U(1)\Big)
     +(1-\alpha)(U(1.5)+U(0.5))
     \Big]\\
     &\markref{(2)}{=}
     \frac{U(1.5)+U(0.5)}{2},
   \end{align*}
   where at (1) we have used that $U(0)=0$, and at (2) we have used Eq. \eqref{eq:
   for example prop lost performance markovianity}.
 
   Thus, all Markovian policies $\pi^\alpha$ have the same performance. Let us
   consider the non-Markovian policy $\overline{\pi}$ that, in state $s_3$, plays
   action $a_1$ w.p. $1$ if $s_3$ is reached with cumulative reward $1$, and it
   plays action $a_2$ w.p. $1$ if $s_3$ is reached with cumulative reward $0$.
   Then, its performance is:
   \begin{align*}
     J^{\overline{\pi}}(U;p,r)&=\frac{1}{2}
       \Big(\frac{x}{3.99}U(2)+(1-\frac{x}{3.99})U(1)\Big)
       +\frac{1}{2}U(0.5).
   \end{align*}
 
   The difference in performance between the optimal performance and that of
   $\pi^\alpha$ is:
   \begin{align*}
     J^*(U;p,r)-J^{\pi^\alpha}(U;p,r)&\ge J^{\overline{\pi}}(U;p,r)
     - J^{\pi^\alpha}(U;p,r)\\
     &=\frac{1}{2}
     \Big(\frac{x}{3.99}U(2)+(1-\frac{x}{3.99})U(1)\Big)
     +\frac{1}{2}U(0.5)-\frac{U(1.5)+U(0.5)}{2}\\
     &=\frac{1}{2}
     \Big(\frac{x}{3.99}U(2)+(1-\frac{x}{3.99})U(1)-U(1.5)\Big)\\
     &\markref{(3)}{=}
     \frac{1}{2}
     \Big(x+x-\frac{x^2}{3.99}-x-0.1\Big)\\
     &=
     \frac{1}{2}
     \Big(x-\frac{x^2}{3.99}-0.1\Big)\\
     &\markref{(4)}{=}
     0.5,
   \end{align*}
   where at (3) we have replaced the values of utility, and at (4) we have used
   the definition of $x$.
 
 \end{proof}

 \envdeterministic*
 \begin{proof}
  If $p$ is deterministic, then the optimal policy in Eq. \eqref{eq: model
  expert rs mdp} is the policy that deterministically plays the trajectory
  $\omega$ with largest value of $U^E(g(\omega))$ ($g(\omega)$ denotes the
  return of $\omega$ under reward $r^E$). Since, by hypothesis, $U^E\in\fU$,
  then it is strictly increasing, thus such trajectory coincides with the
  trajectory with largest return.
 \end{proof}
 
 \policynonstationary*
 \begin{proof}
   \begin{figure}[t]
     \centering
     \begin{tikzpicture}[node distance=3.5cm]
     \node[state,initial] at (0,0) (s0) {$s_{\text{init}}$};
     \node[state] at (3,2) (s1) {$s_1$};
     \node[state] at (3,0) (s2) {$s_2$};
     \node[state] at (3,-2) (s7) {$s_3$};
     \node[state] at (6,0) (s3) {$s_{\text{init}}$};
     \node[state] at (9,2) (s4) {$s_1$};
     \node[state] at (9,0) (s5) {$s_2$};
     \node[state] at (9,-2) (s6) {$s_3$};
     \node[state, draw=none] at (11,2) (ss4) {};
     \node[state, draw=none] at (11,0) (ss5) {};
     \node[state, draw=none] at (11,-2) (ss6) {};
     \node[draw=none, fill=black] at (1.5,-1) (ss0) {};
     \node[draw=none,fill=black] at (7.5,-1) (ss3) {};
     \draw (s0) edge[-, solid, above] node{\scriptsize$a_2$} (s1);
     \draw (s0) edge[-, solid, above] node{\scriptsize$a_1$} (ss0);
     \draw (ss0) edge[->, solid, above] node{\scriptsize$1/3$} (s2);
     \draw (ss0) edge[->, solid, above] node{\scriptsize$2/3$} (s7);
     \draw (s1) edge[->, solid, above, sloped] node{\scriptsize$r=0.5$} (s3);
     \draw (s2) edge[->, solid, above] node{\scriptsize$r=1$} (s3);
     \draw (s7) edge[->, solid, above,sloped] node{\scriptsize$r=0$} (s3);
     \draw (s3) edge[-, solid, above] node{\scriptsize$a_2$} (s4);
     \draw (s3) edge[->, solid, above] node{\scriptsize$a_1$} (ss3);
     \draw (ss3) edge[->, solid, above] node{\scriptsize$1/3$} (s5);
     \draw (ss3) edge[->, solid, above] node{\scriptsize$2/3$} (s6);
     \draw (s4) edge[->, solid, above] node{\scriptsize$r=0.5$} (ss4);
     \draw (s5) edge[->, solid, above] node{\scriptsize$r=1$} (ss5);
     \draw (s6) edge[->, solid, above] node{\scriptsize$r=0$} (ss6);
   \end{tikzpicture}
   \caption{MDP for the proof of Proposition \ref{prop: policy non stationary}.}
   \label{fig: mdp for proof policy non stationary}
   \end{figure}
 
   Consider the stationary RS-MDP $\cM_U=\tuple{\cS,\cA,H,s_0,p,r,U}$ depicted in
  Fig. \ref{fig: mdp for proof policy non stationary}, where
  $\cS=\{s_{\text{init}},s_1,s_2,s_3\}$, $\cA=\{a_1,a_2\}$, $H=4$,
  $s_0=s_{\text{init}}$, stationary transition model $p$ (we omit subscript
  because of stationarity) such that:
  \begin{align*}
    &p(s_2|s_{\text{init}},a_1)=1-p(s_3|s_{\text{init}},a_1)=1/3,\\
    &p(s_1|s_{\text{init}},a_2)=1,\\
    &p(s_{\text{init}}|s,a)=1\quad\forall s\in\{s_1,s_2,s_3\},\forall a\in\cA,
  \end{align*}
  reward function $r$ defined as:
  \begin{align*}
    &r(s_{\text{init}},a)=0\quad\forall a\in\cA,\\
    &r(s_1,a)=0.5\quad\forall a\in\cA,\\
    &r(s_2,a)=1\quad\forall a\in\cA,\\
    &r(s_3,a)=0\quad\forall a\in\cA,
  \end{align*}
  and utility function $U\in\fU$ that satisfies:
  \begin{align*}
    U(y)=\begin{cases}
      0.15 & \text{if }y=0.5\\
      0.2 & \text{if }y=1\\
      1.8 & \text{if }y=1.5\\
      2 & \text{if }y=2
    \end{cases}.
  \end{align*}
 
  Let $\pi^{\alpha,\beta}$ denote the general non-stationary policy that plays
   action $a_1$ at stage 1 w.p. $\alpha\in[0,1]$, and plays action $a_1$ at stage
   2 w.p. $\beta\in[0,1]$. The performance of
   policy $\pi^{\alpha,\beta}$ can be written as:
   \begin{align*}
     J^{\pi^{\alpha,\beta}}(U;p,r)
     &=\alpha\Big\{
     \frac{1}{3}\Big[
     \beta\Big(
     \frac{1}{3}U(2)+\frac{2}{3}U(1)  
     \Big)  
     +(1-\beta)U(1.5)
     \Big]  
     +\frac{2}{3}\Big[
     \beta\frac{1}{3}U(1)+(1-\beta)U(0.5)  
     \Big]
     \Big\}\\
   &\qquad
     + (1-\alpha)\Big[
     \beta\Big(
       \frac{1}{3}U(1.5) +\frac{2}{3}U(0.5)
     \Big)  
     +(1-\beta)U(1)
     \Big]\\
     &=
     \alpha\beta\Big[
     \frac{1}{9}U(2)+\frac{13}{9}U(1)-\frac{2}{3}U(1.5)-\frac{4}{3}U(0.5)
     \Big]\\
     &\qquad
     +(\alpha+\beta)\Big[
       \frac{1}{3}U(1.5)+\frac{2}{3}U(0.5)-U(1)  
     \Big]+U(1)\\
     &= \alpha\beta\Big[
       \frac{2}{9}+\frac{13}{45}-\frac{18}{15}-\frac{1}{5}
       \Big]
       +
       (\alpha+\beta)\Big[
       \frac{1}{5}+\frac{1}{10}-\frac{1}{5}  
     \Big]+\frac{1}{5}\\
     &=-\frac{8}{9}\alpha\beta
       +\frac{1}{10}
       (\alpha+\beta)+\frac{1}{5}.
   \end{align*}
   To show that the best Markovian policy is non-stationary in this example, we
   show that the performance of non-stationary policy $\pi^{0,1}$ is better than
   the performance of all possible Markovian policies. The performance of
   $\pi^{0,1}$ is:
   \begin{align*}
     J^{\pi^{0,1}}(U;p,r)=\frac{1}{10}+\frac{1}{5}=0.3.
   \end{align*}
   Instead, the generic stationary policy is $\pi^{\alpha,\alpha}$, and has
   performance:
   \begin{align*}
     J^{\pi^{\alpha,\alpha}}(U;p,r)=-\frac{8}{9}\alpha^2
     +\frac{1}{5}\alpha+\frac{1}{5}.
   \end{align*}
   The value of $\alpha\in[0,1]$ that maximizes this objective is:
   \begin{align*}
     \frac{d}{d\alpha} J^{\pi^{\alpha,\alpha}}(U;p,r)=-\frac{16}{9}\alpha+\frac{1}{5}=0
     \iff \alpha = \frac{9}{80},
   \end{align*}
   from which we get:
   \begin{align*}
     J^{\pi^{9/80,9/80}}(U;p,r)=\frac{169}{800}\le 0.22,
   \end{align*}
   which is smaller than $0.3=J^{\pi^{0,1}}(U;p,r)$. This concludes the proof of
   the first part of the proposition.
 
   For the second part, simply observe that, in the problem instance considered,
   we just obtained that the best Markovian stationary policy plays action $a_1$
   w.p. $9/80$, i.e., it is stochastic.
 \end{proof}
 
 \section{Additional Results and Proofs for Section \ref{sec: utility learning}}
 \label{apx: section 4}
 
 In this appendix, we provide a more explicit formulation for the feasible
 utility set (Appendix \ref{apx: explicit fs}), we present a property of the
 distance $d^{\text{all}}$ (Appendix \ref{apx: property dall}), and then we
 provide the proofs of all the results presented in Section \ref{sec: utility
 learning} (Appendix \ref{apx: proofs section ul}).
 
 \subsection{A more Explicit Formulation for the Feasible Set}
 \label{apx: explicit fs}

 For any policy $\pi$, we denote by $\cS^{p,r,\pi}$ the set of all
 $(s,h,y)$ state-stage-cumulative reward triples which are covered with non-zero
 probability by policy $\pi$ in the considered (RS-)MDP.
 
 Thanks to this definition, we can rewrite the feasible set as follows:  
 \begin{prop}\label{prop: fs explicit} Let $\cM=\tuple{\cS,\cA,H, s_0,p,r}$ be an
   MDP, and let $\pi^E$ be the expert policy. Then, the feasible utility set
   $\cU_{p,r,\pi^E}$ contains all and only the utility functions that make the
   actions played by the expert policy optimal at all the
   $(s,h,y)\in\cS^{p,r,\pie}$. Formally:
   \begin{align*}
     \cU_{p,r,\pi^E}=\Big\{&U\in\fU\,\Big|\,
     \forall (s,h,y)\in\cS^{p,r,\pie},\forall a\in\cA:
     \\
     &
     Q^*_h(s,y,\pi^E_h(s,y);p,r) \ge Q^*(s,y,a;p,r),
   \end{align*}
   where we used the notation introduced in Appendix \ref{apx: additional
   notation}.
 \end{prop}
 \begin{proofsketch}
   Based on Theorem 3.1 of \citet{bauerle2014more} (or Theorem 1 of
   \citet{wu2023risksensitive}), we have that a utility $U\in\fU$ belongs to the
   feasible set if it makes the expert policy optimal even in the enlarged state
   space MDP (note that it is possible to define a policy $\psi$ for the enlarged
   MDP because we are considering policies $\pi$ whose non-Markovianity lies only
   in the cumulative reward up to now). Therefore, the result follows thanks to a
   proof analogous to that of Lemma E.1 in \citet{lazzati2024offline}, since we
   are simply considering a common MDP with two variables per state.
 \end{proofsketch}

 \subsection{A Property of $d^{\text{all}}$}\label{apx: property dall}
 
We note that closeness under the max norm (restricted to a certain domain)
implies closeness under $d^{\text{all}}$:
\begin{restatable}{prop}{propbounddall}
 Consider an arbitrary MDP with transition model $p$ and reward function
 $r$. Then, for any pair of utilities $U_1,U_2\in\fU$, it holds that
 $d^{\text{all}}_{p,r}(U_1,U_2)\le
 \max_{G\in\cG^{p,r}}|U_1(G)-U_2(G)|$.
\end{restatable}
\begin{proof}
  For the sake of simplicity, we denote the infinity norm and the 1-norm
  w.r.t. set $\cG^{p,r}$ as: $\|f\|_\infty\coloneqq
  \max_{G\in\cG^{p,r}}|f(G)|$ and $\|f\|_1\coloneqq
  \sum_{G\in\cG^{p,r}}|f(G)|$. In addition, we overload notation and use
  symbols $U_1,U_2$ to denote the vectors in $[0,H]^{|\cG^{p,r}|}$ containing,
  respectively, the values assigned by utility functions $U_1,U_2$ to points
  in set $\cG^{p,r}$. Then, we can write:
  
  \begin{align*}
    d^{\text{all}}_{p,r}(U_1,U_2)&\coloneqq \sup\limits_{\pi\in\Pi}
    |J^\pi(U_1;p,r)-J^\pi(U_2;p,r)|\\
    &=\sup\limits_{\pi\in\Pi}
    |\popblue{\E_{G\sim\eta^{p,r,\pi}}[U_1(G)]}-
    \popblue{\E_{G\sim\eta^{p,r,\pi}}[U_2(G)]}|\\
    &=\sup\limits_{\pi\in\Pi}
    |\E_{G\sim\eta^{p,r,\pi}}[\popblue{U_1(G)-U_2(G)}]|\\
    &\markref{(1)}{\le}\sup\limits_{\popblue{\eta\in\Delta^{\cG^{p,r}}}}
    |\E_{G\sim\popblue{\eta}}[U_1(G)-U_2(G)]|\\
    &\markref{(2)}{\le}\sup\limits_{\eta\in\Delta^{\cG^{p,r}}}
    \E_{G\sim\eta}\popblue{|}U_1(G)-U_2(G)\popblue{|}\\
    &\markref{(3)}{=}\popblue{\|}U_1-U_2\popblue{\|_\infty},
\end{align*}
where at (1) we upper bound by considering the set of all possible
distributions over set $\cG^{p,r}$ instead of just those induced by some
policies in the considered MDP, at (2) we apply triangle inequality, and at
(3) we have used the fact that $\|\cdot\|_1$ and $\|\cdot\|_\infty$ are dual
norms.
\end{proof}
 
 \subsection{Proofs for Section \ref{sec: utility learning}}
 \label{apx: proofs section ul}

 \examplepartialidentifiability*
 \begin{proof}
  A utility $U\in\fU$ makes $\pi^E$ optimal for $\cM_U$ if playing $a_1$ is
  better than playing $a_2$: $J^{\pi^E}(U;p,r)=0.1U(2)+0.5U(1.5)+0.4U(1)\ge
  0.8U(1.5)+0.2U(1)$. Thus, all the utilities $U\in\fU$, that assign to
  $G=1,G=1.5$ any of the values coloured in blue in Fig. \ref{fig: example fs}
  (middle), satisfy this condition.
 \end{proof}
 
 \proptransferabilityp*
 \begin{proof}
 
   We will prove the guarantee stated in the proposition using two different
   pairs of MDPs: One that that satisfies $\cG^{p',r}=\cG^{p,r}$, i.e., for which
   the support of the return function coincides, and the other that does not. Let
   us begin with the former.
 
   Consider a simple MDP
   $\cM=\tuple{\cS,\cA,H,s_{\text{init}},p,r}$ with five states
   $\cS=\{s_{\text{init}},s_0,s_{0.25},s_{0.75},s_1\}$, two actions
   $\cA=\{a_1,a_2\}$, horizon $H=2$, initial state $s_{\text{init}}$,
   transition model $p$ such that:
   \begin{align*}
     p_1(s'|s_{\text{init}},a_1)=\begin{cases}
       1/4 & \text{if }s'=s_0\\
       1/4 & \text{if }s'=s_{0.25}\\
       1/4 & \text{if }s'=s_{0.75}\\
       1/4 & \text{if }s'=s_1
     \end{cases},\\
     p_1(s'|s_{\text{init}},a_2)=\begin{cases}
       1/2 & \text{if }s'=s_{0.25}\\
       1/2 & \text{if }s'=s_{0.75}
     \end{cases},
   \end{align*}
   and reward function $r$ that assigns
   $r_1(s_{\text{init}},a_1)=r_1(s_{\text{init}},a_2)=0$, and:
   \begin{align*}
     r_2(s,a)=\begin{cases}
       0 & \text{if }s=s_0\wedge (a=a_1\lor a=a_2)\\
       0.25 & \text{if }s=s_{0.25}\wedge (a=a_1\lor a=a_2)\\
       0.75 & \text{if }s=s_{0.75}\wedge (a=a_1\lor a=a_2)\\
       1 & \text{if }s=s_1\wedge (a=a_1\lor a=a_2)
     \end{cases}.
   \end{align*}
   Note that the support of the return function is $\cG^{p,r}=\{0,0.25,0.75,1\}$.
   We are given an expert's policy $\pi^E$ that prescribes action $a_1$ at stage
   1 in state $s_{\text{init}}$, and arbitrary actions in other states (the
   specific action is not relevant). The MDP
   $\cM$ is represented in Figure \ref{fig: mdp for proof
   transferability p}.
 
   \begin{figure}[t]
     \centering
     \begin{tikzpicture}[node distance=3.5cm]
     \node[state,initial] at (-1,0) (s0) {$s_{\text{init}}$};
     \node[state] at (2.5,2.25) (s1) {$s_0$};
     \node[state] at (2.5,0.75) (s2) {$s_{0.25}$};
     \node[state] at (2.5,-0.75) (s3) {$s_{0.75}$};
     \node[state] at (2.5,-2.25) (s4) {$s_{1}$};
     \node[state, draw=none] at (4.5,2.25) (a1) {};
     \node[state, draw=none] at (4.5,0.75) (a2) {};
     \node[state, draw=none] at (4.5,-0.75) (a3) {};
     \node[state, draw=none] at (4.5,-2.25) (a4) {};
     \node[draw=none, fill=black] at (0.8,1.4) (a01) {};
     \node[draw=none,fill=black] at (0.3,-1.5) (a02) {};
     \draw (s0) edge[-, solid, above] node{$a_1$} (a01);
     \draw (s0) edge[-, solid, above] node{$a_2$} (a02);
     \draw (s1) edge[->, solid, above] node{$a_1,a_2$} (a1);
     \draw (s2) edge[->, solid, above] node{$a_1,a_2$} (a2);
     \draw (s3) edge[->, solid, above] node{$a_1,a_2$} (a3);
     \draw (s4) edge[->, solid, above] node{$a_1,a_2$} (a4);
     \draw (a01) edge[->, solid, above] node{\scriptsize$1/4$} (s1);
     \draw (a01) edge[->, solid, above] node{\scriptsize$1/4$} (s2);
     \draw (a01) edge[->, solid, above] node{\scriptsize$1/4$} (s3);
     \draw (a01) edge[->, solid, above] node{\scriptsize$1/4$} (s4);
     \draw (a02) edge[->, solid, below] node{\scriptsize$1/2$} (s2);
     \draw (a02) edge[->, solid, below] node{\scriptsize$1/2$} (s3);
   \end{tikzpicture}
   \caption{MDP for the proof of Proposition \ref{prop: transferring utilities p}.}
   \label{fig: mdp for proof transferability p}
   \end{figure}
 
   Now, we show that utilities
   $U_1,U_2\in\fU$, defined in points of the support $\cG^{p,r}$ as (and connected in
   arbitrary continuous strictly-increasing manner between these points):
   \begin{align*}
     U_1(G)=\begin{cases}
       0 & \text{if }G=0\\
       0.01 & \text{if }G=0.25\\
       0.02 & \text{if }G=0.75\\
       1.99 & \text{if }G=1
     \end{cases},\qquad
     U_2(G)=\begin{cases}
       0 & \text{if }G=0\\
       0.01 & \text{if }G=0.25\\
       0.99 & \text{if }G=0.75\\
       1.99 & \text{if }G=1
     \end{cases},
   \end{align*}
   belong to the feasible set $\cU_{p,r,\pi^E}$, and, when transferred to the new
   MDP $\cM'=\tuple{\cS,\cA,H, s_{\text{init}},p',r}$, with
   transition model $p'\neq p$ defined as:
   \begin{align*}
     p_1'(\cdot|s_{\text{init}},a_1)=p_1(\cdot|s_{\text{init}},a_1),\\
     p_1'(s'|s_{\text{init}},a_2)=\begin{cases}
       0.7 & \text{if }s'=s_{0}\\
       0.3 & \text{if }s'=s_{1}
     \end{cases},
   \end{align*}
   impose different optimal policies, i.e., utility $U_2$ keeps making action
   $a_1$ optimal from state $s_{\text{init}}$ even in $\cM'$, while $U_1$ makes
   action $a_2$ optimal. This proves the thesis of the proposition.
 
   Let us begin by showing that $U_1,U_2\in\cU_{p,r,\pi^E}$ belong to the feasible
   set of $\cM$ with policy $\pi^E$. Let $\overline{\pi}$ be the policy that
   plays action $a_2$ in state $s_{\text{init}}$. Then, the distribution of
   returns induced by policies $\pi^E$ and $\overline{\pi}$ are (we represent
   values only at points in $\cG^{p,r}=\{0,0.25,0.75,1\}$):
   \begin{align*}
     \eta^{p,r,\pi^E}&=[1/4, 1/4, 1/4, 1/4]^\intercal\\
     \eta^{p,r,\overline{\pi}}&=[0, 1/2, 1/2, 0]^\intercal.
   \end{align*}
   Thus, policy $\pi^E$ is optimal under some utility $U$ if and only if the
   values assigned by $U$ to points in $\cG^{p,r}=\{0,0.25,0.75,1\}$ (denoted,
   respectively, by $U^1,U^2,U^3,U^4$) satisfy:
   \begin{align*}
     U^\intercal (\eta^{p,r,\pi^E}-\eta^{p,r,\overline{\pi}})=
     [1/4, -1/4, -1/4, 1/4] U =U^1-U^2-U^3+U^4\ge 0,
   \end{align*}
   where we have overloaded the notation and denoted with
   $U\coloneqq[U^1,U^2,U^3,U^4]^\intercal$ both the utility and the vector of
   values assigned to points in $\cG^{p,r}$.
   By imposing normalization constraints ($U(0)=0,U(2)=2$), we get $U^1=0$, and
   by imposing also the monotonicity constraints, we get that utility $U$ is in
   the feasible set $\cU_{p,r,\pi^E}$ if and only if:
   \begin{align*}
     \begin{cases}
       U^4\ge U^2+U^3\\
       0<U^2<U^3<U^4<2
     \end{cases}.
   \end{align*}
   Clearly, both utilities $U_1,U_2$ satisfy these constraints, thus they
   belong to the feasible set $\cU_{p,r,\pi^E}$.
   Now, concerning problem $\cM'$, the performances of $\pi^E,\overline{\pi}$
   w.r.t. utilities $U_1,U_2$ are:
   \begin{align*}
     J^{\pi^E}(U_1;p',r)&= \frac{1}{4}U_1(0)+\frac{1}{4}U_1(0.25)+\frac{1}{4}U_1(0.75)+\frac{1}{4}U_1(1)=
     2.02/4=0.505,\\
     J^{\overline{\pi}}(U_1;p',r)&= 0.7 U_1(0)+0.3 U_1(1)=
     0.3\times 1.99=0.597,\\
     J^{\pi^E}(U_2;p',r)&= \frac{1}{4}U_1(0)+\frac{1}{4}U_1(0.25)+\frac{1}{4}U_1(0.75)+\frac{1}{4}U_1(1)=
     2.99/4=0.7475,\\
     J^{\overline{\pi}}(U_2;p',r)&= 0.7 U_1(0)+0.3 U_1(1)=
     0.3\times 1.99=0.597.
   \end{align*}
   Clearly, $J^{\pi^E}(U_1;p',r)<J^{\overline{\pi}}(U_1;p',r)$, but
   $J^{\pi^E}(U_2;p',r)>J^{\overline{\pi}}(U_2;p',r)$, thus we conclude that the
   set of policies induced by utilities $U_1,U_2$ in $\cM'$ do not intersect,
   since they start from $s_{\text{init}}$ with different actions
   $\Pi^*_{p',r}(U_1)\cap\Pi^*_{p',r}(U_2)=\{\}$. This concludes the proof
   with an example that satisfies $\cG^{p',r}=\cG^{p,r}$.
 
   If we want an example that does \emph{not} satisfy $\cG^{p',r}=\cG^{p,r}$,
   then we can consider exactly the same example with $\cM$ and $\cM'$, but
   using $r_1(s_{\text{init}},a_2)=0.001$. In this manner, we see that
   $\cG^{p,r}=\{0,0.25,0.251,0.75,0.751,1\}$, and
   $\cG^{p',r}=\{0,0.001,0.25,0.75,1,1.001\}$, which are different. By choosing
   $U_1',U_2'$ as:
   \begin{align*}
     U_1'(G)=\begin{cases}
       0 & \text{if }G=0\\
       0.001 & \text{if }G=0.001\\
       0.01 & \text{if }G=0.25\\
       0.011 & \text{if }G=0.251\\
       0.02 & \text{if }G=0.75\\
       0.021 & \text{if }G=0.751\\
       1.99 & \text{if }G=1\\
       1.991 & \text{if }G=1.001
     \end{cases},\qquad
     U_2'(G)=\begin{cases}
       0 & \text{if }G=0\\
       0.001 & \text{if }G=0.001\\
       0.01 & \text{if }G=0.25\\
       0.011 & \text{if }G=0.251\\
       0.99 & \text{if }G=0.75\\
       0.991 & \text{if }G=0.751\\
       1.99 & \text{if }G=1\\
       1.991 & \text{if }G=1.001
     \end{cases},
   \end{align*}
   it can be shown that $U_1',U_2'$ belong to the (new) feasible set of $\cM$,
   and that induce different policies in $\cM'$. This concludes the proof.
 \end{proof}
 
   \proptransferabilityr*
   \begin{proof}
     Similarly to the proof of Proposition \ref{prop: transferring utilities p},
     we provide two examples, one with $\cG^{p,r'}=\cG^{p,r}$, and the other with
     $\cG^{p,r'}\neq\cG^{p,r}$. Let us begin with the former.
 
     Consider a simple MDP
     $\cM=\tuple{\cS,\cA,H, s_{\text{init}},p,r}$ with three states
     $\cS=\{s_{\text{init}},s_1,s_2\}$, two actions
     $\cA=\{a_1,a_2\}$, horizon $H=2$, initial state $ s_{\text{init}}$,
     transition model $p$ such that:
     \begin{align*}
       p_1(s'|s_{\text{init}},a_1)=\begin{cases}
         1/2 & \text{if }s'=s_1\\
         1/2 & \text{if }s'=s_{2}
       \end{cases},\\
       p_1(s'|s_{\text{init}},a_2)=\begin{cases}
         0.9 & \text{if }s'=s_{1}\\
         0.1 & \text{if }s'=s_{2}
       \end{cases},
     \end{align*}
     and reward function $r$ that assigns
     $r_1(s_{\text{init}},a_1)=0$, $r_1(s_{\text{init}},a_2)=0.5$, and:
     \begin{align*}
       r_2(s,a)=\begin{cases}
         0 & \text{if }s=s_1\wedge (a=a_1\lor a=a_2)\\
         1 & \text{if }s=s_{2}\wedge (a=a_1\lor a=a_2)
       \end{cases}.
     \end{align*}
     Note that the support of the return function is $\cG^{p,r}=\{0,0.5,1,1.5\}$.
     We are given an expert's policy $\pi^E$ that prescribes action $a_1$ at stage
     1 in state $s_{\text{init}}$, and arbitrary actions in other states (the
     specific action is not relevant). The MDP
     $\cM$ is represented in Figure \ref{fig: mdp for proof
     transferability r}.
   
     \begin{figure}[t]
       \centering
       \begin{tikzpicture}[node distance=3.5cm]
       \node[state,initial] at (-1,0) (s0) {$s_{\text{init}}$};
       \node[state] at (2.5,0.75) (s2) {$s_{1}$};
       \node[state] at (2.5,-0.75) (s3) {$s_{2}$};
       \node[state, draw=none] at (4.5,0.75) (a2) {};
       \node[state, draw=none] at (4.5,-0.75) (a3) {};
       \node[draw=none, fill=black] at (0.8,1.4) (a01) {};
       \node[draw=none,fill=black] at (0.3,-1.5) (a02) {};
       \draw (s0) edge[-, solid, above] node{$a_1$} (a01);
       \draw (s0) edge[-, solid, above] node{$a_2$} (a02);
       \draw (s2) edge[->, solid, above] node{$a_1,a_2$} (a2);
       \draw (s3) edge[->, solid, above] node{$a_1,a_2$} (a3);
       \draw (a01) edge[->, solid, above] node{\scriptsize$1/2$} (s2);
       \draw (a01) edge[->, solid, above] node{\scriptsize$1/2$} (s3);
       \draw (a02) edge[->, solid, below] node{\scriptsize$0.9$} (s2);
       \draw (a02) edge[->, solid, below] node{\scriptsize$0.1$} (s3);
     \end{tikzpicture}
     \caption{MDP for the proof of Proposition \ref{prop: transferring utilities r}.}
     \label{fig: mdp for proof transferability r}
     \end{figure}
   
     Now, we show that the utilities
     $U_1,U_2\in\fU$, defined in points of the support $\cG^{p,r}$ as (and connected in
     arbitrary continuous strictly-increasing manner between these points):
     \begin{align*}
       U_1(G)=\begin{cases}
         0 & \text{if }G=0\\
         0.1 & \text{if }G=0.5\\
         0.9 & \text{if }G=1\\
         1.5 & \text{if }G=1.5
       \end{cases},\qquad
       U_2(G)=\begin{cases}
         0 & \text{if }G=0\\
         0.1 & \text{if }G=0.5\\
         0.8 & \text{if }G=1\\
         1.5 & \text{if }G=1.5
       \end{cases},
     \end{align*}
     belong to the feasible set $\cU_{p,r,\pi^E}$, and, when transferred to the new
     MDP $\cM'=\tuple{\cS,\cA,H, s_{\text{init}},p,r'}$, with
     reward function $r'\neq r$ defined as:
     \begin{align*}
       &r_1'(s_{\text{init}},a_1)=0.5, \qquad r_1(s_{\text{init}},a_2)=0,\\
       &r_2'(s,a)=\begin{cases}
         1 & \text{if }s=s_1\wedge (a=a_1\lor a=a_2)\\
         0 & \text{if }s=s_{2}\wedge (a=a_1\lor a=a_2)
       \end{cases},
     \end{align*}
     impose different optimal policies, i.e., utility $U_2$ keeps making action
     $a_1$ optimal from state $s_{\text{init}}$ even in $\cM'$, while $U_1$ makes
     action $a_2$ optimal. This will demonstrate the thesis of the proposition.
 
     Let us begin by showing that $U_1,U_2\in\cU_{p,r,\pi^E}$ belong to the feasible
     set of $\cM$ with policy $\pi^E$. Let $\overline{\pi}$ be the policy that
     plays action $a_2$ in state $s_{\text{init}}$. Then, the distribution of
     returns induced by policies $\pi^E$ and $\overline{\pi}$ are (we represent
     values only at points in $\cG^{p,r}=\{0,0.5,1,1.5\}$):
     \begin{align*}
       \eta^{p,r,\pi^E}&=[0.5, 0, 0.5, 0]^\intercal\\
       \eta^{p,r,\overline{\pi}}&=[0, 0.9, 0, 0.1]^\intercal.
     \end{align*}
     Thus, policy $\pi^E$ is optimal under some utility $U$ if and only if the
     values assigned by $U$ to points in $\cG^{p,r}=\{0,0.5,1,1.5\}$ (denoted,
     respectively, by $U^1,U^2,U^3,U^4$) satisfy:
     \begin{align*}
       U^\intercal (\eta^{p,r,\pi^E}-\eta^{p,r,\overline{\pi}})=
       [0.5, -0.9, 0.5, -0.1] U =0.5 U^1-0.9 U^2 + 0.5 U^3-0.1 U^4\ge 0,
     \end{align*}
     where we have overloaded the notation and denoted with
     $U\coloneqq[U^1,U^2,U^3,U^4]^\intercal$ both the utility and the vector of
     values assigned to points in $\cG^{p,r}$.
     By imposing normalization constraints ($U(0)=0,U(2)=2$), we get $U^1=0$, and
     by imposing also the monotonicity constraints, we get that utility $U$ is in
     the feasible set $\cU_{p,r,\pi^E}$ if and only if:
     \begin{align*}
       \begin{cases}
         U^4\ge 5 U^3 - 9 U^2\\
         0<U^2<U^3<U^4<2
       \end{cases}.
     \end{align*}
     Clearly, both utilities $U_1,U_2$ satisfy these constraints, thus they
     belong to the feasible set $\cU_{p,r,\pi^E}$.
     Now, concerning problem $\cM'$, the performances of $\pi^E,\overline{\pi}$
     w.r.t. utilities $U_1,U_2$ are:
     \begin{align*}
       J^{\pi^E}(U_1;p,r')&= 0 U_1(0)+0.5 U_1(0.5)+0 U_1(1)+0.5 U_1(1.5)=
       1.6/2=0.8,\\
       J^{\overline{\pi}}(U_1;p,r')&= 0.1 U_1(0)+0 U_1(0.5)+0.9 U_1(1)+0 U_1(1.5)=
       0.9\times 0.9=0.81,\\
       J^{\pi^E}(U_2;p,r')&= 0 U_2(0)+0.5 U_2(0.5)+0 U_2(1)+0.5 U_2(1.5)=
       1.6/2=0.8,\\
       J^{\overline{\pi}}(U_2;p,r')&= 0.1 U_2(0)+0 U_2(0.5)+0.9 U_2(1)+0 U_2(1.5)=
       0.9\times 0.8=0.72.
     \end{align*}
     Clearly, $J^{\pi^E}(U_1;p,r')<J^{\overline{\pi}}(U_1;p,r')$, but
     $J^{\pi^E}(U_2;p,r')>J^{\overline{\pi}}(U_2;p,r')$, thus we conclude that the
     set of policies induced by utilities $U_1,U_2$ in $\cM'$ do not intersect,
     since they start from $s_{\text{init}}$ with different actions
     $\Pi^*_{p,r'}(U_1)\cap\Pi^*_{p,r'}(U_2)=\{\}$. This concludes the proof
     with an example that satisfies $\cG^{p,r'}=\cG^{p,r}$.
 
     If we want an example that does \emph{not} satisfy $\cG^{p,r'}=\cG^{p,r}$,
     then we can consider exactly the same example with $\cM$ and $\cM'$, but
     using $r_1'(s_{\text{init}},a_2)=0.001$. In this manner, we see that
     $\cG^{p,r}=\{0,0.5,1,1.5\}$, and $\cG^{p',r}=\{0.001,0.5,1.001,1.5\}$, which
     are different. Nevertheless, by choosing $U_1',U_2'$ as:
     \begin{align*}
       U_1'(G)=\begin{cases}
         0 & \text{if }G=0\\
         0.001 & \text{if }G=0.001\\
         0.1 & \text{if }G=0.5\\
         0.9 & \text{if }G=1\\
         0.901 & \text{if }G=1.001\\
         1.5 & \text{if }G=1.5
       \end{cases},\qquad
       U_2'(G)=\begin{cases}
         0 & \text{if }G=0\\
         0.001 & \text{if }G=0.001\\
         0.1 & \text{if }G=0.5\\
         0.8 & \text{if }G=1\\
         0.801 & \text{if }G=1.001\\
         1.5 & \text{if }G=1.5
       \end{cases},
     \end{align*}
     it can be shown that $U_1',U_2'$ still belong to the feasible set of $\cM$
     (the constraints are the same), and that induce different policies in
     $\cM'$. This concludes the proof.
   \end{proof}
 
   \propimitation*
   \begin{proof}
     Consider a simple MDP
     $\cM=\tuple{\cS,\cA,H, s_{\text{init}},p,r}$ with four states
     $\cS=\{s_{\text{init}},s_1,s_2,s_3\}$, three actions
     $\cA=\{a_1,a_2,a_3\}$, horizon $H=2$, initial state $ s_{\text{init}}$,
     transition model $p$ such that:
     \begin{align*}
       &p_1(s_2|s_{\text{init}},a_1)=1,\qquad
       p_1(s_1|s_{\text{init}},a_3)=1,\\
       &p_1(s'|s_{\text{init}},a_2)=\begin{cases}
         0.91 & \text{if }s'=s_{1}\\
         0.09 & \text{if }s'=s_{3}
       \end{cases},
     \end{align*}
     and reward function $r$ that assigns
     $r_1(s_{\text{init}},a_1)=r_1(s_{\text{init}},a_2)=r_1(s_{\text{init}},a_3)=0$, and:
     \begin{align*}
       r_2(s,a)=\begin{cases}
         0 & \text{if }s=s_1\wedge (a=a_1\lor a=a_2 \lor a=a_3)\\
         0.5 & \text{if }s=s_{2}\wedge (a=a_1\lor a=a_2 \lor a=a_3)\\
         1 & \text{if }s=s_{3}\wedge (a=a_1\lor a=a_2 \lor a=a_3)
       \end{cases}.
     \end{align*}
     Note that the support of the return function is $\cG^{p,r}=\{0,0.5,1\}$. We
     are given an expert's policy $\pi^E$ that prescribes action $a_1$ at stage 1
     in state $s_{\text{init}}$, and arbitrary actions in other states (the
     specific action is not relevant). The MDP
     $\cM$ is represented in Figure \ref{fig: mdp for proof
     imitation}.
     \begin{figure}[t]
       \centering
       \begin{tikzpicture}[node distance=3.5cm]
       \node[state,initial] at (-1,0) (s0) {$s_{\text{init}}$};
       \node[state] at (2.5,1.5) (s2) {$s_{1}$};
       \node[state] at (2.5,0) (s3) {$s_{2}$};
       \node[state] at (2.5,-1.5) (s4) {$s_{3}$};
       \node[state, draw=none] at (4.5,1.5) (a2) {};
       \node[state, draw=none] at (4.5,0) (a3) {};
       \node[state, draw=none] at (4.5,-1.5) (a4) {};
       \node[draw=none,fill=black] at (0.3,-1.5) (a02) {};
       \draw (s0) edge[->, solid, above] node{$a_1$} (s3);
       \draw (s0) edge[->, solid, above] node{$a_3$} (s2);
       \draw (s0) edge[-, solid, above] node{$a_2$} (a02);
       \draw (s2) edge[->, solid, above] node{\scriptsize $a_1,a_2,a_3$} (a2);
       \draw (s3) edge[->, solid, above] node{\scriptsize $a_1,a_2,a_3$} (a3);
       \draw (s4) edge[->, solid, above] node{\scriptsize $a_1,a_2,a_3$} (a4);
       \draw (a02) edge[->, solid, below] node{\scriptsize$0.91$} (s2);
       \draw (a02) edge[->, solid, below] node{\scriptsize$0.09$} (s4);
     \end{tikzpicture}
     \caption{MDP for the proof of Proposition \ref{prop: imitating the expert}.}
     \label{fig: mdp for proof imitation}
     \end{figure}
 
     Now, we show that the utilities
     $U_1,U_2\in\fU$, defined in points of the support $\cG^{p,r}$ as (and connected in
     arbitrary continuous strictly-increasing manner between these points):
     \begin{align*}
       U_1(G)=\begin{cases}
         0 & \text{if }G=0\\
         0.1 & \text{if }G=0.5\\
         0.1/0.09 & \text{if }G=1
       \end{cases},\qquad
       U_2(G)=\begin{cases}
         0 & \text{if }G=0\\
         1.099 & \text{if }G=0.5\\
         1.1 & \text{if }G=1
       \end{cases},
     \end{align*}
     belong to the feasible set $\cU_{p,r,\pi^E}$, and that, for any
     $\epsilon\in[0,0.1]$, there exists a policy $\pi$ for which it holds both
     that $J^*(U_1;p,r)-J^\pi(U_1;p,r)=\epsilon$ and
     $J^*(U_2;p,r)-J^\pi(U_2;p,r)\ge 1$.
 
     First, let us show that both $U_1,U_2$ belong to the feasible utility set.
     Let $\pi^1,\pi^2,\pi^3$ be the policies that play, respectively, action
     $a_1,a_2,a_3$ in state $s_{\text{init}}$ (note that $\pi^1=\pi^E$). Then,
     their performances for arbitrary utility $U$ are:
     \begin{align*}
       &J^{\pi^1}(U; p,r)=U(0.5),\\
       &J^{\pi^2}(U; p,r)=0.09U(1)+0.91U(0)=0.09U(1),\\
       &J^{\pi^3}(U; p,r)=U(0)=0,
     \end{align*}
     where we have used the normalization condition. Replacing $U$ with $U_1$, we
     get $J^*(U_1;p,r)=J^{\pi^1}(U_1; p,r)=0.1\popblue{=}J^{\pi^2}(U_1; p,r)=0.1>J^{\pi^3}(U_1;
     p,r)=0$. Instead, replacing with $U_2$, we get $J^*(U_2;p,r)=J^{\pi^1}(U_2;
     p,r)=1.099>J^{\pi^2}(U_2; p,r)=0.09\times 1.1>J^{\pi^3}(U_2; p,r)=0$. Therefore, both
     $U_1,U_2\in\cU_{p,r,\pi^E}$.
 
     Now, for any $\alpha\in[0,1]$ let us denote by $\pi_\alpha$ the policy that,
     at state $s_{\text{init}}$, plays action $a_3$ w.p. $\alpha$, and action
     $a_2$ w.p. $1-\alpha$. We show that, for any $\epsilon\in[0,0.1]$, policy
     $\pi_{\epsilon/0.1}$ is $\epsilon$-optimal for utility $U_1$, and its
     suboptimality is at least $1$ under utility $U_2$. For any $\alpha\in[0,1]$,
     the expected utilities of policy $\pi_\alpha$ under $U_1$ and $U_2$ are:
     \begin{align*}
       &J^{\pi_\alpha}(U_1;p,r)=(1-\alpha)\times 0.09\times U_1(1)=(1-\alpha)\times 0.1,\\
       &J^{\pi_\alpha}(U_2;p,r)=(1-\alpha)\times 0.09\times U_2(1)=(1-\alpha)\times 0.099,
     \end{align*}
     from which we derive that the suboptimalities of such policy under $U_1$ and $U_2$ are:
     \begin{align*}
       &J^*(U_1;p,r)-J^{\pi_\alpha}(U_1;p,r)=0.1-(1-\alpha)\times 0.1=0.1\alpha,\\
       &J^*(U_2;p,r)-J^{\pi_\alpha}(U_2;p,r)=1.099-(1-\alpha)\times 0.099=1+0.099\alpha.
     \end{align*}
     Thus, for any $\epsilon\in[0,0.1]$, policy $\pi_{\epsilon/0.1}$ is
     $\epsilon$-optimal for utility $U_1$, but it is at least $1$-suboptimal for
     utility $U_2$.
 
     The intuition is that utilities $U_1$ and $U_2$ assess in completely
     different manners the policies that play action $a_2$, although they both
     describe policy $\pi^E$ as optimal. This concludes the proof.
   \end{proof}
 
   \propnobounddall*
   \begin{proof}
     Consider a simple MDP
     $\cM=\tuple{\cS,\cA,H, s_{\text{init}},p,r}$ with three states
     $\cS=\{s_{\text{init}},s_1,s_2\}$, three actions
     $\cA=\{a_1,a_2,a_3\}$, horizon $H=2$, initial state $ s_{\text{init}}$,
     transition model $p$ such that:
     \begin{align*}
       p_1(s_1|s_{\text{init}},a_1)=1,\qquad 
       p_1(s_2|s_{\text{init}},a_2)=p_1(s_2|s_{\text{init}},a_2)=1,
     \end{align*}
     and reward function $r$ that assigns
     $r_1(s_{\text{init}},a_1)=r_1(s_{\text{init}},a_2)=0$, $r_1(s_{\text{init}},a_2)=1$, and:
     \begin{align*}
       r_2(s,a)=\begin{cases}
         0 & \text{if }s=s_1\wedge (a=a_1\lor a=a_2\lor a_3)\\
         1 & \text{if }s=s_{2}\wedge (a=a_1\lor a=a_2\lor a_3)
       \end{cases}.
     \end{align*}
     Note that the support of the return function is $\cG^{p,r}=\{0,1,2\}$. We
     are given an expert's policy $\pi^E$ that prescribes action $a_3$ at stage 1
     in state $s_{\text{init}}$, and arbitrary actions in the other states (the
     specific action is not relevant). The MDP
     $\cM$ is represented in Figure \ref{fig: mdp for proof no
     bound d all}.
   
     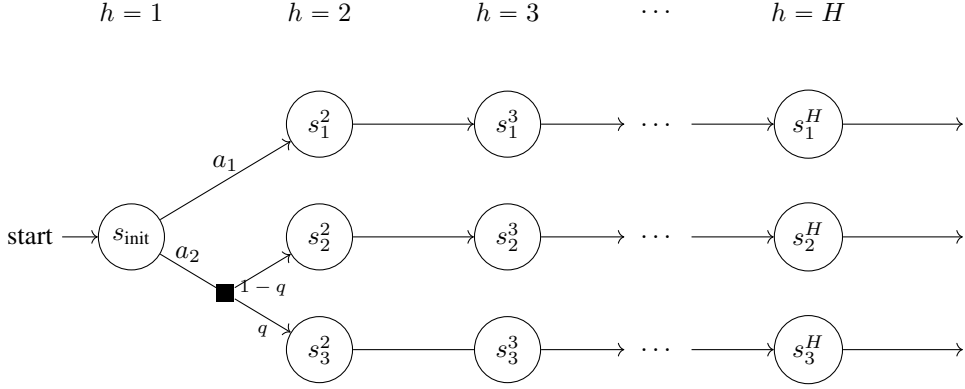
\begin{figure}[t]
       \centering
       \begin{tikzpicture}[node distance=3.5cm]
       \node[state,initial] at (-1,0) (s0) {$s_{\text{init}}$};
       \node[state] at (2.5,0.75) (s2) {$s_{1}$};
       \node[state] at (2.5,-0.75) (s3) {$s_{2}$};
       \node[state, draw=none] at (4.5,0.75) (a2) {};
       \node[state, draw=none] at (4.5,-0.75) (a3) {};
       \draw (s2) edge[->, solid, above] node{\scriptsize$a_1,a_2,a_3$} (a2);
       \draw (s3) edge[->, solid, above] node{\scriptsize$a_1,a_2,a_3$} (a3);
       \draw (s0) edge[->, solid, above] node{\scriptsize$a_1$} (s2);
       \draw (s0) edge[->, solid, above] node{\scriptsize$a_2,a_3$} (s3);
     \end{tikzpicture}
     \caption{MDP for the proof of Proposition \ref{prop: no bound d all}.}
     \label{fig: mdp for proof no bound d all}
     \end{figure}
 
     Consider two utilities $U_1,U_2$, that take on the following values in
     $\cG^{p,r}$:
     \begin{align*}
       &U_1(G)=\begin{cases}
         0& \text{if }G=0\\
         0.1& \text{if }G=1\\
         2& \text{if }G=2
       \end{cases},\\
       &U_2(G)=\begin{cases}
         0& \text{if }G=0\\
         1.1& \text{if }G=1\\
         2& \text{if }G=2
       \end{cases}.
     \end{align*}
     It is immediate that both utilities belong to the feasible set
     $\cU_{p,r,\pi^E}$. Nevertheless, if we denote by $\overline{\pi}$ the policy
     that plays action $a_2$ in state $s_{\text{init}}$, we see that
     $J^{\overline{\pi}}(U_1;p,r)=0.1$, while $J^{\overline{\pi}}(U_2;p,r)=1.1$,
     so that the difference is $1$.
   \end{proof}
 
   \propmultipledemonstrations*
   \begin{proof}
     We provide a constructive proof that shows which values of $ s_0,p,r$ it is
     sufficient to choose for recovering $U^E$ exactly. The construction is
     articulated into two parts. First, we aim to recover the value of $U^E(1)$,
     i.e., for $G=1$; next, we recover the utility for all other possible values
     of return. The intuition is that we construct a Standard Gamble (SG) between
     two policies over the entire horizon \citep{Wakker2010prospect}.
 
     To infer $U^E(1)$, we use the $ s_0,p,r$ values that provide the MDP
     described in Figure \ref{fig: mdp for proof multiple demonstrations}.
 
     \begin{figure}[t]
       \centering
       \begin{tikzpicture}[node distance=3.5cm]
       \node[state,initial] at (-1,0) (s0) {$s_{\text{init}}$};
       \node[state] at (1.5,1.5) (s2) {$s_{1}^2$};
       \node[state] at (1.5,0) (s3) {$s_{2}^2$};
       \node[state] at (1.5,-1.5) (s4) {$s_{3}^2$};
       \node[state] at (4,1.5) (s5) {$s_{1}^3$};
       \node[state] at (4,0) (s6) {$s_{2}^3$};
       \node[state] at (4,-1.5) (s7) {$s_{3}^3$};
       \node[state] at (8,1.5) (s8) {$s_{1}^H$};
       \node[state] at (8,0) (s9) {$s_{2}^H$};
       \node[state] at (8,-1.5) (s10) {$s_{3}^H$};
       \node[state, draw=none] at (6,1.5) (a2) {$\dotsc$};
       \node[state, draw=none] at (6,0) (a3) {$\dotsc$};
       \node[state, draw=none] at (6,-1.5) (a4) {$\dotsc$};
       \node[state, draw=none] at (10.5,1.5) (a5) {};
       \node[state, draw=none] at (10.5,0) (a6) {};
       \node[state, draw=none] at (10.5,-1.5) (a7) {};
       \node[draw=none,fill=black] at (0.25,-0.75) (a02) {};
       \node[state, draw=none] at (-1,3) (h1) {$h=1$};
       \node[state, draw=none] at (1.5,3) (h2) {$h=2$};
       \node[state, draw=none] at (4,3) (h3) {$h=3$};
       \node[state, draw=none] at (6,3) (h4) {$\dotsc$};
       \node[state, draw=none] at (8,3) (hH) {$h=H$};
       \draw (s0) edge[->, solid, above] node{$a_1$} (s2);
       \draw (s0) edge[-, solid, above] node{$a_2$} (a02);
       \draw (a02) edge[->, solid, below] node{\scriptsize$1-q$} (s3);
       \draw (a02) edge[->, solid, below] node{\scriptsize$q$} (s4);
       \draw (s2) edge[->, solid, above] node{} (s5);
       \draw (s3) edge[->, solid, above] node{} (s6);
       \draw (s4) edge[-, solid, above] node{} (s7);
       \draw (s5) edge[->, solid, above] node{} (a2);
       \draw (s6) edge[->, solid, above] node{} (a3);
       \draw (s7) edge[->, solid, above] node{} (a4);
       \draw (a2) edge[->, solid, above] node{} (s8);
       \draw (a3) edge[->, solid, above] node{} (s9);
       \draw (a4) edge[->, solid, above] node{} (s10);
       \draw (s8) edge[->, solid, above] node{} (a5);
       \draw (s9) edge[->, solid, above] node{} (a6);
       \draw (s10) edge[->, solid, above] node{} (a7);
     \end{tikzpicture}
     \caption{MDP for the proof of Proposition \ref{prop: multiple demonstrations}.}
     \label{fig: mdp for proof multiple demonstrations}
     \end{figure}
 
     We consider a single initial state $s_{\text{init}}$. From here, action
     $a_1$ (and all actions other than $a_1$ and $a_2$) brings deterministically
     to state $s_1^2$, while action $a_2$ brings to state $s_3^2$ w.p. $q$ (to
     choose, for some $q\in[0,1]$), and to state $s_2^2$ w.p. $1-q$. From state
     $s_i^2$, for any $i\in\dsb{3}$, all actions bring deterministically to state
     $s_i^3$, and so on, up to state $s_i^H$. We will call the trajectory
     $\{s_{\text{init}},s_i^2,s_i^3,\dotsc,s_i^H\}$ the $i^{\text{th}}$
     trajectory for all $i\in\dsb{3}$, and we will write $G(i)$ to denote the sum
     of rewards along such trajectory. To infer the value $U^E(1)$, we select a
     reward $r':\SAH\to[0,1]$ that provides return $G(1)=1.5$ to the first
     trajectory, return $G(2)=1$ to the second trajectory, and return $G(3)=H$ to
     the third trajectory (this is possible because $H\ge 2$). By selecting,
     successively, all the values of $q\in[0,1]$, we are asking to the expert to
     play either action $a_1$ or action $a_2$ from the initial state
     $s_{\text{init}}$ (we denote policies $\pi^1,\pi^2$, respectively, the
     policies that play actions $a_1,a_2$ in $s_{\text{init}}$). Since we are
     assuming that the expert will demonstrate all the possible deterministic
     optimal policies, there exists a value $q'\in[0,1]$ for which the expert
     demonstrates both policies $\pi^1$ and $\pi^2$. Indeed, the expected
     utilities of policies $\pi^1,\pi^2$ for arbitrary value of $q$ are (we write
     $p(q)$ as the generic transition model):
     \begin{align*}
       \begin{split}
         &J^{\pi^1}(U^E;p(q),r')=U^E(1.5),\\
         &J^{\pi^2}(U^E;p(q),r')=qU^E(H)+(1-q)U^E(1)=qH+(1-q)U^E(1),
       \end{split}
     \end{align*}
     and since $U^E$ is strictly-increasing, we have $U^E(1)<U^E(1.5)<U^E(H)=H$,
     thus there must exist $q'$ that permits to write $U^E(1.5)$ as a convex
     combination of the other two. This allows us to write:
     \begin{align}\label{eq: ut1 multiple demonstrations}
       U^E(1.5)=q'H+(1-q')U^E(1).
     \end{align}
     Next, we select reward $r''$ that provides returns
     $G(1)=1,G(2)=0.5,G(3)=1.5$. Thus, there must exist a $q''\in[0,1]$ for which
     the expert demonstrates both policies $\pi^1$ and $\pi^2$, allowing us to
     write:
     \begin{align}\label{eq: ut2 multiple demonstrations}
       U^E(1)=q''U^E(1.5)+(1-q'')U^E(0.5).
     \end{align}
     Finally, we can repeat the same step with a third reward $r'''$ that
     provides returns $G(1)=0.5,G(2)=0,G(3)=1$, and for some $q'''\in[0,1]$ we
     obtain:
     \begin{align}\label{eq: ut3 multiple demonstrations}
       U^E(0.5)=q'''U^E(1).
     \end{align}
     By putting together Eq. \eqref{eq: ut1 multiple demonstrations}, Eq. \eqref{eq:
     ut2 multiple demonstrations}, and Eq. \eqref{eq: ut3 multiple demonstrations},
     we can retrieve $U^E(1)$:
     \begin{align*}
       \begin{cases}
         U^E(1.5)=q'H+(1-q')U^E(1)\\
         U^E(1)=q''U^E(1.5)+(1-q'')U^E(0.5)\\
         U^E(0.5)=q'''U^E(1)
       \end{cases}.
     \end{align*}
 
     Now that we know $U^E(1)$, we can infer the utility for all the returns
     $\overline{G}\in(1,H)$ by choosing a reward that provides returns
     $G(1)=\overline{G},G(2)=1,G(3)=H$, because for some $\overline{q}\in[0,1]$
     the expert will play both policies $\pi^1$ and $\pi^2$, which allows us to
     write:
     \begin{align*}
       U^E(\overline{G})=\overline{q}H+(1-\overline{q})U^E(1),
     \end{align*}
     and to retrieve $U^E(\overline{G})$.
 
     Similarly, for all $\overline{G}\in(0,1)$, we select a reward that provides
     returns $G(1)=\overline{G},G(2)=0,G(3)=1$, and for some
     $\overline{q}\in[0,1]$ we can write:
     \begin{align*}
       U^E(\overline{G})=\overline{q}U^E(1),
     \end{align*}
     and retrieve $U^E(\overline{G})$.
 
     This concludes the proof. As a final remark, we stress that the initial step
     for inferring $U^E(1)$ cannot be dropped because there is no reward
     $r:\SAH\to[0,1]$ that provides returns $G(2)=0$ and $G(3)=H$, because both
     the first and second trajectories pass through action $a_2$ in state
     $s_{\text{init}}$.
   \end{proof}
 
 \section{Additional Results and Proofs for Section \ref{sec: online utility learning}}
 \label{apx: section 5}
 
 This appendix is divided in 5 parts. First, we show the complexity of
 implementing operator $\Pi_{\overline{\underline{\fU}}_L}$ (Appendix \ref{apx:
 computational complexity projection}). In Appendix \ref{apx: missing algorithms
 and sub routines}, we provide the pseudocode, along with a description, of
 algorithms \texttt{EXPLORE}, \texttt{PLANNING}, \texttt{ERD}, and
 \texttt{ROLLOUT}. In Appendix \ref{apx: comp complexity} we analyze the time
 and space complexities of \caty and \tractor. In Appendix \ref{apx: analysis
 caty}, we provide the proof of Theorem \ref{thr: caty upper bound 1 u}. In
 Appendix \ref{apx: analysis tractor}, we provide the proof of Theorem \ref{thr:
 tractor upper bound}.
 
 \subsection{Projecting onto the Set of Discretized Utilities}
 \label{apx: computational complexity projection}
 
 Let us use the square brackets $[]$ to denote the components of vectors. Then,
 note that set $\overline{\underline{\fU}}_L$ can be represented more explicitly as:
 \begin{align}\label{eq: set of utilities for projection}
   \overline{\underline{\fU}}_L= 
   \{\overline{U}\in[0,H]^{d}\,|\,&\overline{U}[1]=0\wedge \overline{U}[d]=H\wedge \overline{U}[i]\le \overline{U}[i+1]\;
   \forall i\in\dsb{d-1}\notag\\
   &\wedge\;\forall i,j \in\dsb{d}\text{ s.t. } i<j:\,
   |\overline{U}[i]-\overline{U}[j]|\le L(j-i)\epsilon_0\}.
 \end{align}
 Notice that set $\overline{\underline{\fU}}_L$ is closed and convex, since it is defined by linear
 constraints only.
 The amount of constraints scales as $\propto d^2$.

 We remark that in Theorem \ref{thr: tractor upper bound} we assume availability
 of an oracle for computing the projection exactly. In practice, we can use any
 quadratic programming solver to approximate the projection.
 
 \subsection{Missing Algorithms and Sub-routines}
 \label{apx: missing algorithms and sub routines}
 
 \paragraph{\texttt{EXPLORE}} In Algorithm \ref{alg: caty exploration}, we report
 the pseudo-code implementing subroutine \texttt{EXPLORE}. Simply put, we adopt a
 uniform-sampling strategy, i.e., we collect $n=\floor{\tau/(SAH)}$ samples from
 each $(s,a,h)\in\SAH$ triple, that we use to compute the empirical estimate of
 the transition model. We return such estimate.
 
 \input{algorithms/explore.tex}

 \paragraph{\texttt{PLANNING}} The \texttt{PLANNING} sub-routine (Algorithm
 \ref{alg: planning}) takes in input a utility $U$, an environment index $i$, and
 a transition model $p$, that uses to construct the RS-MDP
 $\cM_U\coloneqq\tuple{\cS^i,\cA^i,H,s_0^i,\popblue{p},\popblue{\overline{r}^i},\popblue{U}}$.
 Notice that $\cM_U\neq\cM^i_{U^E}$, for 3 aspects. First, it uses the input
 transition model $p\neq p^i$; next, it consider the discretized reward
 $\overline{r}^i\neq r^i$; finally, it has input utility $U\neq U^E$.
 
 \texttt{PLANNING} outputs two items. The optimal performance $J^*(U;p,r^i)$ for
 RS-MDP $\cM_U$, and the optimal policy $\psi^*=\{\psi^*_h\}_h$ for the enlarged
 state space MDP $\fE[\cM_U]$. However, it should be remarked that, instead of
 computing optimal policy $\psi^*$ for $\fE[\cM_U]$ only at pairs
 $(s,y)\in\cS\times\cG^{p,\overline{r}^i}_h$ for all $h\in\dsb{H}$,
 \texttt{PLANNING} computes the optimal policy $\psi^*$ at all pairs
 $(s,y)\in\cS\times\cY_h$ for all $h\in\dsb{H}$ (note that
 $\cG^{p,\overline{r}^i}_h\subseteq\cY_h$).
 
 The algorithm implemented in \texttt{PLANNING} for computing both $J^*(U;p,r^i)$
 and $\psi^*$ is value iteration. The difference from common implementations of
 value iterations lies in the presence of an additional variable in the state. A
 similar pseudocode is provided in Algorithm 1 of \citet{wu2023risksensitive}.
 
 \input{algorithms/planning.tex}

 \paragraph{\texttt{ERD} (Estimate the Return Distribution)}
 
 The \texttt{ERD} sub-routine (Algorithm \ref{alg: erd}) takes in input a dataset
 $\cD^E=\{\omega_j\}_j$ of state-action trajectories $\omega_j\in\Omega$ and a
 reward function $r$, and it computes an estimate of the return distribution w.r.t.
 $r$.
 
 For every trajectory $\omega_j\in\cD^E$, \texttt{ERD} computes the return $G_j$ of
 $\omega_j$ based on the input reward $r$ (Line \ref{line: compute G erd}). In
 the next lines, \texttt{ERD} simply computes the categorical projection of the
 mixture of Dirac deltas:
 \begin{align*}
   \widehat{\eta}= \text{Proj}_\cC\Big(\sum_j \frac{1}{|\cD^E|} \delta_{G_j}\Big),
 \end{align*}
 where the categorical projection operator $\text{Proj}_\cC$ is defined in Eq.
 \eqref{eq: definition proj c}.
 
 \input{algorithms/erd.tex}

 \paragraph{\texttt{ROLLOUT}} \texttt{ROLLOUT} (Algorithm \ref{alg: rollout})
 takes in input a Markovian policy $\psi$, a transition model $p$, a reward $r$,
 an environment index $i$, and a number of trajectories $K$, to construct the MDP
 $\cM\coloneqq\tuple{\cS^i,\cA^i,H,s_0^i,\popblue{p},\popblue{r}}$ obtained from MDP $\cM^i$
 by replacing the dynamics and reward $p^i,r^i$ with the input $p,r$.
 
 \texttt{ROLLOUT} collects $K$ trajectories by playing policy $\psi$ in $\cM$ for
 $K$ times, computes the return $G$ of each trajectory, and then returns a
 dataset $\cD$ containing these $K$ returns.
 In other words, with abuse of notation, we say that the outputted dataset
 $\cD=\{G_k\}_{k\in\dsb{K}}$ is obtained by collecting $K$ samples $G_k$ from
 distribution $\eta^{p,r,\psi}$.
 
 \input{algorithms/rollout.tex}

\subsection{Time and Space Complexities}\label{apx: comp complexity}

The time and space complexities of the subroutines are:
\begin{itemize}
  \item \texttt{EXPLORE}: \emph{time} = $\mathcal{O}\Big(N\tau\Big)$ for
  collecting $\tau$ samples from the $N$ environments; \emph{space} =
  $\mathcal{O}\Big(SAHN\Big)$ for storing the estimates of the transition model
  of the $N$ environments.
  \item \texttt{ERD}: \emph{time} = $\mathcal{O}\Big(H\tau^E+H/\epsilon_0\Big)$
  for computing the return of each trajectory demontrated by the expert and
  initializing an estimate of the return distribution; \emph{space} =
  $\mathcal{O}\Big(H/\epsilon_0\Big)$ to store an estimate of the return
  distribution.
  \item \texttt{PLANNING}: \emph{time} = $\mathcal{O}\Big(S^2AH^2/\epsilon_0\Big)$ for
  doing backward induction in the enlarged discretized MDP; \emph{space} =
  $\mathcal{O}\Big(SAH^2/\epsilon_0\Big)$ to store a Q-function in the enlarged
  discretized MDP.
  \item \texttt{ROLLOUT}: \emph{time} = $\mathcal{O}\Big(KH\Big)$ for simulating $K$
  trajectories long $H$; \emph{space} = $\mathcal{O}\Big(K\Big)$ for storing the
  returns of the $K$ trajectories.
\end{itemize}
Using these complexities, we derive the complexities of \caty and \tractor
as:
\begin{itemize}
  \item \caty: \emph{time} =
  $\mathcal{O}\Big(N\tau+MN\Big(H\tau^E+S^2AH^2/\epsilon_0\Big)\Big)$ for
  calling \texttt{EXPLORE} once and then both \texttt{ERD} and \texttt{PLANNING}
  $MN$ times, where $M$ denotes the number of input utilities to which \caty is
  applied; \emph{space} = $\mathcal{O}\Big(SAHN+SAH^2/\epsilon_0\Big)$ where the
  dominant terms are for storing a transition model in \texttt{EXPLORE} and a
  Q-function in \texttt{PLANNING}.
  \item \tractor: \emph{time} =
  $\mathcal{O}\Big(N\tau+NH\tau^E+T\Big(NS^2AH^2/\epsilon_0+NKH+Q_{time}\Big)\Big)$
  for calling \texttt{EXPLORE} once, \texttt{ERD} $N$ times, both
  \texttt{PLANNING} and \texttt{ROLLOUT} $TN$ times, and executing $T$ times the
  Euclidean projection onto $\overline{\underline{\mathfrak{U}}}_L$ using some
  optimization solver ($Q_{time}$ represents this term); \emph{space} =
  $\mathcal{O}\Big(NH/\epsilon_0+SAH^2/\epsilon_0+K+Q_{space}\Big)$ for storing
  the $N$ estimates of return distributions, for calling \texttt{PLANNING} and
  \texttt{ROLLOUT}, and for executing some optimization solver for Euclidean projection
  ($Q_{space}$ represents this term).
\end{itemize}
Observe that the time and space complexities of the proposed algorithms are
polynomial in the amount of data ($\tau,\tau^E$), in the number of
environments ($N$), and in the size of the environments ($S,A,H$). Moreover,
both \caty and \tractor have time complexities that grow linearly in the
number of runs (resp. $M$ and $T$), and note that the complexity of \tractor
grows linearly also in the number of simulated trajectories ($K$) and in the
complexity of the optimization solver used for the Euclidean projection
($Q_{time}$). Observe that the complexities depend on $1/\epsilon_0$, where
$\epsilon_0>0$ is the discretization parameter.

From a theoretical perspective, if we want that, with probability at least
$1-\delta$, the outputs of \caty and \tractor are $\epsilon$-accurate, then,
under the assumption that the output of the optimization solver adopted for the
Euclidean projection is \emph{exact}, Theorems \ref{thr: caty upper bound 1 u}
and \ref{thr: tractor upper bound} show that it suffices to take
$\epsilon_0=\Theta(\epsilon^2/(HN^2))$,
$\tau^E\le\widetilde{\mathcal{O}}\Big(\frac{N^4H^4}{\epsilon^4}\log\frac{1}{\delta}\Big)$,
$\tau\le\widetilde{\mathcal{O}}\Big(\frac{N^2SAH^5}{\epsilon^2}\Big(S+\log\frac{1}{\delta}\Big)\Big)$,
$T\le\mathcal{O}\Big(\frac{N^4H^4}{\epsilon^4}\Big)$,
$K\le\widetilde{\mathcal{O}}\Big(\frac{N^2H^2}{\epsilon^2}\log\frac{1}{\delta}\Big)$,
for obtaining a time and space complexity for the algorithms that grow
polynomially in $S,A,H$, $N,$ $\frac{1}{\epsilon}$,
$\log\frac{1}{\delta},Q_{time},Q_{space}$.

 \subsection{Analysis of \caty}
 \label{apx: analysis caty}
 
 \thrupperboundcatyoneu*
 \begin{proof}
   Observe that the classification carried out by \caty complies with the
   statement in the theorem as long as we can demonstrate that:
      \begin{align*} 
      \mathop{\mathbb{P}}\limits_{\{\cM^i\}_i,\{\pi^{E,i}\}_i}\Big(
       \sup\limits_{U\in\cU}\Big|\sum\limits_{i\in\dsb{N}}
        \overline{\cC}_{p^i,r^i,\pi^{E,i}}(U)-
        \sum\limits_{i\in\dsb{N}}\widehat{\cC}^i(U)
        \Big|\le\epsilon\Big)\ge 1-\delta,
  \end{align*}
  where $\mathbb{P}_{\{\cM^i\}_i,\{\pi^{E,i}\}_i}$ represents the joint
  probability distribution induced by the exploration phase of \caty and the
  execution of each $\pi^{E,i}$ in the corresponding $\cM^i$.
 
 We can rewrite this expression as:
 \begin{align*}
   \sup\limits_{U\in\cU}\Big|\sum\limits_{i\in\dsb{N}}
       \overline{\cC}_{p^i,r^i,\pi^{E,i}}(U)-
       \sum\limits_{i\in\dsb{N}}\widehat{\cC}^i(U)
       \Big|&\le
       \sup\limits_{U\in\cU}\popblue{\sum\limits_{i\in\dsb{N}}}\Big|
       \overline{\cC}_{p^i,r^i,\pi^{E,i}}(U)-
       \widehat{\cC}^i(U)
       \Big|\\
       &\markref{(1)}{\le}\sum\limits_{i\in\dsb{N}}\popblue{\sup\limits_{U\in\cU}}\Big|
       \overline{\cC}_{p^i,r^i,\pi^{E,i}}(U)-
       \widehat{\cC}^i(U)
       \Big|,
 \end{align*}
 where at (1) we have upper bounded the maximum of a sum with the sum of the
 maxima. This shows that we can obtain the result as long as we can demonstrate
 that, for all $i\in\dsb{N}$, it holds that:
 \begin{align}\label{eq: bound non comp i}
   \mathop{\mathbb{P}}\limits_{p^i,r^i,\pi^{E,i}}\Big(
       \sup\limits_{U\in\cU}\Big|
       \overline{\cC}_{p^i,r^i,\pi^{E,i}}(U)-
       \widehat{\cC}^i(U)
       \Big|\le\frac{\epsilon}{N}\Big)\ge 1-\frac{\delta}{N};
 \end{align}
 the statement of the theorem would then follow from a union bound.
 Therefore, let us omit the $i$ index for simplicity, and let us try to obtain
 the bound in Eq. \eqref{eq: bound non comp i}. We can write:
 \begin{align*}
   \sup\limits_{U\in\cU}\big|
       \overline{\cC}_{p,r,\pie}(U)-\widehat{\cC}(U)\big|
       &\coloneqq 
       \sup\limits_{U\in\cU}\big|\big(
         J^*(U;p,r)-J^{\pi^E}(U;p,r)\big)-
         \big(
         \widehat{J}^*(U)-\widehat{J}^{E}(U)\big)\big|\\   
       &\markref{(2)}{\le}
       \popblue{\sup\limits_{U\in\cU}\big|}J^{\pi^E}(U;p,r)-\widehat{J}^E(U)\popblue{\big|}
       +\popblue{\sup\limits_{U\in\cU}\big|}J^*(U;p,r)-\widehat{J}^*(U)\popblue{\big|}\\
       &\markref{(3)}{=}\sup\limits_{U\in\cU}\big|
       \mathop{\E}\limits_{G\sim\eta^{p,r,\pi^E}}[U(G)]-
       \mathop{\E}\limits_{G\sim\widehat{\eta}^{E}}[U(G)]\\
       &\qquad\popblue{\pm
       \mathop{\E}\limits_{G\sim\text{Proj}_{\cC}(\eta^{p,r,\pi^E})}[U(G)]}\big|
       +\sup\limits_{U\in\cU}\big|J^*(U;p,r)-\widehat{J}^*(U)\big|\\
       &\markref{(4)}{\le}\popblue{\sup\limits_{U\in\cU}
       \big|}\mathop{\E}\limits_{G\sim\eta^{p,r,\pi^E}}[U(G)]-
       \mathop{\E}\limits_{G\sim\text{Proj}_{\cC}(\eta^{p,r,\pi^E})}[U(G)]\popblue{\big|}\\
       &\qquad+\popblue{\sup\limits_{U\in\cU}
       \big|}
       \mathop{\E}\limits_{G\sim\text{Proj}_{\cC}(\eta^{p,r,\pi^E})}[U(G)]-
       \mathop{\E}\limits_{G\sim\widehat{\eta}^{E}}[U(G)]\popblue{\big|}\\
       &\qquad+\sup\limits_{U\in\cU}\big|J^*(U;p,r)-\widehat{J}^*(U)\big|\\
       &\markref{(5)}{\le}\sup\limits_{\popblue{f:\,f\text{ is }L\text{-Lipschitz}}}
       \big|\mathop{\E}\limits_{G\sim\eta^{p,r,\pi^E}}[f(G)]-
       \mathop{\E}\limits_{G\sim\text{Proj}_{\cC}(\eta^{p,r,\pi^E})}[f(G)]\big|\\
       &\qquad+\sup\limits_{U\in\cU}
       \big|
       \mathop{\E}\limits_{G\sim\text{Proj}_{\cC}(\eta^{p,r,\pi^E})}[U(G)]-
       \mathop{\E}\limits_{G\sim\widehat{\eta}^{E}}[U(G)]\big|\\
       &\qquad+\sup\limits_{U\in\cU}\big|J^*(U;p,r)-\widehat{J}^*(U)\big|\\
       &\markref{(6)}{=} \popblue{L\cdot w_1(\eta^{p,r,\pi^E},\text{Proj}_{\cC}(\eta^{p,r,\pi^E}))}\\
       &\qquad+
       \sup\limits_{U\in\cU}
       \big|\mathop{\E}\limits_{G\sim\text{Proj}_{\cC}(\eta^{p,r,\pi^E})}[U(G)]-
       \mathop{\E}\limits_{G\sim\widehat{\eta}^{E}}[U(G)]\big|\\
       &\qquad
       \sup\limits_{U\in\cU}\big|J^*(U;p,r)-\widehat{J}^*(U)\big|,
 \end{align*}
 where at (2) we have applied triangle inequality, at (3) we use the definition
 of $J^{\pi^E}(U;p,r)$, and that of $\widehat{J}^E(U)$ (Line \ref{line: est JE
 caty} of \caty), and we have added and subtracted a term, where operator
 $\text{Proj}_{\cC}$ is defined in Eq. \eqref{eq: definition proj c}. We remark
 that distribution $\eta^{p,r,\pi^E}$ may have a support that grows exponentially
 in $H$, while both $\widehat{\eta}^{E}$ and
 $\text{Proj}_{\cC}(\eta^{p,r,\pi^E})$ are supported on $\cY$. Note that
 $\widehat{\eta}^{E}$ and $\text{Proj}_{\cC}(\eta^{p,r,\pi^E})$ are different
 distributions, since the former is the projection on $\cY$ of an
 estimate of $\eta^{p,r,\pi^E}$. At (4), we apply triangle inequality, at (5) we
 use the hypothesis that all utilities are $L$-Lipschitz $\cU\subseteq\fU_L$, and
 notice that $\fU_L$ is a subset of all $L$-Lipschitz functions
 $f:[0,H]\to[0,H]$, and at (6) we apply the duality formula for the 1-Wasserstein
 distance $w_1$ (see Eq. (6.3) in Chapter 6 of \citet{Villani2008OptimalTO}).
 
 Concerning the case $|\cU|=1$, we apply, for all $i\in\dsb{N}$, Lemma
 \ref{lemma: bound estimation error 1 U} with probability $\delta/(2N)$ and
 accuracy $\epsilon/(3N)$, and Lemma
 \ref{lemma: bound J star single utility} with probability $\delta/(2N)$ and
 accuracy $\epsilon/(3N)$, while
 we bound the 1-Wasserstein distance through Lemma \ref{lemma: bound w1 etae}, to
 obtain, through an application of the union bound, that:
 \begin{align*} 
   \mathop{\mathbb{P}}\limits_{\{\cM^i\}_i,\{\pi^{E,i}\}_i}\Big(
     \sup\limits_{U\in\cU}&\Big|\sum\limits_{i\in\dsb{N}}
     \overline{\cC}_{p^i,r^i,\pi^{E,i}}(U)-
     \sum\limits_{i\in\dsb{N}}\widehat{\cC}^i(U)
     \Big|\le\\
     &\qquad NL\sqrt{2H\epsilon_0}+\epsilon/3+NHL\epsilon_0+\epsilon/3
     \Big)\ge 1-\delta,
 \end{align*}
 as long as, for all $i\in\dsb{N}$:
 \begin{align*}
   &\tau^{E,i}\ge \widetilde{\cO}\bigg(\frac{N^2H^2\log\frac{N}{\delta}}{\epsilon^2}\bigg),\\
   &\tau^i\ge\widetilde{\cO}\bigg(\frac{N^2SAH^4}{\epsilon^2}
   \log\frac{SAHN}{\delta\epsilon_0}\bigg).
 \end{align*}
 By setting $\epsilon_0= \frac{\epsilon^2}{72HL^2N^2}$, we obtain that:
 \begin{align*}
   NL\sqrt{2H\epsilon_0}+NHL\epsilon_0=
   \frac{\epsilon}{6}+\frac{\epsilon^2}{72LN}\le
   \epsilon/3.
 \end{align*}
 By putting this bound into the bound on $\tau^i$, we get the result.
 
 When $\cU$ is an arbitrary subset of $\fU_L$, we apply, for all $i\in\dsb{N}$, Lemma
 \ref{lemma: bound estimation error tante U} with probability $\delta/(2N)$ and
 accuracy $\epsilon/(3N)$, and Lemma
 \ref{lemma: bound J star all utilities} with probability $\delta/(2N)$ and
 accuracy $\epsilon/(3N)$, while
 we bound the 1-Wasserstein distance through Lemma \ref{lemma: bound w1 etae}, to
 obtain, through an application of the union bound, that:
 \begin{align*} 
   \mathop{\mathbb{P}}\limits_{\{\cM^i\}_i,\{\pi^{E,i}\}_i}\Big(
     \sup\limits_{U\in\cU}&\Big|\sum\limits_{i\in\dsb{N}}
     \overline{\cC}_{p^i,r^i,\pi^{E,i}}(U)-
     \sum\limits_{i\in\dsb{N}}\widehat{\cC}^i(U)
     \Big|\le\\
     &\qquad NL\sqrt{2H\epsilon_0}+\epsilon/3+NHL\epsilon_0+\epsilon/3
     \Big)\ge 1-\delta,
 \end{align*}
 as long as, for all $i\in\dsb{N}$:
 \begin{align*}
   &\tau^{E,i}\ge \widetilde{\cO}\Big(
     \frac{N^2H^3}{\epsilon^2\epsilon_0}\log\frac{HN}{\delta\epsilon_0}   
     \Big),\\
   &\tau^i\ge\widetilde{\cO}\Big(\frac{N^2SAH^5}{\epsilon^2}
   \Big(S+\log\frac{SAHN}{\delta}\Big)\Big).
 \end{align*}
 Again, by setting $\epsilon_0= \frac{\epsilon^2}{72HL^2N^2}$, we obtain that:
 \begin{align*}
   NL\sqrt{2H\epsilon_0}+NHL\epsilon_0=
   \frac{\epsilon}{6}+\frac{\epsilon^2}{72LN}\le
   \epsilon/3.
 \end{align*}
 By putting this bound into the bounds on $\tau^{E,i}$ and $\tau^i$, we get the
 result.
 \end{proof}
 
 \subsubsection{Lemmas on the Expert's Return Distribution}
 
 \begin{lemma}\label{lemma: bound w1 etae} Let the projection operator
   $\text{Proj}_{\cC}$ be defined as in Eq. \eqref{eq: definition proj c}, over set
   $\cY$ with discretization $\epsilon_0$. Then, for all $i\in\dsb{N}$, it holds
   that:
   \begin{align*}
     w_1(\eta^{p^i,r^i,\pi^{E,i}},\text{Proj}_{\cC}(\eta^{p^i,r^i,\pi^{E,i}}))
     \le\sqrt{2H\epsilon_0}.
   \end{align*}
 \end{lemma}
 \begin{proof}
   For the sake of simplicity, we omit index $i\in\dsb{N}$, but the following
   derivation can be applied to all the $N$ demonstrations.
 
   By applying Lemma 5.2 of \citet{rowland2024nearminimaxoptimal}, replacing term
   $1/(1-\gamma)$ with horizon $H$, we get:
   \begin{align*}
     w_1(\eta^{p,r,\pi^{E}},\text{Proj}_{\cC}(\eta^{p,r,\pi^{E}}))
     &\le \sqrt{H}\ell_2(\eta^{p,r,\pi^{E}},\text{Proj}_{\cC}(\eta^{p,r,\pi^{E}})).
   \end{align*}
 
   Similarly to the proof of Proposition 3 of
   \citet{rowland2018analysis}, we can write:
   \begin{align*}
     \ell_2^2(\eta^{p,r,\pi^{E}},\text{Proj}_{\cC}(\eta^{p,r,\pi^{E}}))
     &\markref{(1)}{\coloneqq}
   \int_\RR (F_{\eta^{p,r,\pi^{E}}}(y)-F_{\text{Proj}_{\cC}(\eta^{p,r,\pi^{E}})}(y))^2dy\\
   &\markref{(2)}{=}
   \popblue{\int_{0}^H} (F_{\eta^{p,r,\pi^{E}}}(y)-
   F_{\text{Proj}_{\cC}(\eta^{p,r,\pi^{E}})}(y))^2dy\\
   &\markref{(3)}{=}
   \popblue{\sum\limits_{j\in\dsb{d-1}}\int_{y_j}^{y_{j+1}}} (F_{\eta^{p,r,\pi^{E}}}(y)-
   F_{\text{Proj}_{\cC}(\eta^{p,r,\pi^{E}})}(y))^2dy\\
   &\qquad
   +\popblue{\int_{y_{d}}^{H}} (F_{\eta^{p,r,\pi^{E}}}(y)-
   F_{\text{Proj}_{\cC}(\eta^{p,r,\pi^{E}})}(y))^2dy\\
   &\markref{(4)}{\le}
   \sum\limits_{j\in\dsb{d-1}}\int_{y_j}^{y_{j+1}} (F_{\eta^{p,r,\pi^{E}}}(y)-
   F_{\text{Proj}_{\cC}(\eta^{p,r,\pi^{E}})}(y))^2dy+\popblue{\epsilon_0}\\
   &\markref{(5)}{\le}
   \sum\limits_{j\in\dsb{d-1}}\int_{y_j}^{y_{j+1}} (F_{\eta^{p,r,\pi^{E}}}(\popblue{y_{j+1}})-
   F_{\popblue{\eta^{p,r,\pi^{E}}}}(\popblue{y_{j}}))^2dy+\epsilon_0\\
   &=
   \sum\limits_{j\in\dsb{d-1}}\popblue{(y_{j+1}-y_j)}
   (F_{\eta^{p,r,\pi^{E}}}(y_{j+1})-
   F_{\eta^{p,r,\pi^{E}}}(y_{j}))^2+\epsilon_0\\
   &\markref{(6)}{=}
   \popblue{\epsilon_0}\sum\limits_{j\in\dsb{d-1}}
   (F_{\eta^{p,r,\pi^{E}}}(y_{j+1})-
   F_{\eta^{p,r,\pi^{E}}}(y_{j}))^2+\epsilon_0\\
   &\markref{(7)}{\le}
   \epsilon_0\popblue{\Big(}\sum\limits_{j\in\dsb{d-1}}
   (F_{\eta^{p,r,\pi^{E}}}(y_{j+1})-
   F_{\eta^{p,r,\pi^{E}}}(y_{j})\popblue{\Big)^2}+\epsilon_0\\
   &\markref{(8)}{=}
   \epsilon_0
   \big(F_{\eta^{p,r,\pi^{E}}}(\popblue{y_{d}})-
   F_{\eta^{p,r,\pi^{E}}}(\popblue{y_1})\big)^2+\epsilon_0\\
   &\le2\epsilon_0,
   \end{align*}
   where at (1) we have applied the definition of $\ell_2$ distance (Eq.
   \eqref{eq: definition cramer distance l2}), at (2) we recognize that the two
   distributions $\eta^{p,r,\pi^{E}},\text{Proj}_{\cC}(\eta^{p,r,\pi^{E}})$ are
   defined on $[0,H]$, at (3) we use the additivity property of the integral,
   using notation
   $\cY\coloneqq\{0,\epsilon_0,2\epsilon_0,\dotsc,\floor{H/\epsilon_0}\epsilon_0\}$,
   $d\coloneqq |\cY|=\floor{H/\epsilon_0}+1$, $y_1\coloneqq 0, y_2\coloneqq
   \epsilon_0, y_3\coloneqq 2\epsilon_0,\dotsc, y_d\coloneqq
   \floor{H/\epsilon_0}\epsilon_0$, (notation introduced in Section \ref{sec:
   online utility learning}). At (4) we upper bound $\int_{y_{d}}^{H}
   (F_{\eta^{p,r,\pi^{E}}}(y)-
   F_{\text{Proj}_{\cC}(\eta^{p,r,\pi^{E}})}(y))^2dy\le \int_{y_{d}}^{H}dy
   =H-y_d=H-\floor{H/\epsilon_0}\epsilon_0=\epsilon_0(H/\epsilon_0-\floor{H/\epsilon_0})\le\epsilon_0
   $ since the difference of cumulative distribution functions is bounded by 1.
   At (5), thanks to the definition of the projection operator
   $\text{Proj}_{\cC}$ (Eq. \eqref{eq: definition proj c}), we notice that, for
   $y\in[y_j,y_{j+1}]$, it holds that
   $F_{\text{Proj}_{\cC}(\eta^{p,r,\pi^{E}})}(y)\in[F_{\eta^{p,r,\pi^{E}}}(y_j),
   F_{\eta^{p,r,\pi^{E}}}(y_{j+1})]$, thus we can upper bound the integrand
   through the maximum, constant, difference of cumulative distribution
   functions. At (6) we use the definition of set $\cY$, i.e., an
   $\epsilon_0$-covering of the $[0,H]$ interval, at (7) we use the
   Cauchy-Schwarz's inequality $\sum_j (x_j)^2\le (\sum_j x_j)^2$ for $x_j\ge 0$,
   and noticed that the summands are always non-negative, at (8) we apply a
   telescoping argument.
 
   The result follows by taking the square root of both sides.
 \end{proof}
 
 \begin{lemma}
   \label{lemma: expectation erd is projected true distribution}
   Let $i\in\dsb{N}$, and let $f\in[0,H]^d$ be an arbitrary $d$-dimensional
   vector. Denote by
   $G_1,G_2,\dotsc,G_{\tau^{E,i}}\overset{\text{i.i.d.}}{\sim}\eta^{p^i,r^i,\pi^{E,i}}$
   the random variables representing the returns of the $\tau^{E,i}$ trajectories
   inside dataset $\cD^{E,i}$. Let $\widehat{\eta}^{E,i}$ be the random output of
   Algorithm \ref{alg: erd} that depends on the random variables
   $G_1,G_2,\dotsc,G_{\tau^{E,i}}$. Then, it holds that:
   \begin{align*}
     \E_{G_1,G_2,\dotsc,G_{\tau^{E,i}}\sim \eta^{p^i,r^i,\pi^{E,i}}}\bigg[
     \E_{y\sim\widehat{\eta}^{E,i}}\Big[f(y)\Big]  
     \bigg]
     = \E_{y\sim\text{Proj}_{\cC}(\eta^{p^i,r^i,\pi^{E,i}})
     }\Big[f(y)\Big].
   \end{align*}
 \end{lemma}
 \begin{proof}
   We omit index $i$ for simplicity, but the proof can be carried out for all
   $i\in\dsb{N}$ independently. To prove the statement, we use the notation
   described in Appendix \ref{apx: missing algorithms and sub routines} for the
   Dirac delta, to provide an explicit representation of both the distribution
   $\text{Proj}_{\cC}(\eta^{p,r,\pi^{E}})$ and the ``random'' distribution
   $\widehat{\eta}^E$.
 
   We consider distribution $\eta^{p,r,\pi^{E}}$ supported on
   $\cZ\coloneqq\{z_1,z_2,\dotsc,z_M\}\subseteq[0,H]$, while distributions
   $\text{Proj}_{\cC}(\eta^{p,r,\pi^{E}}),\widehat{\eta}^E$ are supported on set 
   $\cY=\{y_1,y_2,\dotsc,y_d\}\subseteq[0,H]$.
   
   W.r.t. distribution $\text{Proj}_{\cC}(\eta^{p,r,\pi^{E}})$, we can write:
   \begin{align*}
     \text{Proj}_{\cC}(\eta^{p,r,\pi^{E}})
     &=\text{Proj}_{\cC}\Big(
       \sum\limits_{k\in\dsb{M}}
     \eta^{p,r,\pi^E}(z_k) \delta_{z_k}
     \Big)\\
     &\markref{(1)}{=}
       \sum\limits_{k\in\dsb{M}}
       \eta^{p,r,\pi^E}(z_k)\popblue{\text{Proj}_{\cC}
     (\delta_{z_k})}\\
     &\markref{(2)}{=}
     \sum\limits_{k\in\dsb{M}}
     \eta^{p,r,\pi^E}(z_k)\popblue{\Big(
     \delta_{y_1}\indic{z_k\le y_1}
     +\delta_{y_d}\indic{z_k> y_d}}\\
     &\qquad\popblue{+ \sum\limits_{j\in\dsb{d-1}}\Big(
     \frac{y_{j+1}-z_k}{y_{j+1}-y_j}\delta_{y_j}+
     \frac{z_k-y_j}{y_{j+1}-y_j}\delta_{y_{j+1}}
     \Big)\indic{z_k\in(y_j,y_{j+1}]}
     \Big)}\\
     &=
     \popblue{\delta_{y_1}}\sum\limits_{k\in\dsb{M}}\eta^{p,r,\pi^E}(z_k)\Big(\indic{z_k\le y_1}
     +\frac{y_2-z_k}{y_{2}-y_1}\indic{z_k\in(y_1,y_2]}
     \Big)\\
     &\qquad+
     \sum\limits_{j\in\{2,\dotsc,d-1\}}\popblue{\delta_{y_j}}
     \Big(\sum\limits_{k\in\dsb{M}}
     \eta^{p,r,\pi^E}(z_k)\Big(\frac{y_{j+1}-z_k}{y_{j+1}-y_j}\indic{z_k\in(y_i,y_{j+1}]}\\
     &\qquad+
     \frac{z_k-y_{j-1}}{y_{i}-y_{j-1}}\indic{z_k\in(y_{j-1},y_{i}]}\Big)
     \Big)\\
     &\qquad+\popblue{\delta_{y_d}}\sum\limits_{k\in\dsb{M}}\eta^{p,r,\pi^E}(z_k)\Big(
       \indic{z_k> y_d}+\frac{z_k-y_{d-1}}{y_{d}-y_{d-1}}\indic{z_k\in(y_{d-1},y_d]}\Big),
   \end{align*}
   where at (1) we have applied the extension in Eq. \eqref{eq: property proj c} of
   the projection operator $\text{Proj}_{\cC}$ to finite mixtures of Dirac
   distributions, and at (2) we have applied its definition (Eq. \eqref{eq:
   definition proj c}).
 
   Concerning distribution $\widehat{\eta}^E$, based on Algorithm \ref{alg: erd},
   we can write:
   \begin{align*}
     \widehat{\eta}^E
     &=\frac{\popblue{\delta_{y_1}}}{\tau^E}\Big(
     \sum\limits_{t\in\dsb{\tau^E}}
     \Big(
     \indic{G_t\le y_1}+  
     \frac{y_2-G_t}{y_{2}-y_1}\indic{G_t\in(y_1,y_2]}
     \Big)  
     \Big)\\
     &\qquad+
     \sum\limits_{j\in\{2,\dotsc,d-1\}}\frac{\popblue{\delta_{y_j}}}{\tau^E}
     \Big(\sum\limits_{t\in\dsb{\tau^E}}
     \Big(\frac{y_{j+1}-G_t}{y_{j+1}-y_j}\indic{G_t\in(y_i,y_{j+1}]}\\
     &\qquad+
     \frac{G_t-y_{j-1}}{y_{i}-y_{j-1}}\indic{G_t\in(y_{j-1},y_{i}]}\Big)
     \Big)\\
     &\qquad+\frac{\popblue{\delta_{y_d}}}{\tau^E}\Big(\sum\limits_{t\in\dsb{\tau^E}}
     \Big(
       \indic{G_t> y_d}+\frac{G_t-y_{d-1}}{y_{d}-y_{d-1}}\indic{G_t\in(y_{d-1},y_d]}\Big)\Big).
   \end{align*}
 
 Now, if we take the expectation of the random vector $\widehat{\eta}^E$ w.r.t.
 $\eta^{p,r,\pi^{E}}$, we get:
 \begin{align*}
   &\E_{G_1,G_2,\dotsc,G_{\tau^{E}}\sim \eta^{p,r,\pi^{E}}}\Big[
     \widehat{\eta}^{E}
     \Big]\\
     &\qquad=\E_{G_1,G_2,\dotsc,G_{\tau^{E}}\sim \eta^{p,r,\pi^{E}}}\bigg[
       \frac{\popblue{\delta_{y_1}}}{\tau^E}\Big(
     \sum\limits_{t\in\dsb{\tau^E}}
     \Big(
     \indic{G_t\le y_1}+  
     \frac{y_2-G_t}{y_{2}-y_1}\indic{G_t\in(y_1,y_2]}
     \Big)  
     \Big)\\
     &\qquad\qquad+
     \sum\limits_{j\in\{2,\dotsc,d-1\}}\frac{\popblue{\delta_{y_j}}}{\tau^E}
     \Big(\sum\limits_{t\in\dsb{\tau^E}}
     \Big(\frac{y_{j+1}-G_t}{y_{j+1}-y_j}\indic{G_t\in(y_i,y_{j+1}]}\\
     &\qquad\qquad+
     \frac{G_t-y_{j-1}}{y_{i}-y_{j-1}}\indic{G_t\in(y_{j-1},y_{i}]}\Big)
     \Big)\\
     &\qquad\qquad+\frac{\popblue{\delta_{y_d}}}{\tau^E}\Big(\sum\limits_{t\in\dsb{\tau^E}}
     \Big(
       \indic{G_t> y_d}+\frac{G_t-y_{d-1}}{y_{d}-y_{d-1}}\indic{G_t\in(y_{d-1},y_d]}\Big)\Big)
     \bigg]\\
       &\qquad\markref{(3)}{=}
       \E_{G\sim \eta^{p,r,\pi^{E}}}\bigg[
       \popblue{\delta_{y_1}}\Big(
     \indic{G\le y_1}+  
     \frac{y_2-G}{y_{2}-y_1}\indic{G\in(y_1,y_2]}
     \Big)\\
     &\qquad\qquad+
     \sum\limits_{j\in\{2,\dotsc,d-1\}}\popblue{\delta_{y_j}}
     \Big(\frac{y_{j+1}-G}{y_{j+1}-y_j}\indic{G\in(y_i,y_{j+1}]}\\
     &\qquad\qquad+
     \frac{G-y_{j-1}}{y_{i}-y_{j-1}}\indic{G\in(y_{j-1},y_{i}]}\Big)\\
     &\qquad\qquad+\popblue{\delta_{y_d}}\Big(
       \indic{G> y_d}+\frac{G-y_{d-1}}{y_{d}-y_{d-1}}\indic{G\in(y_{d-1},y_d]}
       \Big)
     \bigg]\\
     &\qquad\markref{(4)}{=}
     \popblue{\delta_{y_1}}\sum\limits_{k\in\dsb{M}}\eta^{p,r,\pi^E}(z_k)\Big(\indic{z_k\le y_1}
     +\frac{y_2-z_k}{y_{2}-y_1}\indic{z_k\in(y_1,y_2]}
     \Big)\\
     &\qquad\qquad+
     \sum\limits_{j\in\{2,\dotsc,d-1\}}\popblue{\delta_{y_j}}
     \Big(\sum\limits_{k\in\dsb{M}}
     \eta^{p,r,\pi^E}(z_k)\Big(\frac{y_{j+1}-z_k}{y_{j+1}-y_j}\indic{z_k\in(y_i,y_{j+1}]}\\
     &\qquad\qquad+
     \frac{z_k-y_{j-1}}{y_{i}-y_{j-1}}\indic{z_k\in(y_{j-1},y_{i}]}\Big)
     \Big)\\
     &\qquad\qquad+\popblue{\delta_{y_d}}\sum\limits_{k\in\dsb{M}}\eta^{p,r,\pi^E}(z_k)\Big(
       \indic{z_k> y_d}+\frac{z_k-y_{d-1}}{y_{d}-y_{d-1}}\indic{z_k\in(y_{d-1},y_d]}\Big)\\
     &\qquad\markref{(5)}{=}
     \text{Proj}_{\cC}(\eta^{p,r,\pi^{E}}),
 \end{align*}
 where at (3) we use the fact that $G_1,G_2,\dotsc,G_{\tau^{E}}$ are independent
 and identically distributed, at (4) we apply the linearity of the expectation,
 we notice that $\delta_{y_j}$ does not depend on $G$ for all $j\in\dsb{d}$, and
 we notice that, for any $y\in\cY$, it holds that $\E_{G\sim
 \eta^{p,r,\pi^{E}}}\big[\indic{G\le y}\big]=\eta^{p,r,\pi^E}(G\le
 y)=\sum_{k\in\dsb{M}} \eta^{p,r,\pi^E}(z_k) \indic{z_k\le y}$, where we have
 abused notation by writing $\eta^{p,r,\pi^E}(G\le y)$ to mean the probability,
 under distribution $\eta^{p,r,\pi^E}$, that event $\{G\le y\}$ happens.
 Moreover, similarly, we notice that, for any $y,y'\in\cY$, it holds that $\E_{G\sim
 \eta^{p,r,\pi^{E}}}\big[G\cdot\indic{G\in[y,y']}\big]= \sum_{k\in\dsb{M}} z_k
 \eta^{p,r,\pi^E}(z_k) \indic{z_k\in[y,y']}$. At (5) we simply recognize
 $\text{Proj}_{\cC}(\eta^{p,r,\pi^{E}})$ using the previous expression.
 
 This concludes the proof because the equality of the Dirac delta representations
 means that the expectations of any function w.r.t. these two distributions
 coincide.
 \end{proof}
 
 \begin{lemma}
   \label{lemma: bound estimation error 1 U}
   Let $i\in\dsb{N}$ and let $\epsilon,\delta\in(0,1)$. If $|\cU|=1$, then, with
   probability at least $1-\delta$, we have:
   \begin{align*}
     \sup\limits_{U\in\cU}
       \Big|\mathop{\E}\limits_{G\sim\text{Proj}_{\cC}(\eta^{p^i,r^i,\pi^{E,i}})
       }[U(G)]-
       \mathop{\E}\limits_{G\sim\widehat{\eta}^{E,i}}[U(G)]\Big|\le \epsilon,
   \end{align*}
   as long as:
   \begin{align*}
     \tau^E\ge c \frac{H^2\log\frac{2}{\delta}}{\epsilon^2},
   \end{align*}
   where $c$ is some positive constant.
 \end{lemma}
 \begin{proof}
   Let $U$ be the only function inside $\cU$. Let us omit index $i$ for simplicity. Then, 
   we can write:
   \begin{align*}
     \Big|\mathop{\E}\limits_{G\sim\widehat{\eta}^{E}}[U(G)]-
     \mathop{\E}\limits_{G\sim\text{Proj}_{\cC}(\eta^{p,r,\pi^{E}})
       }[U(G)]\Big|
       &\markref{(1)}{=}
       \Big|\mathop{\E}\limits_{G\sim\widehat{\eta}^{E}}[U(G)]-
      \popblue{\mathop{\E}\limits_{\eta^{p,r,\pi^{E}}}\Big[}
      \mathop{\E}\limits_{G\sim\widehat{\eta}^{E}}[U(G)]\popblue{\Big]}\Big|\\
       &\markref{(2)}{\le} c H\sqrt{\frac{\log\frac{2}{\delta}}{\tau^E}},
   \end{align*}
   where at (1) we have applied Lemma \ref{lemma: expectation erd is projected
   true distribution}, and at (2) we have applied the Hoeffding's inequality
   noticing that function $U$ is bounded in $[0,H]$, and denoting with $c$ some positive
   constant.
 
   By imposing:
   \begin{align*}
     c H\sqrt{\frac{\log\frac{2}{\delta}}{\tau^E}}\le\epsilon,
   \end{align*}
   and solving w.r.t. $\tau^E$, we get the result.
 \end{proof}
 
 \begin{lemma}\label{lemma: bound estimation error tante U}
   Let $i\in\dsb{N}$ and let $\epsilon,\delta\in(0,1)$. Then, with probability at least
   $1-\delta$, we have:
   \begin{align*}
     \sup\limits_{U\in\cU}
       \Big|\mathop{\E}\limits_{G\sim\text{Proj}_{\cC}(\eta^{p^i,r^i,\pi^{E,i}})
       }[U(G)]-
       \mathop{\E}\limits_{G\sim\widehat{\eta}^{E,i}}[U(G)]\Big|\le\epsilon,
   \end{align*}
   as long as:
   \begin{align*}
     \tau^E\ge \widetilde{\cO}\Big(
     \frac{H^3}{\epsilon^2\epsilon_0}\log\frac{H}{\delta\epsilon_0}   
     \Big).
   \end{align*}
 \end{lemma}
 \begin{proof}
   Again, let us omit index $i$ for simplicity. First, for all possible functions
   $U\in\cU$, we denote by $\overline{U}\in\overline{\fU}_L$ the function in
   $\overline{\fU}_L$ that takes on the values that the function $U$ assigns to
   the points of set $\cY$. This permits us to write:
   \begin{align*}
     &\sup\limits_{U\in\cU}\Big|\mathop{\E}\limits_{G\sim\widehat{\eta}^{E}}[U(G)]-
     \mathop{\E}\limits_{G\sim\text{Proj}_{\cC}(\eta^{p,r,\pi^{E}})
       }[U(G)]\Big|\\
       &\qquad=
       \sup\limits_{\overline{U}\in\overline{\fU}_L}
       \Big|\mathop{\E}\limits_{G\sim\widehat{\eta}^{E}}[\overline{U}(G)]-
     \mathop{\E}\limits_{G\sim\text{Proj}_{\cC}(\eta^{p,r,\pi^{E}})
       }[\overline{U}(G)]\Big|\\
       &\qquad\markref{(1)}{\le}
       \sup\limits_{\overline{U}\in[0,H]^d}
       \Big|\mathop{\E}\limits_{G\sim\widehat{\eta}^{E}}[\overline{U}(G)]-
     \mathop{\E}\limits_{G\sim\text{Proj}_{\cC}(\eta^{p,r,\pi^{E}})
       }[\overline{U}(G)]\Big|\\
       &\qquad\markref{(2)}{=}
       \sup\limits_{\overline{U}\in[0,H]^d}
       \Big|\mathop{\E}\limits_{G\sim\widehat{\eta}^{E}}[\overline{U}(G)]-
      \popblue{\mathop{\E}\limits_{\eta^{p,r,\pi^{E}}}\Big[}
      \mathop{\E}\limits_{G\sim\widehat{\eta}^{E}}[\overline{U}(G)]\popblue{\Big]}\Big|,
   \end{align*} 
   where at (1) we upper bound by considering all the possible vectors
   $\overline{U}\in[0,H]^d$, and at (2) we apply Lemma \ref{lemma: expectation
   erd is projected true distribution}.
 
   Now, similarly to the proof of Lemma 7.2 in \citet{agarwal2021RL}, we
   construct an $\epsilon'$-covering of set $[0,H]^d$, call it $\cN_{\epsilon'}$,
   with $|\cN_{\epsilon'}|\le (1+2H\sqrt{d}/\epsilon')^d$ such that, for all
   $f\in[0,H]^d$, there exists $f'\in\cN_{\epsilon'}$ for which
   $\|f-f'\|_2\le\epsilon'$. By applying a union bound over all
   $f'\in\cN_{\epsilon'}$ and Lemma \ref{lemma: bound estimation error 1 U}, we
   have that, with probability at least $1-\delta$, for all
   $f'\in\cN_{\epsilon'}$, it holds that:
   \begin{align}\label{eq: bound norm 1 proof Hoeffding}
     \Big|\mathop{\E}\limits_{G\sim\widehat{\eta}^{E}}[f'(G)]-
      \mathop{\E}\limits_{\eta^{p,r,\pi^{E}}}\Big[
      \mathop{\E}\limits_{G\sim\widehat{\eta}^{E}}[f'(G)]\Big]\Big|
      \le cH\sqrt{\frac{d\log\frac{2(1+2H\sqrt{d}/\epsilon')}{\delta}}{\tau^E}}.
   \end{align}
   Next, for any $f\in[0,H]^d$, denote its closest points (in 2-norm) from
   $\cN_{\epsilon'}$ as $f'$. Then, we have:
   \begin{align*}
     &\Big|\mathop{\E}\limits_{G\sim\widehat{\eta}^{E}}[f(G)]-
      \mathop{\E}\limits_{\eta^{p,r,\pi^{E}}}\Big[
      \mathop{\E}\limits_{G\sim\widehat{\eta}^{E}}[f(G)]\Big]
      \Big|\\
      &\qquad=\Big|\mathop{\E}\limits_{G\sim\widehat{\eta}^{E}}[f(G)]-
      \mathop{\E}\limits_{\eta^{p,r,\pi^{E}}}\Big[
      \mathop{\E}\limits_{G\sim\widehat{\eta}^{E}}[f(G)]\Big]
      \popblue{\pm \Big(\mathop{\E}\limits_{G\sim\widehat{\eta}^{E}}[f'(G)]-
      \mathop{\E}\limits_{\eta^{p,r,\pi^{E}}}\Big[
      \mathop{\E}\limits_{G\sim\widehat{\eta}^{E}}[f'(G)]\Big]\Big)}
      \Big|\\
      &\qquad\markref{(3)}{\le}\Big|\mathop{\E}\limits_{G\sim\widehat{\eta}^{E}}[\popblue{f'}(G)]-
      \mathop{\E}\limits_{\eta^{p,r,\pi^{E}}}\Big[
      \mathop{\E}\limits_{G\sim\widehat{\eta}^{E}}[\popblue{f'}(G)]\Big]\Big|\\
      &\qquad\qquad+\Big|\mathop{\E}\limits_{G\sim\widehat{\eta}^{E}}[\popblue{f}(G)-\popblue{f'}(G)]\Big|
      +\Big|\mathop{\E}\limits_{\eta^{p,r,\pi^{E}}}\Big[
       \mathop{\E}\limits_{G\sim\widehat{\eta}^{E}}[\popblue{f}(G)-\popblue{f'}(G)]\Big]\Big|\\
     &\qquad\markref{(4)}{\le}
     cH\sqrt{\frac{d\log\frac{2(1+2H\sqrt{d}/\epsilon')}{\delta}}{\tau^E}}+2\epsilon'\\
     &\qquad\markref{(5)}{\le}
     \popblue{c'}H\sqrt{\frac{d\log\frac{Hd\tau^E}{\delta}}{\tau^E}}\\
     \end{align*}
     where at (3) we apply triangle inequality, at (4) we apply the result in Eq.
     \eqref{eq: bound norm 1 proof Hoeffding}, and the fact that, by definition of
     $\epsilon'$-covering, $\|f-f'\|_2\le \epsilon'$ entails that
     $|f(y)-f(y')|\le\epsilon'$ for all $y\in\cY$; at (5) we set
     $\epsilon'=1/\tau^E$, and we simplify.
 
     The result follows by upper bounding $d\le H/\epsilon_0+1$, and then by
     setting:
     \begin{align}\label{eq: bound difference etaE for tractor}
       c''H\sqrt{\frac{H\log\frac{H\tau^E}{\delta\epsilon_0}}{\epsilon_0\tau^E}}\le\epsilon,
     \end{align}
     and solving w.r.t. $\tau^E$, and noticing that for all $\tau^E$ greater than
     some constant, we can get rid of the logarithmic terms in $\tau^E$.
 \end{proof}
 
 \subsubsection{Lemmas on the Optimal Performance for Single Utility}
 \label{apx: lemmas optimal performance single utility}
 
 In this section, we will omit index $i\in\dsb{N}$ since the following derivations
 can be carried out for each $i$.
 
 We denote the arbitrary MDP in $\{\cM^i\}_i$ as $\cM=\tuple{\cS,\cA,H,s_0,p,r}$,
 and its analogous with discretized reward $\overline{r}$, defined at all
 $(s,a,h)\in\SAH$ as $\overline{r}_h(s,a)\coloneqq\Pi_\cR[r_h(s,a)]$, as
 $\overline{\cM}\coloneqq \tuple{\cS,\cA,H,s_0,p,\overline{r}}$. We denote the
 analogous MDPs with empirical transition model $\widehat{p}$ as $\widehat{\cM}
 =\tuple{\cS,\cA,H,s_0,\widehat{p},r}$ and $\widehat{\overline{\cM}}\coloneqq
 \tuple{\cS,\cA,H,s_0,\widehat{p},\overline{r}}$.
 
 Given any utility $U\in\fU_L$, we denote the corresponding RS-MDPs,
 respectively, as
 $\cM_U,\overline{\cM}_U,\widehat{\cM}_U,\widehat{\overline{\cM}}_U$. Concerning
 the discretized RS-MDPs $\overline{\cM}_U$ and $\widehat{\overline{\cM}}_U$, we
 denote the corresponding enlarged state space MDPs, respectively, as
 $\fE[\overline{\cM}_U]=\tuple{\{\cS\times\cY_h\}_h,\cA,H,(s_0,0),\fp,\fr}$ and
 $\fE[\widehat{\overline{\cM}}_U]=\tuple{\{\cS\times\cY_h\}_h
 ,\cA,H,(s_0,0),\widehat{\fp},\fr}$, where we decided to define such enlarged
 state space MDPs using the state space $\{\cS\times\cY_h\}_h$ considered by
 Algorithm \ref{alg: planning} (\texttt{PLANNING}) instead of, respectively,
 $\{\cS\times\cG^{p,\overline{r}}_h\}_h$ and
 $\{\cS\times\cG^{\widehat{p},\overline{r}}_h\}_h$. Thus, the transition models
 $\fp$ and $\widehat{\fp}$, from any $h\in\dsb{H}$ and $(s,y,a)\in\SYAh$, assign
 to the next state $(s',y')\in\SYh$ the probability: $\fp_h(s',y'|s,y,a)\coloneqq
 p_h(s'|s,a)\indic{y'=y+\overline{r}_h(s,a)}$ and
 $\widehat{\fp}_h(s',y'|s,y,a)\coloneqq
 \widehat{p}_h(s'|s,a)\indic{y'=y+\overline{r}_h(s,a)}$.
 Moreover, the reward function $\fr$, in any $h\in\dsb{H}$ and $(s,y,a)\in\SYAh$,
 is $\fr_h(s,y,a)=0$ if $h<H$, and $\fr_h(s,y,a)=U(y+\overline{r}_h(s,a))$ if
 $h=H$.
 
 We will make extensive use of notation for $V$- and $Q$- functions introduced in
 Appendix \ref{apx: additional notation}.
 
 We are now ready to proceed with the analysis. In general, the analysis shares
 similarities to that of Theorem 3 of \citet{wu2023risksensitive}, but we use
 results also from \citet{azar2013minimax} to obtain tighter bounds.
 
 \begin{lemma}\label{lemma: bound J star single utility} Let
   $\epsilon,\delta\in(0,1)$. For any fixed $L$-Lipschitz utility function
   $U\in\fU_L$, it suffices to execute \caty with:
   \begin{align*}
     &\tau\le\widetilde{\cO}\Big(\frac{SAH^4}{\epsilon^2}
     \log\frac{SAH}{\delta\epsilon_0}\Big),
   \end{align*}
   to obtain $\big|J^*(U;p,r)-\widehat{J}^*(U)\big|\le HL\epsilon_0+\epsilon$
   w.p. $1-\delta$.
 \end{lemma}
 \begin{proof}
   For an arbitrary utility $U\in\fU_L$, we can write:
   \begin{align*}
     |J^*(U;p,r)-\widehat{J}^*(U)|
     &\markref{(1)}{=}|J^*(U;p,r)-\widehat{J}^*(U)\popblue{\pm
     J^*(\fp,\fr)}|\\
     &\markref{(2)}{\le} \popblue{|}J^*(U;p,r)-J^*(\fp,\fr)\popblue{|}
     +\popblue{|}J^*(\fp,\fr)-\widehat{J}^*(U)\popblue{|}\\
     &\markref{(3)}{=} |J^*(U;p,r)-J^*(\fp,\fr)|+|J^*(\fp,\fr)-\popblue{J^*(\widehat{\fp},\fr)}|\\
     &\markref{(4)}{\le} \popblue{HL\epsilon_0}+|J^*(\fp,\fr)-J^*(\widehat{\fp},\fr)|\\
     &=HL\epsilon_0+\popblue{|V_1^*(s_0,0;\fp,\fr)-V_1^*(s_0,0;\widehat{\fp},\fr)|}\\
     &\le HL\epsilon_0+\popblue{\max\limits_{\hsya}
     |Q^*_h(s,y,a;\fp,\fr)-Q^*_h(s,y,a;\widehat{\fp},\fr)|}\\
     &\markref{(5)}{\le}HL\epsilon_0+\popblue{\epsilon'},
   \end{align*}
   where at (1) we add and subtract the optimal expected utility in the enlarged
   MDP $\fE[\overline{\cM}_U]$ considered by Algorithm \ref{alg: planning}, but
   with the true transition model $\fp$. At (2) we apply triangle inequality, at
   (3) we recognize that the estimate $\widehat{J}^*(U)$ used in \caty and
   outputted by \texttt{PLANNING} (Algorithm \ref{alg: planning}) is the optimal
   expected utility for the discretized problem with estimated dynamics
   $\widehat{\fp}$, at (4) we use Proposition 3 of \citet{wu2023risksensitive},
   since $U$ is $L$-Lipschitz, and at (5) we apply Lemma \ref{lemma: lemma 9} to
   bound the distance between $Q$-functions.
   
   By setting:
   \begin{align*}
     \underbrace{c\sqrt{\frac{H^3\log\frac{4SAHd}{\delta}}{n}}}_{\le\epsilon/3}
     +\underbrace{cH^2\bigg(\frac{\log\frac{16SAHd}{\delta}}{n}\bigg)^{3/4}}_{\le\epsilon/3}
     +\underbrace{cH^3\frac{\log\frac{16SAHd}{\delta}}{n}}_{\le\epsilon/3}\le\epsilon,
   \end{align*}
   and solving w.r.t. $\epsilon$:
   \begin{align*}
     \begin{cases}
       n\ge c'\frac{H^3\log\frac{4SAHd}{\delta}}{\epsilon^2}\\
       n\ge c''\frac{H^{8/3}\log\frac{16SAHd}{\delta}}{\epsilon^{4/3}}\\
       n\ge c'''\frac{H^3\log\frac{16SAHd}{\delta}}{\epsilon}
     \end{cases}.
   \end{align*}
   Taking the largest bound, we get:
   \begin{align*}
     &n\ge c \frac{H^3\log\frac{16SAHd}{\delta}}{\epsilon^2},
   \end{align*}
   for some positive constant $c$. Since $d\le H/\epsilon_0+1$, we can write:
   \begin{align*}
     &\tau\ge c' \frac{SAH^4\log\frac{c''SAH}{\delta\epsilon_0}}{\epsilon^2},
   \end{align*}
   for some positive constants $c',c''$, where we used that $\tau=SAHn$.
 \end{proof}
 
 The proof of the following lemma is organized in many lemmas, and is based on
 the proof of Theorem 1 of \citet{azar2013minimax}.
 \begin{lemma}\label{lemma: lemma 9}
   For any $\delta\in(0,1)$, we have:
   \begin{align*}
     \max\limits_{\hsya}
     |Q^*_h(s,y,a;\fp,\fr)-Q^*_h(s,y,a;\widehat{\fp},\fr)|\le\epsilon',
   \end{align*}
   w.p. at least $1-\delta$, where $\epsilon'$ is defined as:
   \begin{align*}
     \epsilon'\coloneqq c\sqrt{\frac{H^3\log\frac{4SAHd}{\delta}}{n}}
     +cH^2\bigg(\frac{\log\frac{16SAHd}{\delta}}{n}\bigg)^{3/4}
     +cH^3\frac{\log\frac{16SAHd}{\delta}}{n},
   \end{align*}
   for some positive constant $c$.
 \end{lemma}
 \begin{proof}
   We upper bound one side, and then the other. For all the $\hsya$,
   it holds that:
   \begin{align*}
     &Q^*_h(s,a,y;\fp,\fr)-Q^*_h(s,y,a;\widehat{\fp},\fr)\\
     &\qquad\markref{(1)}{\le}\mathop{\E}\limits_{\widehat{\fp},\fr,\psi^*}\bigg[
       \sum\limits_{h'=h}^H \sum\limits_{s'\in\cS}\Big(
         p_{h'}(s'|s_{h'},a_{h'})-\widehat{p}_{h'}(s'|s_{h'},a_{h'})\Big)
         V_{h'+1}^{\psi^*}(s',y_{h'+1};\fp,\fr)\\
         &\qquad\qquad\Big|\,s_h=s,y_h=y,a_h=a
         \bigg]\\
     &\qquad\markref{(2)}{\le}\mathop{\E}\limits_{\widehat{\fp},\fr,\psi^*}\bigg[
   \sum\limits_{h'=h}^H 
   c\sqrt{\frac{c_1 
     \V_{s'\sim \widehat{p}_{h'}(\cdot|s_{h'},a_{h'})}
     [V^{\psi^*}_{h'+1}(s',y_{h'+1};\widehat{\fp},\fr)]}{n}}
     +b_2\\
     &\qquad\qquad\Big|\,s_h=s,y_h=y,a_h=a      
     \bigg]\\
     &\qquad=\popblue{c\sqrt{\frac{c_1}{n}}}\mathop{\E}\limits_{\widehat{\fp},\fr,\psi^*}\bigg[
       \sum\limits_{h'=h}^H 
       \sqrt{
         \V_{s'\sim \widehat{p}_{h'}(\cdot|s_{h'},a_{h'})}[V^{\psi^*}_{h'+1}(s',y_{h'+1};\widehat{\fp},\fr)]}
       \,\Big|\,s_h=s,y_h=y,a_h=a      
         \bigg]\\
         &\qquad\qquad+\popblue{Hb_2}\\
     &\qquad\markref{(3)}{\le}
     c\sqrt{\frac{c_1}{n}}\popblue{\sqrt{H^3}}+Hb_2\\
     &\qquad=c\sqrt{\frac{H^3\log\frac{\popblue{4}SAHY}{\delta}}{n}}
     +c'H^2\bigg(\frac{\log\frac{\popblue{16}SAHY}{\delta}}{n}\bigg)^{3/4}
     +c''H^3\frac{\log\frac{\popblue{16}SAHY}{\delta}}{n}\\
     &\qquad\eqqcolon \epsilon',
   \end{align*}
   where at (1) we have applied Lemma \ref{lemma: lemma 3}, at (2) we have
   applied Lemma \ref{lemma: lemma 6} with $\delta/2$ of probability, at (3) we have applied Lemma \ref{lemma:
   lemma 8}.
 
   The proof for the other side of inequality is completely analogous, and it
   holds w.p. $1-\delta/2$. The result follows through the application of a union bound.
 \end{proof}
 
 \begin{lemma}\label{lemma: lemma 3}
   For any tuple $\hsya$, it holds that:
   \begin{align*}
     &Q^*_h(s,y,a;\fp,\fr)-Q^*_h(s,y,a;\widehat{\fp},\fr)\le 
     \mathop{\E}\limits_{\widehat{\fp},\fr,\psi^*}\bigg[
       \sum\limits_{h'=h}^H \sum\limits_{s'\in\cS}\\
       &\qquad\Big(
         p_{h'}(s'|s_{h'},a_{h'})-\widehat{p}_{h'}(s'|s_{h'},a_{h'})\Big)
         V_{h'+1}^{\psi^*}(s',y_{h'+1};\fp,\fr)\,\Big|\,s_h=s,y_h=y,a_h=a
         \bigg],\\
         &Q^*_h(s,y,a;\fp,\fr)-Q^*_h(s,y,a;\widehat{\fp},\fr)\ge 
     \mathop{\E}\limits_{\widehat{\fp},\fr,\widehat{\psi}^*}\bigg[
       \sum\limits_{h'=h}^H \sum\limits_{s'\in\cS}\\
       &\qquad\Big(
         p_{h'}(s'|s_{h'},a_{h'})-\widehat{p}_{h'}(s'|s_{h'},a_{h'})\Big)
         V_{h'+1}^{\psi^*}(s',y_{h'+1};\fp,\fr)\,\Big|\,s_h=s,y_h=y,a_h=a      
         \bigg],
   \end{align*}
   where $\psi^*,\widehat{\psi}^*$ are the optimal policies respectively in
   problems $\fp,\fr$ and $\widehat{\fp},\fr$.
 \end{lemma}
 \begin{proof}
   For any $\hsya$, we can write:
   \begin{align*}
     &Q^*_h(s,y,a;\fp,\fr)-Q^*_h(s,y,a;\widehat{\fp},\fr)\\
     &\qquad= Q^{\popblue{\psi^*}}_h(s,y,a;\fp,\fr)-
     Q^{\popblue{\widehat{\psi}^*}}_h(s,y,a;\widehat{\fp},\fr)\\
     &\qquad\markref{(1)}{\le} Q^{\psi^*}_h(s,y,a;\fp,\fr)-
     Q^{\popblue{\psi^*}}_h(s,y,a;\widehat{\fp},\fr)\\
     &\qquad\markref{(2)}{=}\fr_h(s,y,a)+\sum\limits_{(s',y')\in\cS\times\cY_{h+1}}
       \fp_{h}(s',y'|s,y,a)V_{h+1}^{\psi^*}(s',y';\fp,\fr)\\
       &\qquad\qquad-\Big(
         \fr_h(s,y,a)+\sum\limits_{(s',y')\in\cS\times\cY_{h+1}}
         \widehat{\fp}_{h}(s',y'|s,y,a)V_{h+1}^{\psi^*}(s',y';\widehat{\fp},\fr)\Big)\\
     &\qquad\markref{(3)}{=}\sum\limits_{(s',y')\in\cS\times\cY_{h+1}}
     \fp_{h}(s',y'|s,y,a)V_{h+1}^{\psi^*}(s',y';\fp,\fr)\\
     &\qquad\qquad-\sum\limits_{(s',y')\in\cS\times\cY_{h+1}}
     \widehat{\fp}_{h}(s',y'|s,y,a)V_{h+1}^{\psi^*}(s',y';\widehat{\fp},\fr)
       \\&\qquad\qquad\popblue{\pm \sum\limits_{(s',y')\in\cS\times\cY_{h+1}}
       \widehat{\fp}_{h}(s',y'|s,y,a)V_{h+1}^{\psi^*}(s',y';\fp,\fr)}\\
     &\qquad=\sum\limits_{(s',y')\in\cS\times\cY_{h+1}}\Big(
       \fp_{h}(s',y'|s,y,a)-\widehat{\fp}_{h}(s',y'|s,y,a)\Big)V_{h+1}^{\psi^*}(s',y';\fp,\fr)\\
     &\qquad\qquad+\sum\limits_{(s',y')\in\cS\times\cY_{h+1}}
     \widehat{\fp}_{h}(s',y'|s,y,a)\Big(V_{h+1}^{\psi^*}(s',y';\fp,\fr)-
       V_{h+1}^{\psi^*}(s',y';\widehat{\fp},\fr)
       \Big)\\
     &\qquad\markref{(4)}{=}\sum\limits_{(s',y')\in\cS\times\cY_{h+1}}\Big(
       \popblue{p_{h}(s'|s,a)\indic{y+\overline{r}_h(s,a)=y'}}\\
       &\qquad\qquad-
       \popblue{\widehat{p}_{h}(s'|s,a) \indic{y+\overline{r}_h(s,a)=y'}}
       \Big)V_{h+1}^{\psi^*}(s',y';\fp,\fr)\\
     &\qquad\qquad+\sum\limits_{(s',y')\in\cS\times\cY_{h+1}}
     \widehat{\fp}_{h}(s',y'|s,y,a)\Big(V_{h+1}^{\psi^*}(s',y';\fp,\fr)-
       V_{h+1}^{\psi^*}(s',y';\widehat{\fp},\fr)
       \Big)\\
       &\qquad\markref{(5)}{=}\popblue{\sum\limits_{s'\in\cS}}\Big(
         p_{h}(s'|s,a)-\widehat{p}_{h}(s'|s,a)\Big)
         \popblue{\sum\limits_{y'\in\cY_{h+1}}}\indic{y+\overline{r}_h(s,a)=y'}V_{h+1}^{\psi^*}(s',y';\fp,\fr)\\
     &\qquad\qquad+\sum\limits_{(s',y')\in\cS\times\cY_{h+1}}
     \widehat{\fp}_{h}(s',y'|s,y,a)\Big(V_{h+1}^{\psi^*}(s',y';\fp,\fr)-
       V_{h+1}^{\psi^*}(s';\widehat{\fp},\fr)
       \Big)\\
       &\qquad\markref{(6)}{=}\sum\limits_{s'\in\cS}\Big(
         p_{h}(s'|s,a)-\widehat{p}_{h}(s'|s,a)\Big)\popblue{V_{h+1}^{\psi^*}(s',y+\overline{r}_h(s,a);\fp,\fr)}\\
     &\qquad\qquad+\sum\limits_{(s',y')\in\cS\times\cY_{h+1}}
     \widehat{\fp}_{h}(s',y'|s,y,a)\Big(V_{h+1}^{\psi^*}(s',y';\fp,\fr)-
       V_{h+1}^{\psi^*}(s',y';\widehat{\fp},\fr)
       \Big)\\
     &\qquad=
     \sum\limits_{s'\in\cS}\Big(
       p_{h}(s'|s,a)-\widehat{p}_{h}(s'|s,a)\Big)V_{h+1}^{\psi^*}(s',y+\overline{r}_h(s,a);\fp,\fr)\\
       &\qquad\qquad+\sum\limits_{(s',y')\in\cS\times\cY_{h+1}}
         \widehat{\fp}_{h}(s',y'|s,y,a)\\
       &\qquad\qquad\cdot\Big(\popblue{Q}_{h+1}^{\psi^*}(s',y',\popblue{\psi^*_{h+1}(s',y')};\fp,\fr)-
         \popblue{Q}_{h+1}^{\psi^*}(s',y',\popblue{\psi^*_{h+1}(s',y')};\widehat{\fp},\fr)
         \Big),
   \end{align*}
   where at (1) we have used that $\widehat{\psi}^*$ is the optimal policy in
   $\widehat{\fp},\fr$, and thus $Q^{\psi^*}_h(s,a;\widehat{\fp},\fr)\le
   Q^{\widehat{\psi}^*}_h(s,a;\widehat{\fp},\fr)$.
   At (2) we apply the Bellman equation, at (3) we add and subtract the expected
   under $\widehat{\fp}$ optimal value function under $\fp$, at (4) we use the
   definition of transition model $\fp,\widehat{\fp}$, at (5) we split the
   summations, at (6) we recognize that the indicator function takes on value $1$
   only when $y+\overline{r}_h(s,a)=y'$. Finally, we unfold the recursion to obtain the result.
 
   Concerning the second equation, for any $\hsya$, we can write:
   \begin{align*}
     &Q^*_h(s,y,a;\fp,\fr)-Q^*_h(s,y,a;\widehat{\fp},\fr)\\
     &\qquad= Q^{\popblue{\psi^*}}_h(s,y,a;\fp,\fr)-
     Q^{\popblue{\widehat{\psi}^*}}_h(s,y,a;\widehat{\fp},\fr)\\
     &\qquad\markref{(7)}{=}\fr_h(s,y,a)+\sum\limits_{(s',y')\in\cS\times\cY_{h+1}}
     \fp_{h}(s',y'|s,y,a)V_{h+1}^{\psi^*}(s',y';\fp,\fr)\\
     &\qquad\qquad-\Big(
       \fr_h(s,y,a)+\sum\limits_{(s',y')\in\cS\times\cY_{h+1}}
       \widehat{\fp}_{h}(s',y'|s,y,a)V_{h+1}^{\widehat{\psi}^*}(s',y';\widehat{\fp},\fr)\Big)\\
     &\qquad\markref{(8)}{=}\sum\limits_{(s',y')\in\cS\times\cY_{h+1}}
     \fp_{h}(s',y'|s,y,a)V_{h+1}^{\psi^*}(s',y';\fp,\fr)\\
     &\qquad\qquad-\sum\limits_{(s',y')\in\cS\times\cY_{h+1}}
       \widehat{\fp}_{h}(s',y'|s,y,a)V_{h+1}^{\widehat{\psi}^*}(s',y';\widehat{\fp},\fr)\\
       &\qquad\qquad\popblue{\pm \sum\limits_{(s',y')\in\cS\times\cY_{h+1}}
       \widehat{\fp}_{h}(s',y'|s,y,a)V_{h+1}^{\psi^*}(s',y';\fp,\fr)}\\
     &\qquad=\sum\limits_{(s',y')\in\cS\times\cY_{h+1}}\Big(
       \fp_{h}(s',y'|s,y,a)-\widehat{\fp}_{h}(s',y'|s,y,a)\Big)V_{h+1}^{\psi^*}(s',y';\fp,\fr)\\
       &\qquad\qquad+\sum\limits_{(s',y')\in\cS\times\cY_{h+1}}
         \widehat{\fp}_{h}(s',y'|s,y,a)\Big(V_{h+1}^{\psi^*}(s',y';\fp,\fr)-
         V_{h+1}^{\widehat{\psi}^*}(s',y';\widehat{\fp},\fr)
         \Big)\\
     &\qquad=\sum\limits_{(s',y')\in\cS\times\cY_{h+1}}\Big(
       \popblue{p_{h}(s'|s,a)\indic{y+\overline{r}_h(s,a)=y'}}\\
       &\qquad\qquad-
       \popblue{\widehat{p}_{h}(s'|s,a)\indic{y+\overline{r}_h(s,a)=y'}}\Big)V_{h+1}^{\psi^*}(s',y';\fp,\fr)\\
     &\qquad\qquad+\sum\limits_{(s',y')\in\cS\times\cY_{h+1}}
     \widehat{\fp}_{h}(s',y'|s,y,a)\Big(V_{h+1}^{\psi^*}(s',y';\fp,\fr)-
       V_{h+1}^{\widehat{\psi}^*}(s',y';\widehat{\fp},\fr)
       \Big)\\
       &\qquad=\popblue{\sum\limits_{s'\in\cS}}\Big(
         p_{h}(s'|s,a)-\widehat{p}_{h}(s'|s,a)\Big)
         \popblue{\sum\limits_{y'\in\cY_{h+1}}}\indic{y+\overline{r}_h(s,a)=y'}V_{h+1}^{\psi^*}(s',y';\fp,\fr)\\
     &\qquad\qquad+\sum\limits_{(s',y')\in\cS\times\cY_{h+1}}
     \widehat{\fp}_{h}(s',y'|s,y,a)\Big(V_{h+1}^{\psi^*}(s',y';\fp,\fr)-
       V_{h+1}^{\widehat{\psi}^*}(s',y';\widehat{\fp},\fr)
       \Big)\\
       &\qquad=\sum\limits_{s'\in\cS}\Big(
         p_{h}(s'|s,a)-\widehat{p}_{h}(s'|s,a)\Big)
         \popblue{V_{h+1}^{\psi^*}(s',y+\overline{r}_h(s,a);\fp,\fr)}\\
     &\qquad\qquad+\sum\limits_{(s',y')\in\cS\times\cY_{h+1}}
     \widehat{\fp}_{h}(s',y'|s,y,a)\Big(V_{h+1}^{\psi^*}(s',y';\fp,\fr)-
       V_{h+1}^{\widehat{\psi}^*}(s',y';\widehat{\fp},\fr)
       \Big)\\
     &\qquad\markref{(9)}{\ge}\sum\limits_{s'\in\cS}\Big(
       p_{h}(s'|s,a)-\widehat{p}_{h}(s'|s,a)\Big)
       V_{h+1}^{\psi^*}(s',y+\overline{r}_h(s,a);\fp,\fr)\\
       &\qquad\qquad+\sum\limits_{(s',y')\in\cS\times\cY_{h+1}}
         \widehat{\fp}_{h}(s',y'|s,y,a)\\
         &\qquad\qquad\cdot\Big(\popblue{Q}_{h+1}^{\psi^*}
         (s',y',\popblue{\widehat{\psi}^*_{h+1}(s',y')};\fp,\fr)-
         \popblue{Q}_{h+1}^{\widehat{\psi}^*}(s',y',\widehat{\psi}^*_{h+1}(s',y');\widehat{\fp},\fr)
         \Big),
   \end{align*}
   where at (7) we have applied the Bellman equation, at (8) we have added and
   subtracted a term, and at (9) we have used that
   $V_{h+1}^{\psi^*}(s',y';\fp,\fr)=Q_{h+1}^{\psi^*}(s',y',\psi^*_{h+1}(s',y');\fp,\fr)\ge
   Q_{h+1}^{\psi^*}(s',y',\widehat{\psi}^*_{h+1}(s',y');\fp,\fr)$, since $\psi^*_{h+1}(s',y')$ is
   the optimal action under $\fp,\fr$, and so, it cannot be worse than action
   $\widehat{\psi}^*_{h+1}(s',y')$. By unfolding the recursion, we obtain the result.
 \end{proof}
 
 \begin{lemma}\label{lemma: lemma 4}
   For any $\delta\in(0,1)$, w.p. at least $1-\delta$, it holds that:
   \begin{align*}
     \max_{\hsy}|V_h^*(s,y;\fp,\fr)-V^{\psi^*}_h(s,y;\widehat{\fp},\fr)|&\le
     cH^2\sqrt{\frac{\log\frac{2SAHd}{\delta}}{n}},\\
     \max_{\hsy}|V_h^*(s,y;\fp,\fr)-V^{*}_h(s,y;\widehat{\fp},\fr)|&\le
     cH^2\sqrt{\frac{\log\frac{2SAHd}{\delta}}{n}}.
   \end{align*}
   where $c$ is some positive constant.
 \end{lemma}
 \begin{proof}
   First, we observe that, for any $\hsy$, by following passages similar
   to those in the proof of Lemma \ref{lemma: lemma 3}:
   \begin{align*}
     &|V_h^*(s,y;\fp,\fr)-V^{\psi^*}_h(s,y;\widehat{\fp},\fr)|\\
     &\qquad=
     |Q_h^{\popblue{\psi^*}}(s,y,\popblue{\psi^*_h(s,y)};\fp,\fr)-
     Q^{\psi^*}_h(s,y,\popblue{\psi^*_h(s,y)};\widehat{\fp},\fr)|\\
     &\qquad=\Big|\fr_h(s,y,\psi^*_h(s,y))+\sum\limits_{(s',y')\in\cS\times\cY_{h+1}}
       \fp_{h}(s',y'|s,y,\psi^*_h(s,y))V_{h+1}^{\psi^*}(s',y';\fp,\fr)\\
       &\qquad\qquad-\Big(
         \fr_h(s,y,\psi^*_h(s,y))+\sum\limits_{(s',y')\in\cS\times\cY_{h+1}}
         \widehat{\fp}_{h}(s',y'|s,y,\psi^*_h(s,y))V_{h+1}^{\psi^*}(s',y';\widehat{\fp},\fr)\Big)\Big|\\
     &\qquad=\Big|\sum\limits_{(s',y')\in\cS\times\cY_{h+1}}
     \fp_{h}(s',y'|s,y,\psi^*_h(s,y))V_{h+1}^{\psi^*}(s',y';\fp,\fr)\\
     &\qquad\qquad-\sum\limits_{(s',y')\in\cS\times\cY_{h+1}}
       \widehat{\fp}_{h}(s',y'|s,y,\psi^*_h(s,y))V_{h+1}^{\psi^*}(s',y';\widehat{\fp},\fr)\\
       &\qquad\qquad\popblue{\pm \sum\limits_{(s',y')\in\cS\times\cY_{h+1}}
       \widehat{\fp}_{h}(s',y'|s,y,\psi^*_h(s,y))V_{h+1}^{\psi^*}(s',y';\fp,\fr)}\Big|\\
     &\qquad=\Big|\sum\limits_{(s',y')\in\cS\times\cY_{h+1}}\Big(
     \fp_{h}(s',y'|s,y,\psi^*_h(s,y))-\widehat{\fp}_{h}(s',y'|s,y,\psi^*_h(s,y))\Big)V_{h+1}^{\psi^*}(s',y';\fp,\fr)\\
     &\qquad\qquad+\sum\limits_{(s',y')\in\cS\times\cY_{h+1}}
       \widehat{\fp}_{h}(s',y'|s,y,\psi^*_h(s,y))\Big(V_{h+1}^{\psi^*}(s',y';\fp,\fr)-
       V_{h+1}^{\psi^*}(s',y';\widehat{\fp},\fr)
       \Big)\Big|\\
   &\qquad=\Big|\popblue{\sum\limits_{s'\in\cS}\Big(
     p_{h}(s'|s,\psi^*_h(s,y))-\widehat{p}_{h}(s'|s,\psi^*_h(s,y))\Big)
     V_{h+1}^{\psi^*}(s',y+\overline{r}_h(s,\psi^*_h(s,y));\fp,\fr)}\\
     &\qquad\qquad+\sum\limits_{(s',y')\in\cS\times\cY_{h+1}}
       \widehat{\fp}_{h}(s',y'|s,y,\psi^*_h(s,y))\Big(V_{h+1}^{\psi^*}(s',y';\fp,\fr)-
       V_{h+1}^{\psi^*}(s',y';\widehat{\fp},\fr)
       \Big)\Big|\\
       &\qquad=\dotsc\\
     &\qquad=\bigg|\mathop{\E}\limits_{\widehat{\fp},\fr,\psi^*}\bigg[
       \sum\limits_{h'=h}^H \sum\limits_{s'\in\cS}\Big(
         p_{h'}(s'|s_{h'},a_{h'})-\widehat{p}_{h'}(s'|s_{h'},a_{h'})\Big)
         V_{h'+1}^{\psi^*}(s',y_{h'+1};\fp,\fr)\\
         &\qquad\qquad\Big|\,s_h=s,y_h=y
         \bigg]\bigg|\\
     &\qquad\markref{(1)}{\le}
     \mathop{\E}\limits_{\widehat{\fp},\fr,\psi^*}\bigg[
       \sum\limits_{h'=h}^H \popblue{\bigg|}\sum\limits_{s'\in\cS}\Big(
         p_{h'}(s'|s_{h'},a_{h'})-\widehat{p}_{h'}(s'|s_{h'},a_{h'})\Big)
         V_{h'+1}^{\psi^*}(s',y_{h'+1};\fp,\fr)\popblue{\bigg|}\\
         &\qquad\qquad\bigg|\,s_h=s,y_h=y  
         \bigg],
   \end{align*}
 where at (1) we have brought the absolute value inside the expectation.
 
 Similarly, for the other term, for any $\hsy$, we can write:
 \begin{align*}
   &|V_h^*(s,y;\fp,\fr)-V^*_h(s,y;\widehat{\fp},\fr)|\\
   &\qquad=|V_h^{\popblue{\psi^*}}(s,y;\fp,\fr)
   -V^{\popblue{\widehat{\psi}^*}}_h(s,y;\widehat{\fp},\fr)|\\
   &\qquad\markref{(2)}{=}
   |\popblue{\max\limits_{a\in\cA}}Q_h^{\psi^*}(s,y,\popblue{a};\fp,\fr)-
   \popblue{\max\limits_{a\in\cA}}Q^{\widehat{\psi}^*}_h(s,y,\popblue{a};\widehat{\fp},\fr)|\\
   &\qquad\markref{(3)}{\le}
   \popblue{\max\limits_{a\in\cA}}|Q_h^{\psi^*}(s,y,a;\fp,\fr)-
   Q^{\widehat{\psi}^*}_h(s,y,a;\widehat{\fp},\fr)|\\
   &\qquad=\max\limits_{a\in\cA}\Big|\fr_h(s,y,a)+\sum\limits_{(s',y')\in\cS\times\cY_{h+1}}
   \fp_{h}(s',y'|s,y,a)V_{h+1}^{\psi^*}(s',y';\fp,\fr)\\
   &\qquad\qquad-\Big(
     \fr_h(s,y,a)+\sum\limits_{(s',y')\in\cS\times\cY_{h+1}}
     \widehat{\fp}_{h}(s',y'|s,y,a)V_{h+1}^{\widehat{\psi}^*}(s',y';\widehat{\fp},\fr)\Big)\Big|\\
 &\qquad=\max\limits_{a\in\cA}\Big|\sum\limits_{(s',y')\in\cS\times\cY_{h+1}}
 \fp_{h}(s',y'|s,y,a)V_{h+1}^{\psi^*}(s',y';\fp,\fr)\\
 &\qquad\qquad-\sum\limits_{(s',y')\in\cS\times\cY_{h+1}}
   \widehat{\fp}_{h}(s',y'|s,y,a)V_{h+1}^{\widehat{\psi}^*}(s',y';\widehat{\fp},\fr)\\
   &\qquad\qquad\popblue{\pm \sum\limits_{(s',y')\in\cS\times\cY_{h+1}}
   \widehat{\fp}_{h}(s',y'|s,y,a)V_{h+1}^{\psi^*}(s',y';\fp,\fr)}\Big|\\
 &\qquad=\max\limits_{a\in\cA}\Big|\sum\limits_{(s',y')\in\cS\times\cY_{h+1}}\Big(
 \fp_{h}(s',y'|s,y,a)-\widehat{\fp}_{h}(s',y'|s,y,a)\Big)V_{h+1}^{\psi^*}(s',y';\fp,\fr)\\
 &\qquad\qquad+\sum\limits_{(s',y')\in\cS\times\cY_{h+1}}
   \widehat{\fp}_{h}(s',y'|s,y,a)\Big(V_{h+1}^{\psi^*}(s',y';\fp,\fr)-
   V_{h+1}^{\widehat{\psi}^*}(s',y';\widehat{\fp},\fr)
   \Big)\Big|\\
   &\qquad\markref{(4)}{\le}\popblue{\Big|}\sum\limits_{(s',y')\in\cS\times\cY_{h+1}}\Big(
     \fp_{h}(s',y'|s,y,\popblue{\overline{a}})-\widehat{\fp}_{h}(s',y'|s,y,
     \popblue{\overline{a}})\Big)V_{h+1}^{\psi^*}(s',y';\fp,\fr)\popblue{\Big|}\\
     &\qquad\qquad+\popblue{\Big|}\sum\limits_{(s',y')\in\cS\times\cY_{h+1}}
       \widehat{\fp}_{h}(s',y'|s,y,\popblue{\overline{a}})\Big(V_{h+1}^{\psi^*}(s',y';\fp,\fr)-
       V_{h+1}^{\widehat{\psi}^*}(s',y';\widehat{\fp},\fr)
       \Big)\popblue{\Big|}\\
   &\qquad=\Big|\popblue{\sum\limits_{s'\in\cS}\Big(
     p_{h}(s'|s,\overline{a})-\widehat{p}_{h}(s'|s,\overline{a})\Big)
     V_{h+1}^{\psi^*}(s',y+\Pi_{\cR}[r_h(s,\overline{a})];\fp,\fr)}\Big|\\
     &\qquad\qquad+\Big|\sum\limits_{(s',y')\in\cS\times\cY_{h+1}}
       \widehat{\fp}_{h}(s',y'|s,y,\overline{a})\Big(V_{h+1}^{\psi^*}(s',y';\fp,\fr)-
       V_{h+1}^{\widehat{\psi}^*}(s',y';\widehat{\fp},\fr)
       \Big)\Big|\\
 &\qquad\le\dotsc\\
 &\qquad\markref{(5)}{\le}
 \mathop{\E}\limits_{\widehat{\fp},\fr,\popblue{\overline{\psi}}}\bigg[
       \sum\limits_{h'=h}^H \bigg|\sum\limits_{s'\in\cS}\Big(
         p_{h'}(s'|s_{h'},a_{h'})-\widehat{p}_{h'}(s'|s_{h'},a_{h'})\Big)
         V_{h'+1}^{\psi^*}(s',y_{h'+1};\fp,\fr)\bigg|\\
         &\qquad\qquad\bigg|\,s_h=s,y_h=y
         \bigg],
 \end{align*}
 where at (2) we have applied the Bellman optimality equation, at (3) we have
 upper bounded the difference of maxima with the maximum of the difference, at
 (4) we denote the maximal action by $\overline{a}$, and we apply triangle
 inequality; at (5) we have unfolded the recursion and called $\overline{\psi}$
 the resulting policy.
 
 Now, for some $\epsilon\in(0,1)$, let us denote by $\cE$ the event defined as:
 \begin{align*}
   \cE\coloneqq\bigg\{&
   \forall \hsya:\\
   &\qquad
   \big|\sum\limits_{s'\in\cS}\Big(
     p_h(s'|s,a)-\widehat{p}_h(s'|s,a)\Big)
     V_{h+1}^{\psi^*}(s',y+\overline{r}_h(s,a);\fp,\fr)\big|\le
     \epsilon
   \bigg\}
 \end{align*}
 We can write:
 \begin{align*}
   \P(\cE^\complement)&=\P\bigg(
     \exists \hsya:\\
     &\qquad\qquad
   \big|\sum\limits_{s'\in\cS}\Big(
     p_h(s'|s,a)-\widehat{p}_h(s'|s,a)\Big)
     V_{h+1}^{\psi^*}(s',y+\overline{r}_h(s,a);\fp,\fr)\big|>
     \epsilon
   \bigg)\\
   &\markref{(6)}{\le}
   \sum\limits_{\hsya}\\
   &\qquad\qquad
   \P\bigg(
   \big|\sum\limits_{s'\in\cS}\Big(
     p_h(s'|s,a)-\widehat{p}_h(s'|s,a)\Big)
     V_{h+1}^{\psi^*}(s',y+\overline{r}_h(s,a);\fp,\fr)\big|>
     \epsilon
   \bigg)\\
   &\markref{(7)}{\le}
   \sum\limits_{\hsya}
   2e^{\frac{-2 n \epsilon^2}{H^2}}\\
   &=2SAHd e^{\frac{-2 n \epsilon^2}{H^2}},
 \end{align*}
 where at (6) we have applied a union bound over all tuples
 $\hsya$, and at (7) we have applied Hoeffding's inequality, by
 recalling that we collect $n$ samples (see Algorithm \ref{alg: caty
 exploration}) for any $(s,a,h)\in\SAH$ triple, and that vector
 $V_{h+1}^{\psi^*}(\cdot,y+\overline{r}_h(s,a);\fp,\fr)$ bounded by $[0,H]$ is
 independent of the randomness in $\widehat{p}_h(\cdot|s,a)$. It should be
 remarked that our collection of samples depends only on $\SAH$, and not on
 $\cY_h$; such term enters the expression only through the union bound, because
 we have to apply Hoeffding's inequality for all the value functions considered,
 which are as many as $|\cY_h|$ . Note that we use $d=|\cY_{H+1}|$ since it is
 the largest $|\cY_h|$ among $h\in\dsb{H+1}$.
 
 This probability is at most $\delta$ if:
 \begin{align*}
   2SAHd e^{\frac{-2 n \epsilon^2}{H^2}}\le\delta\iff
   \epsilon \ge H\sqrt{\frac{\log\frac{2SAHd}{\delta}}{2 n}}.
 \end{align*}
 
 By plugging into the previous expressions, we obtain that, w.p. $1-\delta$:
 \begin{align*}
   &|V_h^*(s,y;\fp,\fr)-V^{\psi^*}_h(s,y;\widehat{\fp},\fr)|\\
   &\qquad\le
   \mathop{\E}\limits_{\widehat{\fp},\fr,\psi^*}\bigg[
       \sum\limits_{h'=h}^H \bigg|\sum\limits_{s'\in\cS}\Big(
         p_{h'}(s'|s_{h'},a_{h'})-\widehat{p}_{h'}(s'|s_{h'},a_{h'})\Big)
         V_{h'+1}^{\psi^*}(s',y_{h'+1};\fp,\fr)\bigg|\\
         &\qquad\qquad\bigg|\,s_h=s,y_h=y
         \bigg]\\
   &\qquad\le \mathop{\E}\limits_{\widehat{\fp},\fr,\psi^*}\bigg[
     \sum\limits_{h'=h}^H 
     H\sqrt{\frac{\log\frac{2SAHd}{\delta}}{2 n}}
     \,\bigg|\,s_h=s,y_h=y  
       \bigg]\\
   &\qquad= H^2\sqrt{\frac{\log\frac{2SAHd}{\delta}}{2 n}},
 \end{align*}
  and also:
  \begin{align*}
   &|V_h^*(s,y;\fp,\fr)-V^*_h(s,y;\widehat{\fp},\fr)|\\
   &\qquad\le
   \mathop{\E}\limits_{\widehat{\fp},\fr,\overline{\psi}}\bigg[
       \sum\limits_{h'=h}^H \bigg|\sum\limits_{s'\in\cS}\Big(
         p_{h'}(s'|s_{h'},a_{h'})-\widehat{p}_{h'}(s'|s_{h'},a_{h'})\Big)
         V_{h'+1}^{\psi^*}(s',y_{h'+1};\fp,\fr)\bigg|\\
         &\qquad\qquad\bigg|\,s_h=s,y_h=y
         \bigg]\\
   &\qquad\le \mathop{\E}\limits_{\widehat{\fp},\fr,\overline{\psi}}\bigg[
     \sum\limits_{h'=h}^H 
     H\sqrt{\frac{\log\frac{2SAHd}{\delta}}{2 n}}
     \,\bigg|\,s_h=s,y_h=y    
       \bigg]\\
   &\qquad= H^2\sqrt{\frac{\log\frac{2SAHd}{\delta}}{2 n}}.
 \end{align*}
 
 This concludes the proof.
 
 \end{proof}
 
 \begin{lemma}\label{lemma: lemma 5}
   For any $\delta\in(0,1)$, w.p. at least $1-\delta$, it holds that, for all $\hsya$:
   \begin{align*}
     &\sqrt{\V_{s'\sim p_h(\cdot|s,a)}[V^*_{h+1}(s',y+\overline{r}_h(s,a);\fp,\fr)]}\le\\
     &\qquad\qquad
     \sqrt{\V_{s'\sim \widehat{p}_h(\cdot|s,a)}
     [V^{\psi^*}_{h+1}(s',y+\overline{r}_h(s,a);\widehat{\fp},\fr)]}+b_1,\\
     &\sqrt{\V_{s'\sim p_h(\cdot|s,a)}[V^*_{h+1}(s',y+\overline{r}_h(s,a);\fp,\fr)]}\le\\
     &\qquad\qquad\sqrt{\V_{s'\sim \widehat{p}_h(\cdot|s,a)}
     [V^{*}_{h+1}(s',y+\overline{r}_h(s,a);\widehat{\fp},\fr)]}+b_1,
   \end{align*}
   where $b_1$ is defined as:
   \begin{align*}
     b_1\coloneqq cH\bigg(\frac{\log\frac{4SAHY}{\delta}}{n}\bigg)^{1/4}
     + c'H^2\sqrt{\frac{\log\frac{4SAHY}{\delta}}{n}},
   \end{align*}
   for some positive constants $c,c'$.
 \end{lemma}
 \begin{proof}
   In the following, we will use $\overline{y}$ as a label for $y+\overline{r}_h(s,a)$.
   We begin with the first expression. We can write, for any $\hsya$:
   \begin{align*}
     &\V_{s'\sim p_h(\cdot|s,a)}[V^*_{h+1}(s',\overline{y};\fp,\fr)]\\
     &\qquad=\V_{s'\sim p_h(\cdot|s,a)}[V^*_{h+1}(s',\overline{y};\fp,\fr)]
 \popblue{\pm \V_{s'\sim \widehat{p}_h(\cdot|s,a)}[V^*_{h+1}(s',\overline{y};\fp,\fr)]}\\
 &\qquad=\Big(\V_{s'\sim p_h(\cdot|s,a)}[V^*_{h+1}(s',\overline{y};\fp,\fr)]-
 \V_{s'\sim \widehat{p}_h(\cdot|s,a)}[V^*_{h+1}(s',\overline{y};\fp,\fr)]\Big)\\
 &\qquad\qquad+
 \V_{s'\sim \widehat{p}_h(\cdot|s,a)}[V^*_{h+1}(s',\overline{y};\fp,\fr)]\\
 &\qquad\markref{(1)}{=}
 \sum\limits_{s'\in\cS}\Big(p_h(s'|s,a)-\widehat{p}_h(s'|s,a)\Big)
 V^{*^{\popblue{2}}}_{h+1}(s',\overline{y};\fp,\fr)\\
 &\qquad\qquad -
 \Big[\Big(\sum\limits_{s'\in\cS}p_h(s'|s,a)V^{*}_{h+1}(s',\overline{y};\fp,\fr)\Big)^{\popblue{2}}\\
 &\qquad\qquad-
 \Big(\sum\limits_{s'\in\cS}\widehat{p}_h(s'|s,a)
 V^{*}_{h+1}(s',\overline{y};\fp,\fr)\Big)^{\popblue{2}}\Big]\\
 &\qquad\qquad +\V_{s'\sim \widehat{p}_h(\cdot|s,a)}[V^*_{h+1}(s',\overline{y};\fp,\fr)\popblue{\pm 
 V^{\psi^*}_{h+1}(s',\overline{y};\widehat{\fp},\fr)}]\\
 &\qquad\markref{(2)}{=}
 \sum\limits_{s'\in\cS}\Big(p_h(s'|s,a)-\widehat{p}_h(s'|s,a)\Big)
 V^{*^{2}}_{h+1}(s',\overline{y};\fp,\fr)\\
 &\qquad\qquad -
 \Big[\Big(\sum\limits_{s'\in\cS}p_h(s'|s,a)V^{*}_{h+1}(s',\overline{y};\fp,\fr)\Big)^{2}\\
 &\qquad\qquad-
 \Big(\sum\limits_{s'\in\cS}\widehat{p}_h(s'|s,a)
 V^{*}_{h+1}(s',\overline{y};\fp,\fr)\Big)^{2}\Big]\\
 &\qquad\qquad +\popblue{\V_{s'\sim \widehat{p}_h(\cdot|s,a)}
 [V^*_{h+1}(s',\overline{y};\fp,\fr)-
 V^{\psi^*}_{h+1}(s',\overline{y};\widehat{\fp},\fr)]}\\
 &\qquad\qquad\popblue{+
 \V_{s'\sim \widehat{p}_h(\cdot|s,a)}[V^{\psi^*}_{h+1}
 (s',\overline{y};\widehat{\fp},\fr)]}\\
 &\qquad\qquad\popblue{+2\text{Cov}_{s'\sim \widehat{p}_h(\cdot|s,a)}
 [V^*_{h+1}(s',\overline{y};\fp,\fr)-
 V^{\psi^*}_{h+1}(s',\overline{y};\widehat{\fp},\fr),}\\
 &\qquad\qquad\popblue{V^{\psi^*}_{h+1}(s',\overline{y};\widehat{\fp},\fr)]}\\
 &\qquad\markref{(3)}{\le}
 \sum\limits_{s'\in\cS}\Big(p_h(s'|s,a)-\widehat{p}_h(s'|s,a)\Big)
 V^{*^{2}}_{h+1}(s',\overline{y};\fp,\fr)\\
 &\qquad\qquad -
 \Big[\Big(\sum\limits_{s'\in\cS}p_h(s'|s,a)V^{*}_{h+1}(s',\overline{y};\fp,\fr)\Big)^{2}\\
 &\qquad\qquad-
 \Big(\sum\limits_{s'\in\cS}\widehat{p}_h(s'|s,a)
 V^{*}_{h+1}(s',\overline{y};\fp,\fr)\Big)^{2}\Big]\\
 &\qquad\qquad +\V_{s'\sim \widehat{p}_h(\cdot|s,a)}[V^*_{h+1}(s',\overline{y};\fp,\fr)-
 V^{\psi^*}_{h+1}(s',\overline{y};\widehat{\fp},\fr)]\\
 &\qquad\qquad+
 \V_{s'\sim \widehat{p}_h(\cdot|s,a)}[V^{\psi^*}_{h+1}(s',\overline{y};\widehat{\fp},\fr)]\\
 &\qquad\qquad+2\popblue{\Big(\V_{s'\sim \widehat{p}_h(\cdot|s,a)}[V^*_{h+1}(s',\overline{y};\fp,\fr)-
 V^{\psi^*}_{h+1}(s',\overline{y};\widehat{\fp},\fr)]}\\
 &\qquad\qquad\popblue{\cdot
 \V_{s'\sim \widehat{p}_h(\cdot|s,a)}[V^{\psi^*}_{h+1}(s',\overline{y};\widehat{\fp},\fr)]\Big)^{1/2}}\\
 &\qquad=\sum\limits_{s'\in\cS}\Big(p_h(s'|s,a)-\widehat{p}_h(s'|s,a)\Big)
 V^{*^{2}}_{h+1}(s',\overline{y};\fp,\fr)\\
 &\qquad\qquad -
 \Big[\Big(\sum\limits_{s'\in\cS}p_h(s'|s,a)V^{*}_{h+1}(s',\overline{y};\fp,\fr)\Big)^{2}\\
 &\qquad\qquad-
 \Big(\sum\limits_{s'\in\cS}\widehat{p}_h(s'|s,a)
 V^{*}_{h+1}(s',\overline{y};\fp,\fr)\Big)^{2}\Big]\\
 &\qquad\qquad +\popblue{\Big[
 \sqrt{\V_{s'\sim \widehat{p}_h(\cdot|s,a)}[V^*_{h+1}(s',\overline{y};\fp,\fr)-
 V^{\psi^*}_{h+1}(s',\overline{y};\widehat{\fp},\fr)]}}\\
 &\qquad\qquad\popblue{+\sqrt{
   \V_{s'\sim \widehat{p}_h(\cdot|s,a)}[V^{\psi^*}_{h+1}(s',\overline{y};\widehat{\fp},\fr)]
 }
 \Big]^2},
   \end{align*}
 where at (1) we have used the common formula for the variance
 $\V[X]=\E[X^2]-\E[X]^2$,
 at (2) we have decomposed the variance of a sum as $\V[X+Y] = \V[X] + \V[Y] +
 2\text{Cov}[X,Y]$,
 at (3) we have applied Cauchy-Schwarz's inequality to bound the covariance with
 the product of the variances $|\text{Cov}[X,Y]|\le\sqrt{\V[X]\V[Y]}$.
 
 Next, observe that:
 \begin{align*}
   &\V_{s'\sim \widehat{p}_h(\cdot|s,a)}[V^*_{h+1}(s',\overline{y};\fp,\fr)-
 V^{\psi^*}_{h+1}(s',\overline{y};\widehat{\fp},\fr)]\\
 &\qquad\markref{(4)}{=}
 \E_{s'\sim \widehat{p}_h(\cdot|s,a)}[(V^*_{h+1}(s',\overline{y};\fp,\fr)-
 V^{\psi^*}_{h+1}(s',\overline{y};\widehat{\fp},\fr))^{\popblue{2}}]\\
 &\qquad\qquad-
 \E_{s'\sim \widehat{p}_h(\cdot|s,a)}[V^*_{h+1}(s',\overline{y};\fp,\fr)-
 V^{\psi^*}_{h+1}(s',\overline{y};\widehat{\fp},\fr)]^{\popblue{2}}\\
 &\qquad\markref{(5)}{\le}
 \E_{s'\sim \widehat{p}_h(\cdot|s,a)}[(V^*_{h+1}(s',\overline{y};\fp,\fr)-
 V^{\psi^*}_{h+1}(s',\overline{y};\widehat{\fp},\fr))^2]\\
 &\qquad\markref{(6)}{\le}
 \|(V^*_{h+1}(\popblue{\cdot},\overline{y};\fp,\fr)-
 V^{\psi^*}_{h+1}(\popblue{\cdot},\overline{y};\widehat{\fp},\fr))^2\|_\infty\\
 &\qquad=
 \|V^*_{h+1}(\cdot,\overline{y};\fp,\fr)-
 V^{\psi^*}_{h+1}(\cdot,\overline{y};\widehat{\fp},\fr)\|_\infty^{\popblue{2}},
 \end{align*}
 where at (4) we have used $\V[X]=\E[X^2]-\E[X]^2$, at (5) we recognize that the
 second term is a square, thus always positive, and we remove it, and at (6) we
 have upper bounded the expected value, an average, through the infinity norm.
 
 Thanks to this expression, we can continue to upper bound the previous term
 as:
 \begin{align*}
   &\V_{s'\sim p_h(\cdot|s,a)}[V^*_{h+1}(s',\overline{y};\fp,\fr)]\\
   &\qquad\le
   \sum\limits_{s'\in\cS}\Big(p_h(s'|s,a)-\widehat{p}_h(s'|s,a)\Big)
 V^{*^{2}}_{h+1}(s',\overline{y};\fp,\fr)\\
 &\qquad\qquad -
 \Big[\Big(\sum\limits_{s'\in\cS}p_h(s'|s,a)V^{*}_{h+1}(s',\overline{y};\fp,\fr)\Big)^{2}\\
 &\qquad\qquad-
 \Big(\sum\limits_{s'\in\cS}\widehat{p}_h(s'|s,a)
 V^{*}_{h+1}(s',\overline{y};\fp,\fr)\Big)^{2}\Big]\\
 &\qquad\qquad
 +\Big[\popblue{\|V^*_{h+1}(\cdot,\overline{y};\fp,\fr)-
 V^{\psi^*}_{h+1}(\cdot,\overline{y};\widehat{\fp},\fr)\|_\infty}\\
 &\qquad\qquad+\sqrt{
   \V_{s'\sim \widehat{p}_h(\cdot|s,a)}[V^{\psi^*}_{h+1}(s',\overline{y};\widehat{\fp},\fr)]}\Big]^2\\
 &\qquad\markref{(7)}{=}
   \sum\limits_{s'\in\cS}\Big(p_h(s'|s,a)-\widehat{p}_h(s'|s,a)\Big)
 V^{*^{2}}_{h+1}(s',\overline{y};\fp,\fr)\\
 &\qquad\qquad -
 \Big[\popblue{\Big(\sum\limits_{s'\in\cS}(p_h(s'|s,a)-\widehat{p}_h(s'|s,a))
 V^{*}_{h+1}(s',\overline{y};\fp,\fr)\Big)}\\
 &\qquad\qquad\popblue{\cdot
 \Big(\sum\limits_{s'\in\cS}(p_h(s'|s,a)+\widehat{p}_h(s'|s,a))V^{*}_{h+1}(s',\overline{y};\fp,\fr)\Big)}\Big]\\
 &\qquad\qquad
 +\Big[\|V^*_{h+1}(\cdot,\overline{y};\fp,\fr)-
 V^{\psi^*}_{h+1}(\cdot,\overline{y};\widehat{\fp},\fr)\|_\infty\\
 &\qquad\qquad+\sqrt{
   \V_{s'\sim \widehat{p}_h(\cdot|s,a)}[V^{\psi^*}_{h+1}(s',\overline{y};\widehat{\fp},\fr)]}\Big]^2\\
 &\qquad\markref{(8)}{\le}
   \sum\limits_{s'\in\cS}\Big(p_h(s'|s,a)-\widehat{p}_h(s'|s,a)\Big)
 V^{*^{2}}_{h+1}(s',\overline{y};\fp,\fr)\\
 &\qquad\qquad -
 \Big[\Big(\sum\limits_{s'\in\cS}(p_h(s'|s,a)-\widehat{p}_h(s'|s,a))
 V^{*}_{h+1}(s',\overline{y};\fp,\fr)\Big)\\
 &\qquad\qquad
 \cdot\Big(\sum\limits_{s'\in\cS}(p_h(s'|s,a)+\widehat{p}_h(s'|s,a))V^{*}_{h+1}(s',\overline{y};\fp,\fr)\Big)\Big]\\
 &\qquad\qquad
 +\Big[\popblue{cH^2\sqrt{\frac{\log\frac{4SAHd}{\delta}}{ n}}}+\sqrt{
   \V_{s'\sim \widehat{p}_h(\cdot|s,a)}[V^{\psi^*}_{h+1}(s',\overline{y};\widehat{\fp},\fr)]}\Big]^2\\
   &\qquad\markref{(9)}{\le}
   \sum\limits_{s'\in\cS}\Big(p_h(s'|s,a)-\widehat{p}_h(s'|s,a)\Big)
 V^{*^{2}}_{h+1}(s',\overline{y};\fp,\fr)\\
 &\qquad\qquad+\popblue{2H
 \Big|}\sum\limits_{s'\in\cS}(p_h(s'|s,a)-\widehat{p}_h(s'|s,a))V^{*}_{h+1}(s',\overline{y};\fp,\fr)\popblue{\Big|}\\
 &\qquad\qquad
 +\Big[cH^2\sqrt{\frac{\log\frac{4SAHd}{\delta}}{ n}}+\sqrt{
   \V_{s'\sim \widehat{p}_h(\cdot|s,a)}[V^{\psi^*}_{h+1}(s',\overline{y};\widehat{\fp},\fr)]}\Big]^2\\
   &\qquad\markref{(10)}{\le}
   \sum\limits_{s'\in\cS}\Big(p_h(s'|s,a)-\widehat{p}_h(s'|s,a)\Big)
 V^{*^{2}}_{h+1}(s',\overline{y};\fp,\fr)\\
 &\qquad\qquad+\popblue{2cH^2\sqrt{\frac{\log\frac{4SAHd}{\delta}}{ n}}}\\
 &\qquad\qquad
 +\Big[cH^2\sqrt{\frac{\log\frac{4SAHd}{\delta}}{ n}}+\sqrt{
   \V_{s'\sim \widehat{p}_h(\cdot|s,a)}[V^{\psi^*}_{h+1}(s',\overline{y};\widehat{\fp},\fr)]}\Big]^2\\
   &\qquad\markref{(11)}{\le}
   \popblue{cH^2\sqrt{\frac{\log\frac{4SAHd}{\delta}}{ n}}}
 +2cH^2\sqrt{\frac{\log\frac{4SAHd}{\delta}}{ n}}\\
 &\qquad\qquad
 +\Big[cH^2\sqrt{\frac{\log\frac{4SAHd}{\delta}}{ n}}+\sqrt{
   \V_{s'\sim \widehat{p}_h(\cdot|s,a)}[V^{\psi^*}_{h+1}(s',\overline{y};\widehat{\fp},\fr)]}\Big]^2\\
   &\qquad=\popblue{3}cH^2\sqrt{\frac{\log\frac{4SAHd}{\delta}}{ n}}\\
   &\qquad\qquad
   +\Big[cH^2\sqrt{\frac{\log\frac{4SAHd}{\delta}}{ n}}+\sqrt{
     \V_{s'\sim \widehat{p}_h(\cdot|s,a)}[V^{\psi^*}_{h+1}(s',\overline{y};\widehat{\fp},\fr)]}\Big]^2,
 \end{align*}
 where at (7) we have applied the common formula $x^2-y^2=(x-y)(x+y)$, at (8) we
 have applied Lemma \ref{lemma: lemma 4} using probability $\delta'=\delta/2$,
 and noticing that, for how the discretized MDP is constructed, we have that
 $\overline{y}\in\cY$, at (9) we have upper bounded the second term with the
 absolute value and recognized that the value function does not exceed $H$ and
 the sum of probabilities is no greater than $2$; at (10) we recognize that, in
 the proof of Lemma \ref{lemma: lemma 4}, we had already bounded that term, thus,
 under the event $\cE$ which holds w.p. $1-\delta/2$, we have that bound; at (11)
 we have applied Hoeffding's inequality to all tuples $\hsya$ with
 probability $\delta/(2SAHd)$, and noticed that the square of the value function does
 not exceed $H^2$.
 
 Observe that the previous formula holds for all $\hsya$ w.p.
 $1-\delta$ (by summing the two $\delta/2$ through a union bound). By taking the
 square root of both sides, we obtain:
 \begin{align*}
   &\sqrt{\V_{s'\sim p_h(\cdot|s,a)}[V^*_{h+1}(s',\overline{y};\fp,\fr)]}\\
   &\qquad\le
   \Big(3cH^2\sqrt{\frac{\log\frac{4SAHd}{\delta}}{ n}}
   +\Big[cH^2\sqrt{\frac{\log\frac{4SAHd}{\delta}}{ n}}\\
   &\qquad\qquad+\sqrt{
     \V_{s'\sim \widehat{p}_h(\cdot|s,a)}
     [V^{\psi^*}_{h+1}(s',\overline{y};\widehat{\fp},\fr)]}\Big]^2\Big)^{1/2}\\
     &\qquad\markref{(12)}{\le}\underbrace{c'H\sqrt[4]{\frac{\log\frac{4SAHY}{\delta}}{ n}}
     +cH^2\sqrt{\frac{\log\frac{4SAHY}{\delta}}{ n}}}_{\popblue{\eqqcolon b_1}}\\
     &\qquad\qquad+\sqrt{
       \V_{s'\sim \widehat{p}_h(\cdot|s,a)}[V^{\psi^*}_{h+1}(s',\overline{y};\widehat{\fp},\fr)]}\\
     &\qquad= \sqrt{
       \V_{s'\sim \widehat{p}_h(\cdot|s,a)}[V^{\psi^*}_{h+1}(s',\overline{y};\widehat{\fp},\fr)]}+ b_1,
 \end{align*}
 where at (12) we have used the fact that $\sqrt{a+b}\le\sqrt{a}+\sqrt{b}$.
 
 To prove the second formula, the passages are basically the same, the only
 difference is that, at passage (1), we sum and subtract
 $V^{\widehat{\psi}^*}_{h+1}(s',\overline{y};\widehat{\fp},\fr)$ instead of
 $V^{\psi^*}_{h+1}(s',\overline{y};\widehat{\fp},\fr)$, and that at passage (8) we
 apply the other expression in Lemma \ref{lemma: lemma 4}. This concludes the
 proof.
 \end{proof}
 
 \begin{lemma}\label{lemma: lemma 6}
   For any $\delta\in(0,1)$, define:
   \begin{align*}
     c_1&\coloneqq \log\frac{2SAHd}{\delta},\\
     b_2&\coloneqq cH\bigg(\frac{\log\frac{8SAHd}{\delta}}{n}\bigg)^{3/4}
     +c'H^2\frac{\log\frac{8SAHd}{\delta}}{n},
   \end{align*}
   for some positive constants $c,c'$.
   Then, w.p. at least $1-\delta$, we have, for all $\hsya$:
   \begin{align*}
     &\sum\limits_{s'\in\cS}\Big(p_h(s'|s,a)-\widehat{p}_h(s'|s,a)\Big)
     V^*_{h+1}(s',y+\overline{r}_h(s,a);\fp,\fr)\\
     &\qquad\le c''\sqrt{\frac{c_1 
     \V_{s'\sim \widehat{p}_h(\cdot|s,a)}[V^{\psi^*}_{h+1}(s',y+\overline{r}_h(s,a);\widehat{\fp},\fr)]}{n}}
     +b_2,\\
     &\sum\limits_{s'\in\cS}\Big(p_h(s'|s,a)-\widehat{p}_h(s'|s,a)\Big)
     V^*_{h+1}(s',y+\overline{r}_h(s,a);\fp,\fr)\\
     &\qquad\ge -c'''\sqrt{\frac{c_1 
     \V_{s'\sim \widehat{p}_h(\cdot|s,a)}[V^{*}_{h+1}(s',y+\overline{r}_h(s,a);\widehat{\fp},\fr)]}{n}}
     +b_2,
   \end{align*}
   for some positive constants $c'',c'''$.
 \end{lemma}
 \begin{proof}
   Again, we will write $\overline{y}$ instead of $y+\overline{r}_h(s,a)$ for simplicity.
   For all $\hsya$, we can write:
   \begin{align*}
     &\sum\limits_{s'\in\cS}\Big(p_h(s'|s,a)-\widehat{p}_h(s'|s,a)\Big)
     V^*_{h+1}(s',\overline{y};\fp,\fr)\\
     &\qquad\markref{(1)}{\le}
     \sqrt{\frac{2 \V_{s'\sim p_h(\cdot|s,a)}[V^*_{h+1}
     (s',\overline{y};\fp,\fr)]\log\frac{2SAHd}{\delta}}{n}}
     +\frac{2H\log\frac{2SAHd}{\delta}}{3n}\\
     &\qquad\markref{(2)}{\le}
     \sqrt{\frac{2\log\frac{2SAHd}{\delta}}{n}}
     \Big(\sqrt{\V_{s'\sim \popblue{\widehat{p}}_h(\cdot|s,a)}
     [V^{\popblue{\psi^*}}_{h+1}(s',\overline{y};\popblue{\widehat{\fp}},R)]}+b_1\Big)
     +\frac{2H\log\frac{2SAHd}{\delta}}{3n}\\
     &\qquad\markref{(3)}{=}c\sqrt{\frac{\popblue{c_1} 
     \V_{s'\sim \widehat{p}_h(\cdot|s,a)}[V^{\psi^*}_{h+1}(s',\overline{y};\widehat{\fp},\fr)]}{n}}
     +c'\sqrt{\frac{\popblue{c_1}}{n}}
     H\bigg(\frac{\log\frac{\popblue{8}SAHd}{\delta}}{n}\bigg)^{1/4}\\
     &\qquad\qquad
     + c''\sqrt{\frac{\popblue{c_1}}{n}}H^2\sqrt{\frac{\log\frac{\popblue{8}SAHd}{\delta}}{n}}
     +c'''H\frac{\popblue{c_1}}{n}\\
     &\qquad\le c\sqrt{\frac{c_1 
     \V_{s'\sim \widehat{p}_h(\cdot|s,a)}[V^{\psi^*}_{h+1}(s',\overline{y};\widehat{\fp},\fr)]}{n}}
     +
     c'H\bigg(\frac{\log\frac{8SAHd}{\delta}}{n}\bigg)^{3/4}
     +c''''H^2\frac{\log\frac{8SAHd}{\delta}}{n},
   \end{align*}
   where at (1) we have applied the Bernstein's inequality using $\delta/(2SAHd)$
   as probability for all $\hsya$, and at (2) we have applied Lemma
   \ref{lemma: lemma 5} with $\delta/2$ of probability, and a union bound to
   guarantee the event to hold w.p. $1-\delta$, at (3) we use the definition of
   $c_1\coloneqq\log\frac{2SAHd}{\delta}$, and denoted by $c,c',c'',c'''$ some
   positive constants.
 
  For the other expression, an analogous derivation can be carried out. In
  particular, we use the other side of the Bernstein's inequality, and the other
  expression in Lemma \ref{lemma: lemma 5}.
 \end{proof}
 
 \begin{lemma}\label{lemma: lemma 7}
 For any $\hsya$ and deterministic policy $\psi$, let $\Sigma^\psi_h(s,y,a)$ be defined as:
 \begin{align*}
   \Sigma^\psi_h(s,y,a)\coloneqq \mathop{\E}\limits_{\fp,\fr,\psi}\Big[
   \Big|\sum\limits_{h'=h}^H \fr_{h'}(s_{h'},y_{h'},a_{h'})
   -Q^\psi_h(s,y,a;\fp,\fr)\Big|^2\,|\,s_h=s,y_h=y,a_h=a  
   \Big].
 \end{align*}
 Then, function $\Sigma$ satisfies the Bellman equation, i.e., for any
 $\hsya$ and deterministic policy $\psi$:
 \begin{align*}
   \Sigma^\psi_h(s,y,a)=&\V_{s'\sim p_h(\cdot|s,a)}[V^{\psi}_{h+1}(s',y+\overline{r}_h(s,a);\fp,\fr)]\\
   &+ \mathop{\E}\limits_{s'\sim p_h(\cdot|s,a)}
   [\Sigma^\psi_{h+1}(s',y+\overline{r}_h(s,a),\psi_{h+1}(s',y+\overline{r}_h(s,a)))].
 \end{align*}
 \end{lemma}
 \begin{proof}
   For all $\hsya$ and deterministic policy $\psi$, we can write (we
   denote $a'\coloneqq \psi_{h+1}(s',y+\overline{r}_h(s,a))$ and
   $\overline{y}\coloneqq y+\overline{r}_h(s,a)$ for notational simplicity, and
   we remark that $\overline{y}$ is \emph{not} a random variable):
   \begin{align*}
     \Sigma^\psi_h(s,y,a)&\coloneqq \mathop{\E}\limits_{\fp,\fr,\psi}\Big[
       \Big|\sum\limits_{h'=h}^H \fr_{h'}(s_{h'},y_{h'},a_{h'}) -
        Q^\psi_h(s,y,a;\fp,\fr)\Big|^2\,|\,s_h=s,y_h=y,a_h=a  
       \Big]\\
       &\markref{(1)}{=}
       \popblue{\mathop{\E}\limits_{s'\sim p_h(\cdot|s,a)}}
       \bigg[\mathop{\E}\limits_{\fp,\fr,\psi}\Big[
       \Big|\sum\limits_{h'=h}^H \fr_{h'}(s_{h'},y_{h'},a_{h'})
       \popblue{-Q^\psi_{h+1}(s',\overline{y},a';\fp,\fr)}\\
       &\qquad- \big(Q^\psi_h(s,y,a;\fp,\fr)\popblue{-Q^\psi_{h+1}(s',\overline{y},a';\fp,\fr)}\big)\Big|^2
       \\
       &\qquad|\,s_h=s,a_h=a,y_h=y,s_{h+1}=s'\Big]\bigg]\\
       &\markref{(2)}{=}
       \mathop{\E}\limits_{s'\sim p_h(\cdot|s,a)}
       \bigg[\mathop{\E}\limits_{\fp,\fr,\psi}\Big[
       \Big|\sum\limits_{h'=\popblue{h+1}}^H \fr_{h'}(s_{h'},y_{h'},a_{h'})
       -Q^\psi_{h+1}(s',\overline{y},a';\fp,\fr)\\
       &\qquad- \big(Q^\psi_h(s,y,a;\fp,\fr)-\popblue{\fr_h(s,y,a)}-
       Q^\psi_{h+1}(s',\overline{y},a';\fp,\fr)\big)\Big|^2\\
       &\qquad|\,s_{h+1}=s',y_{h+1}=\overline{y}\Big]\bigg]\\
       &\markref{(3)}{=}
       \mathop{\E}\limits_{s'\sim p_h(\cdot|s,a)}
       \bigg[\mathop{\E}\limits_{\fp,\fr,\psi}\Big[
       \Big|\sum\limits_{h'=h+1}^H \fr_{h'}(s_{h'},y_{h'},a_{h'})
       -Q^\psi_{h+1}(s',\overline{y},a';\fp,\fr)\Big|^2\\
       &\qquad|\,s_{h+1}=s',y_{h+1}=\overline{y}\Big]\bigg]\\
       &\qquad
       -2 \mathop{\E}\limits_{s'\sim p_h(\cdot|s,a)}
       \bigg[\big(Q^\psi_h(s,y,a;\fp,\fr)-\fr_h(s,y,a)-Q^\psi_{h+1}(s',\overline{y},a';\fp,\fr)\big)\\
       &\qquad\cdot\underbrace{\mathop{\E}\limits_{\fp,\fr,\psi}\Big[
         \sum\limits_{h'=h+1}^H \fr_{h'}(s_{h'},y_{h'},a_{h'})
       -Q^\psi_{h+1}(s',\overline{y},a';\fp,\fr)\,\big|\,s_{h+1}=s',y_{h+1}=\overline{y}
       \Big]}_{\popblue{=0}}\bigg]\\
       &\qquad+
       \mathop{\E}\limits_{s'\sim p_h(\cdot|s,a)}
       \Big[\big|Q^\psi_h(s,y,a;\fp,\fr)-\fr_h(s,y,a)-Q^\psi_{h+1}(s',\overline{y},a';\fp,\fr)\big|^2\Big]\\
       &\markref{(4)}{=}
       \mathop{\E}\limits_{s'\sim p_h(\cdot|s,a)}
       \bigg[\\
       &\qquad\underbrace{\mathop{\E}\limits_{\fp,\fr,\psi}\Big[
       \Big|\sum\limits_{h'=h+1}^H \fr_{h'}(s_{h'},y_{h'},a_{h'})
       -Q^\psi_{h+1}(s',\overline{y},a';\fp,\fr)\Big|^2\,|\,s_{h+1}=s',y_{h+1}=\overline{y}\Big]}_{
         \popblue{= \Sigma^\psi_{h+1}(s',\overline{y},a')}}\bigg]\\
       &\qquad+
       \underbrace{\mathop{\E}\limits_{s'\sim p_h(\cdot|s,a)}
       \Big[\big|Q^\psi_h(s,y,a;\fp,\fr)-\fr_h(s,y,a)-Q^\psi_{h+1}(s',\overline{y},a';\fp,\fr)\big|^2\Big]}_{
         \popblue{\eqqcolon \V_{s'\sim p_h(\cdot|s,a)}[Q^{\psi}_{h+1}(s',\overline{y},a';\fp,\fr)]
         =\V_{s'\sim p_h(\cdot|s,a)}[V^{\psi}_{h+1}(s',\overline{y};\fp,\fr)]}
       }\\
       &=
       \mathop{\E}\limits_{s'\sim p_h(\cdot|s,a)}
       [\Sigma^\psi_{h+1}(s',\overline{y},a')]+\V_{s'\sim p_h(\cdot|s,a)}[V^{\psi}_{h+1}(s',\overline{y};\fp,\fr)],
   \end{align*}
   at (1) we add and subtract a term, at (2) we bring out the non-random reward
   received at $h$, at (3) we compute the square and use the linearity of
   expectation, at (4) we use the fact that $\E_{\fp,\fr,\psi}\big[ \sum_{h'=h+1}^H
   \fr_{h'}(s_{h'},y_{h'},a_{h'}) -Q^\psi_{h+1}(s',\overline{y},a';\fp,\fr)\,\big|\,s_{h+1}=s'
   \big]=Q^\psi_{h+1}(s',\overline{y},a';\fp,\fr)-Q^\psi_{h+1}(s',\overline{y},a';\fp,\fr)=0$ because of linearity of
   expectation. 
 \end{proof}
 
 \begin{lemma}\label{lemma: lemma 8}
   Let $\psi$ be any policy, and let $\fp$ be any
 transition model associated to an arbitrary inner dynamics $p$. Then, for all
 $\hsya$, it holds that:
 \begin{align*}
   \bigg|\mathop{\E}\limits_{\fp,\fr,\psi}\bigg[\sum\limits_{h'=h}^H
   \sqrt{\V_{s'\sim p_{h'}(\cdot|s_{h'},a_{h'})}[V^\psi_{h'+1}(s',y_{h'+1};\fp,\fr)]}
   \,\bigg|\,s_h=s,y_h=y,a_h=a\bigg]\bigg| \le \sqrt{H^3}.
 \end{align*}
 \end{lemma}
 \begin{proof}
   For all $\hsya$, we can write (note that this derivation is
   independent of $\fp,p$, so we might use even $\widehat{\fp},\widehat{p}$ in the
   proof):
   \begin{align*}
     &\Big|\mathop{\E}\limits_{\fp,\fr,\psi}\Big[\sum\limits_{h'=h}^H
   \sqrt{\V_{s'\sim p_{h'}(\cdot|s_{h'},a_{h'})}[V^\psi_{h'+1}(s',y_{h'+1};\fp,\fr)]}
   \,|\,s_h=s,y_h=y,a_h=a\Big]\Big|\\
   &\qquad\markref{(1)}{\le}
   \Big|\mathop{\E}\limits_{\fp,\fr,\psi}\Big[\sqrt{\popblue{H\sum\limits_{h'=h}^H}
   \V_{s'\sim p_{h'}(\cdot|s_{h'},a_{h'})}[V^\psi_{h'+1}(s',y_{h'+1};\fp,\fr)]}
   \,|\,s_h=s,y_h=y,a_h=a\Big]\Big|\\
   &\qquad\markref{(2)}{\le}
   \popblue{\sqrt{H}}\sqrt{\popblue{\mathop{\E}\limits_{\fp,\fr,\psi}}\Big[\sum\limits_{h'=h}^H
   \V_{s'\sim p_{h'}(\cdot|s_{h'},a_{h'})}[V^\psi_{h'+1}(s',y_{h'+1};\fp,\fr)]
   \,|\,s_h=s,y_h=y,a_h=a\Big]}\\
   &\qquad\markref{(3)}{=}
   \sqrt{H}\bigg(\mathop{\E}\limits_{\fp,\fr,\psi}\Big[\sum\limits_{h'=h}^H
   \popblue{\Sigma^\psi_{h'}(s_{h'},y_{h'},a_{h'})-\E_{s'\sim p_{h'}(\cdot|s_{h'},a_{h'})}
   \big[\Sigma^\psi_{h'+1}(s',y_{h'+1},\psi_{h'+1}(s',y_{h'+1}))\big]}\\
   &\qquad\qquad|\,s_h=s,y_h=y,a_h=a\Big]\bigg)^{1/2}\\
   &\qquad=
   \sqrt{H}\sqrt{\mathop{\E}\limits_{\fp,\fr,\psi}\Big[\sum\limits_{h'=h}^H
   \Sigma^\psi_{h'}(s_{h'},y_{h'},a_{h'})-\popblue{\Sigma^\psi_{h'+1}(s_{h'+1},y_{h'+1},a_{h'+1})}
   \,|\,s_h=s,y_h=y,a_h=a\Big]}\\
   &\qquad\markref{(4)}{=}
   \sqrt{H}\sqrt{\mathop{\E}\limits_{\fp,\fr,\psi}\Big[
   \popblue{\Sigma^\psi_{h}(s_h,y_h,a_h)-\underbrace{\Sigma^\psi_{H+1}(s_{H+1},y_{H+1},a_{H+1})}_{=0}}
   \,|\,s_h=s,y_h=y,a_h=a\Big]}\\
   &\qquad=
   \sqrt{H}\sqrt{
   \Sigma^\psi_{h}(s,y,a)}\\
   &\qquad\markref{(5)}{\le}
   \sqrt{H}\sqrt{
   H^2}\\
   &\qquad=
   \sqrt{H^3},
   \end{align*}
   where at (1) we have applied the Cauchy-Schwarz's inequality, at (2) we have
   applied Jensen's inequality, at (3) we have applied Lemma \ref{lemma: lemma
   7}, at (4) we have used telescoping, and at (5) we have bounded
   $\Sigma^\psi_{h}(s,y,a)\le H^2$ for all $\hsya$.
 \end{proof}
 
 \subsubsection{Lemmas on the Optimal Performance for Multiple Utilities}
 
 To prove the following results, we will make use of the notation introduced in
 the previous section.
 
 \begin{lemma}\label{lemma: bound J star all utilities}
   Let $\epsilon,\delta\in(0,1)$. It suffices to execute \caty with:
   \begin{align*}
     &\tau\le\widetilde{\cO}\Big(\frac{SAH^5}{\epsilon^2}
     \Big(S+\log\frac{SAH}{\delta}\Big)\Big),
   \end{align*}
   to obtain $\sup_{U\in\fU_L}\big|J^*(U;p,r)-\widehat{J}^*(U)\big|\le HL\epsilon_0+\epsilon$
   w.p. $1-\delta$.
 \end{lemma}
 \begin{proof}
   Similarly to the proof of Lemma \ref*{lemma: bound J star all utilities}, we
   can write:
   \begin{align*}
     \sup\limits_{U\in\fU_L}&|J^*(U;p,r)-\widehat{J}^*(U)|\\
     &=\sup\limits_{U\in\fU_L}|J^*(U;p,r)-\widehat{J}^*(U)\popblue{\pm
     J^*(\fp,\fr)}|\\
     &\le \popblue{\sup\limits_{U\in\fU_L}|}J^*(U;p,r)-J^*(\fp,\fr)\popblue{|}
     +\popblue{\sup\limits_{U\in\fU_L}|}J^*(\fp,\fr)-\widehat{J}^*(U)\popblue{|}\\
     &=\sup\limits_{U\in\fU_L}|J^*(U;p,r)-J^*(\fp,\fr)|+
     \sup\limits_{U\in\fU_L}|J^*(\fp,\fr)-\popblue{J^*(\widehat{\fp},\fr)}|\\
     &\le \popblue{HL\epsilon_0}+\sup\limits_{U\in\fU_L}|J^*(\fp,\fr)-J^*(\widehat{\fp},\fr)|\\
     &\markref{(1)}{\le} HL\epsilon_0+\popblue{H^2 \sqrt{\frac{2}{n}\Big(\log\frac{SAH}{\delta}
     +(S-1)\log\big(e(1+n/(S-1))\big)\Big)}}\\
     &\le HL\epsilon_0+\popblue{\epsilon},
   \end{align*}
   where at (1) we have applied the formula in Lemma \ref{lemma: for multiple
   utilities}.
   
   By enforcing such quantity to be smaller than $\epsilon$, we get:
   \begin{align*}
     &H^2 \sqrt{\frac{2}{n}\Big(\log\frac{SAH}{\delta}
     +(S-1)\log\big(e(1+n/(S-1))\big)\Big)}\le\\
     &\qquad\frac{H^2\sqrt{\log\big(e(1+n/(S-1))\big)}}{\sqrt{n}}
     \sqrt{2\Big(\log\frac{SAH}{\delta}+(S-1)\Big)}\le\epsilon\\
     &\qquad\iff n \ge 2\frac{H^4}{\epsilon^2}\Big(\log\frac{SAH}{\delta}+(S-1)\Big)
     \log\big(e(1+n/(S-1))\big).
   \end{align*}
   By summing over all $(s,a,h)\in\SAH$,
   and
   by applying Lemma J.3 of \citet{lazzati2024offline},
   we obtain that:
   \begin{align*}
     \tau=SAHn\ge\widetilde{\cO}\bigg(
     \frac{SAH^5}{\epsilon^2}\Big(\log\frac{SAH}{\delta}+S\Big)
     \bigg).
   \end{align*}
 \end{proof}
 
 \begin{lemma}\label{lemma: for multiple utilities} For any $\delta\in(0,1)$, for
   all utility functions $U\in\fU_L$ at the same time, we have:
   \begin{align*}
     |J^*_h(\fp,\fr)-J^*_h(\widehat{\fp},\fr)|\le
     H^2 \sqrt{\frac{2}{n}\Big(\log\frac{SAH}{\delta}
     +(S-1)\log\big(e(1+n/(S-1))\big)\Big)},
   \end{align*}
   w.p. at least $1-\delta$.
 \end{lemma}
 \begin{proof}
   Let us denote by $\cE$ the event defined as:
 \begin{align*}
   \cE\coloneqq\bigg\{&
   \forall n\in\Nat,\,\forall \hsya:\\
   &n\text{KL}\big(\widehat{p}_h(\cdot|s,a)\|p_h(\cdot|s,a)\big)\le \log\frac{SAH}{\delta}
   +(S-1)\log\big(e(1+n/(S-1))\big)
   \bigg\}.
 \end{align*}
 We can write:
 \begin{align*}
   \P(\cE^\complement)&=\P\bigg(\exists n\in\Nat,\,
     \exists \hsya:\\
   &\qquad
   n\text{KL}\big(\widehat{p}_h(\cdot|s,a)\|p_h(\cdot|s,a)\big)> \log\frac{SAH}{\delta}
   +(S-1)\log\big(e(1+n/(S-1))\big)
   \bigg)\\
   &\markref{(1)}{=}\P\bigg(\exists n\in\Nat,\,
     \popblue{\exists (s,a,h)\in\SAH}:\\
   &\qquad
   n\text{KL}\big(\widehat{p}_h(\cdot|s,a)\|p_h(\cdot|s,a)\big)> \log\frac{SAH}{\delta}
   +(S-1)\log\big(e(1+n/(S-1))\big)
   \bigg)\\
   &\markref{(2)}{\le}
   \sum\limits_{(s,a,h)\in\SAH}
   \P\bigg(
     \exists n\in\Nat,\,
     n\text{KL}\big(\widehat{p}_h(\cdot|s,a)\|p_h(\cdot|s,a)\big)>\\
     &\qquad\log\frac{SAH}{\delta}
   +(S-1)\log\big(e(1+n/(S-1))\big)
   \bigg)\\
   &\markref{(3)}{\le}
   \sum\limits_{(s,a,h)\in\SAH}\frac{\delta}{SAH}\\
   &\le\delta,
 \end{align*}
 where at (1) we realize that there is no dependence on variable $y$, thus we can
 drop it,\footnote{Therefore, differently from the event for a single utility,
 now there is no dependence on $d$ in the bound. Intuitively, $d$ appeared in the
 case of a single utility because we had to apply Hoeffding's inequality $d$
 times, because we had, potentially, $d$ different value functions (as many as
 the states). Since now we provide the bound for all the possible value functions
 (1-norm bound), then the dependence on $d$ disappears.}
 at (2) we have applied a union bound over all triples $(s,a,h)\in\SAH$,
 and at (3) we have applied Proposition 1 of
 \citet{jonsson2020planningmarkovdecisionprocesses}.
 
 Next, for all utilities $U\in\fU_L$ at the same time, for all the tuples
 $\hsy$, we can write:
 \begin{align*}
   &|V_h^*(s,y;\fp,\fr)-V^{*}_h(s,y;\widehat{\fp},\fr)|\\
   &\qquad\markref{(4)}{\le}
   \mathop{\E}\limits_{\widehat{\fp},\fr,\overline{\psi}}\bigg[
     \sum\limits_{h'=h}^H \bigg|\sum\limits_{s'\in\cS}\Big(
       p_{h'}(s'|s_{h'},a_{h'})-\widehat{p}_{h'}(s'|s_{h'},a_{h'})\Big)
       V_{h'+1}^{\psi^*}(s',y_{h'+1};\fp,\fr)\bigg|\\
       &\qquad\qquad\Big|\,s_h=s,y_h=y,a_h=a
       \bigg]\\
   &\qquad\markref{(5)}{\le} \popblue{H} \mathop{\E}\limits_{\widehat{\fp},\fr,\overline{\psi}}\bigg[
     \sum\limits_{h'=h}^H \popblue{\|}p_{h'}(\cdot|s_{h'},a_{h'})-
     \widehat{p}_{h'}(\cdot|s_{h'},a_{h'})\popblue{\|_1}
     \,\Big|\,s_h=s,y_h=y,a_h=a
       \bigg]\\
   &\qquad\markref{(6)}{\le}H \mathop{\E}\limits_{\widehat{\fp},\fr,\overline{\psi}}\bigg[
     \sum\limits_{h'=h}^H \popblue{\sqrt{2\text{KL}(\widehat{p}_{h'}
     (\cdot|s_{h'},a_{h'})\|p_{h'}(\cdot|s_{h'},a_{h'}))}}
     \,\Big|\,s_h=s,y_h=y,a_h=a
       \bigg]\\
   &\qquad\markref{(7)}{\le}H \mathop{\E}\limits_{\widehat{\fp},\fr,\overline{\psi}}\bigg[
     \sum\limits_{h'=h}^H \sqrt{\frac{2}{n}\Big(\log\frac{SAH}{\delta}
     +(S-1)\log\big(e(1+n/(S-1))\big)\Big)}
     \\
     &\qquad\qquad\Big|\,s_h=s,y_h=y,a_h=a
       \bigg]\\
   &\qquad=\popblue{H^2} \sqrt{\frac{2}{n}\Big(\log\frac{SAH}{\delta}
     +(S-1)\log\big(e(1+n/(S-1))\big)\Big)},
 \end{align*}
 where at (4) we apply the formula derived in the proof of Lemma \ref*{lemma:
 lemma 4} and triangle inequality, at (5) we have upper bounded with the 1-norm,
 defined as $\|f\|_1\coloneqq\sum_x |f(x)|$, at (6) we have applied Pinsker's
 inequality, at (7) we assume that concentration event $\cE$ holds.
 
 We remark that the guarantee provided by this theorem holds not only for
 $L$-Lipschitz utilities, but for all functions with the same dimensionality
 (since it is a bound in 1-norm).
 \end{proof}
 
 \subsection{Analysis of \tractor}
 \label{apx: analysis tractor}
 
 \thrupperboundtractor*
 \begin{proof}
   The proof draws inspiration from those of
   \citet{syed2007game} and \citet{schlaginhaufen2024transferability}.
 
   Given any distribution $\eta$ supported on $\cY$, and given any two utilities
   $U\in\underline{\fU}_L,\overline{U}\in\overline{\underline{\fU}}_L$ (where $U$
   is a function on $[0,H]$ and $\overline{U}$ is a vector on $\cY$), we will
   abuse notation and write both $U^\intercal \eta$ and $\overline{U}^\intercal
   \eta$, with obvious meaning. 
   
   Moreover, for $L>0$, we define operator $\fC_L:\overline{\underline{\fU}}_L\to
   2^{\underline{\fU}_L}$ (where $2^\cX$ denotes the power set of set $\cX$)
   that, given vector $\overline{U}\in\overline{\underline{\fU}}_L$, returns the
   set $\fC_L(\overline{U})\coloneqq \{U\in\underline{\fU}_L\,|\,\forall
   y\in\cY:\,U(y)=\overline{U}(y)\}$.
 
 First of all, we observe that the guarantee provided by the theorem follows
 directly by the following expression:
    \begin{align*} 
     \mathop{\mathbb{P}}\limits_{\cM^1,\cM^2,\dotsc,\cM^N}\Big(
       \sup\limits_{U\in\fC_L(\widehat{U})}\sum\limits_{i\in\dsb{N}}
       \overline{\cC}_{p^i,r^i,\pi^{E,i}}(U)\le\epsilon\Big)\ge 1-\delta,
  \end{align*}
  where $\mathbb{P}_{\cM^1,\cM^2,\dotsc,\cM^N}$ denotes the joint
  probability distribution obtained by the $N$ MDPs $\{\cM^i\}_i$.
   
   Let us denote by $\widehat{U}\coloneqq(\sum_{t=0}^{T-1}\overline{U}_t)/T$ the output of
   \tractor. Note that $\widehat{U}\in\overline{\underline{\fU}}_L$. We can
   write:
   \begin{align*}
     &\sup\limits_{U\in\fC_L(\widehat{U})}
     \sum\limits_{i\in\dsb{N}}\overline{\cC}_{p^i,r^i,\pi^{E,i}}
     (U)\\
           &\qquad\markref{(1)}{=}
           \sup\limits_{U\in\fC_L(\widehat{U})}
           \sum\limits_{i\in\dsb{N}}
           \bigg(J^*(U;p^i,r^i)-J^{\pi^{E,i}}(U;p^i,r^i)
           \popblue{\pm \widehat{U}^{\intercal}\widehat{\eta}^{E,i}}\bigg)\\    
           &\qquad\markref{(2)}{\le}
           \sup\limits_{U\in\fC_L(\widehat{U})}
           \sum\limits_{i\in\dsb{N}}
           \bigg(J^*(U;p^i,r^i)-\popblue{\widehat{U}^{\intercal}\widehat{\eta}^{E,i}}\bigg)
           +\popblue{\epsilon_1}\\ 
           &\qquad\markref{(3)}{=}
           \sup\limits_{U\in\fC_L(\widehat{U})}
           \sum\limits_{i\in\dsb{N}}\bigg(
 \popblue{\max\limits_{\eta\in \fD_i}U^\intercal \eta}
 -\widehat{U}^\intercal \widehat{\eta}^{E,i}\bigg)+\epsilon_1
           \\
           &\qquad\markref{(4)}{=}
           \popblue{\sup\limits_{\substack{
             U_0\in\fC_L(\overline{U}_0),\\
           \dotsc,\\U_{T-1}\in\fC_L(\overline{U}_{T-1})
             }}}
           \popblue{\frac{1}{T}}\sum\limits_{i\in\dsb{N}}
 \max\limits_{\eta\in \fD_i}
 \popblue{\sum\limits_{t=0}^{T-1}}
 \bigg(
   \popblue{U_t}^\intercal \eta
 -\popblue{\overline{U}_t}^\intercal \widehat{\eta}^{E,i}\bigg)+\epsilon_1
           \\
       &\qquad\markref{(5)}{\le}
           \frac{1}{T}\popblue{\sum\limits_{t=0}^{T-1}}
           \popblue{\sup\limits_{
             U_t\in\fC_L(\overline{U}_t)}}
           \sum\limits_{i\in\dsb{N}}
 \bigg(
   \max\limits_{\eta\in \fD_i}U_t^\intercal\eta
   \popblue{\pm \overline{U}_t^\intercal \widehat{\eta}^i_t}
 -\overline{U}_t^\intercal \widehat{\eta}^{E,i}\bigg)+\epsilon_1
           \\
           &\qquad\markref{(6)}{\le}
           \frac{1}{T}\sum\limits_{t=0}^{T-1}
           \sum\limits_{i\in\dsb{N}}
   \overline{U}_t^\intercal \Big(\widehat{\eta}^i_t
 - \widehat{\eta}^{E,i}\Big)\popblue{\pm
 \frac{1}{T}\min\limits_{\overline{U}\in\overline{\underline{\fU}}_L}\sum\limits_{t=0}^{T-1}
           \sum\limits_{i\in\dsb{N}}
           \overline{U}^\intercal \Big(\widehat{\eta}^i_t
 - \widehat{\eta}^{E,i}\Big)}+\epsilon_1+\popblue{\epsilon_2}
           \\
           &\qquad\markref{(7)}{\le}
 \frac{1}{T}\min\limits_{\overline{U}\in\overline{\underline{\fU}}_L}
 \sum\limits_{t=0}^{T-1}
           \sum\limits_{i\in\dsb{N}}
   \overline{U}^\intercal \Big(\widehat{\eta}^i_t
 - \widehat{\eta}^{E,i}\Big)+\epsilon_1+\epsilon_2+
 \underbrace{\popblue{\frac{2HN\sqrt{H/\epsilon_0}}{\sqrt{T}}}}_{\eqqcolon \epsilon_3}
           \\
 &\qquad\markref{(8)}{\le}
 \frac{1}{T}\sum\limits_{t=0}^{T-1}
 \sum\limits_{i\in\dsb{N}}
 \popblue{\overline{U}^{E,\intercal}} \Big(\widehat{\eta}^i_t
 - \widehat{\eta}^{E,i}\Big) \popblue{\pm U^{E,\intercal}
 \eta^{p^i,r^i,\pi^{E,i}}}+\epsilon_1+\epsilon_2+\epsilon_3
 \\
 &\qquad\markref{(9)}{\le}
 \frac{1}{T}\sum\limits_{t=0}^{T-1}
           \sum\limits_{i\in\dsb{N}}
           \overline{U}^{E,\intercal} \widehat{\eta}^i_t\popblue{\pm
           U^{E,\intercal}
   \eta^{p^i,r^i,\overline{\pi}^i_t}}
 - U^{E,\intercal}
 \eta^{p^i,r^i,\pi^{E,i}}+\popblue{2}\epsilon_1+\epsilon_2+\epsilon_3
 \\
 &\qquad\markref{(10)}{\le}
 \frac{1}{T}\sum\limits_{t=0}^{T-1}
           \sum\limits_{i\in\dsb{N}}
   \underbrace{U^{E,\intercal}
   \Big(\eta^{p^i,r^i,\overline{\pi}^i_t}-
 \eta^{p^i,r^i,\pi^{E,i}}\Big)}_{\popblue{\le0}}+2\epsilon_1+\epsilon_2+\epsilon_3
 +\popblue{\epsilon_4}
 \\
 &\qquad\markref{(11)}{\le}
 2\epsilon_1+\epsilon_2+\epsilon_3+\epsilon_4,
 \\
         \end{align*}
         where
         at (1) we apply the definition of (non)compatibility, at (2) we first
         upper bound the supremum of a sum with the sum of the supremum, and
         then we apply Lemma \ref{lemma: bound JE for tractor} w.p. $\delta/3$,
         and denote $\epsilon_1\coloneqq
         NL\sqrt{2H\epsilon_0}+\sum\limits_{i\in\dsb{N}}c H
         \sqrt{\frac{H\log\frac{NH\tau^{E,i}}
         {\delta\epsilon_0}}{\epsilon_0\tau^{E,i}}}$, at (3) we denote by
         $\fD_i$ the set of possible return distributions in environment $i$, at
         (4) we use the definition of $\widehat{U}$, and realize that all
         functions $U\in\fC_L(\widehat{U})$ can be constructed based on $T$
         functions $U_0\in\fC_L(\overline{U}_0),
         \dotsc,U_{T-1}\in\fC_L(\overline{U}_{T-1})$. At (5) we upper bound the
         maximum of the sum with the sum of maxima, and exchange the two
         summations, and we add and subtract the dot product between the
         (discretized) utility $U_t$ and the estimate of the return distribution
         computed at Line \ref{line: line compute etati tractor}; moreover, we
         bring the sup inside the summation. At (6) we upper bound the supremum
         of the sum with the sum of the supremum, and we apply Lemma \ref{lemma:
         bound J star for tractor} w.p. $\delta/3$, defining
         $\epsilon_2\coloneqq cNH^2
         \sqrt{\frac{1}{n}\Big(\log\frac{SAHN}{\delta}
         +(S-1)\log\big(e(1+n/(S-1))\big)\Big)} +NHL\epsilon_0+
         c'HN\sqrt{\frac{\log\frac{NT}{\delta}}{ K}}$,
         and we add and subtract a term, at (7) we apply Theorem H.2 from
         \citet{schlaginhaufen2024transferability} since set
         $\overline{\underline{\fU}}_L$ is closed and convex, where $D\coloneqq
         \max_{\overline{U},\overline{U}'\in \overline{\underline{\fU}}_L}
         \|\overline{U}-\overline{U}'\|_2=\sqrt{d-2}H =
         \sqrt{\floor{H/\epsilon_0}-1}H \le H\sqrt{H/\epsilon_0}$ (recall that we
         consider increasing and not strictly-increasing utilities),\footnote{The
         maximum is attained by discretized utilities
         $\overline{U},\overline{U}'$ that assign, respectively,
         $\overline{U}(y)=0$ and $\overline{U}'(y)=H$ to all the
         $y\in\cY\setminus\{y_1,y_d\}$.} and
         $\max_{\overline{U}\in\overline{\underline{\fU}}_L} \|\nabla
         \sum_{i\in\dsb{N}}\overline{U}^\intercal
         (\widehat{\eta}^i_t-\widehat{\eta}^{E,i})\|_2=\|
         \sum_{i\in\dsb{N}}\widehat{\eta}^i_t-\widehat{\eta}^{E,i}\|_2\le
         \sum_{i\in\dsb{N}}
         \|\widehat{\eta}^i_t\|_1+\|\widehat{\eta}^{E,i}\|_1=2N\eqqcolon G$
         (because $\widehat{\eta}^i_t$ and $\widehat{\eta}^{E,i}$ are probability
         distributions), with learning rate
         $\alpha=D/(G\sqrt{T})=H\sqrt{d-2}/(2N\sqrt{T})=
         \sqrt{\floor{H/\epsilon_0}-1}H/(2N\sqrt{T})$,
         at (8) we upper bound the minimum over utilities with a specific choice
         of utility, $\overline{U}^E$, and we add and subtract a term;  note that
         $\overline{U}^E\in\overline{\underline{\fU}}_L$ corresponds to the
         expert's utility $U^E\in \underline{\fU}_L$ (by hypothesis), i.e., for
         all $y\in\cY:\,\overline{U}^E(y)=U^E(y)$. Note that, by hypothesis,
         $U^E$ makes all the expert policies optimal, i.e., $\forall
         i\in\dsb{N}:\, U^{E,\intercal}\eta^{p^i,r^i,\pi^{E,i}} =\sup_{\pi}
         U^{E,\intercal}\eta^{p^i,r^i,\pi}$.
         At (9) we note that, under the good event of Lemma \ref{lemma: bound JE
         for tractor}, we can provide an upper bound using the term in Lemma
         \ref{lemma: bound JE for tractor} (since $U^E\in\underline{\fU}_L$); in
         addition, we sum and subtract a term that depends on some policy
         $\overline{\pi}^i_t$, whose existence is guaranteed by Lemma \ref{lemma:
         bound arbitrary difference of return distributions for tractor}, which
         we apply at the next step.
         At (10) we apply Lemma \ref{lemma: bound arbitrary difference of return
         distributions for tractor} w.p. $\delta/3$, and we define as
         $\epsilon_4$ the upper bound times $N$.
         Finally, at (11) we use the hypothesis that utility $U^E$
         makes the expert policy optimal in all environments.
 
         We want that $2\epsilon_1+\epsilon_2+\epsilon_3+\epsilon_4\le\epsilon$.
         We can rewrite the sum as:
         \begin{align*}
           &2\epsilon_1+\epsilon_2+\epsilon_3+\epsilon_4\\
           &\qquad=\Big(
             2NL\sqrt{2H\epsilon_0}+\frac{3}{2}LNH\epsilon_0\Big)+
             c\frac{HN\sqrt{H}}{\sqrt{\epsilon_0T}}\\
           &\qquad\qquad+
             c'\sum\limits_{i\in\dsb{N}}H
             \sqrt{\frac{H\log\frac{NH\tau^{E,i}}
             {\delta\epsilon_0}}{\epsilon_0\tau^{E,i}}}
         +c''NH\sqrt{\frac{\log\frac{NT}{\delta}}{K}} \\
           &\qquad\qquad+c'''NH^2
           \sqrt{\frac{1}{n}\Big(\log\frac{SAHN}{\delta}
           +(S-1)\log\big(e(1+n/(S-1))\big)\Big)}.
         \end{align*}
         By imposing each term smaller than $\epsilon/5$, we find that it
         suffices that
         \begin{align*}
           \begin{cases}
             \epsilon_0=\frac{\epsilon^2}{80 N^2L^2H}\\
             T\ge \cO\Big(\frac{N^2H^3}{\epsilon_0\epsilon^2}\Big)\ge
             \cO\Big(\frac{N^4H^4L^2}{\epsilon^4}\Big)\\
             \tau^{E,i}\ge \widetilde{\cO}\Big(
               \frac{H^3N^2\log\frac{NH}{\delta\epsilon_0}}{\epsilon_0\epsilon^2}  
             \Big)\ge
             \widetilde{\cO}\Big(
               \frac{H^4N^4L^2\log\frac{NHL}{\delta\epsilon}}{\epsilon^4}  
             \Big)\qquad\forall i\in\dsb{N}\\
             K\ge\widetilde{\cO}\Big(
               \frac{N^2H^2\log\frac{NT}{\delta}}{\epsilon^2}
             \Big)\ge 
             \widetilde{\cO}\Big(
               \frac{N^2H^2\log\frac{NHL}{\delta\epsilon}}{\epsilon^2}
             \Big)\\
             \tau^i\ge\widetilde{\cO}\Big(
               \frac{N^2SAH^5}{\epsilon^2}\Big(
               S+\log\frac{SAHN}{\delta}  
               \Big)\qquad\forall i\in\dsb{N}
           \end{cases},
         \end{align*}
         where we have used that $\tau^i=SAHn$ for all $i\in\dsb{N}$, and also
         used Lemma J.3 of \citet{lazzati2024offline}.
 
         The statement of the theorem follows through the application of a union bound.
 \end{proof}
 
 \begin{lemma}
   \label{lemma: bound JE for tractor}
   Let $\delta\in(0,1)$.
   Then, it holds that, w.p. at least $1-\delta$:
   \begin{align*}
     \sup\limits_{U\in\underline{\fU}_L}
     \sum\limits_{i\in\dsb{N}}\bigg|
       U^{\intercal}\widehat{\eta}^{E,i}-
       J^{\pi^{E,i}}(U;p^i,r^i)\bigg|\le 
     NL\sqrt{2H\epsilon_0}+\sum\limits_{i\in\dsb{N}}c H\sqrt{\frac{H\log\frac{NH\tau^{E,i}}
     {\delta\epsilon_0}}{\epsilon_0\tau^{E,i}}},
   \end{align*}
   where $c$ is some positive constant.
 \end{lemma}
 \begin{proof}
   We can make the same derivation as in the proof of Theorem \ref{thr: caty
   upper bound 1 u} to upper bound the objective with the sum of two terms, which
   can then be bounded using Lemma \ref{lemma: bound w1 etae} and the expression
   (Eq. \eqref{eq: bound difference etaE for tractor}) obtained in the proof of
   Lemma \ref{lemma: bound estimation error tante U} w.p. $\delta/N$:
   \begin{align*}
     &\sup\limits_{U\in\underline{\fU}_L}
     \sum\limits_{i\in\dsb{N}}\bigg|
       U^{\intercal}\widehat{\eta}^{E,i}-
       J^{\pi^{E,i}}(U;p^i,r^i)\bigg|\\
     &\qquad\le
     L\sum\limits_{i\in\dsb{N}}w_1(\eta^{p^i,r^i,\pi^{E,i}},
     \text{Proj}_{\cC}(\eta^{p^i,r^i,\pi^{E,i}}))\\
     &\qquad\qquad+
     \sum\limits_{i\in\dsb{N}}
     \popblue{\sup\limits_{\overline{U}'\in[0,H]^d}}
     \big|\mathop{\E}\limits_{G\sim\text{Proj}_{\cC}(\eta^{p^i,r^i,\pi^{E,i}})}[\overline{U}'(G)]-
     \mathop{\E}\limits_{G\sim\widehat{\eta}^{E,i}}[\overline{U}'(G)]\big|\\
     &\qquad \le L N\sqrt{2H\epsilon_0}
     +\sum\limits_{i\in\dsb{N}}c H\sqrt{\frac{H\log\frac{NH\tau^{E,i}}
     {\delta\epsilon_0}}{\epsilon_0\tau^{E,i}}}.
   \end{align*}
   The result follows through the application of the union bound.
 \end{proof}
 
 \begin{lemma}\label{lemma: bound J star for tractor}
   Let $\delta\in(0,1)$. With
   probability at least $1-\delta$, for all $t\in\{0,1,\dotsc,T-1\}$, for all
   $i\in\dsb{N}$, it holds that:
   \begin{align*}
     \sup\limits_{
             U_t\in\fC_L(\overline{U}_t)}
     \max\limits_{\eta\in \fD_i}U_t^\intercal\eta
   - \overline{U}_t^\intercal \widehat{\eta}^i_t&\le 
   cH^2 \sqrt{\frac{1}{n}\Big(\log\frac{SAHN}{\delta}
   +(S-1)\log\big(e(1+n/(S-1))\big)\Big)}\\
   &\qquad+HL\epsilon_0+
   c'H\sqrt{\frac{\log\frac{NT}{\delta}}{ K}},
   \end{align*}
   where $c,c'$ are some positive constants.
 \end{lemma}
 \begin{proof}
   We use the notation in Section \ref{sec: online utility learning}. In
   particular, 
   let
   policy $\widehat{\pi}^{*,i}_t$ be the optimal policy in the RS-MDP
   $\widehat{\cM}^i_{\overline{U}_t}\coloneqq\tuple{\cS^i,\cA^i,H,s_0^i,
   \widehat{p}^i, \overline{r}^i,\overline{U}_t}$, i.e.:
 \begin{align*}
   J^{\widehat{\pi}^{*,i}_t}(\overline{U}_t;\widehat{p}^i,\overline{r}^i)
   = J^*(\overline{U}_t;\widehat{p}^i,\overline{r}^i)
   = J^*(U_t;\widehat{p}^i,\overline{r}^i),
 \end{align*}
 where the last passage holds trivially for all $U_t\in\fC_L(\overline{U}_t)$
 (because there is no evaluation of utility outside $\cY$).

   Thus,
   for all $t\in\{0,1,\dotsc,T-1\}$, we have:
   \begin{align*}
     &\sup\limits_{
       U_t\in\fC_L(\overline{U}_t)}
 \max\limits_{\eta\in \fD_i}U_t^\intercal\eta
 - \overline{U}_t^\intercal \widehat{\eta}^i_t\popblue{\pm 
 J^*(U_t;\widehat{p}^i,\overline{r}^i)}\\
   &\qquad\markref{(1)}{\le}\sup\limits_{
       U_t\in\fC_L(\overline{U}_t)}
   \Big|J^*(U_t;p^i,r^i)-
   J^*(U_t;\widehat{p}^i,\overline{r}^i)\Big|
   +
   \Big|\popblue{\overline{U}_t}^\intercal\Big(\widehat{\eta}^i_t-
   \popblue{\eta^{\widehat{p}^i,\overline{r}^i,\widehat{\pi}^{*,i}_t}}\Big)\Big|\\
   &\qquad \markref{(2)}{\le}
   HL\epsilon_0+cH^2 \sqrt{\frac{1}{n}\Big(\log\frac{SAHN}{\delta}
     +(S-1)\log\big(e(1+n/(S-1))\big)\Big)}\\
     &\qquad\qquad+
     \Big|\overline{U}_t^\intercal \Big(\widehat{\eta}^i_t-
     \eta^{\widehat{p}^i,\overline{r}^i,\widehat{\pi}^{*,i}_t}\Big)\Big|\\
   &\qquad \markref{(3)}{\le}
   HL\epsilon_0+cH^2 \sqrt{\frac{1}{n}\Big(\log\frac{SAHN}{\delta}
     +(S-1)\log\big(e(1+n/(S-1))\big)\Big)}\\
     &\qquad\qquad+
     c'H\sqrt{\frac{\log\frac{NT}{\delta}}{ K}},
   \end{align*}
   where at (1) we have applied the triangle inequality, and realized that in the
  second term there is no dependence on the value of utility outside of $\cY$;
  moreover, we have used that $J^*(U_t;\widehat{p}^i,\overline{r}^i)=
  \overline{U}_t^\intercal
  \eta^{\widehat{p}^i,\overline{r}^i,\widehat{\pi}^{*,i}_t}$ by definition of
  policy $\widehat{\pi}^{*,i}_t$.
  At (2) we apply Lemma \ref{lemma: bound J star all utilities} (our
  $J^*(U_t;\widehat{p}^i,\overline{r}^i)$ has the same meaning of
  $\widehat{J}^*(U)$ in the lemma, and we upper bound $\sup_{
  U_t\in\fC_L(\overline{U}_t)}$ with $\sup_{ U\in\underline{\fU}_L}$) w.p.
  $\delta/(2N)$,\footnote{We remark that, in doing so, we can still apply
  Proposition 3 of \citet{wu2023risksensitive} inside the proof of Lemma
  \ref{lemma: bound J star all utilities} even though we consider
  \emph{increasing} utilities instead of \emph{strictly-increasing} utilities;
  indeed, it is trivial to observe that the proof of Proposition 3 of
  \citet{wu2023risksensitive} does not depend on such property.} and we keep the
  confidence bound explicit, and we upper bound $d\le H/\epsilon_0+1$, 
  and at (3) we observe that $\widehat{\eta}^i_t$ is the empirical estimate of
  distribution $\eta^{\widehat{p}^i,\overline{r}^i,\widehat{\pi}^{*,i}_t}$ (see
  Line \ref{line: line compute etati tractor}) obtained through the sampling of
  $K$ sample returns $G_1,G_2,\dotsc,G_K \overset{\text{i.i.d.}}{\sim}
  \eta^{\widehat{p}^i,\overline{r}^i,\widehat{\pi}^{*,i}_t}$. Indeed, note that
  the policy $\widehat{\psi}^{*,i}_t$, computed at Line \ref{line: planning
  tractor} and optimal for
  $\fE[\widehat{\cM}^i_{\overline{U}_t}]=\tuple{\{\cS^i\times\cY_h\}_h,\cA^i,H,s_0^i,
  \widehat{\fp}^i, \fr^i_t}$,\footnote{See Section \ref{sec: preliminaries} for
  the meaning of $\widehat{\fp}^i$ and $\fr^i_t$; we use $\cY_h$ for all $h$ in
  the state space instead of the sets of partial returns
  $\{\cG^{\widehat{p}^i,\overline{r}^i}_h\}_h$ in order to obtain policy
  $\widehat{\psi}^{*,i}_t$ supported on the entire $\cS\times\cY_h$ space, and to
  make it compliant with Algorithm \ref{alg: planning}} provides policy
  $\widehat{\pi}^{*,i}_t$ through the formula in Section \ref{sec:
  preliminaries}, thus Line \ref{line: rollout tractor} is actually simulating
  $\widehat{\pi}^{*,i}_t$ in MDP $\widehat{\cM}^i$. Therefore, we can apply
  Hoeffding's inequality (e.g., see Lemma \ref{lemma: bound estimation error 1
  U}) w.p. $\delta/(2TN)$.

   The result follows through the application of the union bound.
 
   We remark that in one case we use probability $\delta/(2N)$ (without $T$)
   while in the other we use $\delta/(2NT)$ (with $T$), because in the former we
   provide a guarantee for all possible utilities w.r.t. the optimal performance,
   thus all the $T$ steps are already included; instead, in the latter, we
   provide a guarantee for a single utility and for a single policy at a specific
   $t\in\{0,\dotsc,T-1\}$, thus we have to compute a union bound with $T$.
 \end{proof}
 
 \begin{lemma}
   \label{lemma: bound arbitrary difference of return distributions for tractor}
   Let $\delta\in(0,1)$. With
   probability at least $1-\delta$, for all $i\in\dsb{N}$ and $t\in\{0,\dotsc,T-1\}$, under
   the good event in Lemma \ref{lemma: bound J star for tractor}, there exists a
   policy $\overline{\pi}^i_t$ such that:
   \begin{align*}
     \overline{U}^{E,\intercal} \widehat{\eta}^i_t- U^{E,\intercal}
   \eta^{p^i,r^i,\overline{\pi}^i}&\le LH\epsilon_0/2+
   cH\sqrt{\frac{\log\frac{NT}{\delta}}{ K}}\\
   &\qquad+c'H^2\sqrt{\frac{1}{n}\Big(\log\frac{SAHN}{\delta}
   +(S-1)\log\big(e(1+n/(S-1))\big)\Big)},
   \end{align*}
   where $c,c'$ are positive constants.
 \end{lemma}
 \begin{proof}
   First, simply observe that $\widehat{\eta}^i_t$ is the empirical estimate (see
   Line \ref{line: line compute etati tractor}) of
   $\eta^{\widehat{p}^i,\overline{r}^i,\widehat{\pi}^{*,i}_t}$, thus,
   similarly to the proof of Lemma \ref{lemma: bound J star for tractor}, for all $i\in\dsb{N}$
   and $t\in\{0,1,\dotsc,T-1\}$, we can apply Hoeffding's inequality w.p. $\delta/(2TN)$:
   \begin{align*}
     \Big|\overline{U}^{E,\intercal} \Big(\widehat{\eta}^i_t
     -\eta^{\widehat{p}^i,\overline{r}^i,\widehat{\pi}^{*,i}_t}
     \Big)\Big|\le
     cH\sqrt{\frac{\log\frac{NT}{\delta}}{ K}}.
   \end{align*}
 
   Now, we compare distributions
   $\eta^{\widehat{p}^i,\overline{r}^i,\widehat{\pi}^{*,i}_t}$ and
   $\eta^{\popblue{p^i},\overline{r}^i,\widehat{\pi}^{*,i}_t}$. 
   Through straightforward passages, we can write:
   \begin{align*}
     &|U^{E,\intercal}\Big(\eta^{\widehat{p}^i,\overline{r}^i,\widehat{\pi}^{*,i}_t}-
     \eta^{p^i,\overline{r}^i,\widehat{\pi}^{*,i}_t}\Big)|\\
     &\qquad=
     |J^{\widehat{\pi}^{*,i}_t}(\overline{U}^E;\widehat{p}^i,\overline{r}^i)
     -
     J^{\widehat{\pi}^{*,i}_t}(\overline{U}^E;p^i,\overline{r}^i)|\\
     &\qquad=\Big|\sum\limits_{s'\in\cS}
     p^i_{1}(s'|s_0^i,\widehat{\pi}^{*,i}_{t,1}(s_0^i))
     V_{2}^{\widehat{\pi}^{*,i}_t}(s';p^i,\overline{r}^i)\\
     &\qquad\qquad-
     \sum\limits_{s'\in\cS}
     \widehat{p}^i_{1}(s'|s_0^i,\widehat{\pi}^{*,i}_{t,1}(s_0^i))
     V_{2}^{\widehat{\pi}^{*,i}_t}(s';\widehat{p}^i,\overline{r}^i)\Big|\\
   &\qquad\le \bigg|\sum\limits_{s'\in\cS}
     \Big(p^i_{1}(s'|s_0^i,\widehat{\pi}^{*,i}_{t,1}(s_0^i))-
     \widehat{p}^i_{1}(s'|s_0^i,\widehat{\pi}^{*,i}_{t,1}(s_0^i))\Big)
     V_{2}^{\widehat{\pi}^{*,i}_t}(s';p^i,\overline{r}^i)
   \bigg|\\
   &\qquad\qquad+\sum\limits_{s'\in\cS}
     \widehat{p}^i_{1}(s'|s_0^i,\widehat{\pi}^{*,i}_{t,1}(s_0^i,0))\Big|
       V_{2}^{\widehat{\pi}^{*,i}_t}(s';p^i,\overline{r}^i)-
       V_{2}^{\widehat{\pi}^{*,i}_t}(s';\widehat{p}^i,\overline{r}^i)
     \Big|\\
   &\qquad\le\dotsc\\
     &\qquad\le
     \mathop{\E}\limits_{\widehat{p}^i,\overline{r}^i,\widehat{\pi}^{*,i}_t}\bigg[
       \sum\limits_{h'=1}^H \bigg|\sum\limits_{s'\in\cS}\Big(
         p^i_{h'}(s'|s_{h'},a_{h'})-\widehat{p}^i_{h'}(s'|s_{h'},a_{h'})\Big)
         V_{h'+1}^{\widehat{\pi}^{*,i}_t}(s';p^i,\overline{r}^i)\bigg|\\
         &\qquad\qquad\bigg|\,s_1=s_0^i  
         \bigg]\\
 &\qquad\le
 \popblue{H}\mathop{\E}\limits_{\widehat{p}^i,\overline{r}^i,\widehat{\pi}^{*,i}_t}\bigg[
   \sum\limits_{h'=1}^H \popblue{\Big\|}
     p^i_{h'}(\cdot|s_{h'},a_{h'})-\widehat{p}^i_{h'}(\cdot|s_{h'},a_{h'})\popblue{\Big\|_1}
     \bigg|\,s_1=s_0^i  
     \bigg]\\
   &\qquad\le
   H\mathop{\E}\limits_{\widehat{p}^i,\overline{r}^i,\widehat{\pi}^{*,i}_t}\bigg[
     \sum\limits_{h'=1}^H \popblue{\sqrt{2\text{KL}(p^i_{h'}(\cdot|s_{h'},a_{h'})
     \|\widehat{p}^i_{h'}(\cdot|s_{h'},a_{h'}))}}
       \bigg|\,s_1=s_0^i  
       \bigg],
   \end{align*}
   where at the last passage we applied the Pinsker's inequality.
   Note that the previous derivation was possible as long as 
   as policy $\widehat{\pi}^{*,i}_t$ is defined over all the possible pairs
   state-cumulative reward $(s,y)\in\cS\times\cY_h$ for all $h\in\dsb{H}$.
   Since we construct it through policy $\widehat{\psi}^{*,i}_t$, obtained at
   Line \ref{line: planning tractor}, i.e., over the entire enlarged state space
   $\{\cS\times\cY_h\}_h$, then policy $\widehat{\pi}^{*,i}_t$ satsifies such
   property.
   Now, in the proof of Lemma \ref{lemma: bound J star for tractor} we used Lemma
   \ref{lemma: for multiple utilities}, in which event $\cE$ bounds the
   KL-divergence between transition models. Therefore, under the application of
   Lemma \ref{lemma: bound J star for tractor}, it holds that:
   \begin{align*}
     |U^{E,\intercal}\Big(\eta^{\widehat{p}^i,\overline{r}^i,\widehat{\pi}^{*,i}_t}-
     \eta^{p^i,\overline{r}^i,\widehat{\pi}^{*,i}_t}\Big)|
     \le H^2 \sqrt{\frac{2}{n}\Big(\log\frac{SAHN}{\delta}
     +(S-1)\log\big(e(1+n/(S-1))\big)\Big)},
   \end{align*}
   where $n$ is the number of samples takes at each $(s,a,h)\in\SAH$ in the
   $i\in\dsb{N}$ MDP.
 
   Therefore, we can finally write:
   \begin{align*}
     &\overline{U}^{E,\intercal} \widehat{\eta}^i_t- U^{E,\intercal}
     \eta^{p^i,r^i,\overline{\pi}^i}\popblue{\pm \overline{U}^{E,\intercal}
     \eta^{\widehat{p}^i,\overline{r}^i,\widehat{\pi}^{*,i}_t}
     \pm \overline{U}^{E,\intercal}
     \eta^{p^i,\overline{r}^i,\widehat{\pi}^{*,i}_t}}\\
     &\qquad=
     U^{E,\intercal}\Big(
       \eta^{p^i,\overline{r}^i,\widehat{\pi}^{*,i}_t}-\eta^{p^i,r^i,\overline{\pi}^i}
       \Big)
     +\overline{U}^{E,\intercal}\Big(
       \eta^{\widehat{p}^i,\overline{r}^i,\widehat{\pi}^{*,i}_t}
       -\eta^{p^i,\overline{r}^i,\widehat{\pi}^{*,i}_t}
     \Big)\\
     &\qquad\qquad
     + \overline{U}^{E,\intercal}
     \Big(\widehat{\eta}^i_t-
     \eta^{\widehat{p}^i,\overline{r}^i,\widehat{\pi}^{*,i}_t}
     \Big)
     \\
   &\qquad \markref{(1)}{\le}
   U^{E,\intercal}\Big(
       \eta^{p^i,\overline{r}^i,\widehat{\pi}^{*,i}_t}-\eta^{p^i,r^i,\overline{\pi}^i}
       \Big)+cH\sqrt{\frac{\log\frac{NT}{\delta}}{ K}}\\
   &\qquad\qquad
   +c'H^2 \sqrt{\frac{2}{n}\Big(\log\frac{SAHN}{\delta}
   +(S-1)\log\big(e(1+n/(S-1))\big)\Big)}\\
   &\qquad \markref{(2)}{\le} \popblue{LH\epsilon_0/2}
   +cH\sqrt{\frac{\log\frac{NT}{\delta}}{ K}}
   \\
   &\qquad\qquad
   +c'H^2 \sqrt{\frac{2}{n}\Big(\log\frac{SAHN}{\delta}
   +(S-1)\log\big(e(1+n/(S-1))\big)\Big)},
   \end{align*}
   where at (1) we have used the bounds derived earlier, and at (2) we have
   applied Lemma \ref{lemma: bound performance same policy different rewards for
   tractor}, noticing that we can choose policy $\overline{\pi}^i$ as we wish,
   and using that $k\le\epsilon_0/2$.
   
 \end{proof}
 
 \begin{lemma}\label{lemma: bound performance same policy different rewards for
   tractor} Let $\cM_1=\tuple{\cS,\cA,H,s_0,p,r^1}$ and
   $\cM_2=\tuple{\cS,\cA,H,s_0,p,r^2}$ be two MDPs with deterministic rewards
   that differ only in the reward function $r^1\neq r^2$, and assume that, for
   all $(s,a,h)\in\SAH$, it holds that $|r^1_h(s,a)-r^2_h(s,a)|\le k$, for some
   $k\ge 0$. Let $\pi^1$ be an arbitrary (potentially non-Markovian) policy that
   induces, in $\cM_1$, the distribution over returns $\eta^{p,r^1,\pi^1}$. Then,
   there exists a policy $\pi^2$ that induces in $\cM_2$ the distribution
   $\eta^{p,r^2,\pi^2}$ such that:
   \begin{align*}
     \sup\limits_{U\in\underline{\fU}_L}\Big|
     \E_{G\sim \eta^{p,r^1,\pi^1}}[U(G)]-
     \E_{G\sim \eta^{p,r^2,\pi^2}}[U(G)]
     \Big|\le LHk.
   \end{align*}
 \end{lemma}
 \begin{proof}
   A non-Markovian policy like $\pi^1$, in its most general form, prescribes
   actions at stages $h\in\dsb{H}$ depending on the sequence of
   state-action-reward
   $\tuple{s_1,a_1,r_1,s_2,a_2,r_2,\dotsc,s_{h-1},a_{h-1},r_{h-1},s_h}$ received
   so far. Since, by hypothesis, the reward functions are deterministic (see also
   Section \ref{sec: preliminaries}), then it is clear that the information
   contained in the rewards received so far ($\{r_1,r_2,\dotsc,r_{h-1}\}$) is
   already contained in the state-action pairs received
   $\tuple{s_1,a_1,s_2,a_2,\dotsc,s_{h-1},a_{h-1},s_h}$ (indeed, for
   deterministic reward $r^1$, we have that
   $r_1=r^1_1(s_1,a_1),r_2=r^1_2(s_2,a_2)$, and so on). This means that, for any
   non-Markovian policy in the MDP $\cM_1$, since it coincides with $\cM_2$
   except for the deterministic reward function, it is possible to construct a
   policy $\pi^2$ that induces the same distribution over \emph{state-action}
   trajectories, i.e., for any state-action trajectory
   $\omega=\tuple{s_1,a_1,s_2,a_2,\dotsc,s_{H-1},a_{H-1},s_H,a_H,s_{H+1}}\in\Omega$, it holds
   $\P_{p,r^1,\pi^1}(\omega)=\P_{p,r^2,\pi^2}(\omega)$.
 
 Therefore, we can write:
 \begin{align*}
   &\sup\limits_{U\in\underline{\fU}_L}\Big|
     \E_{G\sim \eta^{p,r^1,\pi^1}}[U(G)]-
     \E_{G\sim \eta^{p,r^2,\pi^2}}[U(G)]\Big|\\
     &\qquad\markref{(1)}{=}\sup\limits_{U\in\underline{\fU}_L}
     \Big|
     \sum\limits_{\omega\in\Omega}
     \P_{p,r^1,\pi^1}(\omega)
     U\Big(\sum\limits_{(s,a,h)\in\omega}r^1_h(s,a)\Big)\\
     &\qquad\qquad-
     \sum\limits_{\omega\in\Omega}
     \P_{p,r^2,\pi^2}(\omega)
     U\Big(\sum\limits_{(s,a,h)\in\omega}r^2_h(s,a)\Big)
     \Big|\\
     &\qquad\markref{(2)}{=}\sup\limits_{U\in\underline{\fU}_L}
     \Big|
     \sum\limits_{\omega\in\Omega}
     \P_{p,r^1,\pi^1}(\omega)
     U\Big(\sum\limits_{(s,a,h)\in\omega}r^1_h(s,a)\Big)\\
     &\qquad\qquad-
     \sum\limits_{\omega\in\Omega}
     \popblue{\P_{p,r^1,\pi^1}}(\omega)
     U\Big(\sum\limits_{(s,a,h)\in\omega}r^2_h(s,a)\Big)
     \Big|\\
     &=\sup\limits_{U\in\underline{\fU}_L}
     \Big|
     \sum\limits_{\omega\in\Omega}
     \P_{p,r^1,\pi^1}(\omega)
     \Big(U\Big(\sum\limits_{(s,a,h)\in\omega}r^1_h(s,a)\Big)-
     U\Big(\sum\limits_{(s,a,h)\in\omega}r^2_h(s,a)\Big)\Big)
     \Big|\\
     &\markref{(3)}{\le}
     \sup\limits_{U\in\underline{\fU}_L}
     \sum\limits_{\omega\in\Omega}
     \P_{p,r^1,\pi^1}(\omega)
     \popblue{\Big|}U\Big(\sum\limits_{(s,a,h)\in\omega}r^1_h(s,a)\Big)-
     U\Big(\sum\limits_{(s,a,h)\in\omega}r^2_h(s,a)\Big)\popblue{\Big|}
     \\
     &\markref{(4)}{\le}
     \sum\limits_{\omega\in\Omega}
     \P_{p,r^1,\pi^1}(\omega)
     \popblue{L}\popblue{\Big|\sum\limits_{(s,a,h)\in\omega}(r^1_h(s,a)-r^2_h(s,a))\Big|}
     \\
     &\markref{(5)}{\le}
     \sum\limits_{\omega\in\Omega}
     \P_{p,r^1,\pi^1}(\omega)
     L\sum\limits_{(s,a,h)\in\omega}\popblue{\Big|}r^1_h(s,a)-r^2_h(s,a)\popblue{\Big|}
     \\
     &\markref{(6)}{\le}
     \sum\limits_{\omega\in\Omega}
     \P_{p,r^1,\pi^1}(\omega)
     L\sum\limits_{(s,a,h)\in\omega}\popblue{k}
     \\
     &=LHk,
 \end{align*}
 where at (1) we use the fact that the expected utility w.r.t. the distribution
 over returns can be computed using the probability distribution over
 state-action trajectories (since the rewards are deterministic), at (2) we use
 that policy $\pi^2$ is constructed exactly to match the distribution over
 state-action trajectories, at (3) we apply triangle inequality, at (4) we use
 the fact that all utilities $U\in\underline{\fU}_L$ are
 $L$-Lipschitz, i.e., for all $x,y\in[0,H]$:
 $|U(x)-U(y)|\le L|x-y|$, at (5) we apply again the triangle
 inequality, and at (6) we use the hypothesis that $r^1,r^2$ are close to each
 other by parameter $k$.
 \end{proof}
 
 \section{Experimental Details}
 \label{apx: experimental details}
 
 In this appendix, we collect additional information about the experiments
 described in Section \ref{sec: experiments}. Appendix \ref{apx: data
 description} presents formally the MDP used for the collection of the data
 along with the questions posed to the participants. Appendix \ref{apx: details
 exp2} contains additional details on Experiment 2. Finally, Appendix \ref{apx:
 additional experiment} presents an additional experiment conducted on the
 collected data.
 
\subsection{Data Description}
   \label{apx: data description}
 
   Below, we describe the data collected.
 
 \subsubsection{Considered MDP}
 
   The 15 participants analyzed in the study have been provided with complete
   access to the MDP in Figure \ref{fig: mdp for data}, which we will denote by
   $\cM$. In other words, the participants \emph{know the transition model and
   the reward function of $\cM$ everywhere}.
 
   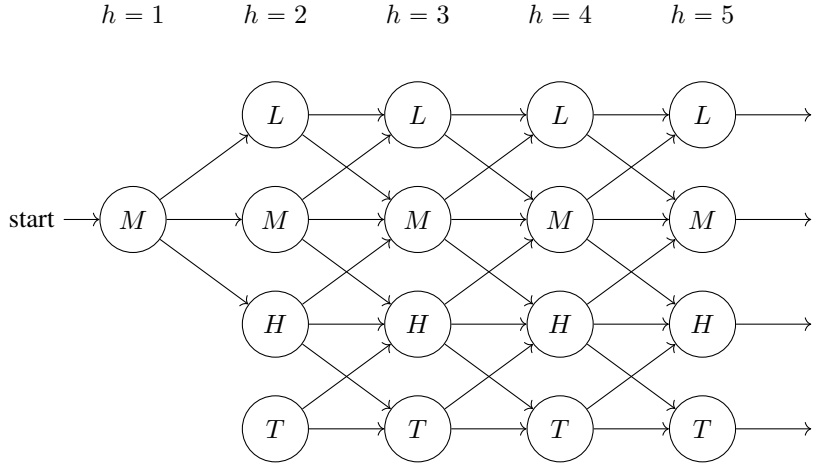
\begin{figure}[h!]
     \centering
     \begin{tikzpicture}[node distance=3.5cm]
     \node[initial,state] (M0) {$M$};
     \node[state, right=1cm of M0] (M1) {$M$};
     \node[state, right=1cm of M0, above=0.5cm of M1] (L1) {$L$};
     \node[state, right=1cm of M0, below=0.5cm of M1] (H1) {$H$};
     \node[state, right=1cm of M0, below=0.5cm of H1] (T1) {$T$};
     \node[state, right=1cm of L1] (L2) {$L$};
     \node[state, right=1cm of L1, below=0.5cm of L2] (M2) {$M$};
     \node[state, right=1cm of L1, below=0.5cm of M2] (H2) {$H$};
     \node[state, right=1cm of L1, below=0.5cm of H2] (T2) {$T$};
     \node[state, right=1cm of L2] (L3) {$L$};
     \node[state, right=1cm of L2, below=0.5cm of L3] (M3) {$M$};
     \node[state, right=1cm of L2, below=0.5cm of M3] (H3) {$H$};
     \node[state, right=1cm of L2, below=0.5cm of H3] (T3) {$T$};
     \node[state, right=1cm of L3] (L4) {$L$};
     \node[state, right=1cm of L3, below=0.5cm of L4] (M4) {$M$};
     \node[state, right=1cm of L3, below=0.5cm of M4] (H4) {$H$};
     \node[state, right=1cm of L3, below=0.5cm of H4] (T4) {$T$};
     \node[state, right=1cm of L4, draw=none] (L5) {};
     \node[state, right=1cm of M4, draw=none] (M5) {};
     \node[state, right=1cm of H4, draw=none] (H5) {};
     \node[state, right=1cm of T4, draw=none] (T5) {};
     \draw (M0) edge[->, solid, right] node{} (L1);
     \draw (M0) edge[->, solid, right] node{} (M1);
     \draw (M0) edge[->, solid, right] node{} (H1);
     \draw (L1) edge[->, solid, right] node{} (L2);
     \draw (L1) edge[->, solid, right] node{} (M2);
     \draw (M1) edge[->, solid, right] node{} (L2);
     \draw (M1) edge[->, solid, right] node{} (M2);
     \draw (M1) edge[->, solid, right] node{} (H2);
     \draw (H1) edge[->, solid, right] node{} (M2);
     \draw (H1) edge[->, solid, right] node{} (H2);
     \draw (H1) edge[->, solid, right] node{} (T2);
     \draw (T1) edge[->, solid, right] node{} (H2);
     \draw (T1) edge[->, solid, right] node{} (T2);
     \draw (L2) edge[->, solid, right] node{} (L3);
     \draw (L2) edge[->, solid, right] node{} (M3);
     \draw (M2) edge[->, solid, right] node{} (L3);
     \draw (M2) edge[->, solid, right] node{} (M3);
     \draw (M2) edge[->, solid, right] node{} (H3);
     \draw (H2) edge[->, solid, right] node{} (M3);
     \draw (H2) edge[->, solid, right] node{} (H3);
     \draw (H2) edge[->, solid, right] node{} (T3);
     \draw (T2) edge[->, solid, right] node{} (H3);
     \draw (T2) edge[->, solid, right] node{} (T3);
     \draw (L3) edge[->, solid, right] node{} (L4);
     \draw (L3) edge[->, solid, right] node{} (M4);
     \draw (M3) edge[->, solid, right] node{} (L4);
     \draw (M3) edge[->, solid, right] node{} (M4);
     \draw (M3) edge[->, solid, right] node{} (H4);
     \draw (H3) edge[->, solid, right] node{} (M4);
     \draw (H3) edge[->, solid, right] node{} (H4);
     \draw (H3) edge[->, solid, right] node{} (T4);
     \draw (T3) edge[->, solid, right] node{} (H4);
     \draw (T3) edge[->, solid, right] node{} (T4);
     \draw (L4) edge[->, solid, right] node{} (L5);
     \draw (M4) edge[->, solid, right] node{} (M5);
     \draw (H4) edge[->, solid, right] node{} (H5);
     \draw (T4) edge[->, solid, right] node{} (T5);
     \node[state, above=1.7cm of M0, draw=none] (h1) {$h=1$};
     \node[state, above=0.3cm of L1, draw=none] (h2) {$h=2$};
     \node[state, above=0.3cm of L2, draw=none] (h3) {$h=3$};
     \node[state, above=0.3cm of L3, draw=none] (h4) {$h=4$};
     \node[state, above=0.3cm of L4, draw=none] (h5) {$h=5$};
   \end{tikzpicture}
   \caption{\small The MDP used for data collection.}
   \label{fig: mdp for data}  
       \end{figure}
 
       Intuitively, states L (Low), M (Medium), H (High), and T (Top), represent
       4 ``levels'' so that the received reward increases when playing actions in
       ``higher'' states instead of ``lower'' states. Formally, MDP
       $\cM=\tuple{\cS,\cA,H,s_0,p,r}$ has four states $\cS=\{L,M,H,T\}$, and
       three actions for each state $\cA=\{a_0,a_+,a_-\}$. The horizon is $H=5$,
       i.e., the agent has to take 5 actions. The initial state is $s_0=M$. The
       transition model $p$ is stationary, i.e., it does not depend on the stage
       $h\in\dsb{H}$. Specifically, $p$ is depicted in Table \ref{table: p of
       mdp}. The intuition is that action $a_0$ keeps the agent in the same state
       deterministically, while action $a_+$ tries to bring the agent to the
       higher state with probability $1/3$, and action $a_-$ sometimes make the
       agent ``fall down'' to the lower state with probability $1/5$.
 
   \begin{table}[h!]
     \centering
      \begin{tabular}{||c | c c c c ||}
      \hline
      $p$ & $L$ & $M$ & $H$ & $T$\\
      \hline\hline
      $(L,a_0)$ & 1 & 0 & 0 & 0\\ 
      $(L,a_+)$ & $2/3$ & $1/3$ & 0 & 0\\ 
      $(L,a_-)$ & 1 & 0 & 0 & 0\\
      $(M,a_0)$ & 0 & 1 & 0 & 0\\ 
      $(M,a_+)$ & 0 & $2/3$ & $1/3$ & 0\\ 
      $(M,a_-)$ & $1/5$ & $4/5$ & 0 & 0\\
      $(H,a_0)$ & 0 & 0 & 1 & 0\\ 
      $(H,a_+)$ & 0 & 0 & $2/3$ & $1/3$\\ 
      $(H,a_-)$ & 0 & $1/5$ & $4/5$ & 0\\
      $(T,a_0)$ & 0 & 0 & 0 & 1\\ 
      $(T,a_+)$ & 0 & 0 & 0 & 1\\ 
      $(T,a_-)$ & 0 & 0 & $1/5$ & $4/5$\\
      \hline
      \end{tabular}
      \caption{\small The transition model $p$ of MDP $\cM$.}
      \label{table: p of mdp}
     \end{table}
 
     The reward function $r:\SAH\to\RR$ is deterministic, stationary, and depends
     only the state-action pair played. The specific values are depicted in Table
     \ref{table: r of mdp}. Note that we have written the reward values as
     numbers in $[0\text{\texteuro}, 1000\text{\texteuro}]$, to provide a monetary
     interpretation. Nevertheless, we will rescale the interval to $[0,1]$ during
     the analysis for normalization.    
     Observe that the same actions played in ``higher'' states
     (e.g., $H$ or $T$) provide higher rewards than when played in ``lower''
     states (e.g., $L$ or $M$). Moreover, notice that action $a_+$, which is the
     only action that tries to increase the state, does not provide reward at
     all, while the risky action $a_-$, which sometimes decreases the state,
     always provides double the reward than ``default'' action $a_0$.
 
     \begin{table}[h!]
       \centering
        \begin{tabular}{||c | c c c c ||}
        \hline
         & $L$ & $M$ & $H$ & $T$\\ [0.5ex] 
        \hline\hline
        $a_0$ & $0$\texteuro & $30$\texteuro & $100$\texteuro & $500$\texteuro\\ 
        $a_+$ & $0$\texteuro & $0$\texteuro & $0$\texteuro & $0$\texteuro\\ 
        $a_-$ & $0$\texteuro & $60$\texteuro & $200$\texteuro & $1000$\texteuro\\[1ex]
        \hline
        \end{tabular}
        \caption{The reward function $r$ of MDP $\cM$.}
        \label{table: r of mdp}
       \end{table}
 
 \subsubsection{Intuition behind agents behavior}
     The reward is interpreted as money. Playing MDP $\cM$ involves a trade-off
     between playing action $a_+$, which gives no money but potentially allows to
     collect more money in the future (by reaching ``higher'' states), and action
     $a_-$, which provides the greatest amount of money immediately, but
     potentially reduces the amount of money which can be earned in the future.
     Action $a_0$, being deterministic, provides a reference point, so that
     deterministically playing action $a_0$ for all the $H=5$ stages gives to the
     agent $30\times 5=150$\texteuro. Thus, playing actions $a_+,a_-$ other than $a_0$
     means that the agent accepts some risk to try to increase its earnings.
 
 \subsubsection{Questions asked to the participants}
     
     We remark that the participants have enough background knowledge to
     understand the MDP described.
     To each participant, we ask which action in $\{a_0,a_+,a_-\}$ it would play
     if it was in a certain state $s$, stage $h$, with cumulative reward up to
     now $y$, for many different values of triples
     $(s,h,y)\in\cS\times\dsb{H}\times[0\text{\texteuro},5000\text{\texteuro}]$. Specifically,
     the values of triples $s,h,y$ considered are:
     \begin{align*}
       &\tuple{M,1,0\text{\texteuro}}\qquad \tuple{M,2,0\text{\texteuro}}\qquad \tuple{M,2,30\text{\texteuro}}
       \qquad \tuple{M,2,60\text{\texteuro}} \qquad \tuple{H,2,0\text{\texteuro}}\\
       &\tuple{M,3,0\text{\texteuro}}\qquad \tuple{M,3,30\text{\texteuro}}\qquad \tuple{M,3,60\text{\texteuro}}
       \qquad \tuple{M,3,200\text{\texteuro}} \qquad \tuple{H,3,0\text{\texteuro}}\\
       &\tuple{H,3,30\text{\texteuro}}\qquad \tuple{H,3,60\text{\texteuro}}\qquad \tuple{H,3,200\text{\texteuro}}
       \qquad \tuple{T,3,0\text{\texteuro}} \qquad \tuple{M,4,0\text{\texteuro}}\\
       &\tuple{M,4,30\text{\texteuro}}\qquad \tuple{M,4,60\text{\texteuro}}\qquad \tuple{M,4,90\text{\texteuro}}
       \qquad \tuple{M,4,120\text{\texteuro}} \qquad \tuple{M,4,150\text{\texteuro}}\\
       &\tuple{M,4,180\text{\texteuro}}\qquad \tuple{M,4,300\text{\texteuro}}\qquad \tuple{M,4,400\text{\texteuro}}
       \qquad \tuple{H,4,0\text{\texteuro}} \qquad \tuple{H,4,30\text{\texteuro}}\\
       &\tuple{H,4,60\text{\texteuro}}\qquad \tuple{H,4,100\text{\texteuro}}\qquad \tuple{H,4,130\text{\texteuro}}
       \qquad \tuple{H,4,200\text{\texteuro}} \qquad \tuple{H,4,300\text{\texteuro}}\\
       &\tuple{H,4,1000\text{\texteuro}}\qquad \tuple{T,4,0\text{\texteuro}}\qquad \tuple{T,4,60\text{\texteuro}}.
     \end{align*}
 
     From state $L$, we assume all participants always play action $a_+$ since it
     is the only rational strategy. Moreover, from stage $h=5$, we assume that
     all participants always play action $a_-$ since, again, it is the only
     rational strategy.
 
     In all other possible combinations of values of $s,h,y$, we ``interpolate''
     by considering the action recommended by the participant in the closest $y'$
     to $y$, in the same $s,h$.
 
 \subsubsection{The return distribution of the participants' policies}
 \label{apx: return distribution of the participants}
 
 We now present the return distribution of the policies prescribed by the
 participants. Specifically, we have simulated 10000 times the policies of the
 participants, and we have computed the empirical estimate of their return
 distributions. Such values are reported in Figures \ref{fig: 1}, \ref{fig: 2},
 \ref{fig: 3}, \ref{fig: 4}, and \ref{fig: 5}, where we use notation $\eta^E_i$
 to denote the return distribution of participant $i$, with $i\in\dsb{15}$.
 
 \begin{figure}[!h]
   \centering
   \begin{minipage}[t!]{0.32\textwidth}
     \centering
     \includegraphics[width=0.95\linewidth]{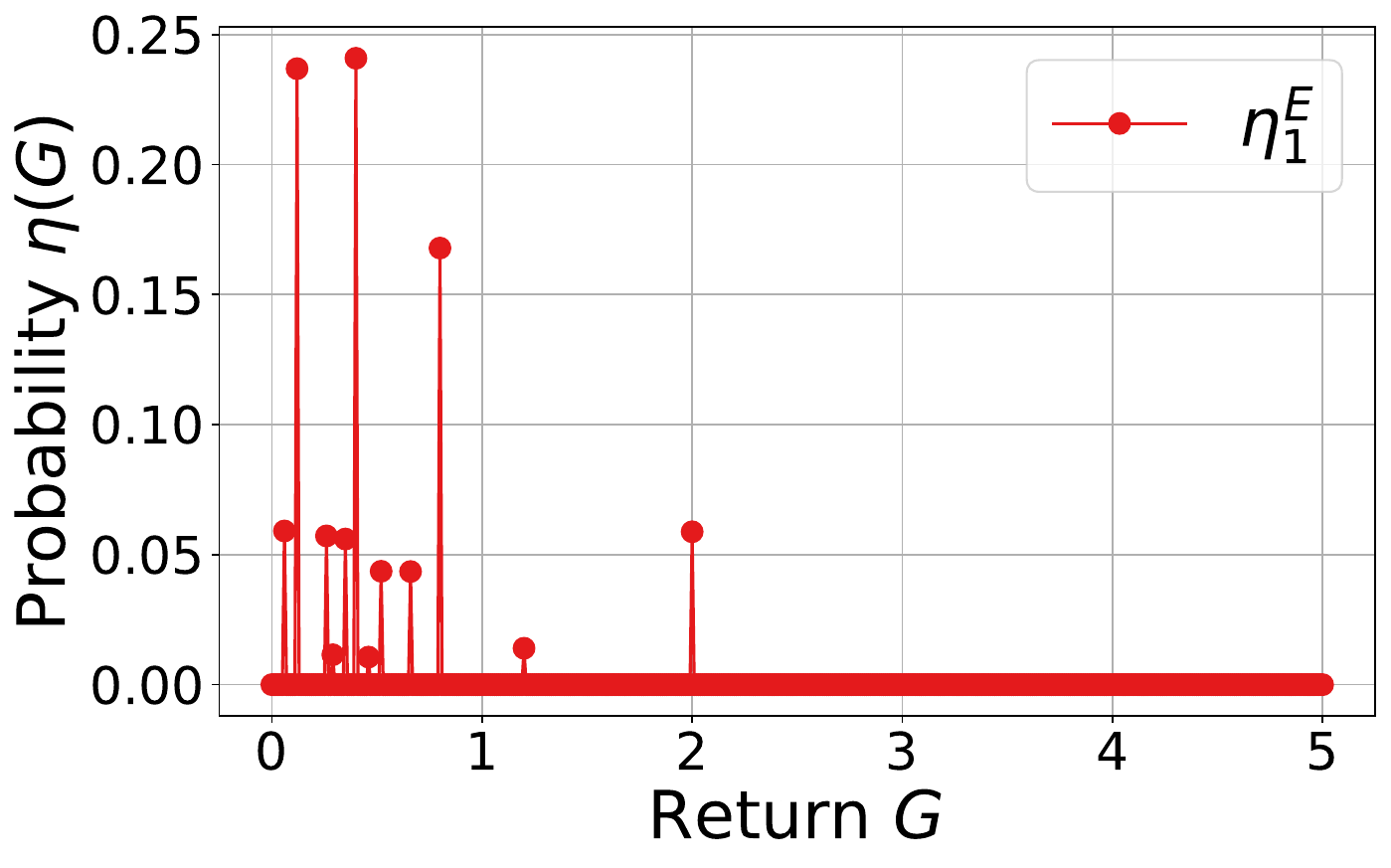}
 \end{minipage}
 \begin{minipage}[t!]{0.32\textwidth}
     \centering
     \includegraphics[width=0.95\linewidth]{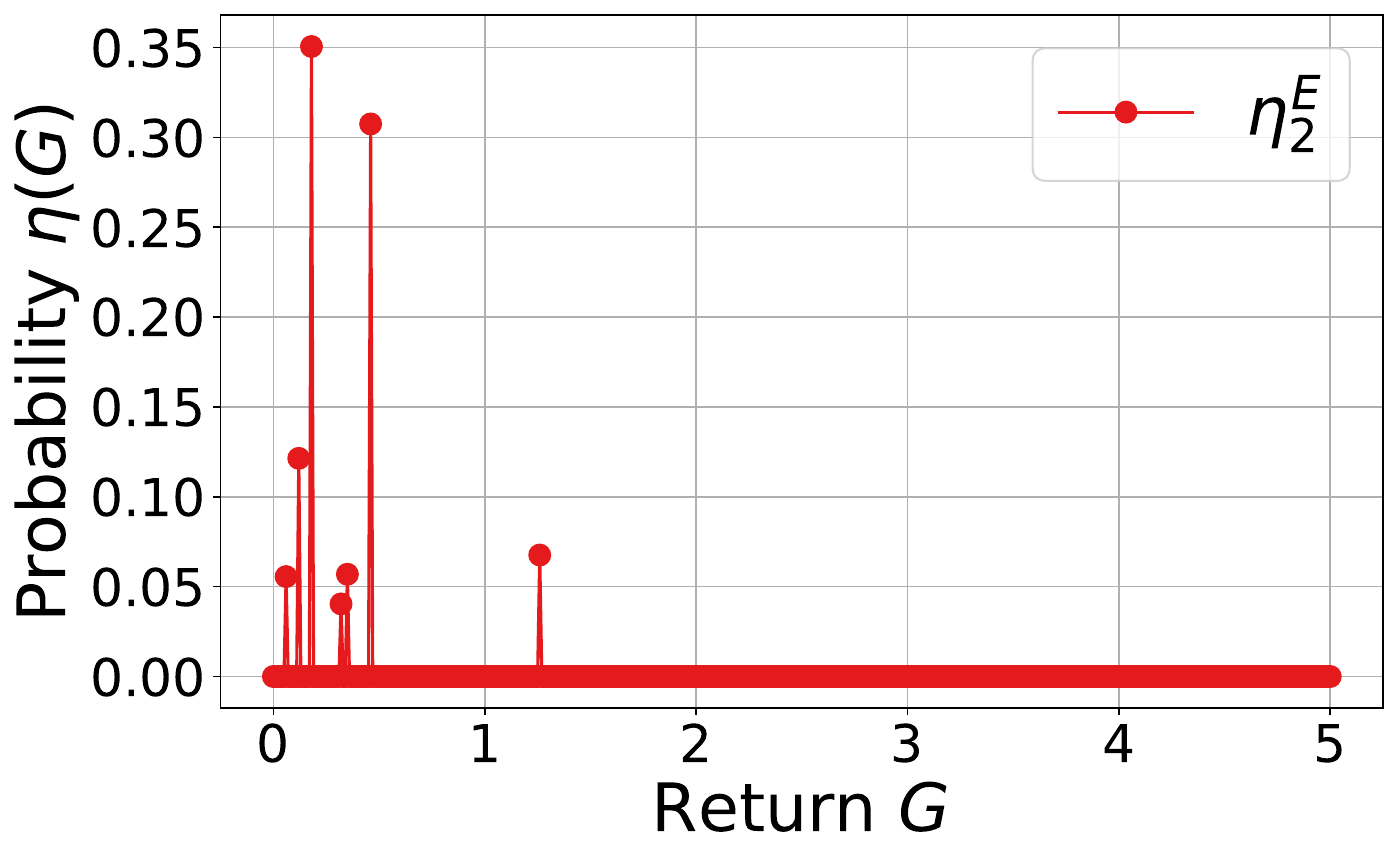}
 \end{minipage}
 \begin{minipage}[t!]{0.32\textwidth}
   \centering
   \includegraphics[width=0.95\linewidth]{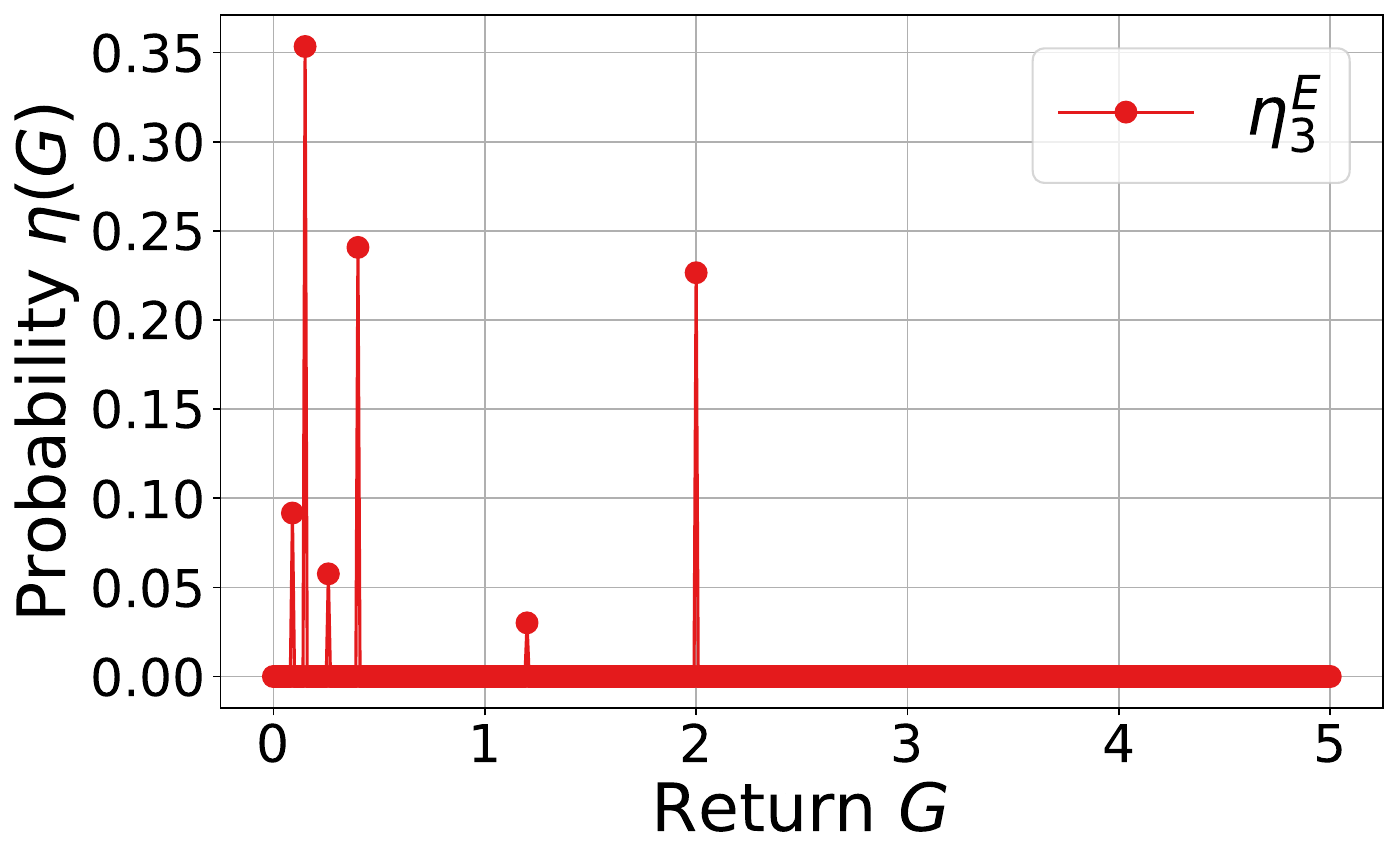}
 \end{minipage}
 \caption{\footnotesize Plot of $\eta^E_1,\eta^E_2$, and $\eta^E_3$.}
 \label{fig: 1}
  \end{figure}
 
  \begin{figure}[!h]
   \centering
   \begin{minipage}[t!]{0.32\textwidth}
     \centering
     \includegraphics[width=0.95\linewidth]{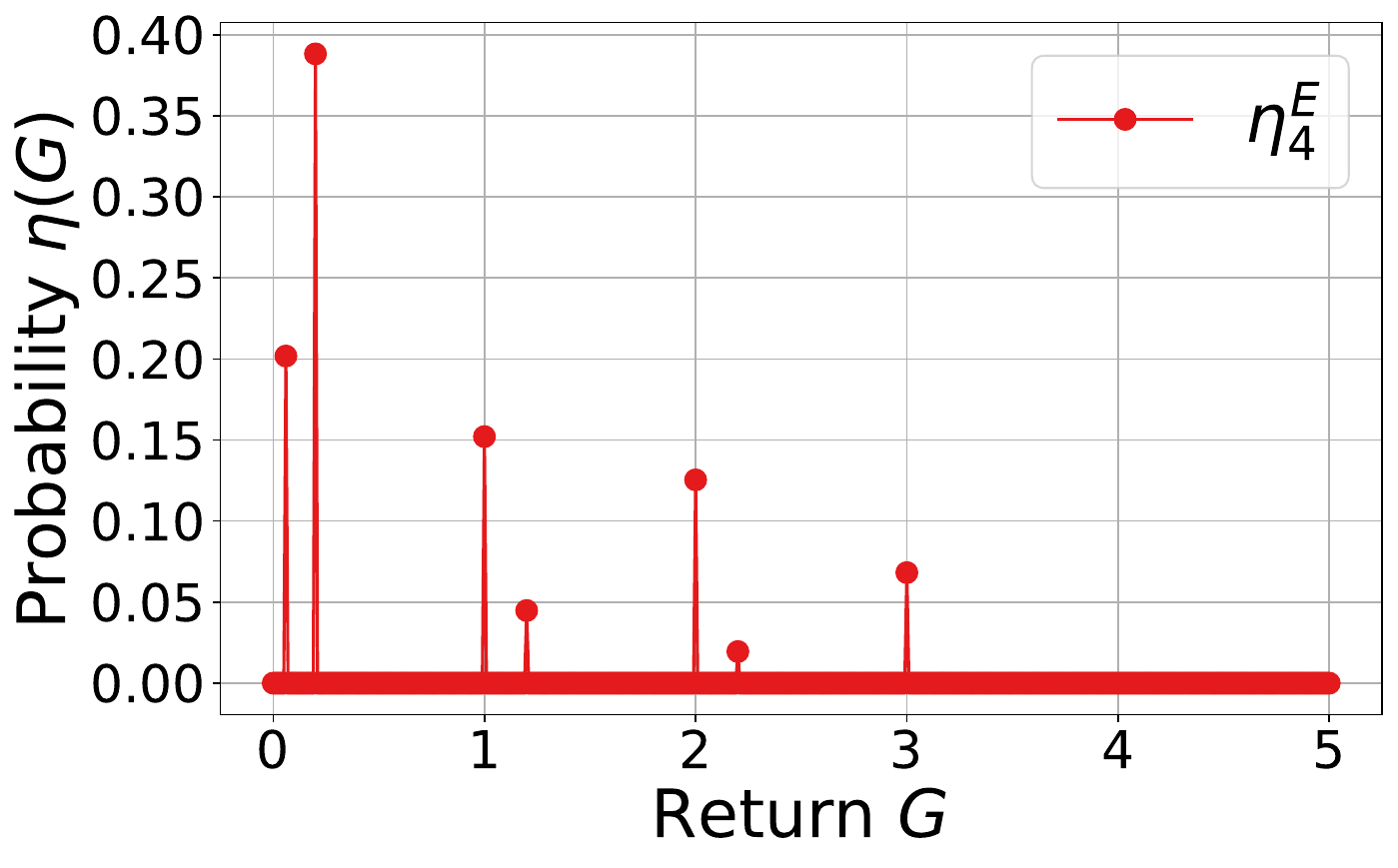}
 \end{minipage}
 \begin{minipage}[t!]{0.32\textwidth}
     \centering
     \includegraphics[width=0.95\linewidth]{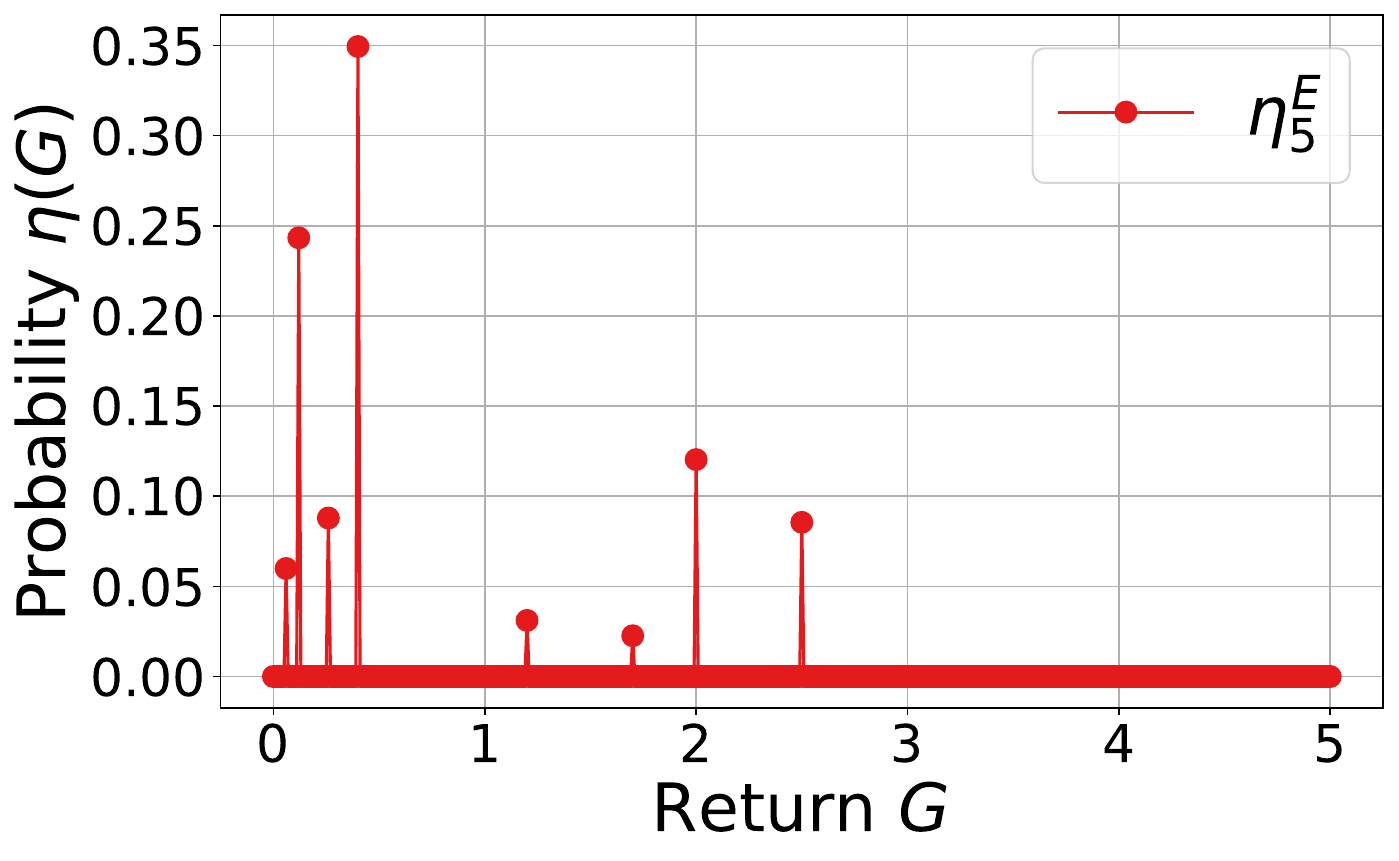}
 \end{minipage}
 \begin{minipage}[t!]{0.32\textwidth}
   \centering
   \includegraphics[width=0.95\linewidth]{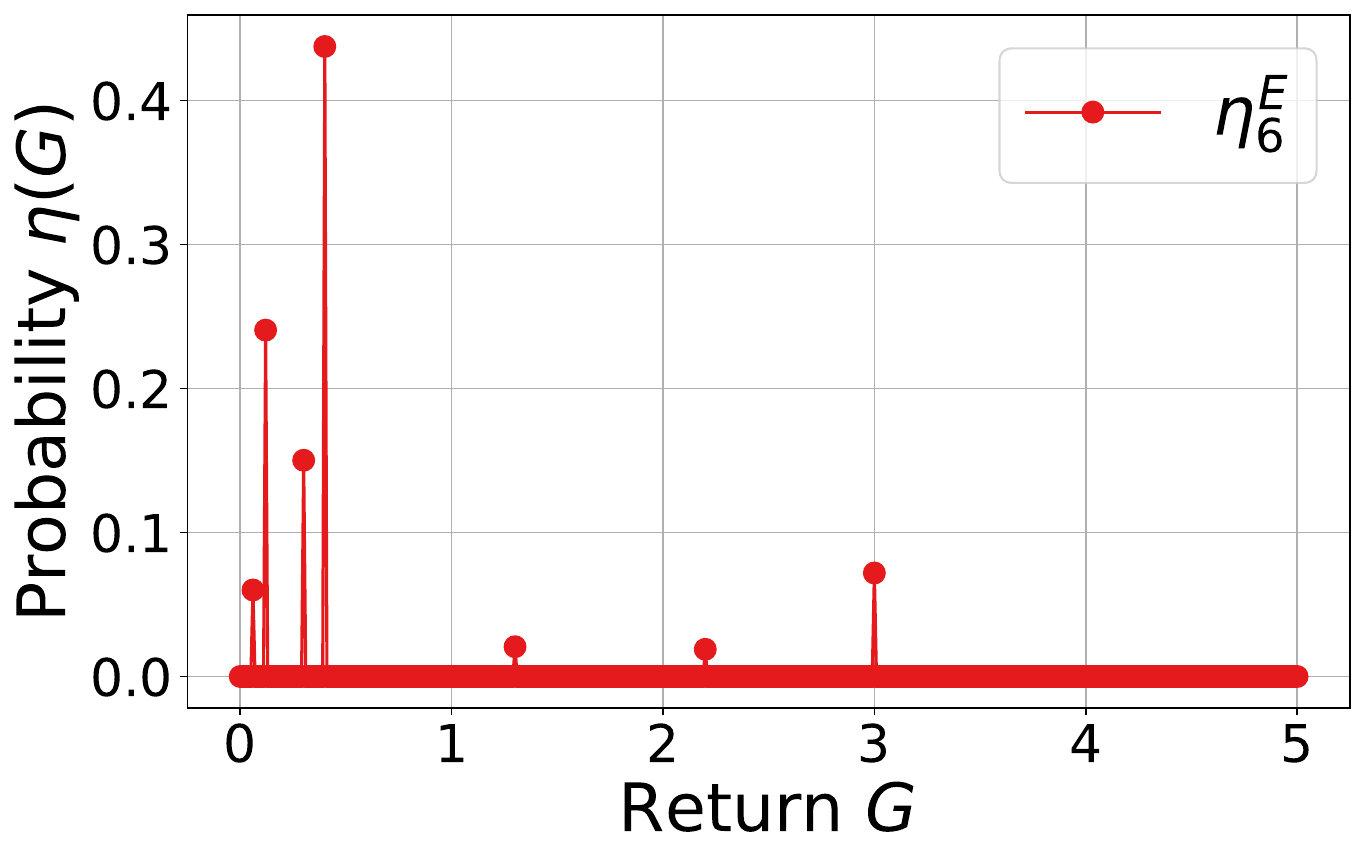}
 \end{minipage}
 \caption{\footnotesize Plot of $\eta^E_4,\eta^E_5$, and $\eta^E_6$.}
 \label{fig: 2}
  \end{figure}
 
  \begin{figure}[!h]
   \centering
   \begin{minipage}[t!]{0.32\textwidth}
     \centering
     \includegraphics[width=0.95\linewidth]{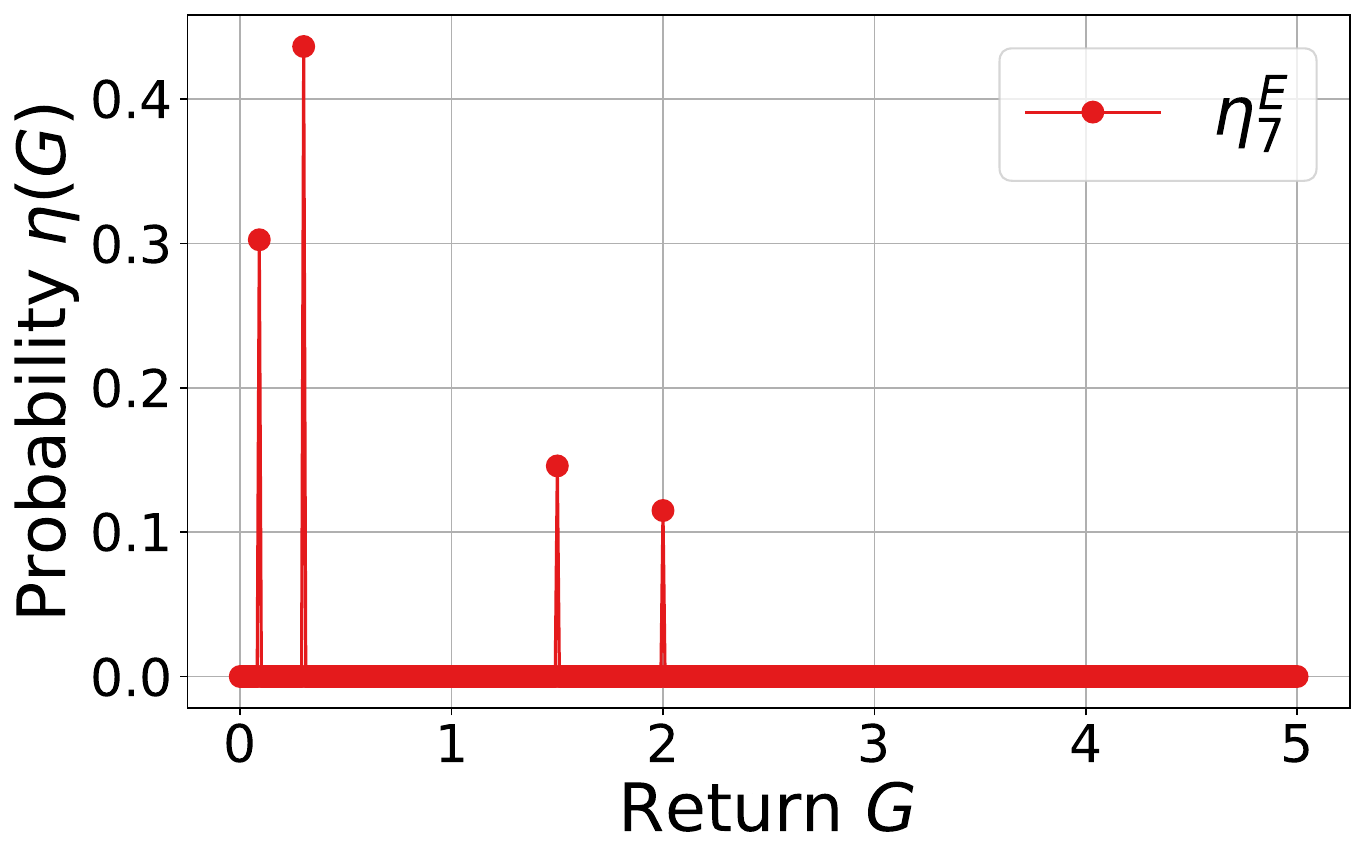}
 \end{minipage}
 \begin{minipage}[t!]{0.32\textwidth}
     \centering
     \includegraphics[width=0.95\linewidth]{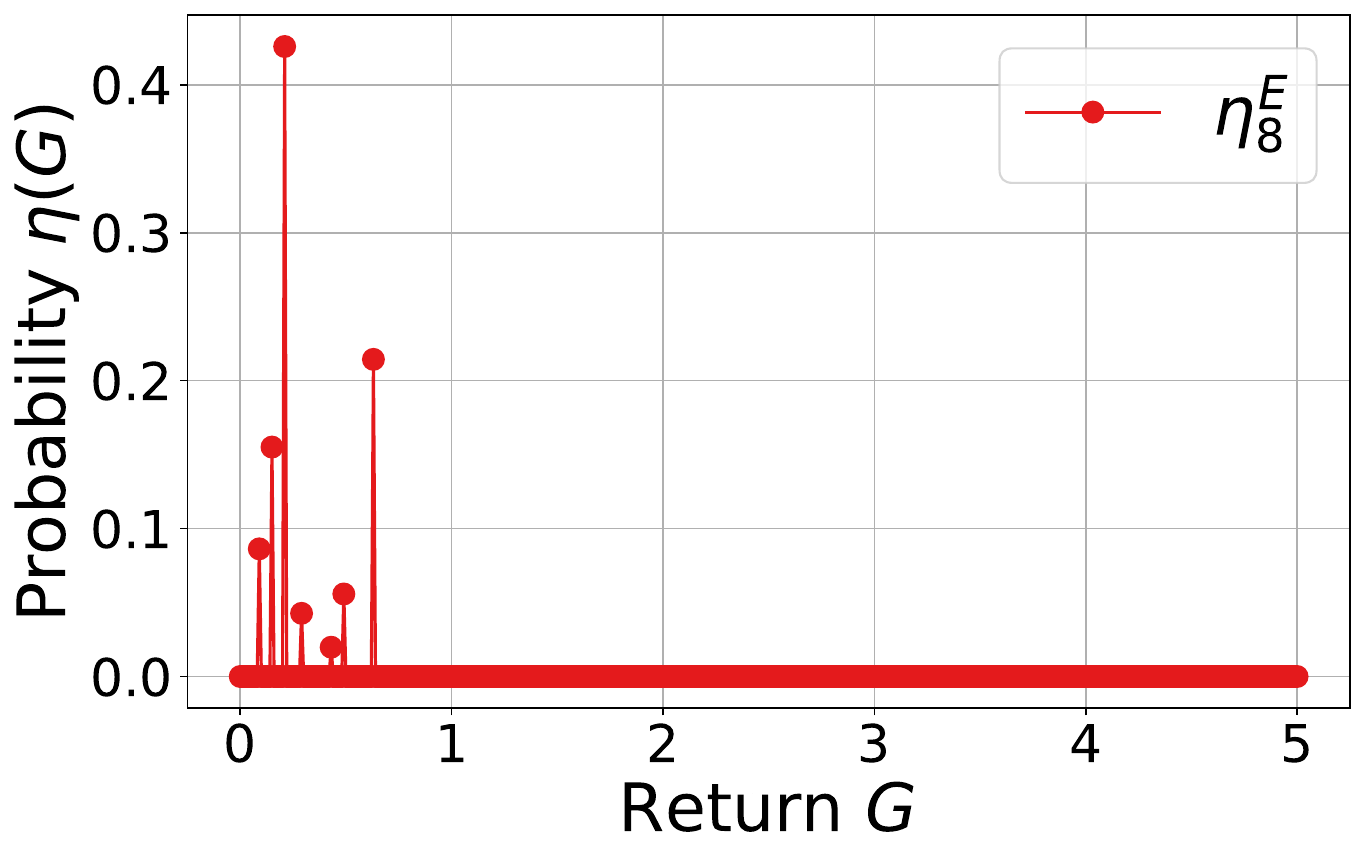}
 \end{minipage}
 \begin{minipage}[t!]{0.32\textwidth}
   \centering
   \includegraphics[width=0.95\linewidth]{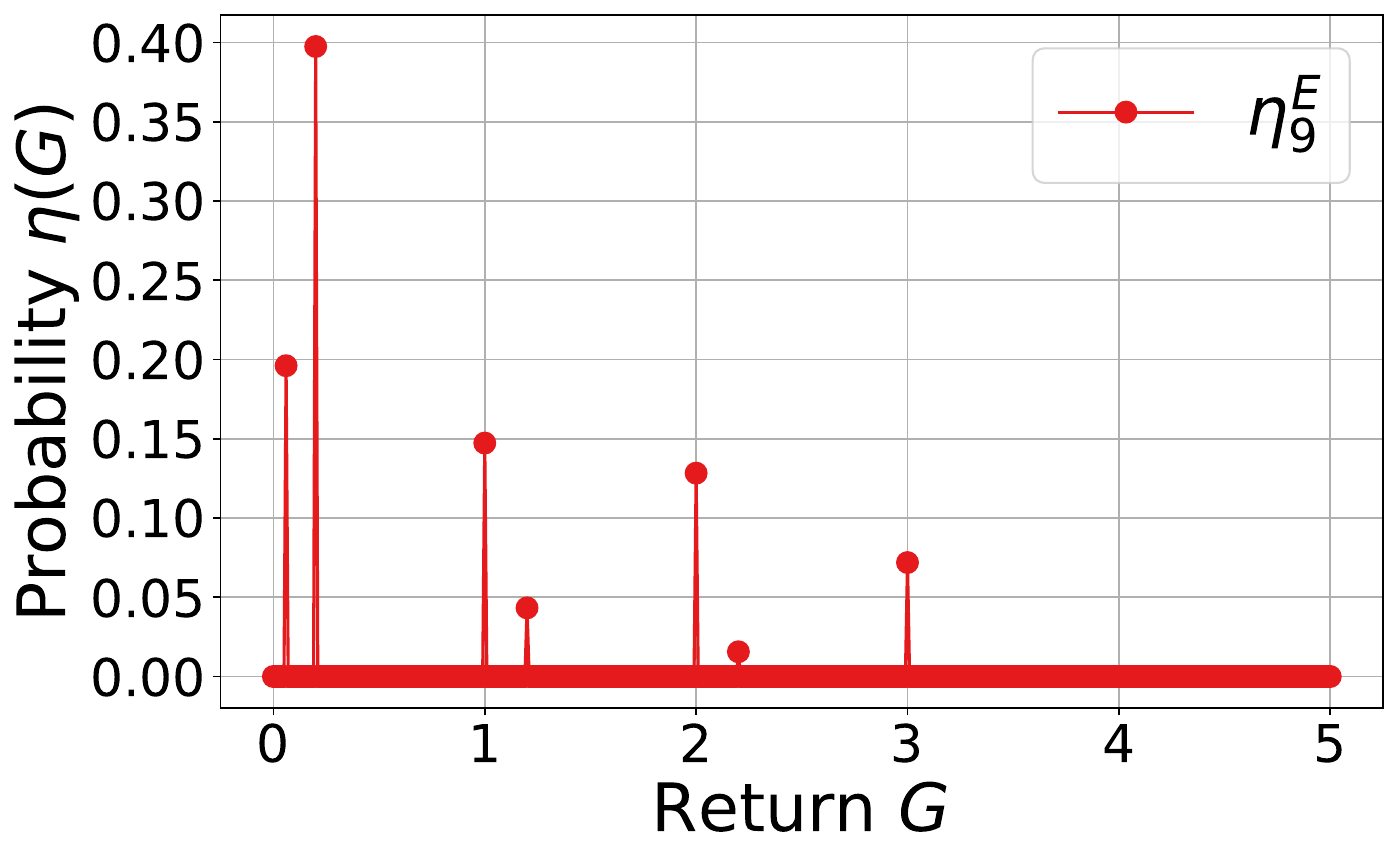}
 \end{minipage}
 \caption{\footnotesize Plot of $\eta^E_7,\eta^E_8$, and $\eta^E_9$.}
 \label{fig: 3}
  \end{figure}
 
  \begin{figure}[!h]
   \centering
   \begin{minipage}[t!]{0.32\textwidth}
     \centering
     \includegraphics[width=0.95\linewidth]{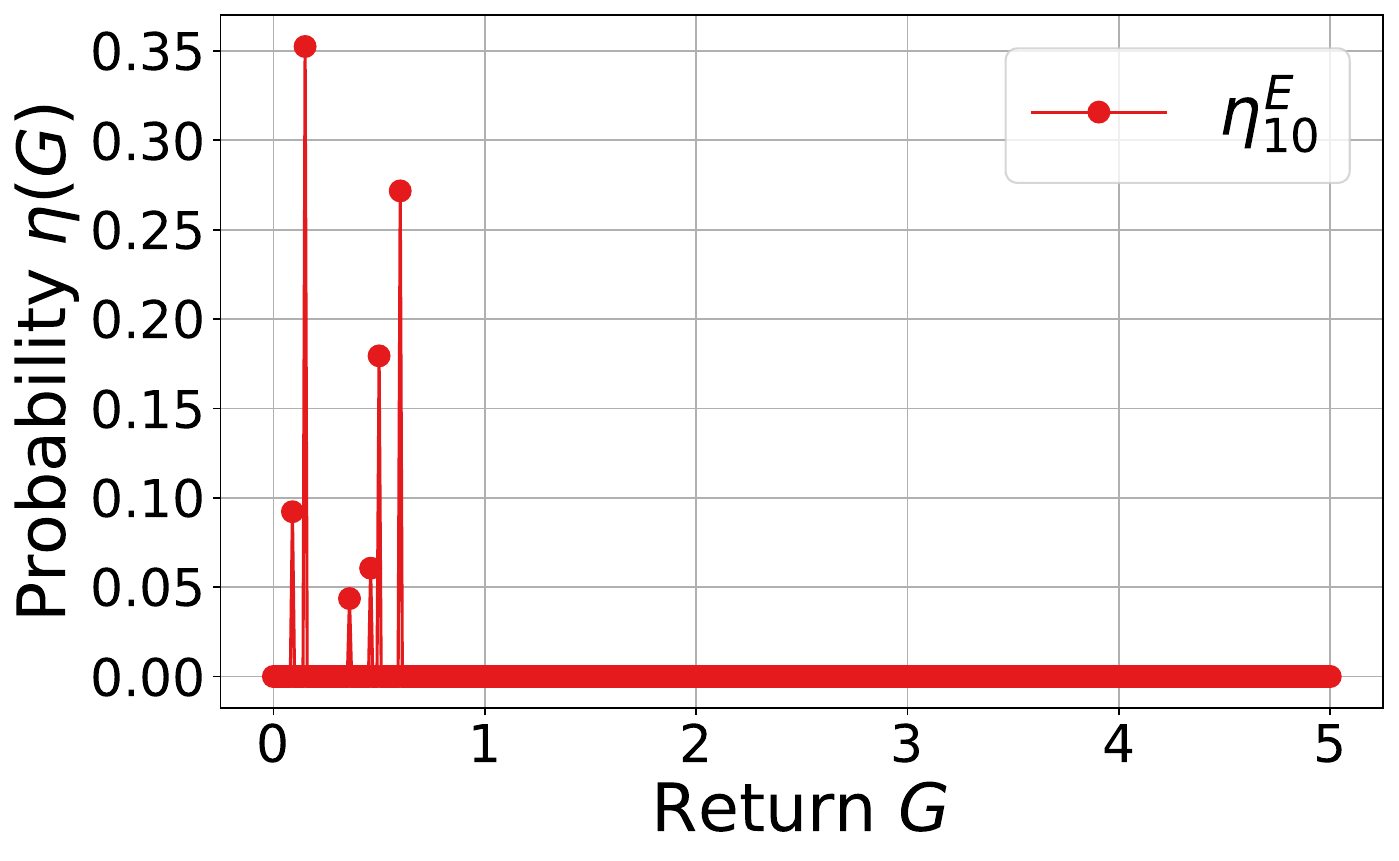}
 \end{minipage}
 \begin{minipage}[t!]{0.32\textwidth}
     \centering
     \includegraphics[width=0.95\linewidth]{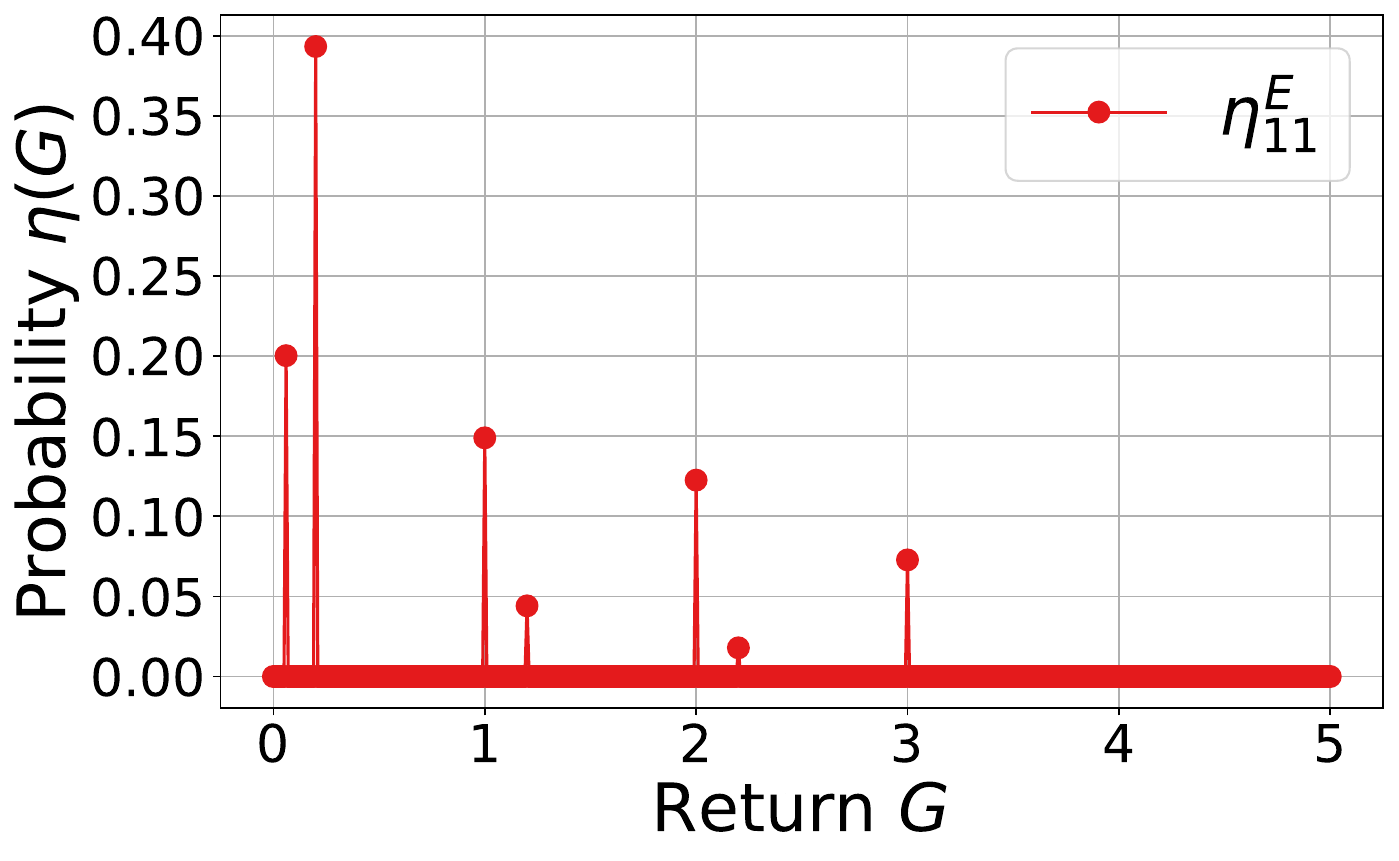}
 \end{minipage}
 \begin{minipage}[t!]{0.32\textwidth}
   \centering
   \includegraphics[width=0.95\linewidth]{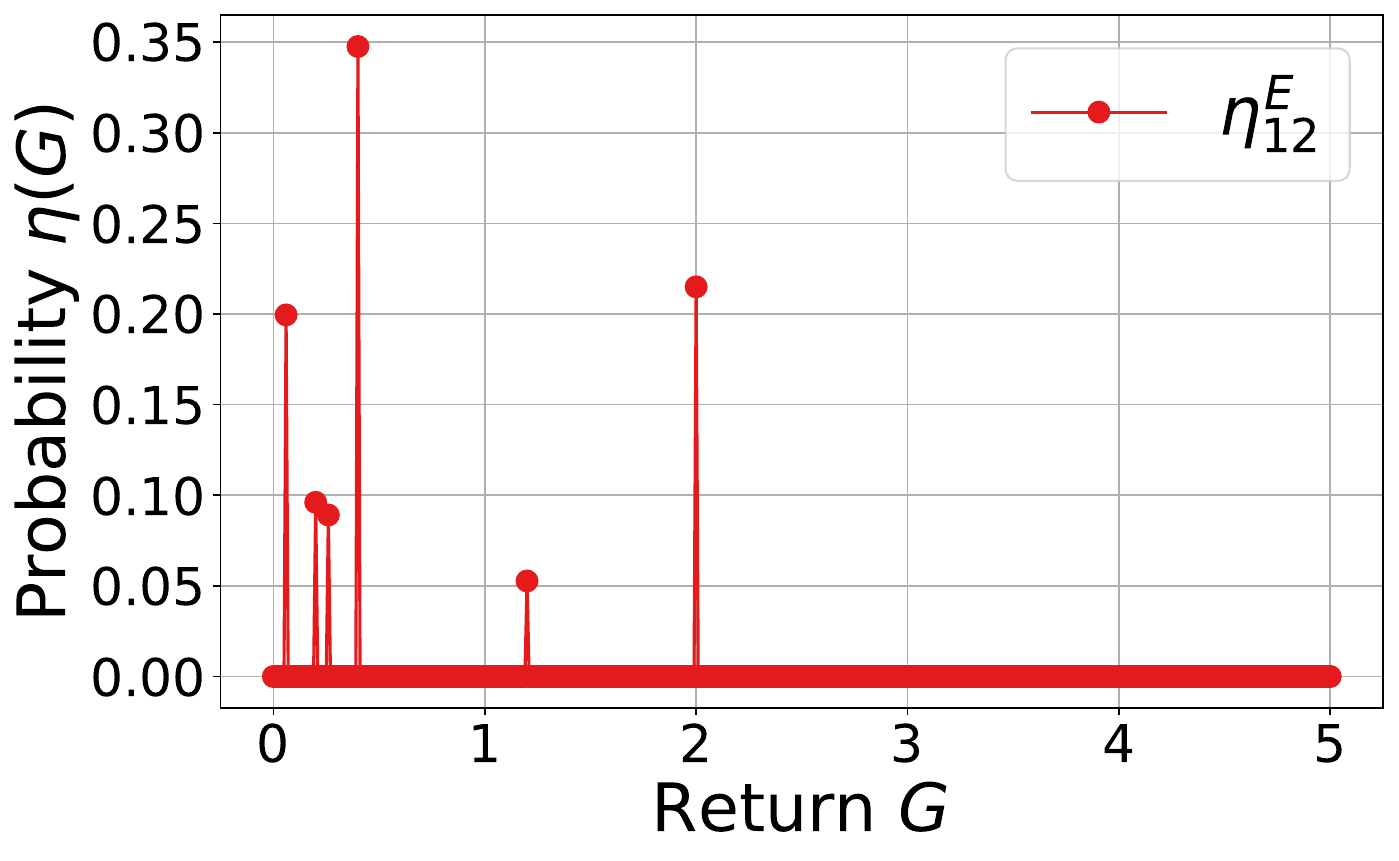}
 \end{minipage}
 \caption{\footnotesize Plot of $\eta^E_{10},\eta^E_{11}$, and $\eta^E_{12}$.}
 \label{fig: 4}
  \end{figure}
 
  \begin{figure}[!h]
   \centering
   \begin{minipage}[t!]{0.32\textwidth}
     \centering
     \includegraphics[width=0.95\linewidth]{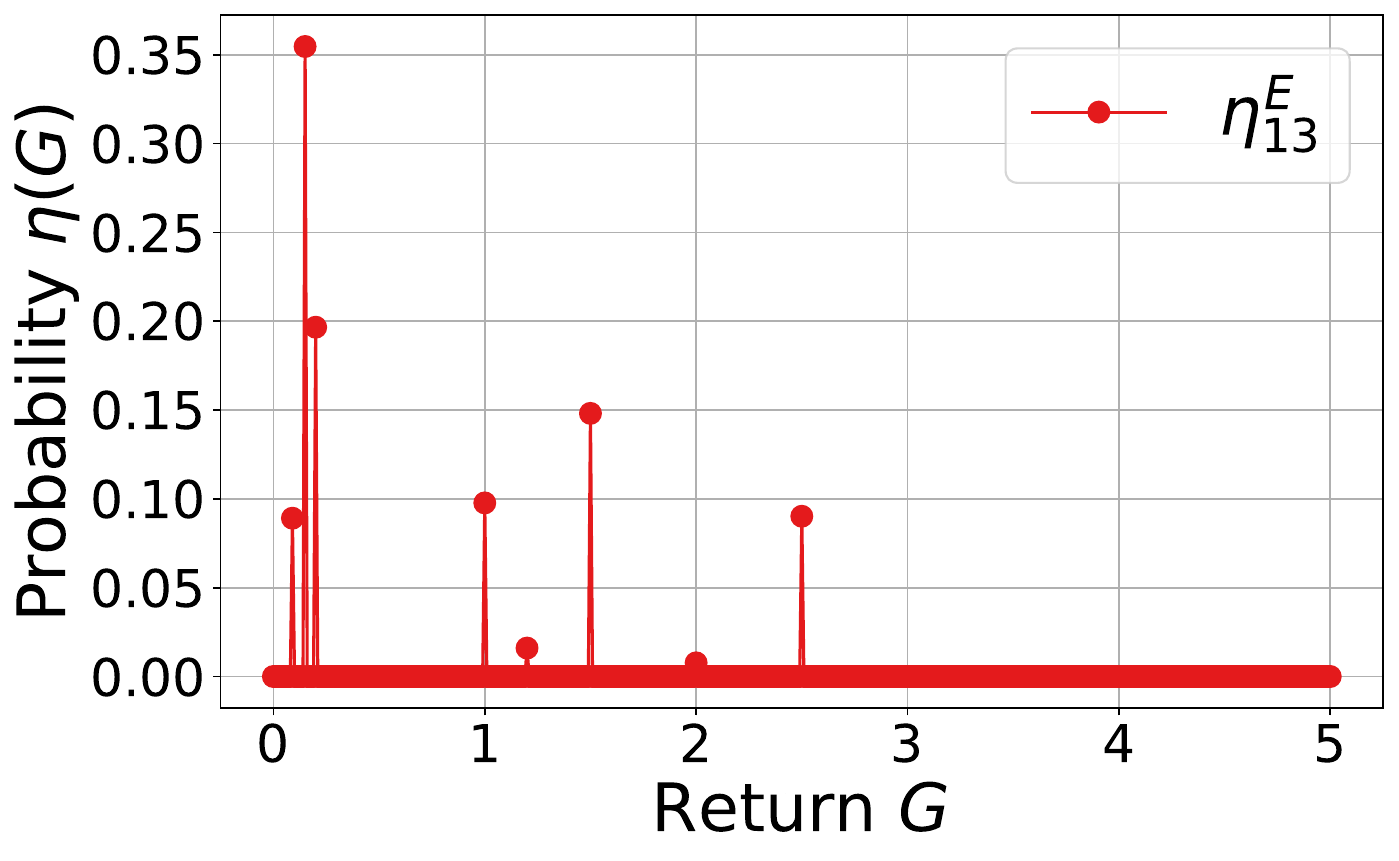}
 \end{minipage}
 \begin{minipage}[t!]{0.32\textwidth}
     \centering
     \includegraphics[width=0.95\linewidth]{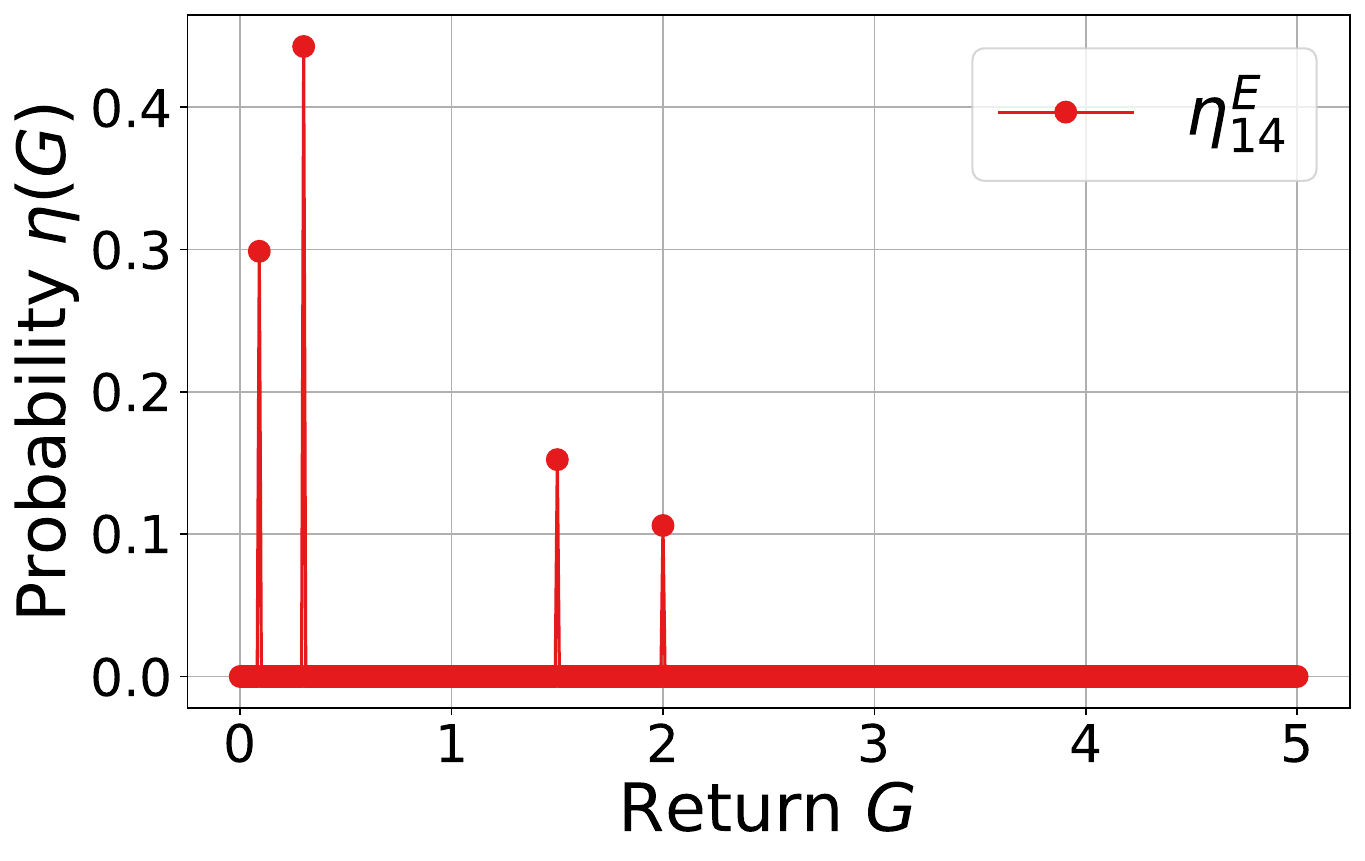}
 \end{minipage}
 \begin{minipage}[t!]{0.32\textwidth}
   \centering
   \includegraphics[width=0.95\linewidth]{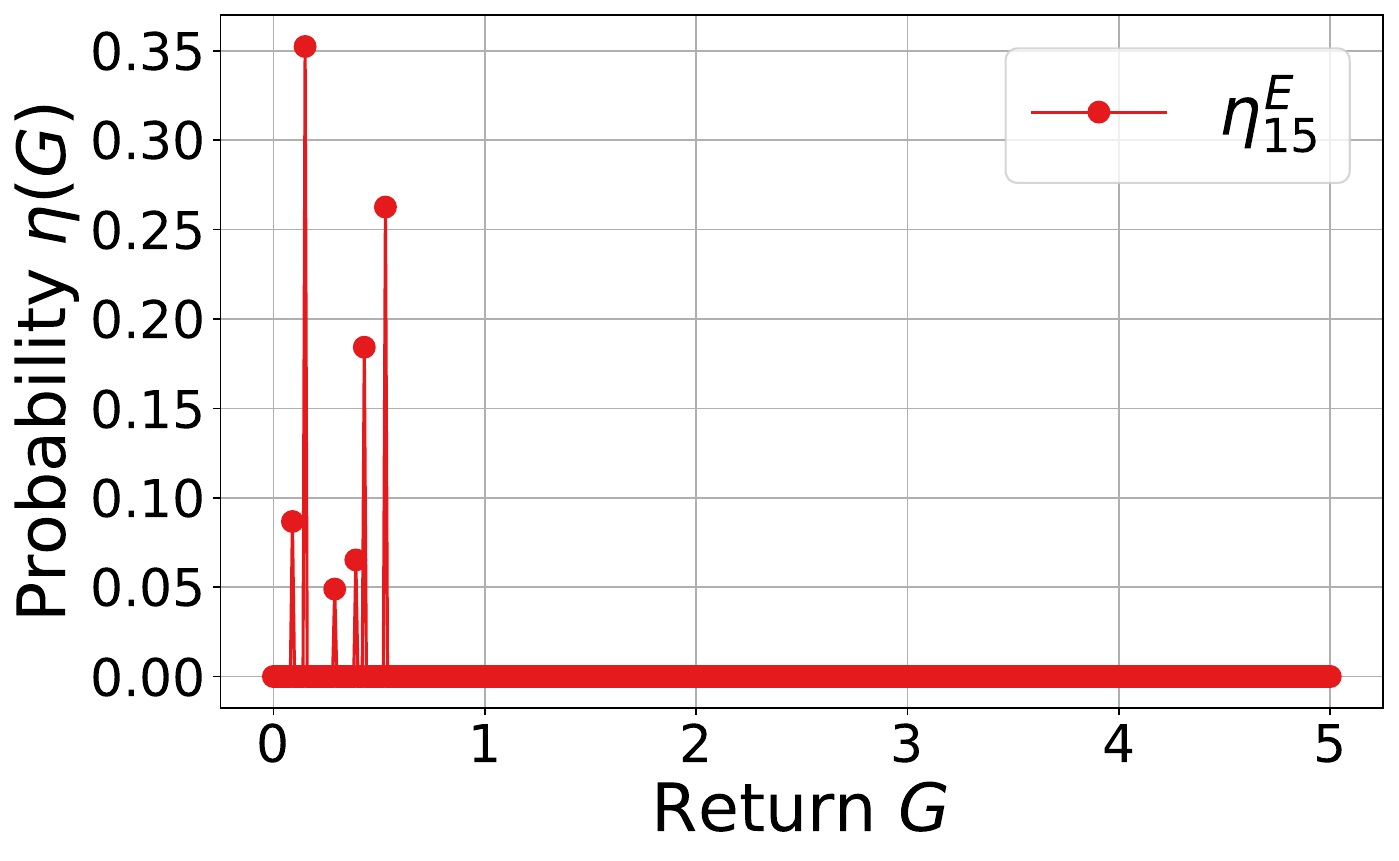}
 \end{minipage}
 \caption{\footnotesize Plot of $\eta^E_{13},\eta^E_{14}$, and $\eta^E_{15}$.}
 \label{fig: 5}
  \end{figure}
   
\subsection{Details Experiment 2}
\label{apx: details exp2}
 
Experiment 2 is made of two parts, the first in which we execute \tractor on the
 MDP (and data) adopted also in Experiment 1, and the other where we use
 simulated data. We describe here the former, while we present the latter more
 in detail in Appendix \ref{apx: analysis on simulated data}.
 
 We consider the policy of the 10th participant (chosen arbitrarily) to the
 survey, and we execute \tractor multiple times with varying values of the input
 parameters, specifically: we always use $K=$10000 trajectories for estimating
 the return distribution of the 10th participant's policy, and the return
 distribution of the optimal policies computed along the way; we make 5 runs
 with each combination of parameters with different seeds. We execute for $T=70$
 iterations using Lipschitz constant $L=10$, which means that we consider only
 utilities $U\in\overline{\underline{\fU}}_L$ satisfying $|U(G)-U(G')|\le
 10|G-G'|$ for all $G,G'\in[0,5]$ (the horizon is 5). As initial utility
 $\overline{U}_0$, we try $U_{\text{sqrt}}$, $U_{\text{square}}$, and
 $U_{\text{linear}}$ (see Appendix \ref{apx: additional experiment}), and as
 learning rates we try $0.01,0.5,5,100,1000,10000$.

 The experiment has been conducted in some hours on a personal computer with
processor AMD Ryzen 5 5500U with Radeon Graphics (2.10 GHz), with 8,00 GB of
RAM.
 
 We note that the choice of $\overline{U}_0$ is rather irrelevant for the shape
 of the extracted $\widehat{U}$, but it matters for its ``location'', as shown in
 Fig. \ref{fig: plots exp2 main paper}.
 
 \begin{figure}[h!]
   \centering
   \includegraphics[scale=0.35]{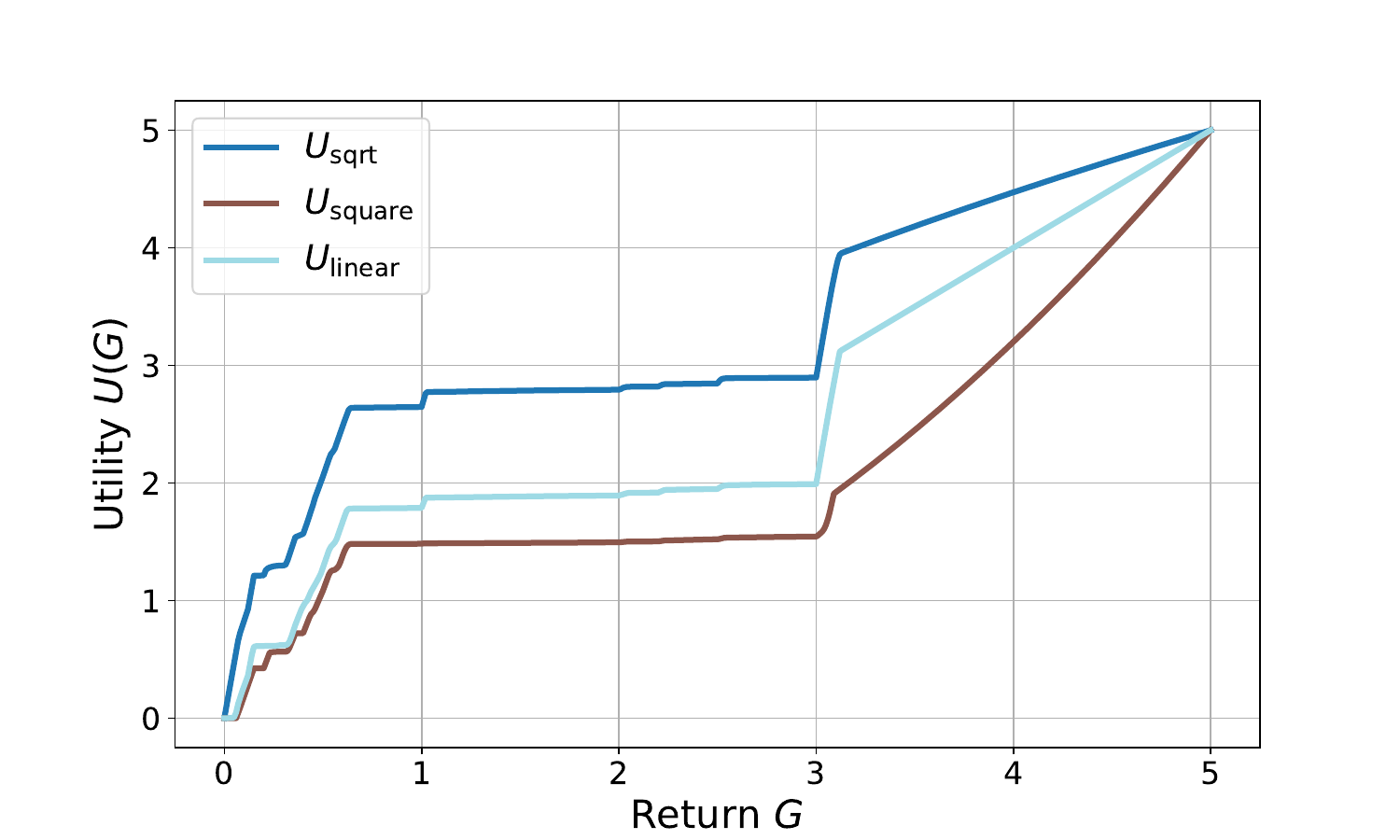}
 \caption{\footnotesize Utilities computed by \tractor starting with the
 $\overline{U}_0$ in the legend ($\alpha=100$).}
 \label{fig: plots exp2 main paper}
  \end{figure}
 
  To view the sequence of utilities extracted by \tractor during the run, see
  Appendix \ref{apx: sequence utilities tractor}, while in Appendix \ref{apx:
  explanation large learning rate} we explain better why the best learning rate
  is large.
 
 \subsubsection{The sequence of utilities extracted by \tractor on the collected data}
 \label{apx: sequence utilities tractor}
 
 We now present some plots representing the sequence of utilities
 extracted by \tractor during its execution. Specifically, we consider initial
 utility $\overline{U}_0=U_{\text{square}}$, and we use learning rates
 $\alpha\in[0.01,0.5,5,100,1000,10000]$. We plot the sequence of utilities
 considered by \tractor during its execution in Figures \ref{fig: sequence 001
 05}, \ref{fig: sequence 5 100}, and \ref{fig: sequence 1000 10000}, where we
 adopt notation that $U_t$ denotes the utility extracted at iteration $t$, and
 the number in the legend represents the (non)compatibility of that utility.
 We consider again participant 10.
 
 We observe that, for smaller learning rates (e.g., $\alpha\in[0.01,0.5,5]$),
 the utilities as well as the (non)compatibilities) do not change much (Figure
 \ref{fig: sequence 001 05} and Figure \ref{fig: sequence 5 100} left), while
 for larger learning rates, we obtain more consistent changes (Figure \ref{fig:
 sequence 5 100} left and Figure \ref{fig: sequence 1000 10000}).
 
 Clearly, larger learning rates require less iterations to achieve small values
 of (non)compatibilities. Nevertheless, too large values (e.g., $\alpha=10000$)
 are outperformed by intermediate values (e.g., $\alpha=100$).
 
 \begin{figure}[!h]
   \centering
   \begin{minipage}[t!]{0.49\textwidth}
     \centering
     \includegraphics[width=0.95\linewidth]{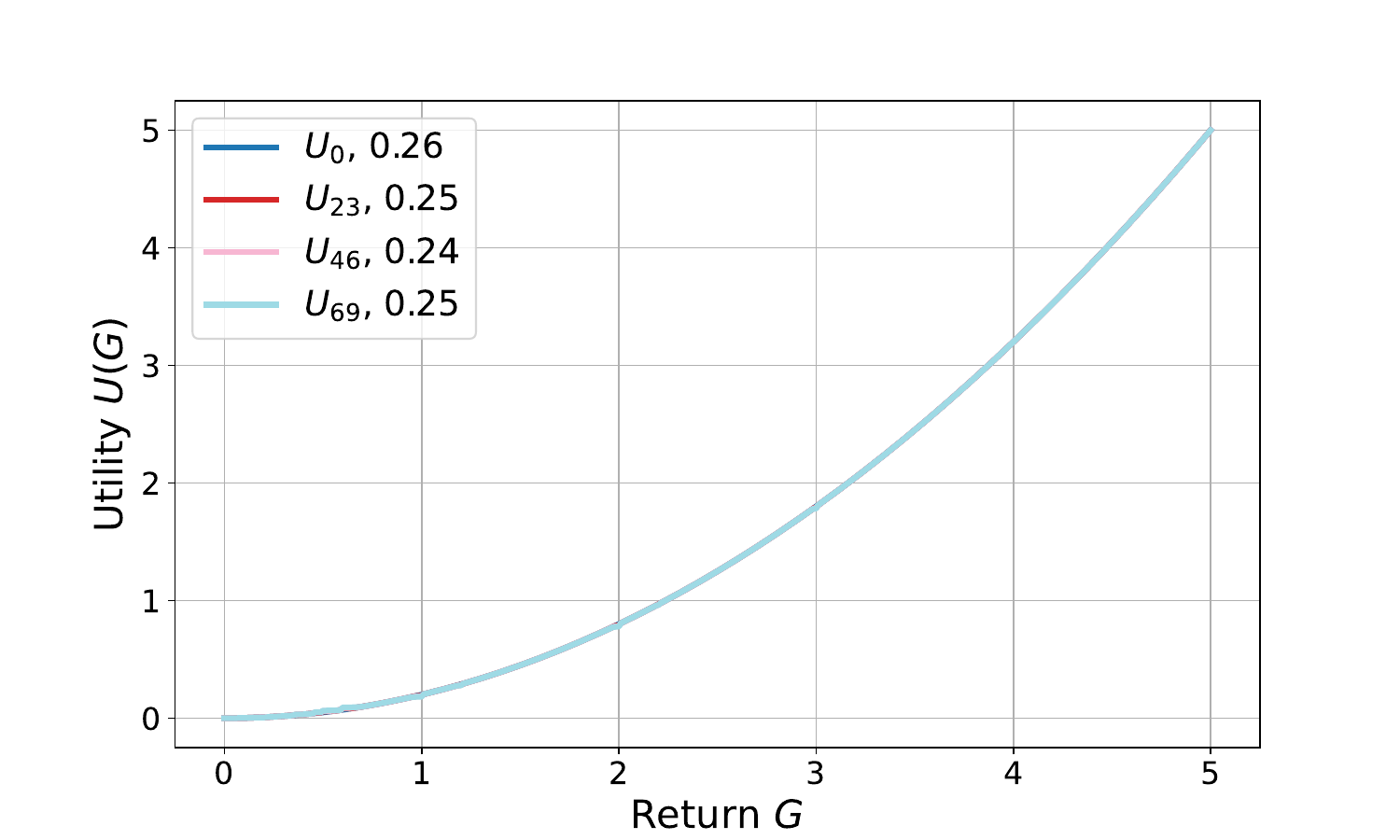}
 \end{minipage}
 \begin{minipage}[t!]{0.49\textwidth}
     \centering
     \includegraphics[width=0.95\linewidth]{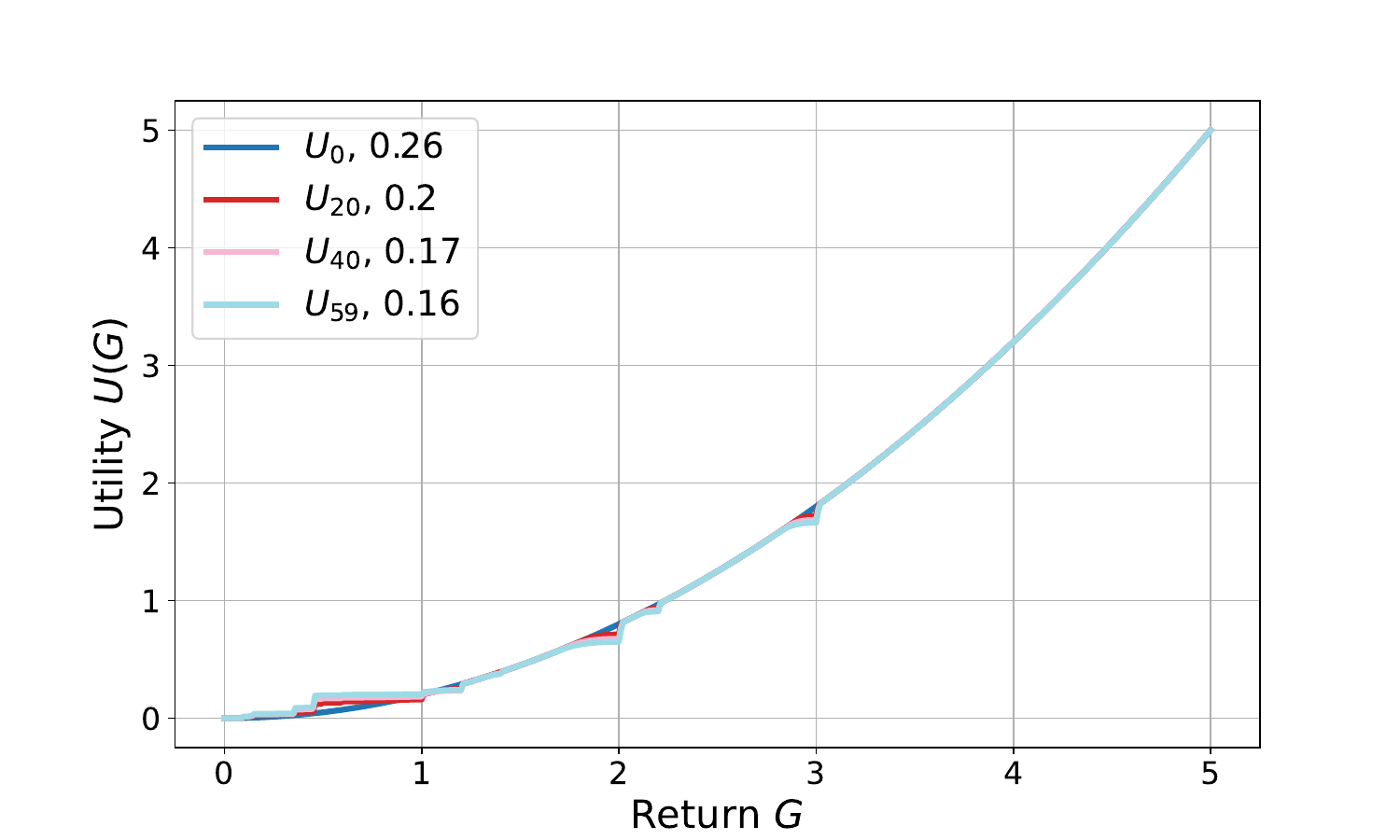}
 \end{minipage}
 \caption{\footnotesize (Left) $\alpha=0.01$.
 (Right) $\alpha=0.5$.}
 \label{fig: sequence 001 05}
  \end{figure}
 
  \begin{figure}[!h]
   \centering
   \begin{minipage}[t!]{0.49\textwidth}
     \centering
     \includegraphics[width=0.95\linewidth]{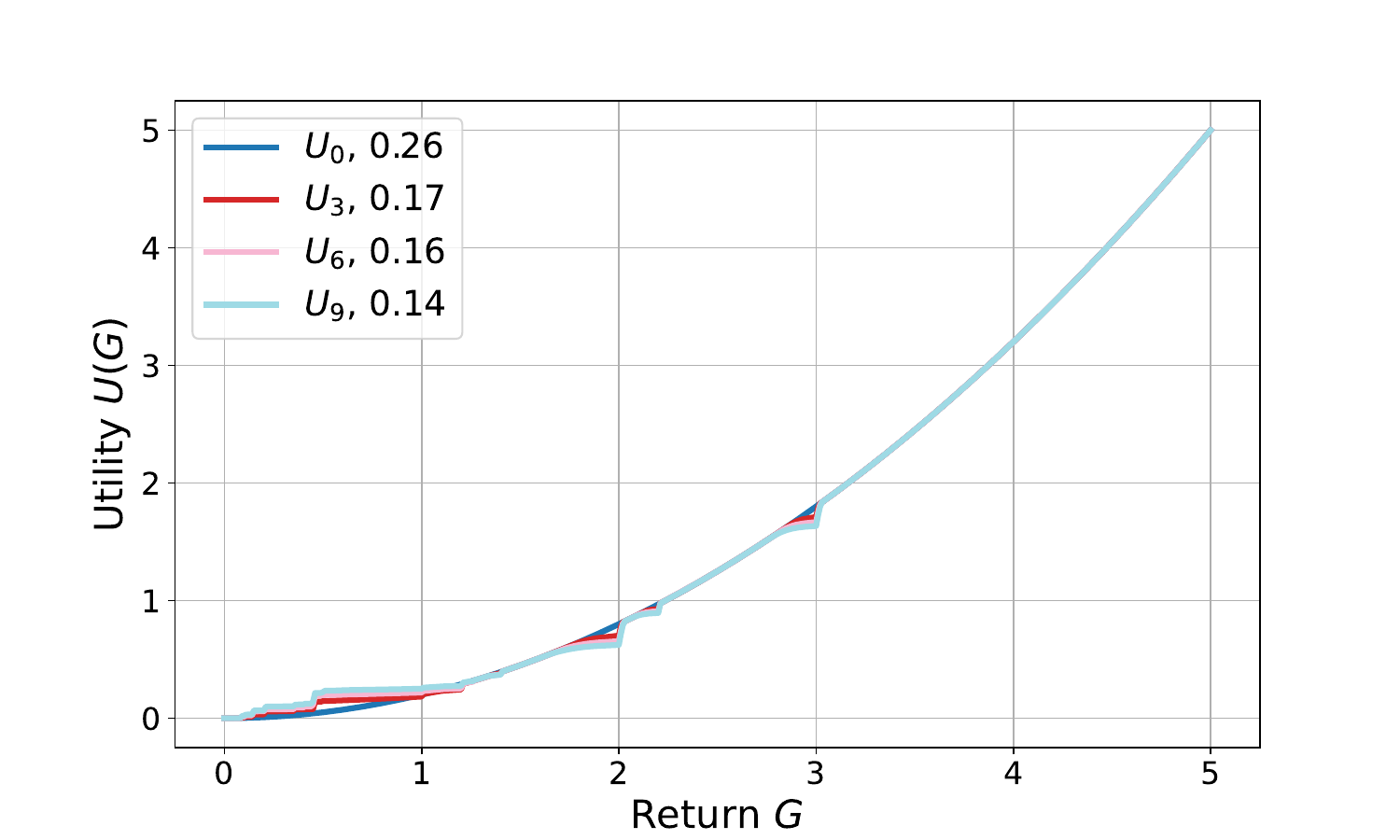}
 \end{minipage}
 \begin{minipage}[t!]{0.49\textwidth}
     \centering
     \includegraphics[width=0.95\linewidth]{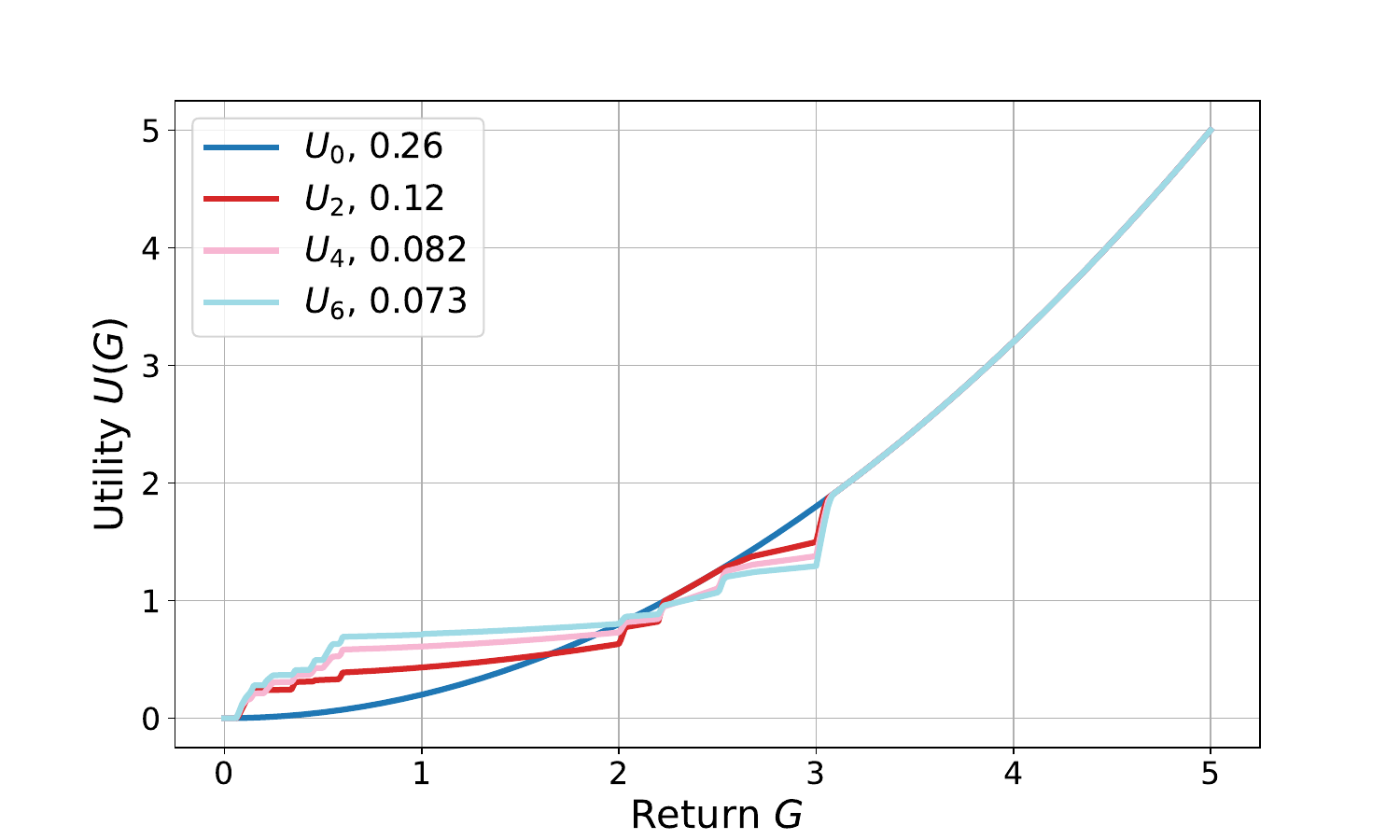}
 \end{minipage}
 \caption{\footnotesize (Left) $\alpha=5$.
 (Right) $\alpha=100$.}
 \label{fig: sequence 5 100}
  \end{figure}
 
  \begin{figure}[!h]
   \centering
   \begin{minipage}[t!]{0.49\textwidth}
     \centering
     \includegraphics[width=0.95\linewidth]{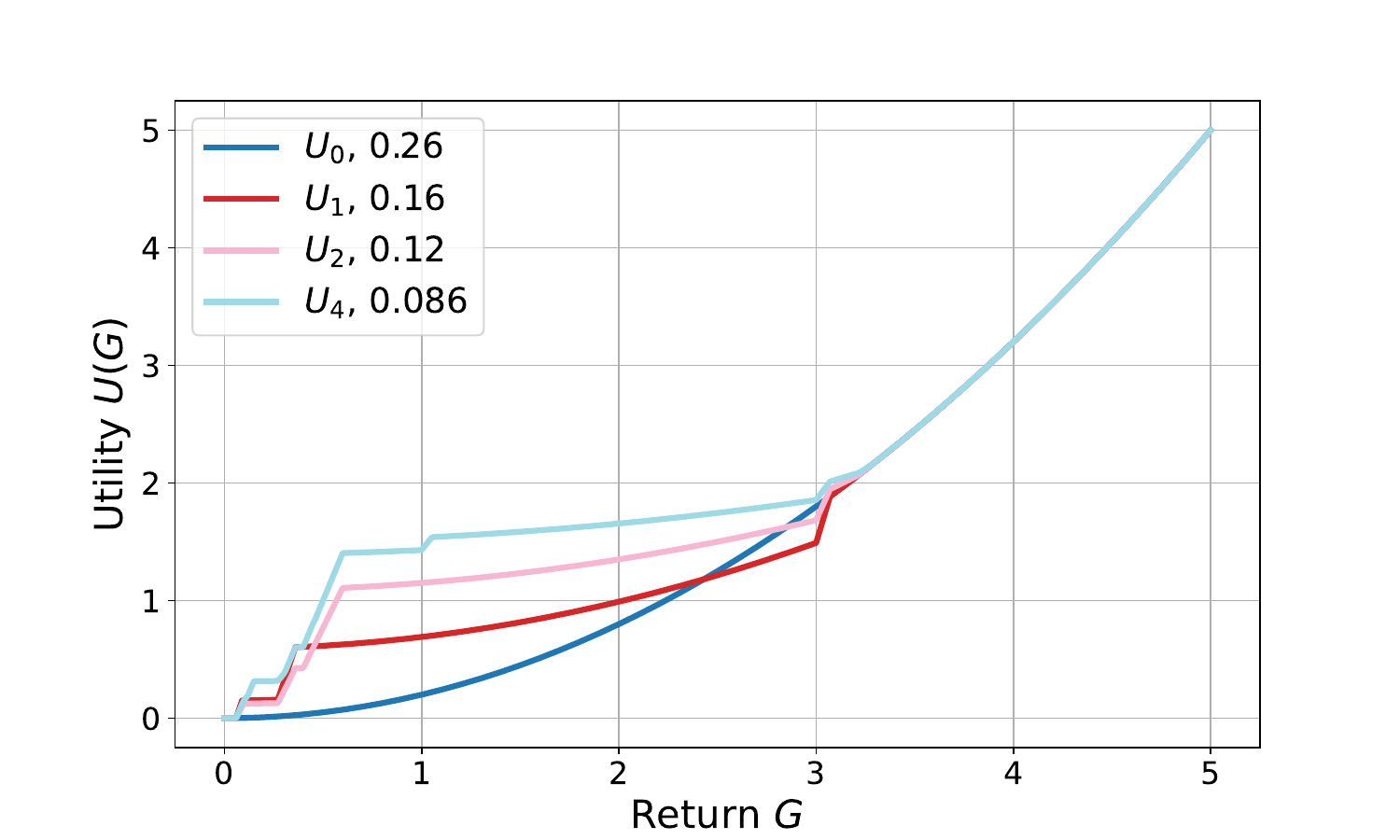}
 \end{minipage}
 \begin{minipage}[t!]{0.49\textwidth}
     \centering
     \includegraphics[width=0.95\linewidth]{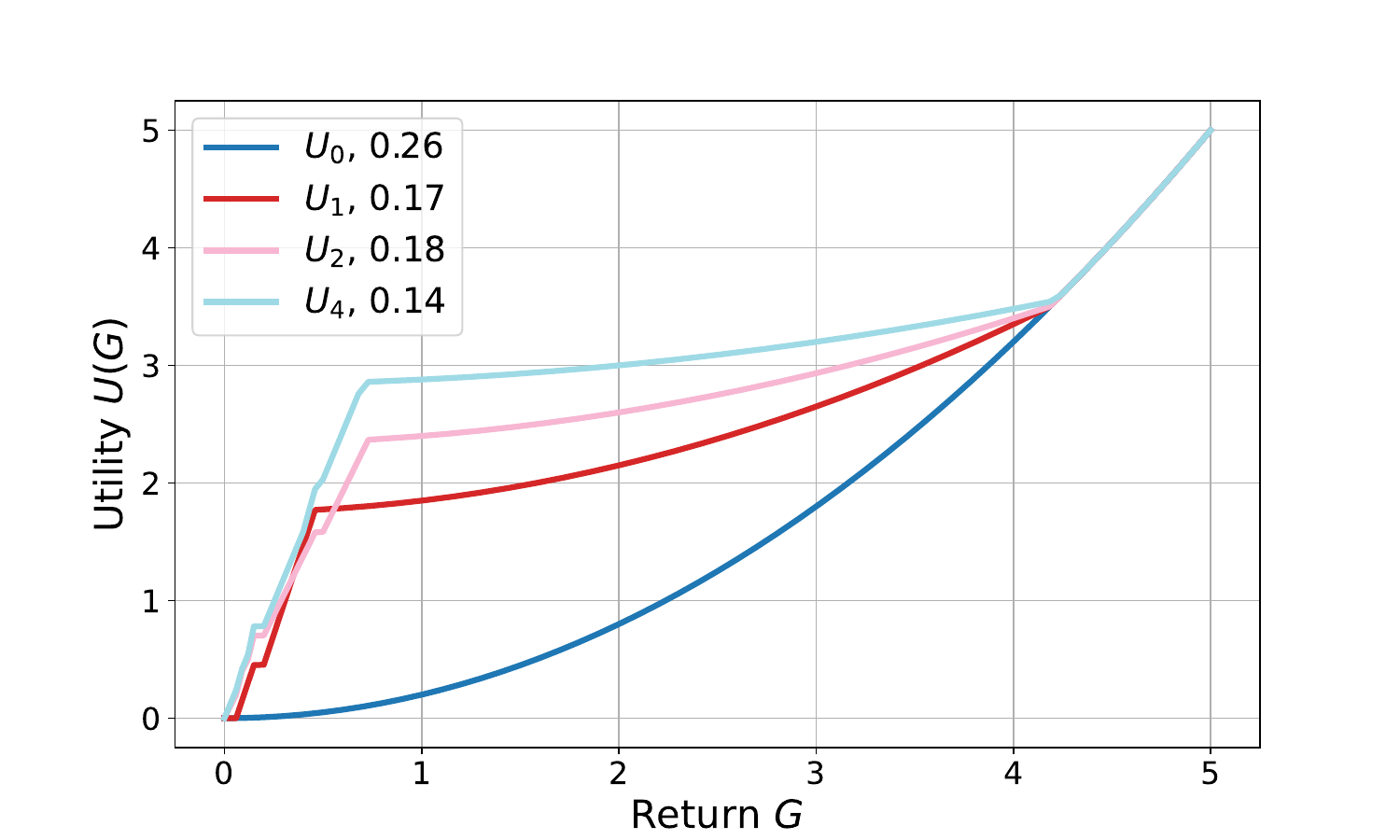}
 \end{minipage}
 \caption{\footnotesize (Left) $\alpha=1000$.
 (Right) $\alpha=10000$.}
 \label{fig: sequence 1000 10000}
  \end{figure}
 
 \subsubsection{An explanation for a large learning rate}
 \label{apx: explanation large learning rate}

 As mentioned in the main paper, there two reasons why a large learning rate is
 required: $(i)$ the feasible set is large and contains utilities that lie on
 the boundaries of set
 $\overline{\underline{\fU}}_L$,\footnote{$\overline{\underline{\fU}}_L$ forces
 utilities to be \emph{increasing}, i.e., with constraints $U(G_1)\le
 U(G_2)\,\forall G_1\le G_2$. The plateau in Fig. \ref{fig: plots exp2 main
 paper} (right) indicates that $U(G_1)= U(G_2)\,\forall G_1\le G_2,
 G_1,G_2\in[1,3]$, thus, it represents a boundary.} causing larger step sizes to
 converge sooner;$(ii)$ the projection onto $\overline{\underline{\fU}}_L$
 results in minimal changes of utility even with very large steps.
 
 Now, we show visually that the projection update represented by operator
 $\Pi_{\overline{\underline{\fU}}_L}$ crucially neglects small variations in the
 (non-projected) utilities, requiring us to increase the step size.
 
 Thus, the intuition is that we need a large learning rate because the projection
 step neglects small variations. To show this, we take as initial utility
 $\overline{U}_0=U_{\text{sqrt}}$, two return distributions $\eta^*_0,\eta^E$,
 where $\eta^*$ coincides with the distribution of an optimal policy for
 $U_{\text{sqrt}}$, and $\eta^E$ is the return distribution of the policy played
 by participant 10. These distributions are plotted in Figure \ref{fig:
 difference return distributions} left, and their difference is plotted in Figure
 \ref{fig: difference return distributions} right. In particular, we note that
 the two distributions are rather different, with the expert's distribution
 $\eta^E$ that is more risk-averse, in that it provides higher probability to
 returns around $G=0.5$, while the optimal distribution $\eta^*_0$ is more
 risk-lover, in that it assigns some probability to higher returns $G\ge 1$, but
 suffering from also high probability to small returns $G\le 0.3$.
 
 \begin{figure}[!h]
   \centering
   \begin{minipage}[t!]{0.49\textwidth}
     \centering
     \includegraphics[width=0.95\linewidth]{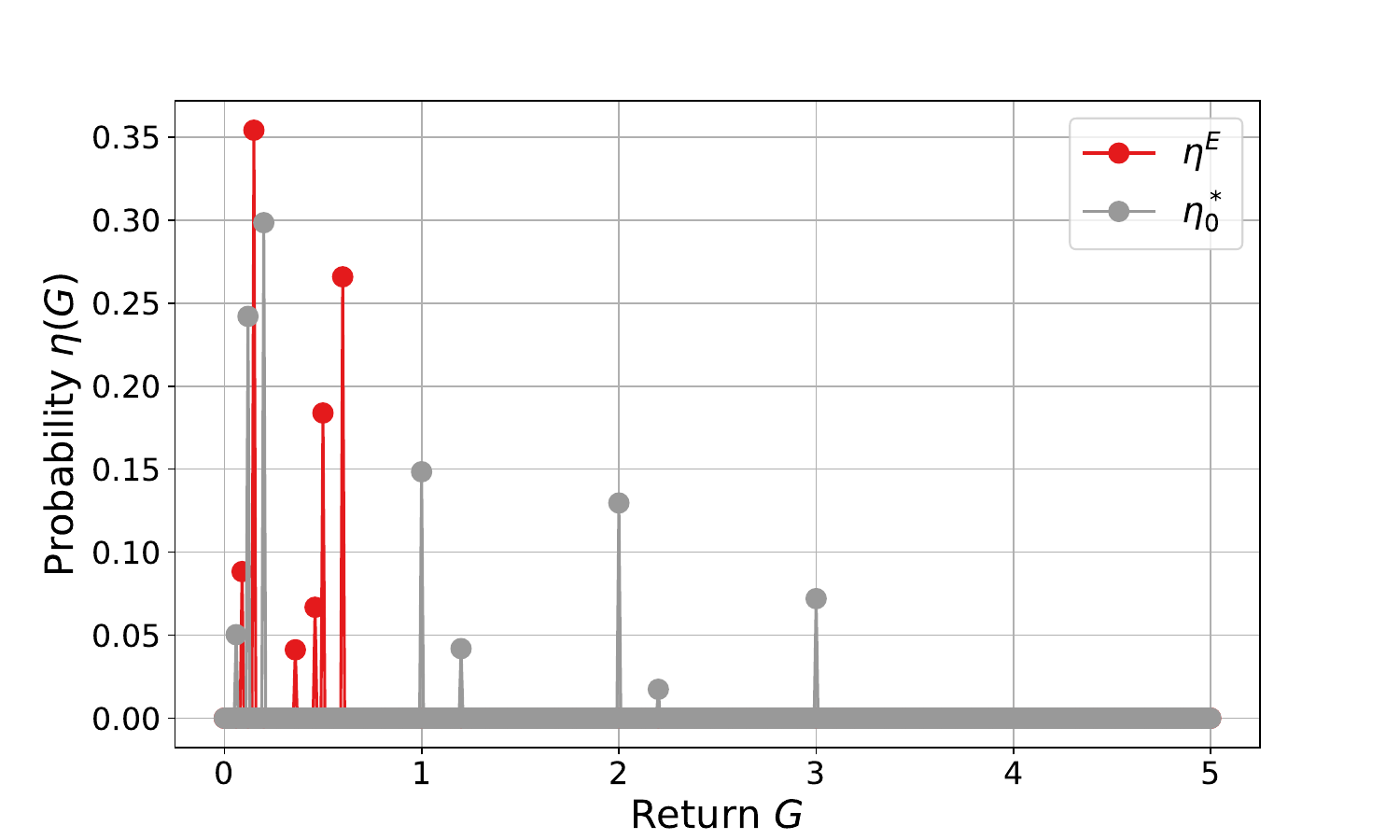}
 \end{minipage}
 \begin{minipage}[t!]{0.49\textwidth}
     \centering
     \includegraphics[width=0.95\linewidth]{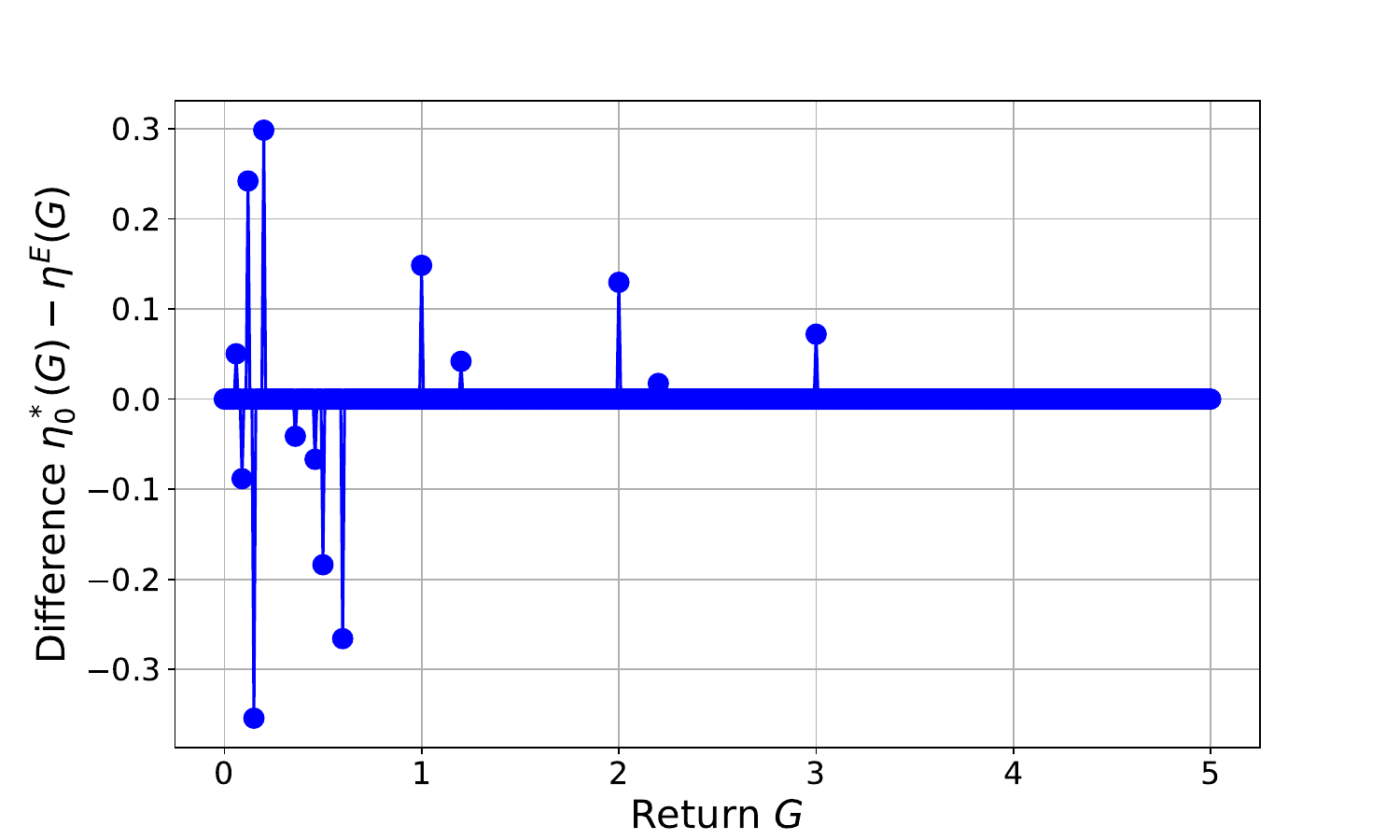}
 \end{minipage}
 \caption{\footnotesize (Left) Plot of $\eta^*_0$ and $\eta^E$.
 (Right) Plot of $\eta^*_0-\eta^E$.}
 \label{fig: difference return distributions}
  \end{figure}
 
 We aim to perform the \tractor update rule:
 \begin{align*}
   \overline{U}_1'\gets \overline{U}_0 -\alpha (\eta^*-\eta^E),
 \end{align*}
 with some learning rate $\alpha$, and then to perform the projection:
 \begin{align*}
   \overline{U}_1\gets \Pi_{\overline{\underline{\fU}}_L}[\overline{U}_1'].
 \end{align*} 
 We execute the update with the following values of steps size:
 $\alpha\in\{0.01,0.5,5,100,1000,10000\}$, and we plot the corresponding updated
 utilities $\overline{U}_1'$ and $\overline{U}_1$ in Figures \ref{fig: explain
 utilities 001 05}, \ref{fig: explain utilities 5 100}, and \ref{fig: explain
 utilities 1000 10000}.
 
 \begin{figure}[!h]
   \centering
   \begin{minipage}[t!]{0.49\textwidth}
     \centering
     \includegraphics[width=0.95\linewidth]{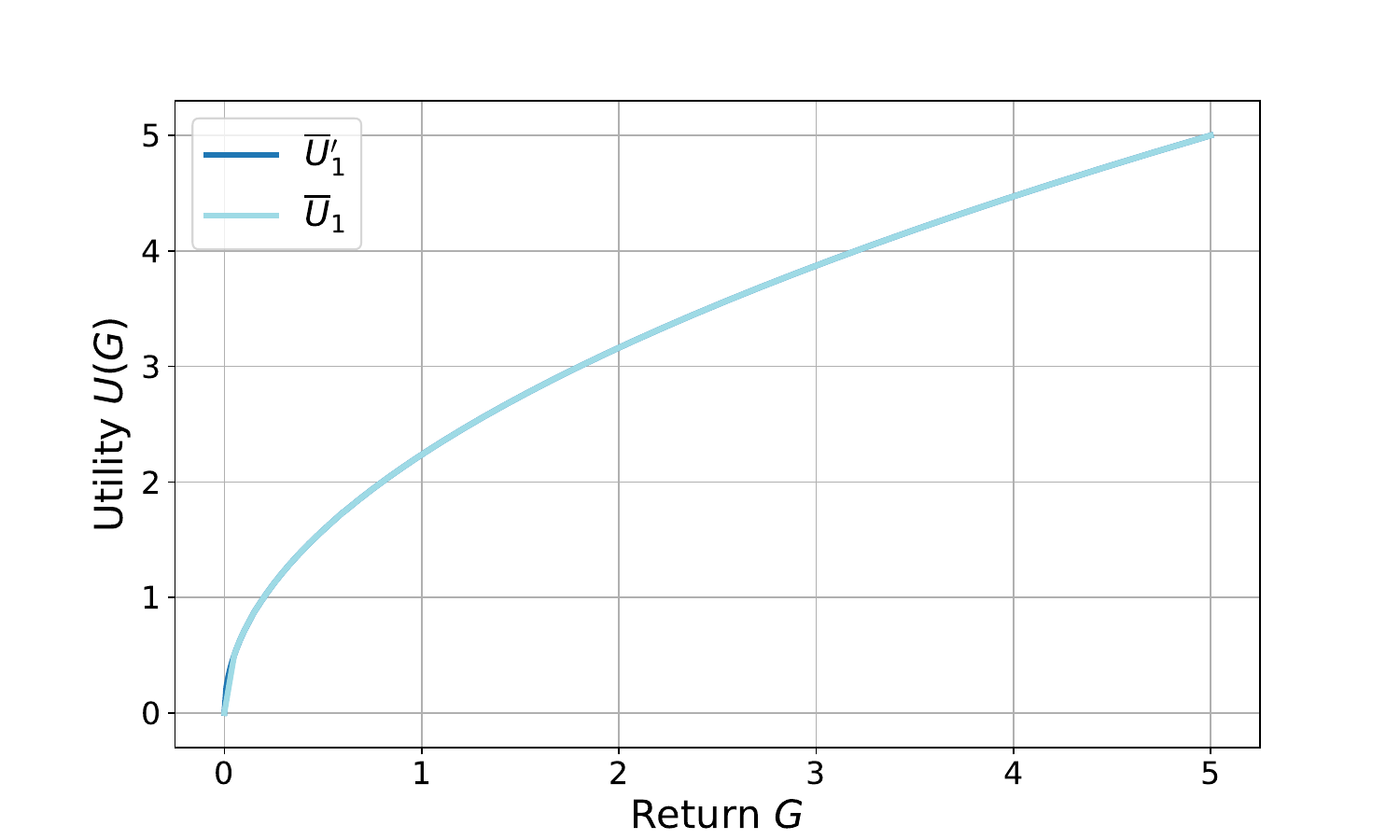}
 \end{minipage}
 \begin{minipage}[t!]{0.49\textwidth}
     \centering
     \includegraphics[width=0.95\linewidth]{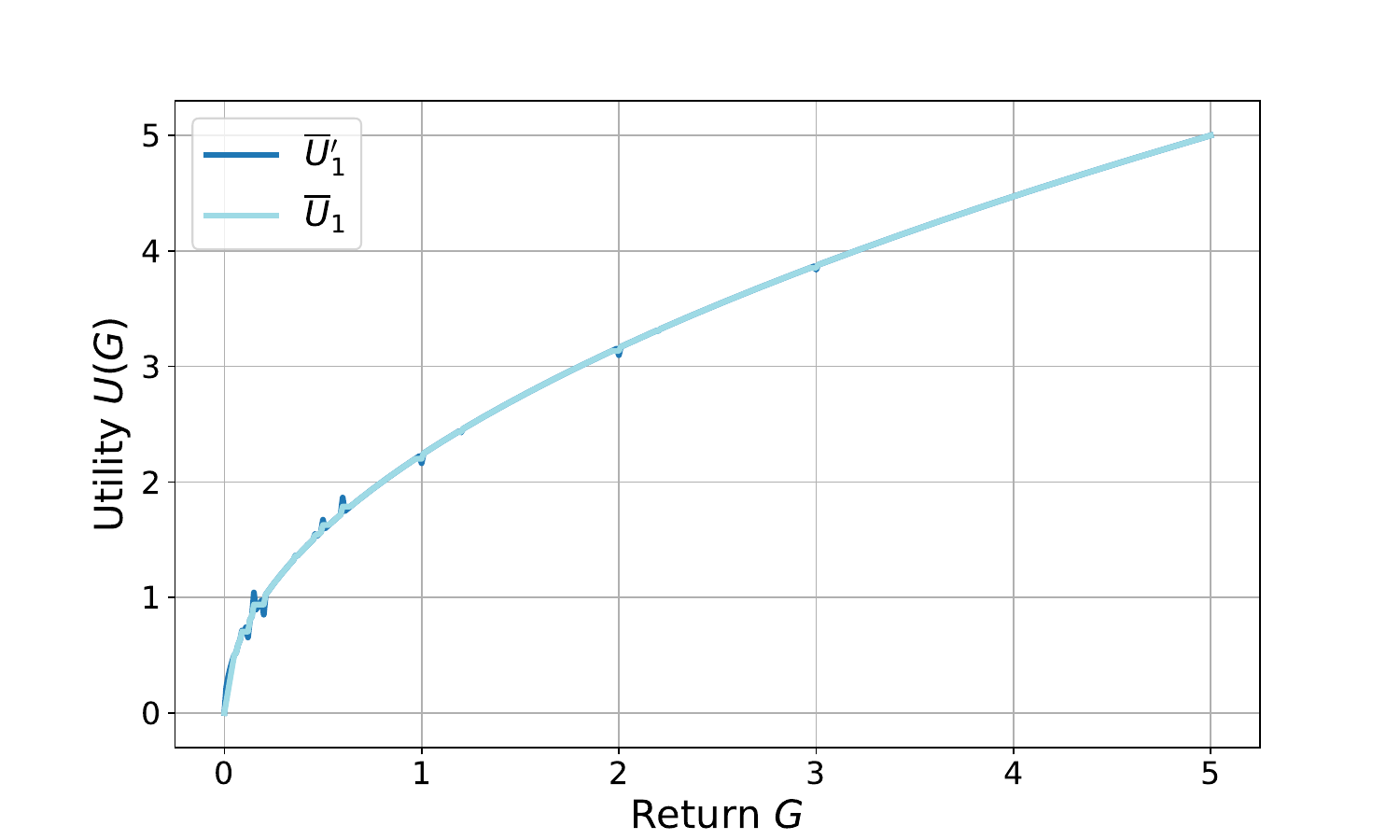}
 \end{minipage}
 \caption{\footnotesize (Left) $\alpha=0.01$.
 (Right) $\alpha=0.5$.}
 \label{fig: explain utilities 001 05}
  \end{figure}
 
  \begin{figure}[!h]
   \centering
   \begin{minipage}[t!]{0.49\textwidth}
     \centering
     \includegraphics[width=0.95\linewidth]{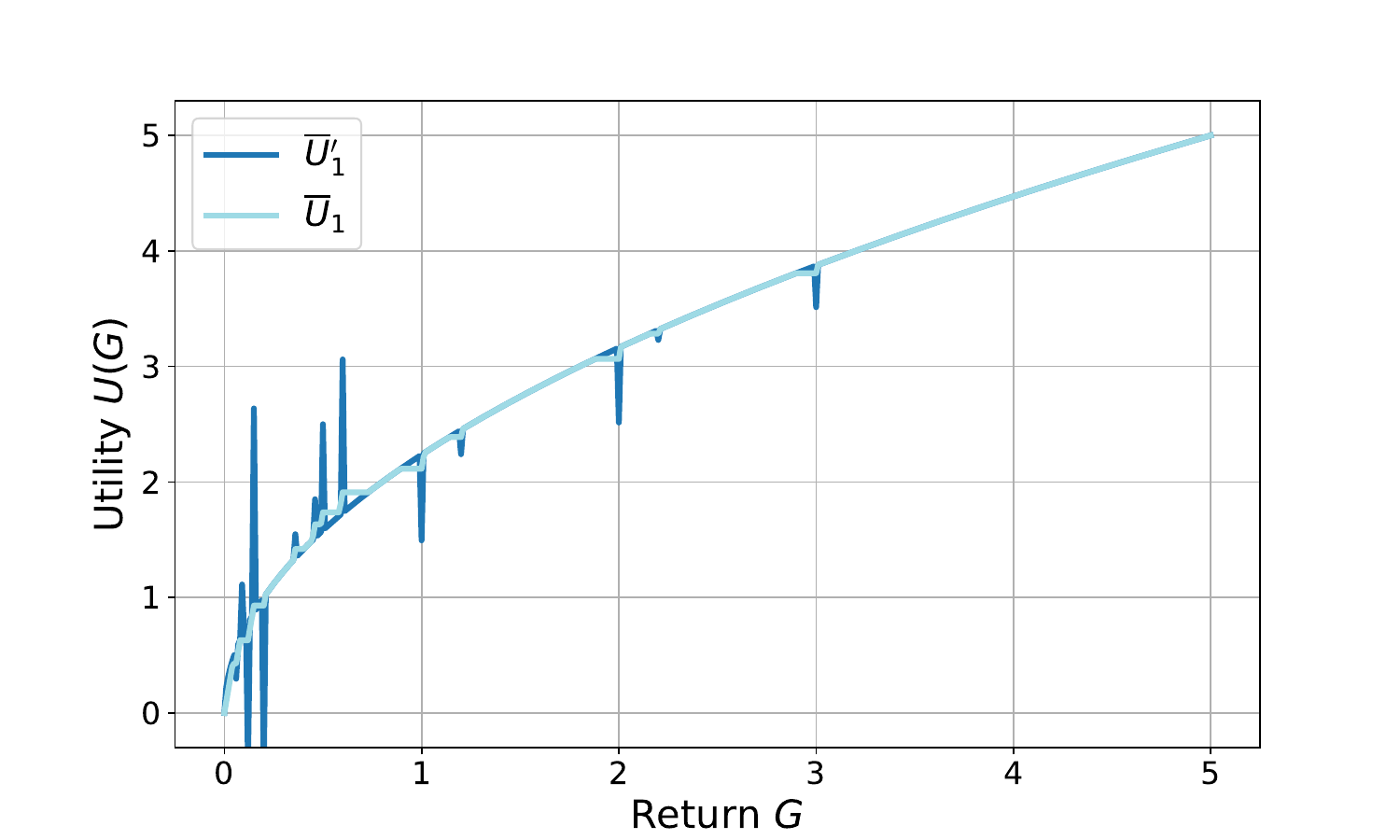}
 \end{minipage}
 \begin{minipage}[t!]{0.49\textwidth}
     \centering
     \includegraphics[width=0.95\linewidth]{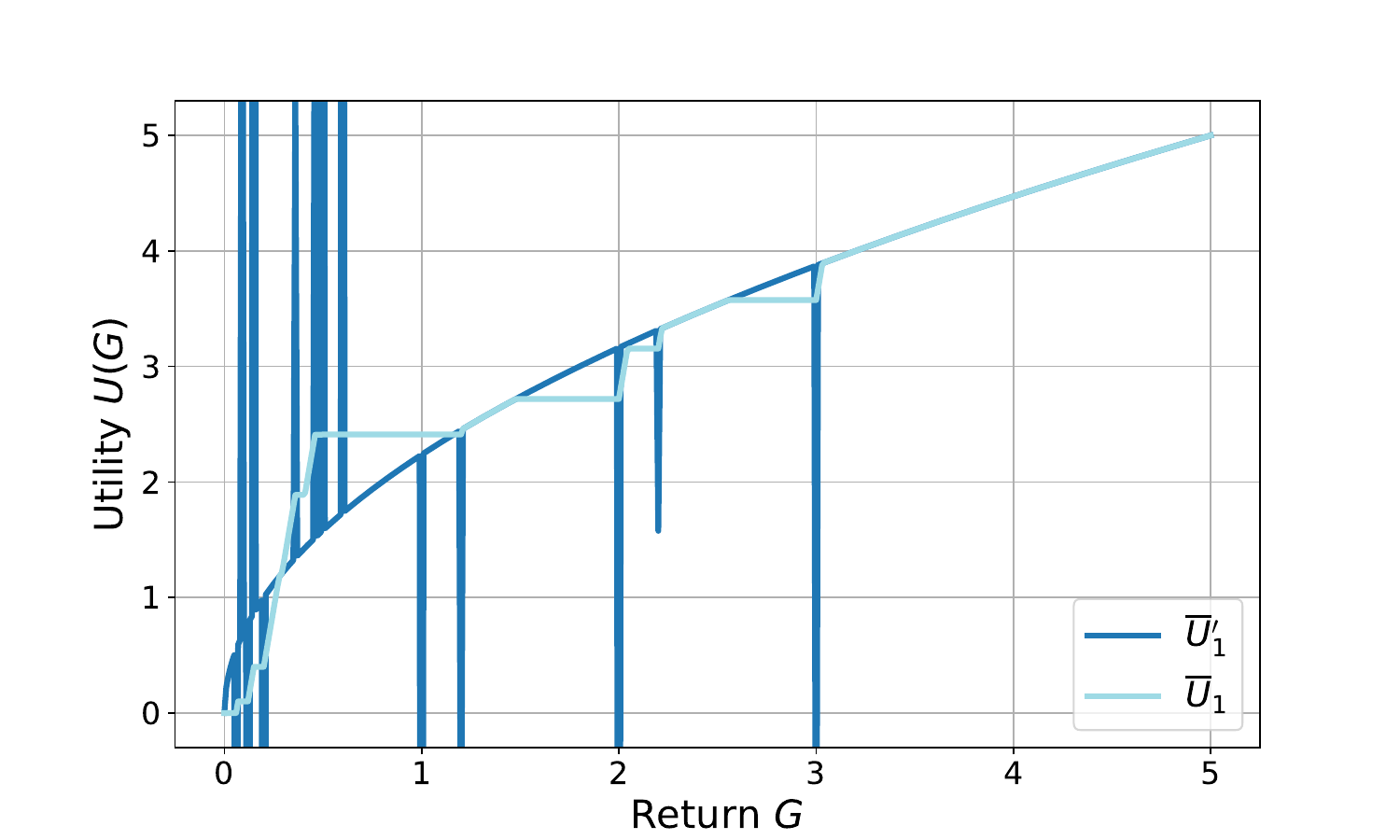}
 \end{minipage}
 \caption{\footnotesize (Left) $\alpha=5$.
 (Right) $\alpha=100$.}
 \label{fig: explain utilities 5 100}
  \end{figure}
 
  \begin{figure}[!h]
   \centering
   \begin{minipage}[t!]{0.49\textwidth}
     \centering
     \includegraphics[width=0.95\linewidth]{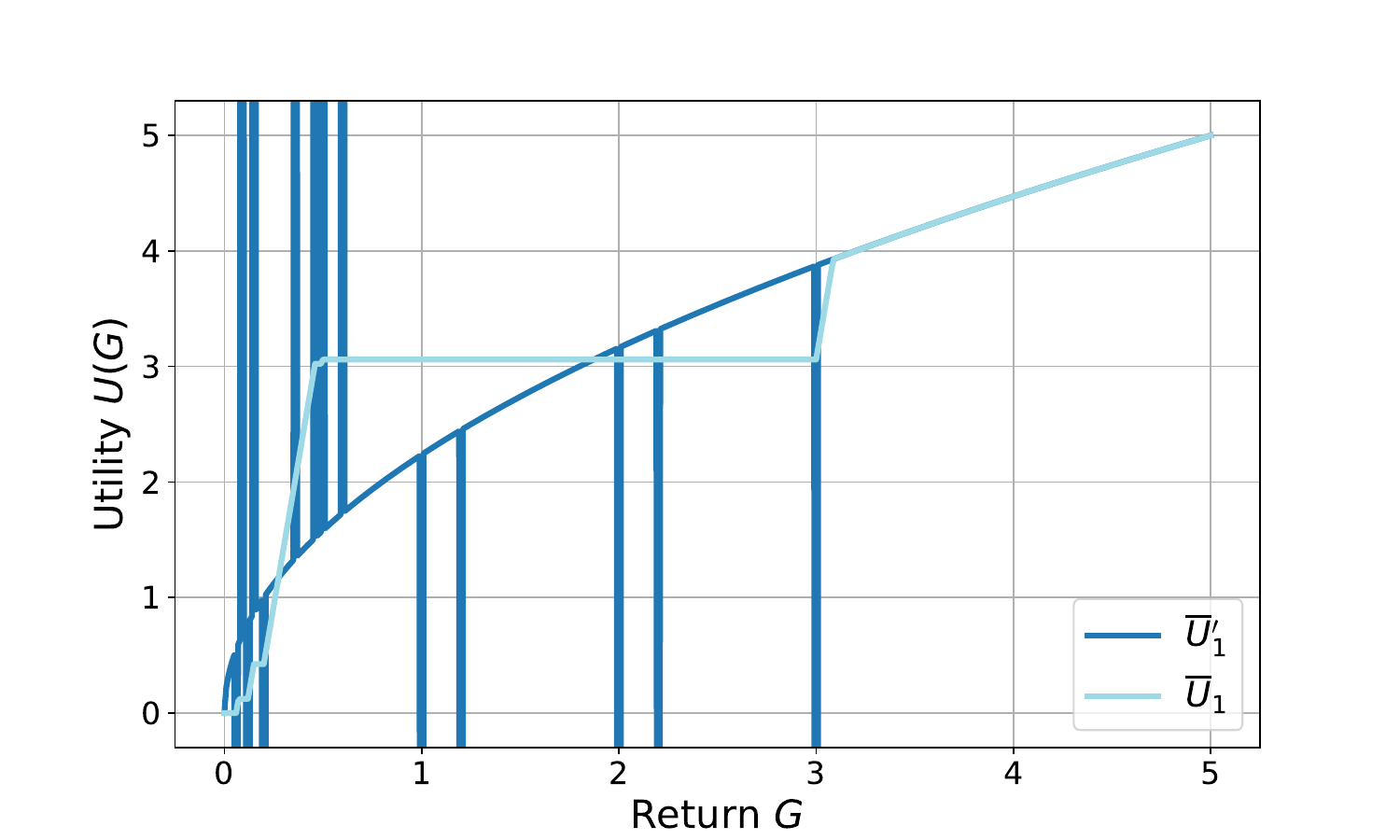}
 \end{minipage}
 \begin{minipage}[t!]{0.49\textwidth}
     \centering
     \includegraphics[width=0.95\linewidth]{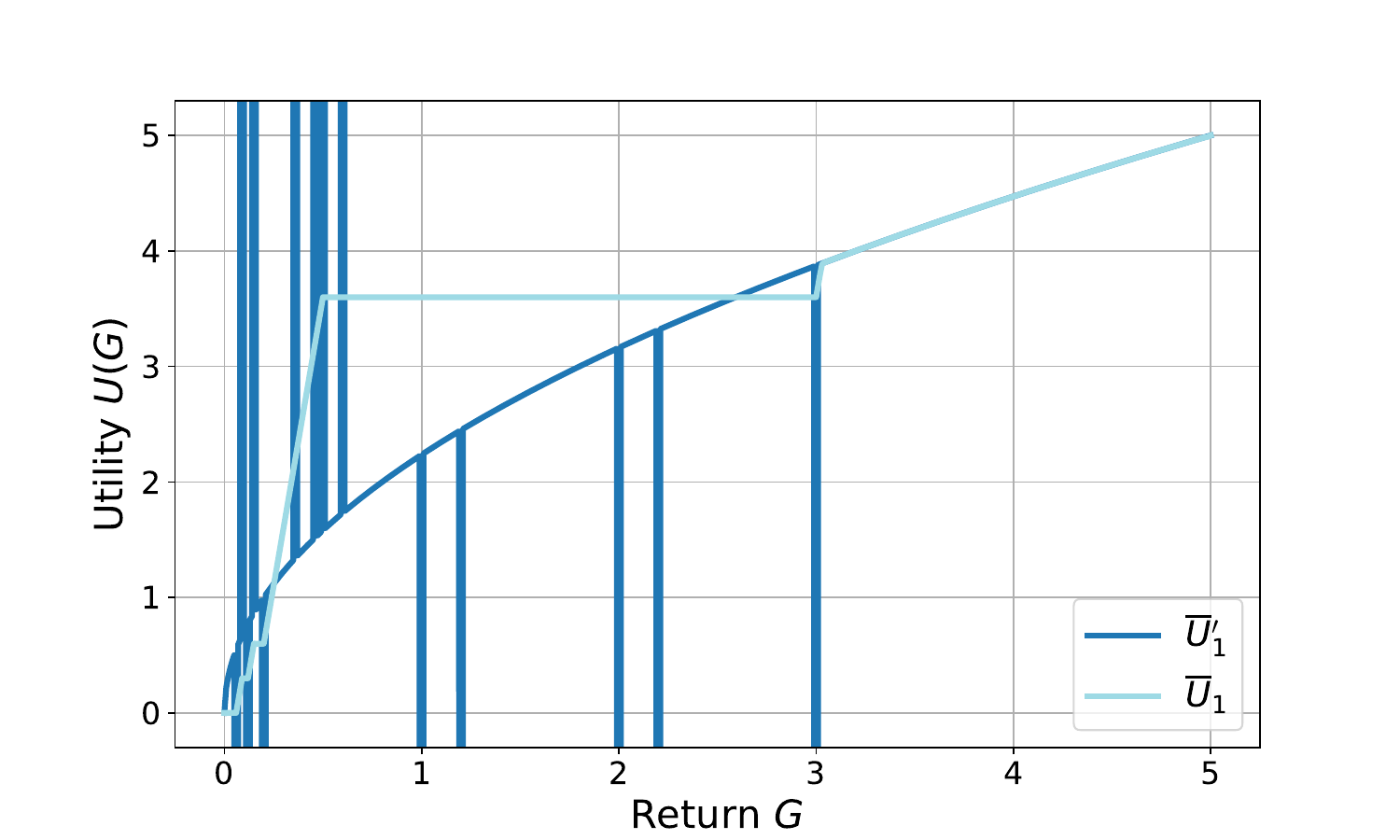}
 \end{minipage}
 \caption{\footnotesize (Left) $\alpha=1000$.
 (Right) $\alpha=10000$.}
 \label{fig: explain utilities 1000 10000}
  \end{figure}
 
  As we can see from Figures \ref{fig: explain utilities 001 05}, \ref{fig:
  explain utilities 5 100}, and \ref{fig: explain utilities 1000 10000}, the
  update $\overline{U}_0\to\overline{U}_1$ obtained with step sizes $< 5$ are
  rather neglectable, so that the return distribution of the new optimal policy
  $\eta^*_1$ for $\overline{U}_1$ still coincides with the previous one
  $\eta^*_0$, and the gradient at the next step is the same. For $\alpha=5$, we
  begin to notice some changes. See Figure \ref{fig: new gradient 5}.
 
  \begin{figure}[!h]
   \centering
   \begin{minipage}[t!]{0.49\textwidth}
     \centering
     \includegraphics[width=0.95\linewidth]{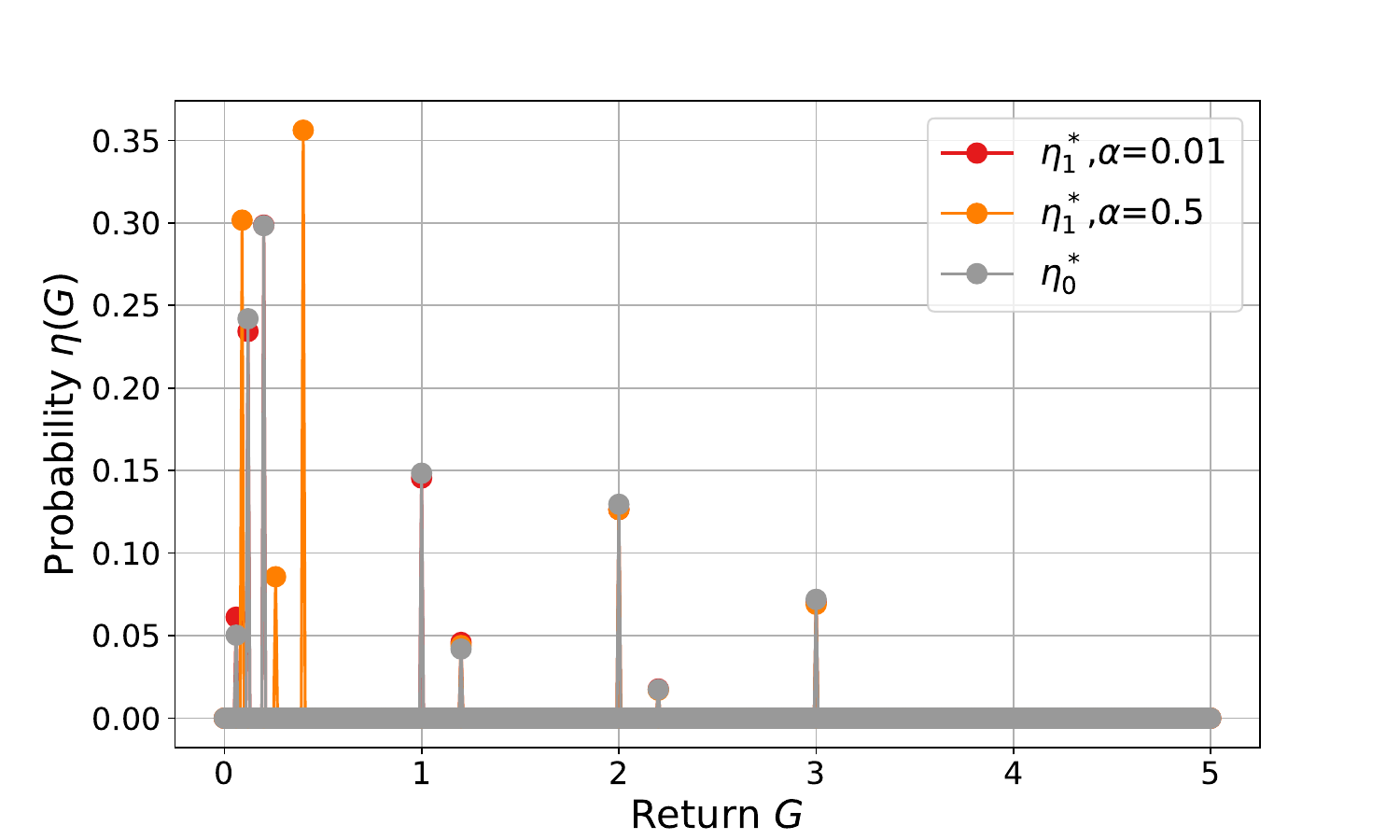}
 \end{minipage}
 \begin{minipage}[t!]{0.49\textwidth}
     \centering
     \includegraphics[width=0.95\linewidth]{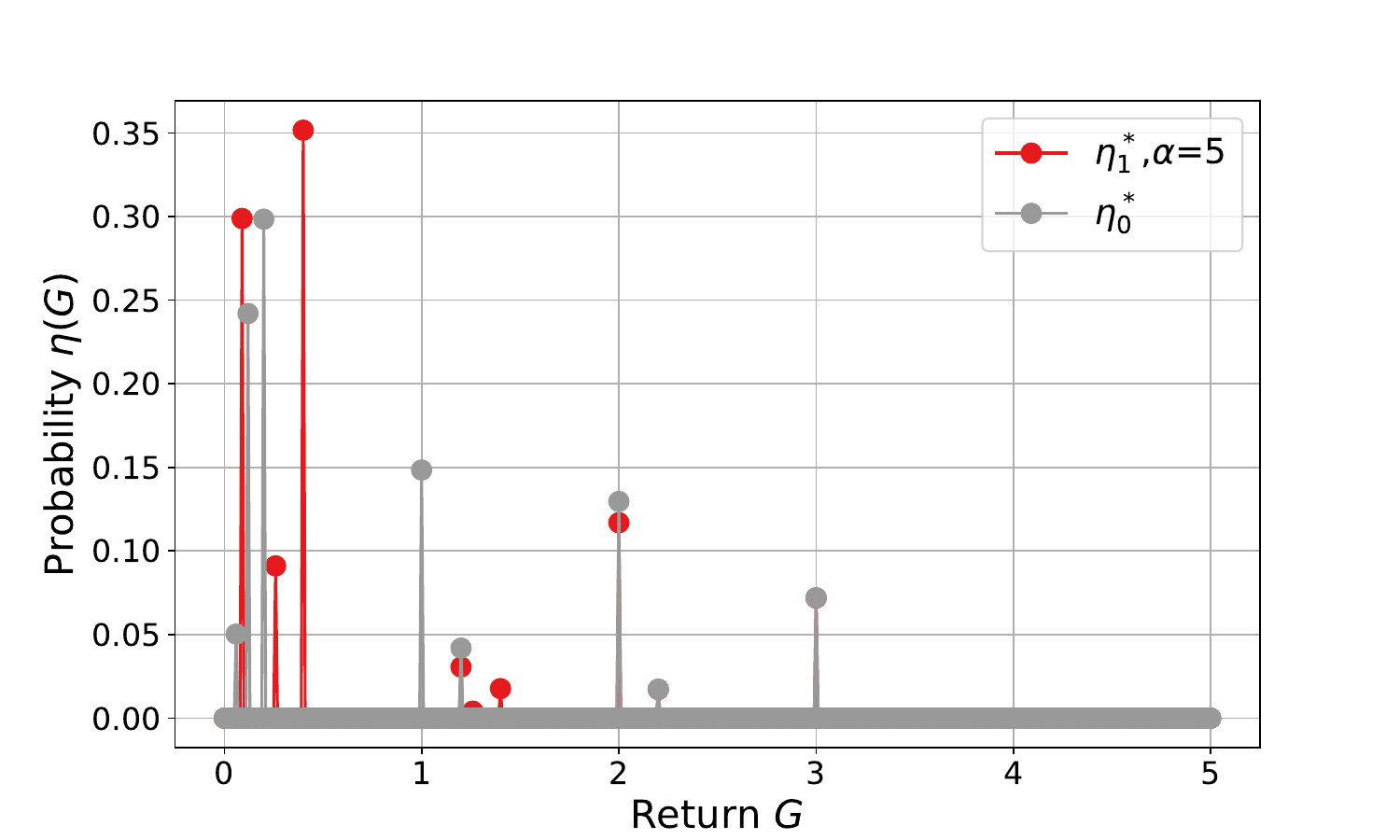}
 \end{minipage}
 \caption{\footnotesize (Left) Comparison of the return distributions $\eta^*_1$
 obtained with $\alpha=0.01$ and $\alpha=0.5$, with $\eta^*_0$. (Right)
 Comparison of the return distribution $\eta^*_1$ obtained with $\alpha=5$, with
 $\eta^*_0$.}
 \label{fig: new gradient 5}
  \end{figure}
  
  Instead, with larger gradients, we observe a non-neglectable change in utility,
  which provides a consistent change in the return distribution for $\alpha=100$,
  and a huge change for $\alpha\in[1000,10000]$ (see Figure \ref{fig: new
  gradient 1000}).
  
  Since neglectable changes in both the utility and the optimal return
  distribution (obtained with small learning rates) mean that we have to update
  the utility many times along the same direction, then the update is equivalent
  to performing a single update in that direction with a huge step size. This
  justifies the use of large learning rates.
 
  \begin{figure}[!h]
   \centering
   \begin{minipage}[t!]{0.49\textwidth}
     \centering
     \includegraphics[width=0.95\linewidth]{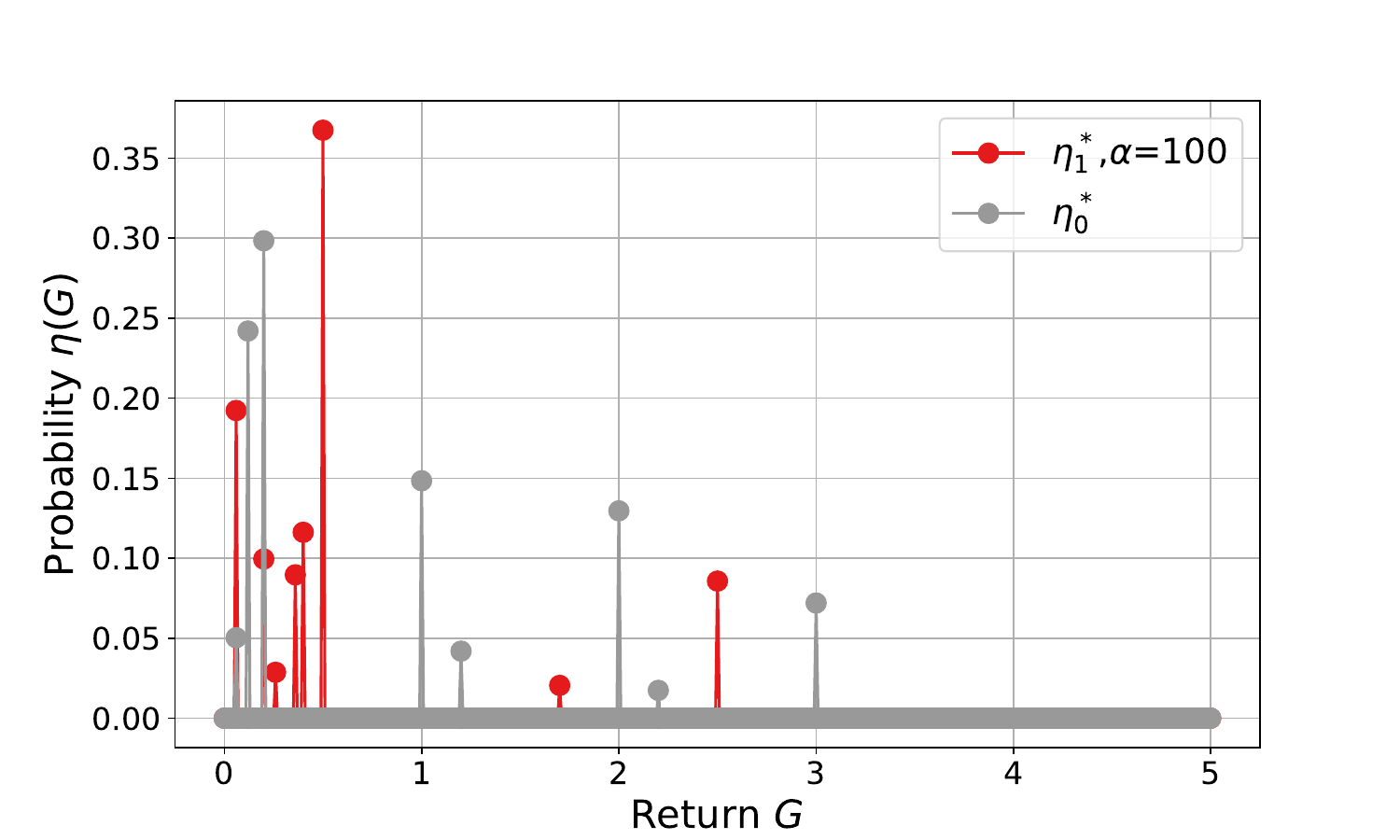}
 \end{minipage}
 \begin{minipage}[t!]{0.49\textwidth}
     \centering
     \includegraphics[width=0.95\linewidth]{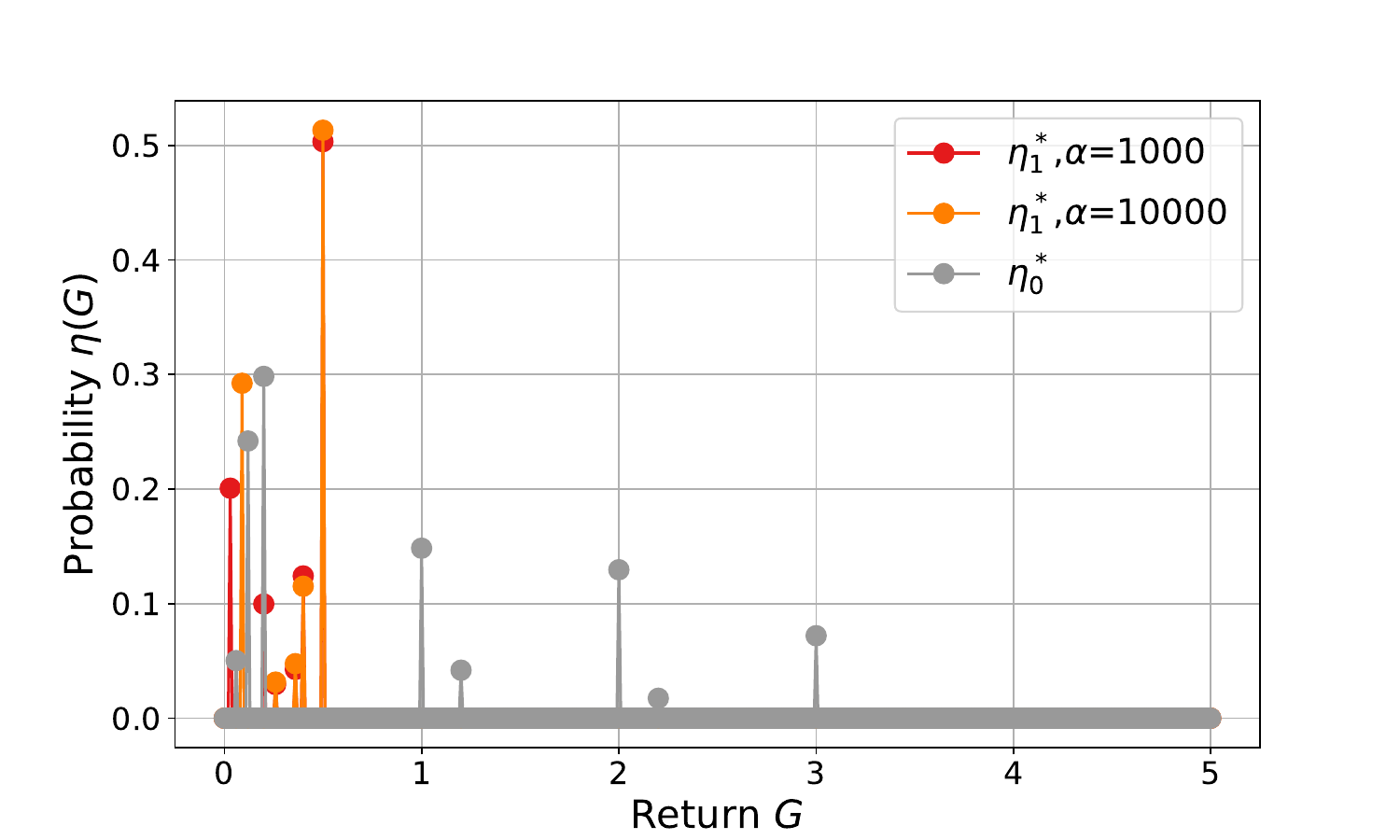}
 \end{minipage}
 \caption{\footnotesize (Left) Comparison of the return distribution $\eta^*_1$
 obtained with $\alpha=100$, with $\eta^*_0$. (Right) Comparison of the return
 distributions $\eta^*_1$ obtained with $\alpha=1000$ and $\alpha=10000$, with
 $\eta^*_0$.}
 \label{fig: new gradient 1000}
  \end{figure}
 
 \subsubsection{Analysis on Simulated Data}\label{apx: analysis on simulated
 data}
 
We have executed \tractor on MDPs generated at random. Below (Figures \ref{fig:
1 sim data}-\ref{fig: 3 sim data}), we report the truncated (non)compatibility
values of the utilities extracted by the algorithm as a function of the number
of iterations, in the five different experiments conducted. For the experiments,
we executed for $T=70$ gradient iterations, with parameters $K=10000$ and
$L=10$, as in the first part of the experiment. We found that the best learning
rate is $\alpha=1$.

 To comply with the assumption that there exists a utility
  function for which the expert's policy is (almost) optimal, we compute, in
  each environment, the optimal policy for an S-shaped utility function that is
  convex for small returns, and concave for large returns, and then we inject
  some noise. 
 
 \begin{figure}[!h]
   \centering
   \begin{minipage}[t!]{0.49\textwidth}
     \centering
     \includegraphics[width=0.95\linewidth]{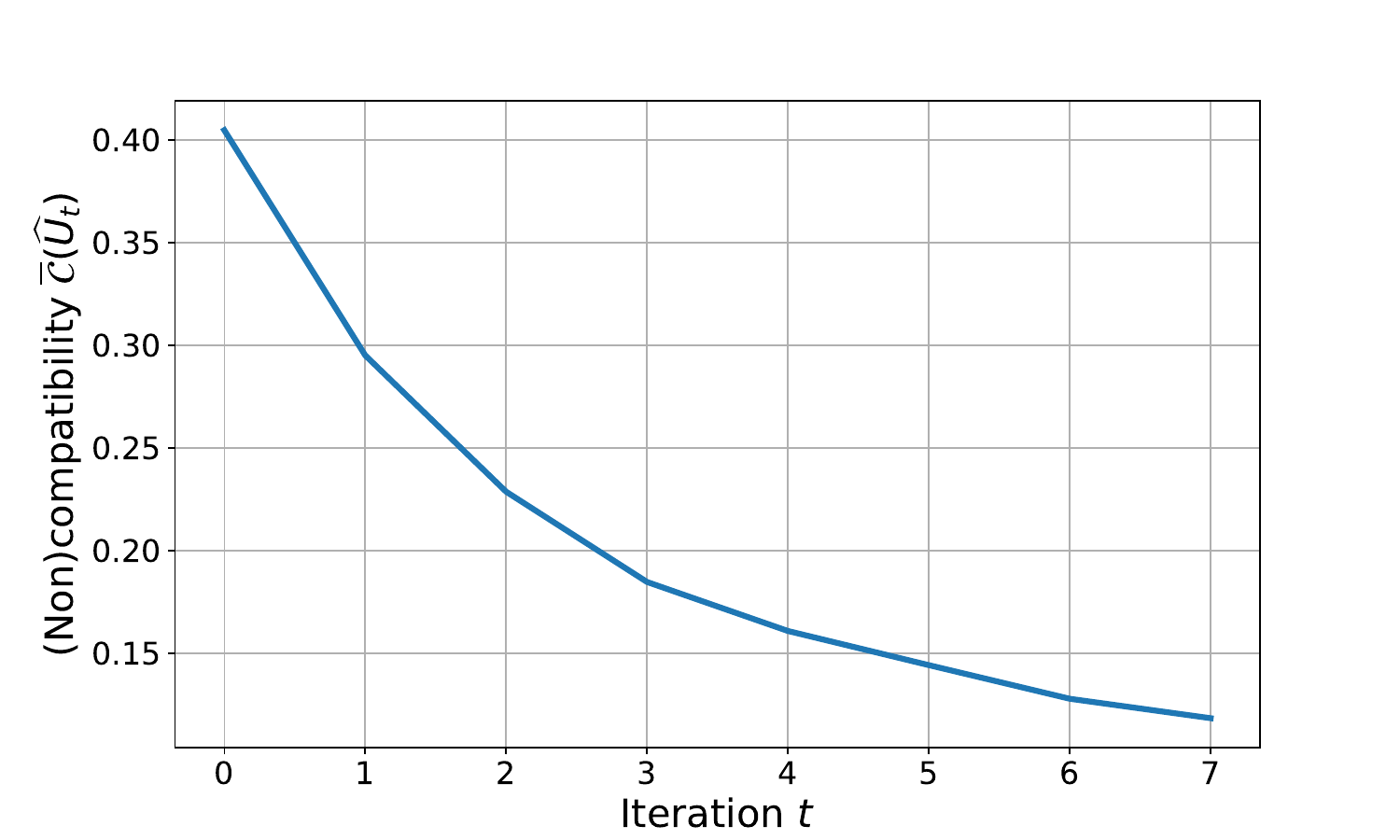}
 \end{minipage}
 \begin{minipage}[t!]{0.49\textwidth}
     \centering
     \includegraphics[width=0.95\linewidth]{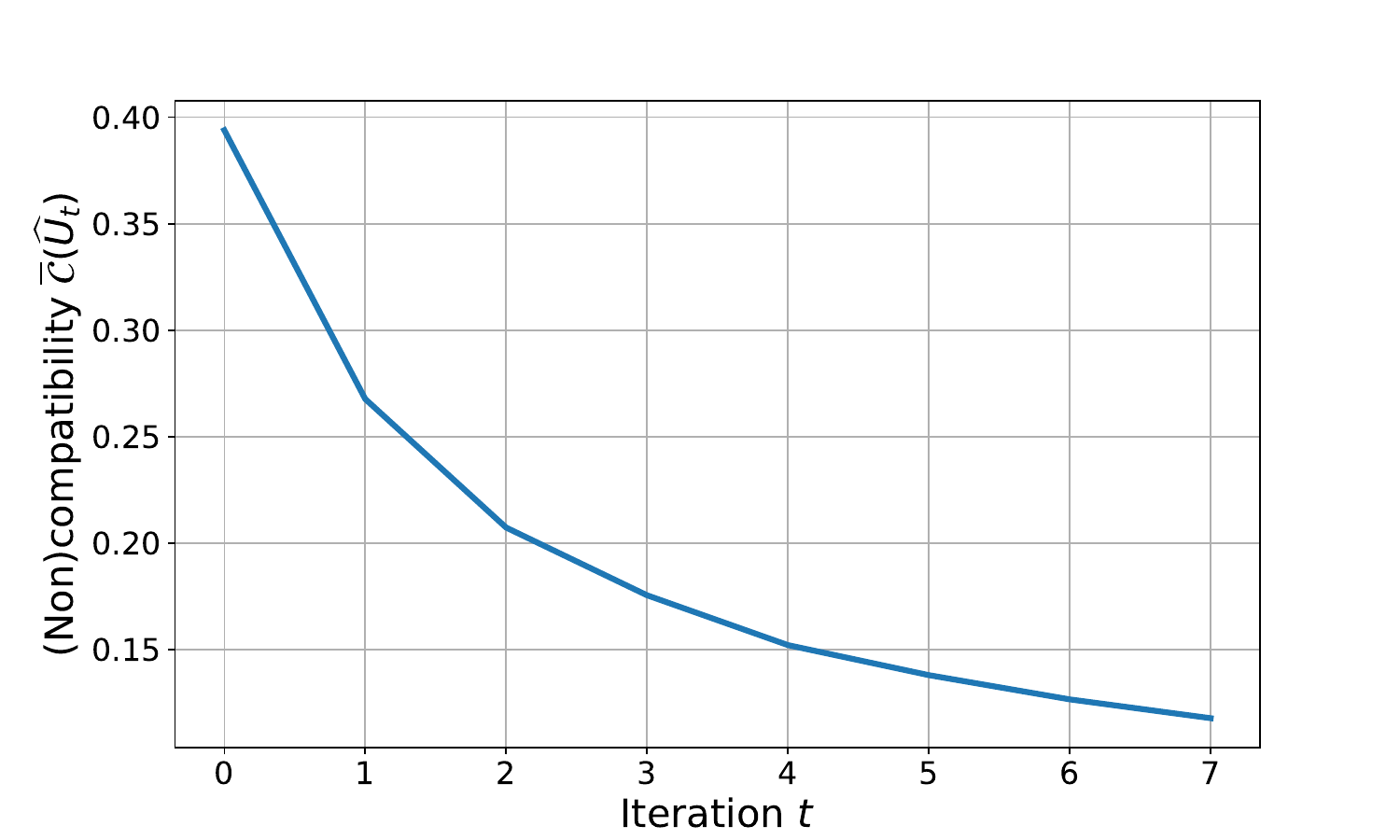}
 \end{minipage}
 \caption{\footnotesize (Left) Simulation with $S=20$ and $A=5$. (Right) Simulation with $S=100$ and $A=10$.}
 \label{fig: 1 sim data}
  \end{figure}
 
  \begin{figure}[!h]
   \centering
   \begin{minipage}[t!]{0.49\textwidth}
     \centering
     \includegraphics[width=0.95\linewidth]{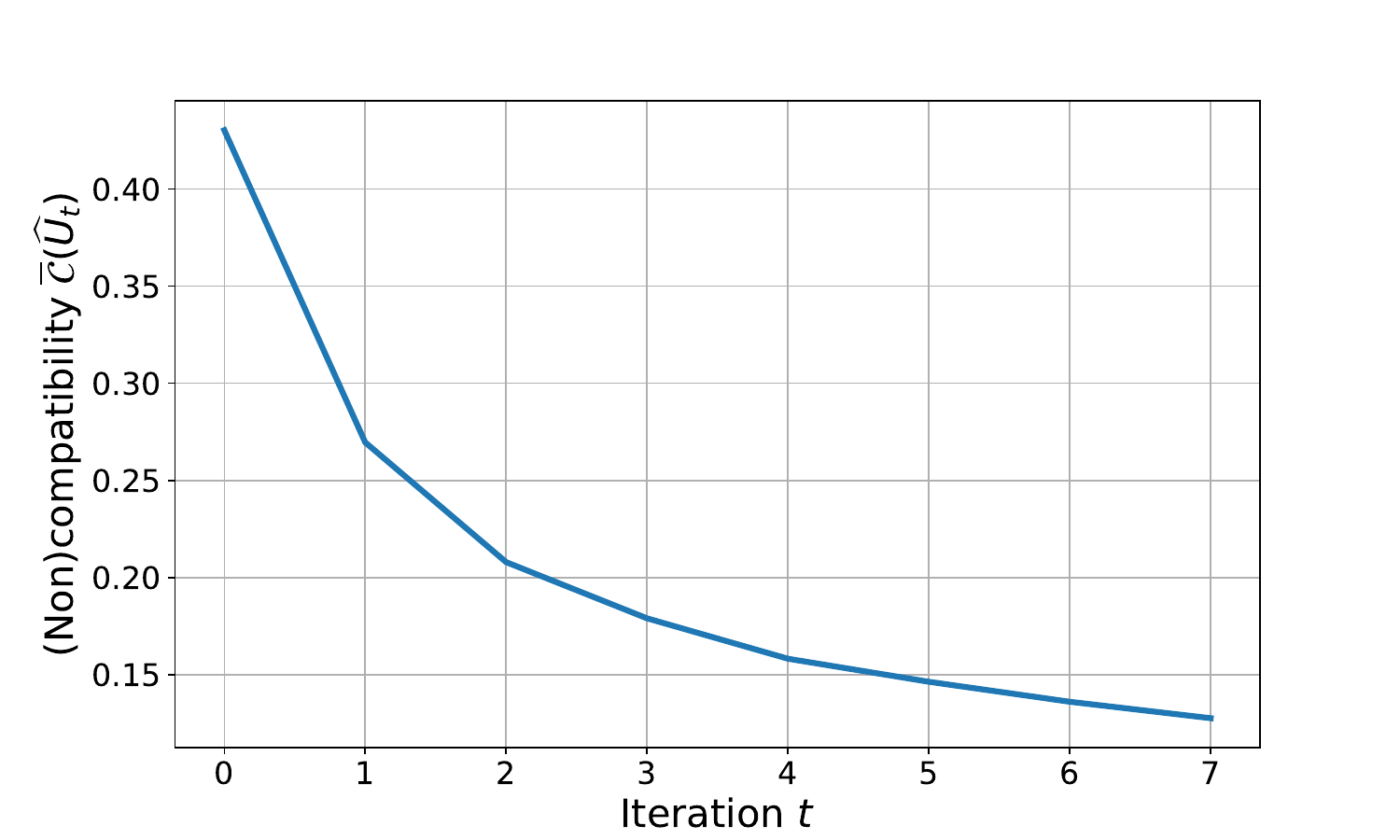}
 \end{minipage}
 \begin{minipage}[t!]{0.49\textwidth}
     \centering
     \includegraphics[width=0.95\linewidth]{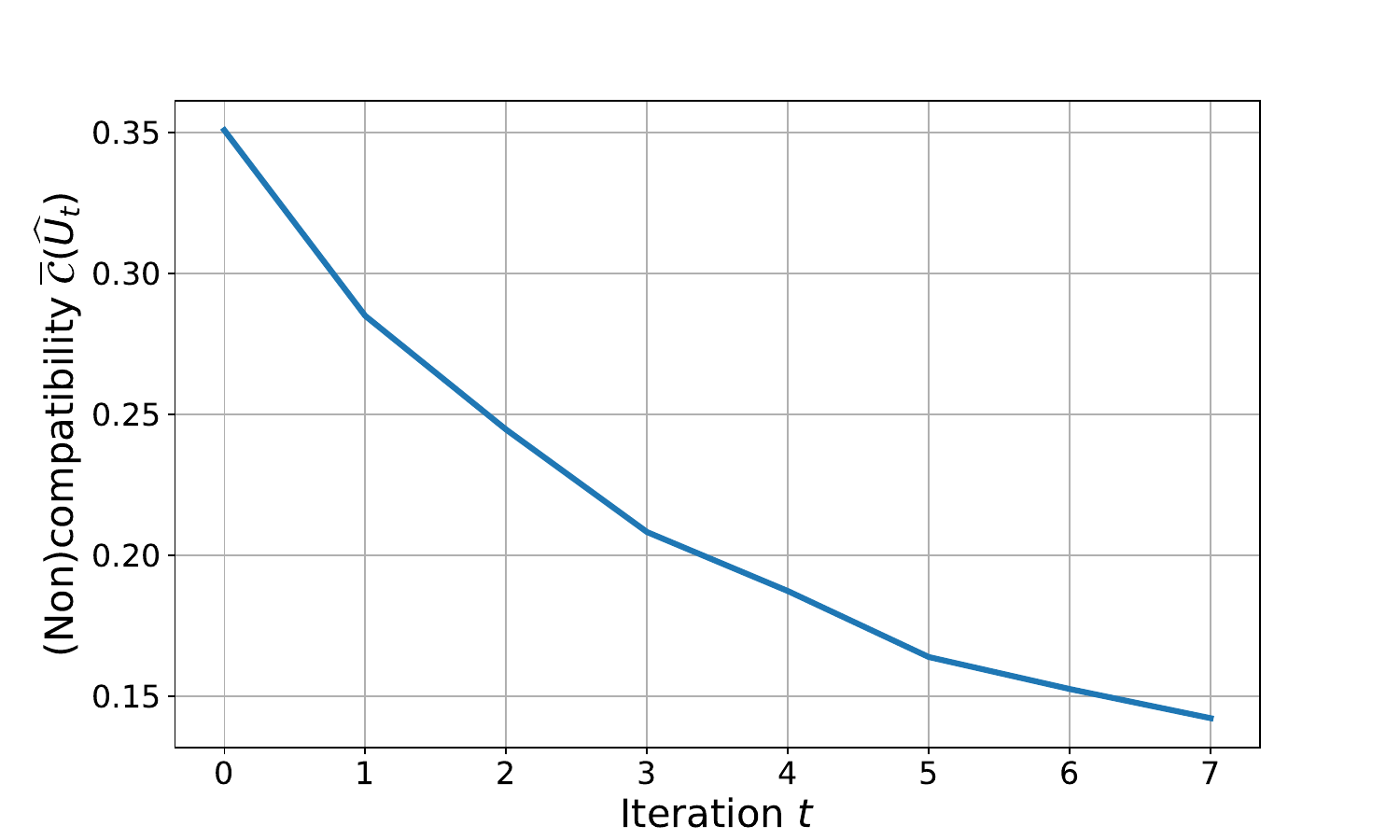}
 \end{minipage}
 \caption{\footnotesize (Left) Simulation with $S=1000$ and $A=20$. (Right) Simulation with $N=5$.}
 \label{fig: 2 sim data}
  \end{figure}
 
  \begin{figure}[!h]
   \centering
     \includegraphics[width=0.49\linewidth]{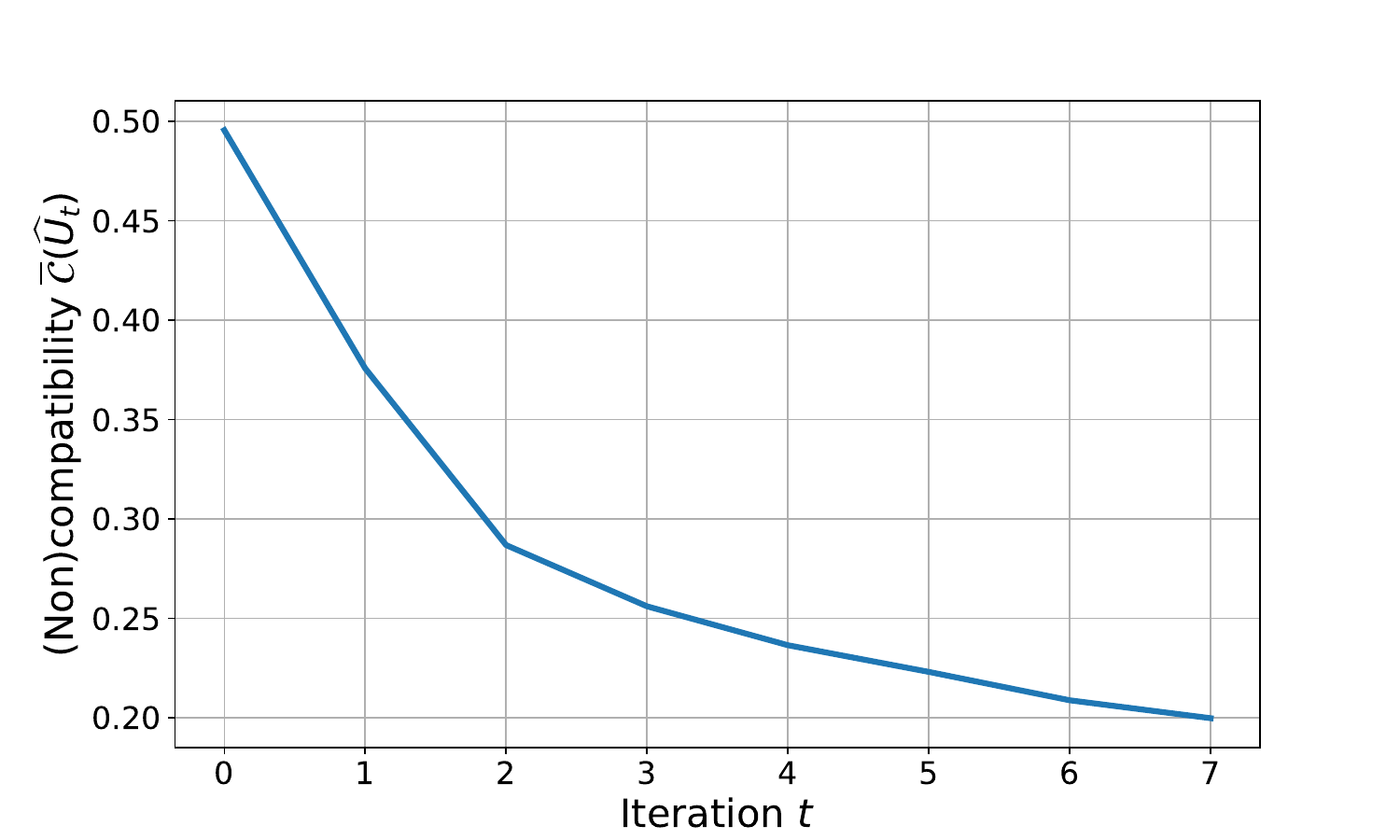}
 \caption{\footnotesize Simulation with $N=20$.}
 \label{fig: 3 sim data}
  \end{figure}

\subsection{Additional Experiment}\label{apx: additional experiment}

We conducted an additional experiment using the collected data to understand
which utility is more representative of all the participants' behaviors under
the model of Eq. \eqref{eq: model expert rs mdp}. 

The utilities considered for comparison are: $U_{\text{sqrt}}$,
$U_{\text{square}}$, $U_{\text{linear}}$, and $U_{\text{SG}}$. The first three
can be formally defined as: $U_{\text{sqrt}}(G)\coloneqq \sqrt{5G}$,
$U_{\text{square}}(G)\coloneqq G^2/5$, $U_{\text{linear}}(G)\coloneqq G$, and
they are depicted in Figure \ref{fig: utilities}. Instead, utility
$U_{\text{SG}}$ differs from each participant and is defined in the next
section.

\begin{figure}[t]
  \centering
  \includegraphics[scale=0.35]{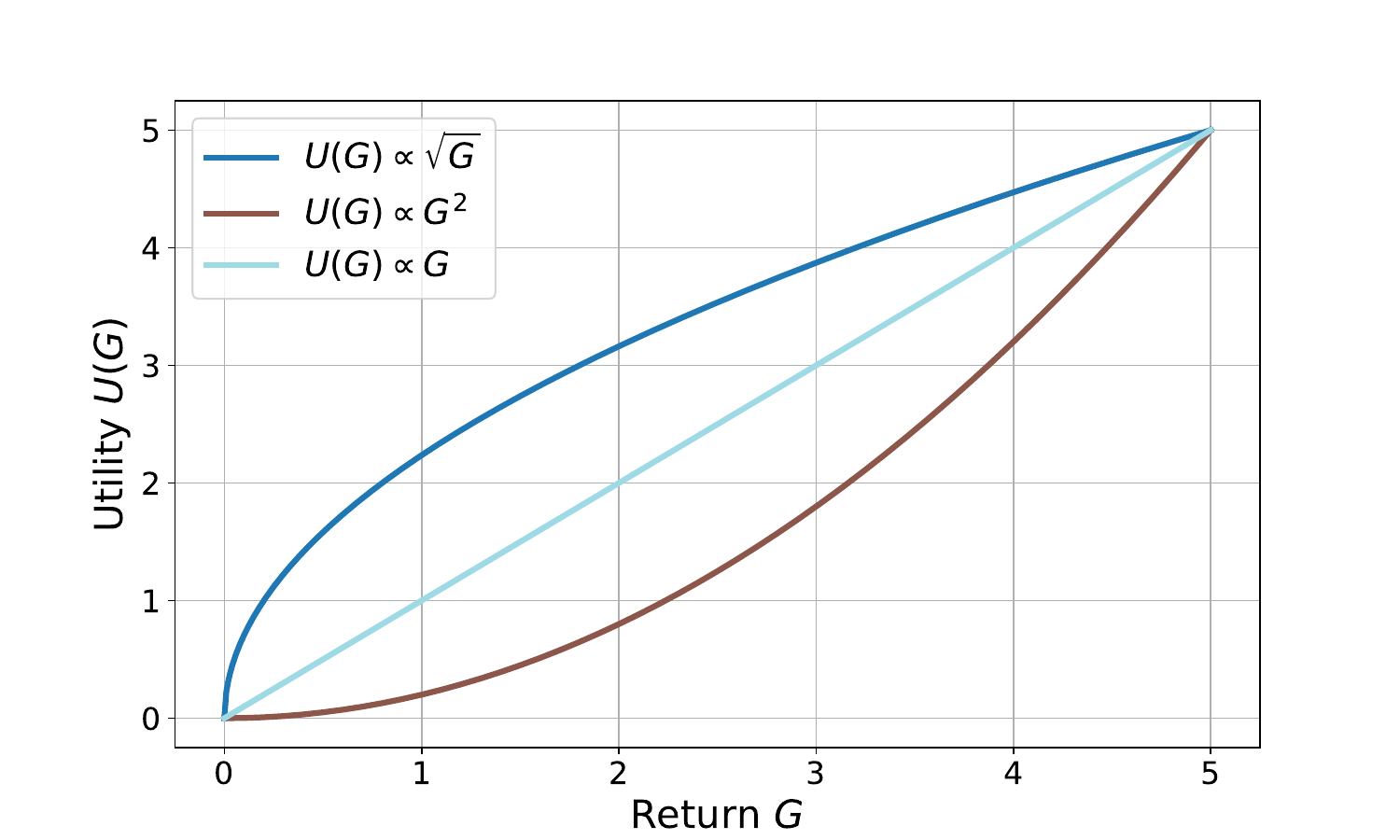}
  \caption{\footnotesize A plot of utilities
  $U_{\text{sqrt}},U_{\text{square}},U_{\text{linear}}$.}
  \label{fig: utilities}
\end{figure}

\subsubsection{Standard Gamble Data}
   \label{apx: standard gamble data}
  
   Utility $U_{\text{SG}}$ corresponds to the utility of each participant as
   fitted using the \emph{standard gamble method} \cite{Wakker2010prospect}.
 
   \begin{figure}[t]
     \centering
     \includegraphics[scale=0.35]{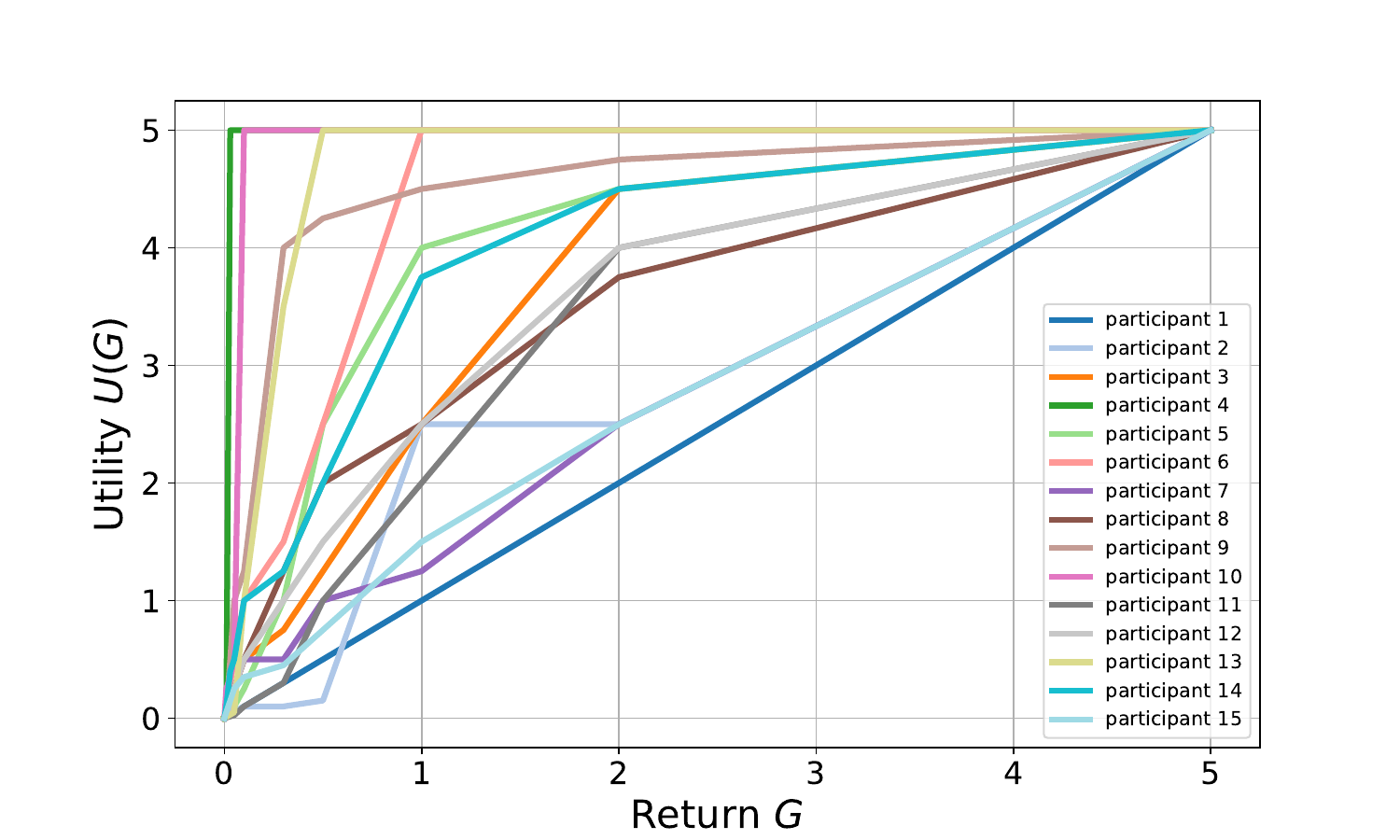}
     \caption{\footnotesize The SG utilities of the participants.}
     \label{fig: SG utilities}
   \end{figure}
 
     \paragraph{Standard Gamble (SG).}
       The Standard Gamble (SG) method (e.g., see Section 2.5 of
       \cite{Wakker2010prospect}) is a common method for inferring the von
       Neumann-Morgenstern (vNM) utility function of an agent. Observe Figure
       \ref{fig: SG for data}. In a SG, the agent has to decide between two
       options: a sure option (e.g., $x=30$\texteuro), in which the prize is
       obtained with probability 1, and a lottery between two prizes (e.g.,
       $5000$\texteuro and $0$\texteuro), in which the best prize
       ($5000$\texteuro) is received with probability $p$. For any value of $x$,
       the agent has to answer what is the probability $p$ that, from his
       perspective, makes the two options (i.e., $x$ for sure, or the lottery)
       \emph{indifferent}.
 
     \begin{figure}[h!]
       \centering
       \begin{tikzpicture}
         \node[rectangle, draw=none, minimum size=1pt] at (-2, 0) {\large $x$};
         \node[rectangle, draw=none, minimum size=1pt] at (-1, 0) {\LARGE $\sim$};
         \draw (0,0) circle (0.15);
         \node[rectangle, draw=none, minimum size=1pt] at (2, 0.5) {$5000$\texteuro};
         \node[rectangle, draw=none, minimum size=1pt] at (2, -0.5)
         {$0$\texteuro};
         \draw (0.05,0.14) -- (0.5,0.5) -- (1.5,0.5);
         \draw (0.05,-0.14) -- (0.5,-0.5) -- (1.5,-0.5);
         \node[rectangle, draw=none, minimum size=1pt] at (1,0.7) {\small $p$};
         \node[rectangle, draw=none, minimum size=1pt] at (1,-0.7) {\small $1-p$};
     \end{tikzpicture}
     \caption{\small The SG used for data collection.}
     \label{fig: SG for data}  
         \end{figure}
 
     Given the probability $p$, we have that the utility $U$ of the agent for $x$
     is:
     \begin{align*}
       U(x)=p\cdot U(5000)+(1-p)\cdot U(0)=p,
     \end{align*}
     since, by normalization conditions, we have $U(0)=0$ and $U(5000)=1$.
     
     \paragraph{Our SG.}
 
     We have asked the 15 participants to the study to answer some SG questions,
     which allows us to fit a vNM utility function $U_{\text{SG}}$ for each of
     them. Specifically, we have asked to answer 8 different SG questions, in
     which the $x$ value in Figure \ref{fig: SG for data} has been replaced by:
     \begin{align*}
       10\text{\texteuro}, 30\text{\texteuro}, 50\text{\texteuro},
       100\text{\texteuro}, 300\text{\texteuro}, 500\text{\texteuro},
       1000\text{\texteuro}, 2000\text{\texteuro}.
     \end{align*}
     Next, we linearly interpolate the computed utilities, obtaining the
     functions in Figure \ref{fig: SG utilities}.
 
     It should be remarked that this model considers single decisions (i.e.,
     $H=1$), while in MDPs there is a sequence of decisions to be taken over
     time, specifically over a certain time horizon $H$.

\subsubsection{Results}

To measure the fitness of a utility $U$ to the data (policy $\pi$)
\emph{fairly}, we  consider a \emph{relative} notion of (non)compatibility (we
omit $p,r$ for simplicity): $\overline{\cC}_{\pi}^{\text{r}}(U)\coloneqq
(J^*(U)-J^{\pi}(U))/J^*(U)$.
Intuitively, $\overline{\cC}_{\pi}^{\text{r}}(U)$ measures the \emph{quality}
of $\pi$ as perceived by the demonstrating agent, \emph{if $U$ was its true utility
function}.

We execute \caty (without exploration) for the 15 participants comparing the
IRL risk-neutral utility $U_{\text{linear}}$ with 3 ``baselines'': A
risk-averse $U_{\text{sqrt}}$ (concave) and a risk-lover $U_{\text{square}}$
(convex) utilities, and the utility $U_{\text{SG}}$ fitted through the SG
method (see Appendix \ref{apx: experimental details} for details).
We report the (non)compatibilities in \emph{percentage} in Table \ref{table: ex
 exp1}, where we have used colors to highlight the \textcolor{brightGreen}{best}
 and \textcolor{vibrantRed}{worst} values for each participant (the last column
 contains the average over the participants).

 \begin{table*}[t]
  \centering
 \scalebox{0.9}{
  \begin{tabular}{||c | c c c c c c c c c c c c c c c| c ||} 
   \hline
   & 1 & 2 & 3 & 4 & 5 & 6 & 7 & 8 & 9 & 10 & 11 & 12 & 13 & 14 &
   15 & mean\\
   \hline
   $U_{\text{linear}}$ &\small 39 &\small 58 &\small 18 &\small
   \textcolor{vibrantRed}{1} &\small 9 &\small 33 &\small 25 &\small 62
   &\small \textcolor{brightGreen}{1} &\small 56 &\small
   \textcolor{brightGreen}{1} &\small 16 &\small 16 &\small 25 &\small 60 &
   \small \textbf{28$\pm$22}\\
   \hline
    $U_{\text{sqrt}}$ &\small \textcolor{brightGreen}{16} &\small
    \textcolor{brightGreen}{28} &\small \textcolor{brightGreen}{8} &\small
    \textcolor{vibrantRed}{1} &\small \textcolor{brightGreen}{3} &\small
    \textcolor{brightGreen}{16} &\small \textcolor{brightGreen}{11} &\small
    \textcolor{brightGreen}{30} &\small \textcolor{brightGreen}{1} &\small 25
    &\small \textcolor{brightGreen}{1} &\small \textcolor{brightGreen}{6}
    &\small \textcolor{brightGreen}{8} &\small \textcolor{brightGreen}{11}
    &\small \textcolor{brightGreen}{28} & \small \textcolor{brightGreen}{\textbf{13$\pm$10}}\\
    $U_{\text{square}}$ &\small \textcolor{vibrantRed}{70} &\small
    \textcolor{vibrantRed}{86} &\small \textcolor{vibrantRed}{32} &\small
    \textcolor{vibrantRed}{1} &\small \textcolor{vibrantRed}{19} &\small
    \textcolor{vibrantRed}{41} &\small \textcolor{vibrantRed}{44} &\small
    \textcolor{vibrantRed}{91} &\small \textcolor{brightGreen}{1} &\small
    \textcolor{vibrantRed}{88} &\small \textcolor{brightGreen}{1} &\small
    \textcolor{vibrantRed}{35} &\small \textcolor{vibrantRed}{28} &\small
    \textcolor{vibrantRed}{44} &\small \textcolor{vibrantRed}{91} & \small
    \textcolor{vibrantRed}{\textbf{45$\pm$32}}\\
    $U_{\text{SG}}$ &\small 39 &\small 76 &\small 11 &\small
    \textcolor{brightGreen}{0} &\small 5 &\small 28 &\small 20 &\small 34
    &\small \textcolor{vibrantRed}{10} &\small \textcolor{brightGreen}{2}
    &\small \textcolor{brightGreen}{1} &\small 8 &\small 21 &\small 17 &\small
    51 & \small \textbf{22$\pm$21}\\
   \hline
  \end{tabular}%
  }
  \caption{ Values of $\overline{\cC}_{\pi}^{\text{r}}$ of various
  utilities with the demonstrations of the participants in percentage.}
  \label{table: ex exp1}
  \end{table*}

  Some observations are in order.
First, \emph{this} data shows that $U_{\text{linear}}$ (i.e., IRL) is overcome
by $U_{\text{sqrt}}$, which reduces
\scalebox{0.9}{$\overline{\cC}_{\pi}^{\text{r}}$}$(\cdot)$ from $28\%$ to $13\%$
on the average of the participants.
Next, note that $U_{\text{sqrt}}$ outperforms the $U_{\text{SG}}$ of \emph{each}
participant. This is due to both the bounded rationality of humans, who can
\emph{not} apply the $H=1$ utility $U_{\text{SG}}$ to $H>1$ problems, and the
fact that $U_{\text{sqrt}}$ probably ``overfits'' the simple MDP considered, but
it might generalize worse than $U_{\text{SG}}$ to new environments.
\footnote{Further analysis should be carried out on this, that we leave to
future works.}
Finally, observe that all the utilities are compatible with policies $4$ and
$11$, providing empirical evidence on the \emph{partial identifiability} of the
expert's utility from single demonstrations.

The experiment has been conducted on the same personal computer as experiment 2,
in less than one hour.

The experiment has been conducted collecting 10000 trajectories to estimate the
return distribution of each participant's policy, and 10000 trajectories for
estimating the return distribution of the optimal policy, which has been
computed exactly through value iteration. We have executed 5 simulations with
different seeds, and the relative (non)compatibility values written in Table
\ref{table: ex exp1} are the average over the 5 simulations.

For the experiment, we used the true transition model, and we remark that the
reward function considered, when discretized, coincides with itself, i.e., we
did not incur in estimation error of the transition model nor in approximation
error for the discretization.

\end{document}

%% file: algorithms/caty_classification.tex
\RestyleAlgo{ruled}
\LinesNumbered
\begin{algorithm}[t]\small
    \caption{\caty}
    \label{alg: caty classification}
\small
\DontPrintSemicolon
\SetKwInOut{Input}{Input}

\Input{data $\{\cD^E_i\}_i$, threshold $\Delta$,
    utility $U$, discretization
    $\epsilon_0$, dynamics $\{\widehat{p}^i\}_i$}
    \nonl \texttt{// Discretize $U$:}\;

    $\overline{U}(y)\gets U(y)\quad$ for all $y\in\cY$\label{line: discretize
    U}\;
    
    \For{$i = 1,2,\dotsc,N$}{
    \nonl \texttt{// Estimate $J^{\pi^{E,i}}(U;p^i,r^i)$:}\;

    $\widehat{\eta}^{E,i}\gets$ \texttt{ERD}($\cD^E_i,r^i$)
    \label{line: erd caty}\;

    $\widehat{J}^{E,i}(U)\gets
    \sum_{y\in\cY}\widehat{\eta}^{E,i}(y)\overline{U}(y)$
    \label{line: est JE caty}\;

    \nonl \texttt{// Estimate $J^*(U;p^i,r^i)$:}\;

    $\widehat{J}^{*,i}(U),\_\gets$ \texttt{PLANNING}($
    \overline{U},i,\widehat{p}^i$)
    \label{line: planning caty}\;

    \nonl \texttt{// Estimate $\overline{\cC}_{p^i,r^i,\pi^{E,i}}(U)$:}\;

    $\widehat{\cC}^i(U)\gets \widehat{J}^{*,i}(U)-\widehat{J}^{E,i}(U)$
    \label{line: estimate C}\;

    }
    class $\gets$ True \textbf{if} $\sum_{i\in\dsb{N}}\widehat{\cC}^i(U)\le\Delta$
          \textbf{else} False\label{line: classify}\;
          
    \textbf{Return} class\;
\end{algorithm}

%% file: algorithms/tractor.tex
\RestyleAlgo{ruled}
\LinesNumbered
\begin{algorithm}[t]\small
    \caption{\tractor}
    \label{alg: tractor}
\small
\DontPrintSemicolon
\SetKwInOut{Input}{Input}

\Input{
    data $\{\cD^{E}_i\}_i$,
parameters $T,K,\alpha,\overline{U}_0$, discretization $\epsilon_0$, dynamics
$\{\widehat{p}^i\}_i$}
$\widehat{\eta}^{E,i}\gets$ \texttt{ERD}($\cD^E_i,r^i$)
$\quad$ for $i\in\dsb{N}$\label{line: erd tractor}\;

\For{$t=0,1,\dotsc,T-1$}{
    \nonl \texttt{// Compute distributions $\{\widehat{\eta}_t^{i}\}_i$:}\;

    \For{$i$ = $1,2,\dotsc,N$\label{line: for2 tractor}}{
$\_,\widehat{\psi}^{*,i}_t\gets$
\texttt{PLANNING}($\overline{U}_t,i,\widehat{p}^i$)
        \label{line: planning tractor}\;

    $\cD\gets$
    \texttt{ROLLOUT}($\widehat{\psi}^{*,i}_t,\widehat{p}^i,\overline{r}^i,i,K$)\label{line:
    rollout tractor}\;

        $\widehat{\eta}^i_t(y)\gets\frac{1}{K}\sum_{G\in\cD}\indic{G=y},\forall
        y\in\cY$
        \label{line: line compute etati tractor}\;

        }
    \label{line: end for2 tractor}
    \nonl \texttt{// Update $\overline{U}_{t+1}$:}\;

    $g_t\gets
    \sum_{i\in\dsb{N}}\big(\widehat{\eta}_t^i-\widehat{\eta}^{E,i}\big)$
    \label{line: compute gt tractor}\;

    $\overline{U}_{t+1}\gets
    \Pi_{\overline{\underline{\fU}}_L}(\overline{U}_t-\alpha g_t)$
    \label{line: gd update tractor}\;
    
}
$\widehat{U}\gets\frac{1}{T}\sum_{t=0}^{T-1}\overline{U}_t$\;

\textbf{Return} $\widehat{U}$\;
\end{algorithm}

%% file: algorithms/explore.tex
\RestyleAlgo{ruled}
\LinesNumbered
\begin{algorithm}[!h]
    \caption{\texttt{EXPLORE}}
    \label{alg: caty exploration}
    \DontPrintSemicolon
    \SetKwInOut{Input}{Input}
    \Input{samples budget $\tau$}
    $n\gets \floor{\tau/(SAH)}$\;

    \For{$i$ $\in$ $\{1,2,\dotsc,N\}$}{
    \nonl \texttt{// Initialize the transition model estimate:}\;

    $\widehat{p}^i_h(s'|s,a)=0$ for all $(s,a,h,s')\in\SAH\times\cS$\;

    \nonl \texttt{// Collect samples:}\;

    \For{$(s,a,h)\in\SAH$}{
        \For{$\_$ $\in$ $\{1,2,\dotsc,n\}$}{
            $s'\gets$ sample from $p^i_h(\cdot|s,a)$\;

            $\widehat{p}^i_h(s'|s,a)\gets \widehat{p}^i_h(s'|s,a)+1$\;

        }
    }
    $\widehat{p}^i_h(\cdot|s,a)\gets \widehat{p}^i_h(\cdot|s,a)/n$\;
    
    }
    \textbf{Return} $\{\widehat{p}^i\}_i$\;
\end{algorithm}

%% file: algorithms/planning.tex
\RestyleAlgo{ruled}
\LinesNumbered
\begin{algorithm}[!h]
    \caption{\texttt{PLANNING}}
    \label{alg: planning}
    \DontPrintSemicolon
    \SetKwInOut{Input}{Input}
    \Input{utility $U$, environment index $i$, transition model $p$}
    \nonl \texttt{// Initialize the $Q$ and value function at the last stage:}\;

    \For{$(s,y)\in\cS^i\times\cY_{H}$}{
      \For{$a\in\cA^i$}{
        $Q_H(s,y,a)\gets U(y+\overline{r}^i_H(s,a))$\;

      }
      $V_{H}(s,y)\gets \max\limits_{a\in\cA^i}Q_H(s,y,a)$\;

      $\psi_H(s,y)\gets \argmax\limits_{a\in\cA^i}Q_H(s,y,a)$\Comment*[r]{Keep just one action}
    }
    \nonl \texttt{// Backward induction:}\;

    \For{$h=H-1,\dotsc,2,1$}{
       \For{$(s,y)\in\cS^i\times\cY_{h}$}{
          \For{$a\in\cA^i$}{
    $Q_h(s,y,a)\gets \E_{s'\sim p_h(\cdot|s,a)}\Big[
          V_{h+1}(s',y+\overline{r}^i_h(s,a))            
    \Big]$
    \label{line: bellman planning}\;

          }
          $V_h(s,y)\gets \max\limits_{a\in\cA^i}Q_h(s,y,a)$\;

          $\psi_h(s,y)\gets
          \argmax\limits_{a\in\cA^i}Q_h(s,y,a)$\Comment*[r]{Keep just one action}
       }
    }
    \nonl \texttt{// Return optimal performance and policy:}\;
    
    \textbf{Return} $V_1(s_0^i,0),\psi$\;
\end{algorithm}

%% file: algorithms/erd.tex
\RestyleAlgo{ruled}
\LinesNumbered
\begin{algorithm}[t]\small
    \caption{\texttt{ERD} - Estimate the Return Distribution}
    \label{alg: erd}
\small
\DontPrintSemicolon
\SetKwInOut{Input}{Input}

\Input{dataset $\cD^E$, reward $r$}
    \nonl \texttt{// Initialize $\widehat{\eta}$:}\;

    \For{$y$ $\in$ $\cY$}{
        $\widehat{\eta}(y)\gets 0$\;

    }
    \nonl \texttt{// Loop over all trajectories in $\cD^E$:}\;

    \For{$\omega$ $\in$ $\cD^E$}{
        \nonl \texttt{// Compute return of $\omega=\{s_1,a_1,\dotsc,s_H,a_H,s_{H+1}\}$:}\;

        $G\gets\sum_{h=1}^H r_h(s_h,a_h)$\label{line: compute G erd}\;

        \nonl \texttt{// Update estimate $\widehat{\eta}$:}\;

        \If{$G\le 0$}{
            $\widehat{\eta}(0)\gets \widehat{\eta}(0)+1$\;

        }
        \ElseIf{$G>\floor{\frac{H}{\epsilon_0}}\epsilon_0$}{
            $\widehat{\eta}(\floor{\frac{H}{\epsilon_0}}\epsilon_0)
            \gets \widehat{\eta}(\floor{\frac{H}{\epsilon_0}}\epsilon_0)+1$\;

        }
        \Else{
            $L\gets \max_{y\in\cY\wedge y<G}y$\;    

            $U\gets \min_{y\in\cY\wedge y\ge G}y$\;

            $\widehat{\eta}(L) \gets \widehat{\eta}(L)
            + \frac{U-G}{U-L}$\;

            $\widehat{\eta}(U) \gets \widehat{\eta}(U)
            + \frac{G-L}{U-L}$\;

        }
    }
    \nonl \texttt{// Normalize:}\;

    $\widehat{\eta}\gets \widehat{\eta}/|\cD^E|$\;
    
    \textbf{Return} $\widehat{\eta}$
\end{algorithm}

%% file: algorithms/rollout.tex
\RestyleAlgo{ruled}
\LinesNumbered
\begin{algorithm}[t]\small
    \caption{\texttt{ROLLOUT}}
    \label{alg: rollout}
\small
\DontPrintSemicolon
\SetKwInOut{Input}{Input}

\Input{policy $\psi$, transition model $p$, reward $r$,
    environment index $i$, number of trajectories $K$}
    $\cD\gets \{\}$\;
    \nonl \texttt{// Loop over the number of trajectories:}\;

    \For{\_ $\in$ $\{1,2,\dotsc,K\}$}{
        $s\gets s_0^i$\;
        $y\gets 0$\Comment*[r]{$y$ keeps track of the accumulated reward}
        \For{$h=1$ to $H$}{
            $a\gets \psi_h(s,y)$\;

            $y\gets y+r_h(s,a)$\;

            $s\gets s'$ where $s'\sim p_h(\cdot|s,a)$\;

        }
        $\cD\gets \cD \cup\{y\}$\;
        
    }
    \textbf{Return} $\cD$\;
\end{algorithm}